\documentclass[nohyperref]{article}

\usepackage{bbm}
\usepackage{amsmath}
\usepackage{amssymb}
\usepackage{fancyhdr}
\usepackage{graphicx}
\usepackage{tabularx}
\usepackage{mathtools}
\usepackage{listings}
\usepackage{xcolor}
\usepackage{xspace}
\usepackage{courier}
\usepackage{comment}
\usepackage{tcolorbox}
\usepackage{cancel}
\usepackage{hyperref}
\usepackage{import}
\usepackage{subfiles}
\usepackage{lscape}
\usepackage{multirow}
\usepackage{multicol}
\usepackage{afterpage}

\tcbuselibrary{skins,xparse,breakable}
\newtcolorbox{quoted}[1][]{skin=enhanced,colback=green!5!white,colframe=green!75!black,breakable=true,title=#1}
\newtcolorbox{technical}[1][]{skin=enhanced,colback=blue!5!white,colframe=blue!75!black,breakable=true,title=#1}
\newtcolorbox{thought}[1][]{skin=enhanced,colback=red!5!white,colframe=red!75!black,breakable=true,title=#1}
\newtcolorbox{highlight}[1][]{skin=enhanced,colback=yellow!5!white,colframe=yellow!75!black,breakable=true,title=#1}

\newcommand{\eps}{\varepsilon}

\newcommand{\inv}{^{-1}}

\newcommand{\transpose}{^\top}

\newcommand{\nat}{\mathbb{N}}
\newcommand{\real}{\mathbb{R}}

\newcommand{\indic}[1]{\mathbbm{1}[#1]}

\newcommand{\normdist}{\mathcal{N}}

\newcommand{\tensf}{\tens{f}}

\newcommand{\tensy}{\tens{y}}

\newcommand{\cD}{\mathcal{D}}

\newcommand{\cH}{\mathcal{H}}
\newcommand{\cK}{\mathcal{K}}
\newcommand{\cL}{\mathcal{L}}
\newcommand{\cM}{\mathcal{M}}
\newcommand{\cO}{\mathcal{O}}

\newcommand{\cU}{\mathcal{U}}
\newcommand{\cX}{\mathcal{X}}

\newcommand{\half}{\frac{1}{2}}

\newcommand{\entropy}{\mathbb{H}}
\newcommand{\mutualinf}{\mathbb{I}}
\newcommand{\kldiv}[2]{D_{KL}[#1\| #2]}

\newcommand{\del}{\nabla}

\newcommand{\partialdiff}[2]{\frac{\partial #1}{\partial #2}}
\newcommand{\partialdiffil}[2]{{\partial #1}/{\partial #2}}

\newcommand{\innerprod}[2]{\langle #1 , #2 \rangle}

\newcommand{\norm}[1]{\| #1 \|}

\newcommand{\loss}{\ell}

\newcommand{\expected}{\mathbb{E}}

\newcommand{\tens}[1]{\mathbf{#1}\xspace}

\newcommand{\eigmin}[1]{\lambda_{\min} (#1)}
\newcommand{\eigmax}[1]{\lambda_{\max} (#1)}

\usepackage{microtype}
\usepackage{graphicx}
\usepackage{subfigure}
\usepackage{booktabs} 

\usepackage{hyperref}

\usepackage[accepted]{icml2023}

\usepackage{amsmath}
\usepackage{amssymb}
\usepackage{mathtools}
\usepackage{amsthm}

\usepackage[capitalize]{cleveref}

\theoremstyle{plain}
\newtheorem{theorem}{Theorem}[section]

\newtheorem{lemma}[theorem]{Lemma}
\newtheorem{corollary}[theorem]{Corollary}
\theoremstyle{definition}
\newtheorem{definition}[theorem]{Definition}
\newtheorem{assumption}[theorem]{Assumption}
\theoremstyle{remark}

\newcommand{\alg}{EV-GP}
\newcommand{\algms}{EV-GP+MS}

\crefname{figure}{Fig.}{Figs.}%
\crefname{section}{Sec.}{Secs.}%
\crefname{appendix}{App.}{Apps.}%
\crefname{algorithm}{Algo.}{Algos.}%

\usepackage[textsize=tiny]{todonotes}

\icmltitlerunning{Training-Free Neural Active Learning with Initialization-Robustness Guarantees}

\begin{document}

\twocolumn[
\icmltitle{Training-Free Neural Active Learning with Initialization-Robustness Guarantees}




\begin{icmlauthorlist}
\icmlauthor{Apivich Hemachandra}{nus}
\icmlauthor{Zhongxiang Dai}{nus}
\icmlauthor{Jasraj Singh}{ntu}
\icmlauthor{See-Kiong Ng}{nus}
\icmlauthor{Bryan Kian Hsiang Low}{nus}
\end{icmlauthorlist}

\icmlaffiliation{nus}{Department of Computer Science, National University of Singapore, Republic of Singapore}
\icmlaffiliation{ntu}{School of Computer Science and Engineering, Nanyang Technological University, Republic of Singapore}

\icmlcorrespondingauthor{Zhongxiang Dai}{dzx@nus.edu.sg}

\icmlkeywords{Machine Learning, ICML}

\vskip 0.3in
]



\printAffiliationsAndNotice{}  

\begin{abstract}
Existing neural active learning algorithms have aimed to optimize the predictive performance of neural networks (NNs) by selecting data for labelling. 
However, other than a good predictive performance, being robust against random parameter initializations is also a crucial requirement in safety-critical applications.
To this end, we introduce our \emph{expected variance with Gaussian processes} (EV-GP) criterion for neural active learning, which is theoretically guaranteed to select data points which lead to trained NNs with both (a) good predictive performances and (b) initialization robustness.
Importantly, our EV-GP criterion is training-free, i.e., it does not require any training of the NN during data selection, which makes it computationally efficient.
We empirically demonstrate that our EV-GP criterion is highly correlated with both initialization robustness and generalization performance, and show that it consistently outperforms baseline methods in terms of both desiderata, especially in situations with limited initial data or large batch sizes.\vspace{-2.8mm}
\end{abstract}


\section{Introduction}
\label{sec:intro}
Deep neural networks (NNs) have recently achieved impressive performances in various applications thanks to the availability of massive amount of labelled data.
However, in some applications, data labelling is so costly that obtaining large datasets is infeasible. In this regard, a number of works have used the method of \emph{neural active learning} to select a small number of data points to be labelled and subsequently be used for the training of the NNs \cite{renSurveyDeepActive2021, senerActiveLearningConvolutional2018, kirschBatchBALDEfficientDiverse2019, ashDeepBatchActive2020}.

When selecting the data for labelling, existing neural active learning methods have only aimed to optimize the predictive performance (e.g., generalization performance) of the trained NNs.
However, the output from a neural network is heavily dependent on its training process, especially on its \emph{initialization} \cite{cyrRobustTrainingInitialization2019, leeWideNeuralNetworks2020}. Therefore in some safety-critical applications, \emph{having small variability in the trained NNs with respect to model initialization} is also a crucial requirement.
For example, given a training set of patient data in a healthcare application (e.g., disease diagnosis), if the NNs trained using different random initializations have drastically disparate predictions, the resulting NNs would be too unreliable to be used for clinical diagnosis due to the high-stakes nature of the application \cite{estevaGuideDeepLearning2019}.
As a simple illustration, \cref{fig:randinit} shows that different training sets can indeed result in distinct sensitivities of the trained NNs to model initialization, and hence a badly constructed training set (e.g., the rightmost figure in \cref{fig:randinit}) can cause the trained NN to be highly variable (and hence untrustworthy) w.r.t. the model initialization. 
Therefore, in these important applications, it is paramount for a neural active learning algorithm to select data points which lead to trained NNs with not only (a) good predictive performances but also (b) \emph{initialization robustness}.

\begin{figure}
\centering
\includegraphics[width=\linewidth]{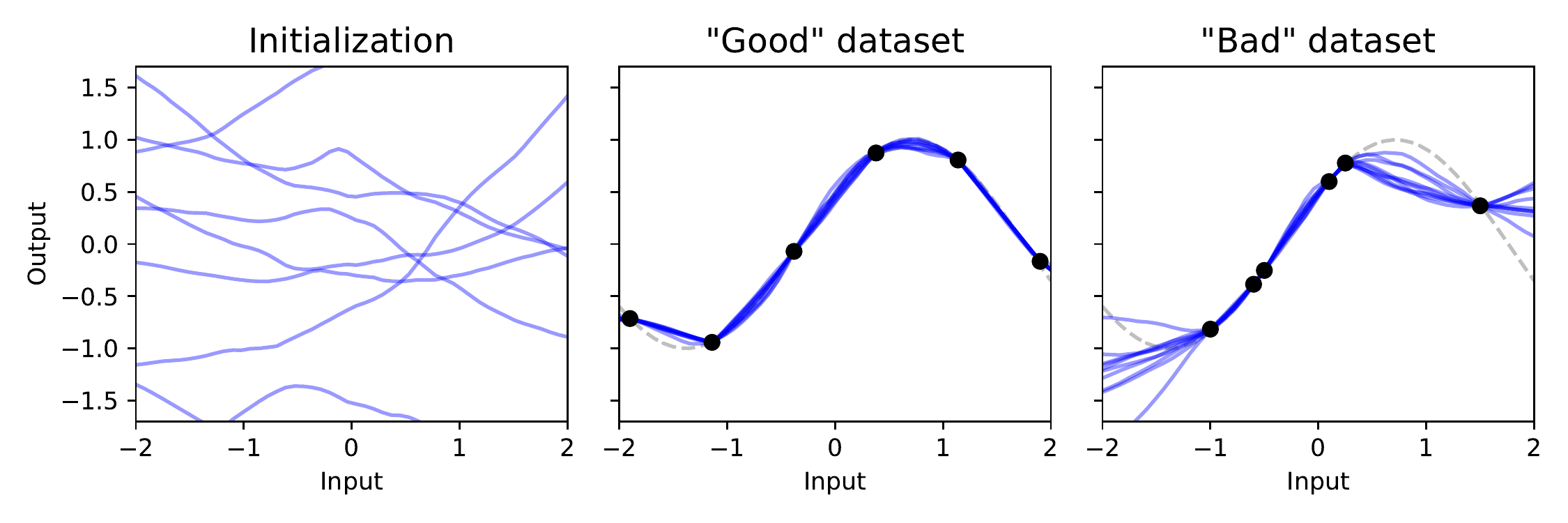}
\vspace{-9.6mm}
\caption{
Neural networks initialized randomly with different seeds (left), after training, can give consistent outputs (center) or highly varying outputs (right) depending on the training data.
}
	\label{fig:randinit}\vspace{-3mm}
\end{figure}

To this end, we introduce our \emph{Expected Variance with Gaussian Processes} (\alg) criterion, which is theoretically shown to select data points which satisfy both criteria.
Specifically, we firstly leverage the theory of \emph{neural tangent kernel} (NTK) \cite{jacotNeuralTangentKernel2018} and overparameterized NNs \cite{leeWideNeuralNetworks2020} to characterize the predictive distribution of trained NNs (\cref{subsec:approximation:of:mu:sigma}) with respect to random model initialization.
This allows us to measure the robustness of 
the trained NNs against model initialization via the output variance, i.e., a smaller output variance indicates a more initialization-robust NN.
Next, we introduce an efficiently computable approximation of the output variance, and derive a theoretical guarantee on the approximation quality (\cref{subsec:approximation:of:mu:sigma}).
Subsequently, we show that this approximate output variance, which is an indicator of (the inverse of) initialization robustness, is also, interestingly, an upper bound on the generalization error of the trained NNs (\cref{subsec:connection:with:gene:error}).

As a result, our \alg~active learning criterion (\cref{subsec:active:learning:criterion}), which is based on the minimization of the above-mentioned approximate output variance, is able to select data points which lead to both (a) good predictive performances (i.e., small generalization errors) and (b) initialization robustness.
Moreover, our \alg~criterion also has the additional advantages of \emph{computational efficiency} and \emph{generality}.
Firstly, during data selection, the computation of our \alg~criterion \emph{does not require training of the NN}, which is usually computationally costly.
Therefore, our \alg~criterion is computationally efficient and can be used to select all data points in a small number of batches or even a single batch.
Secondly, our \alg~criterion does not impose any requirement on the training process of the NN, and is hence applicable to generic NNs rather than requiring specific types of models such as Bayesian NNs.
Furthermore, we introduce a theoretically grounded method for model selection which is able to automatically optimize the architecture of the NN based on the selected data points (\cref{sec:model-sel}).

We empirically show that our \alg~criterion can accurately characterize the output variance (w.r.t.~the random model initialization) and hence the initialization robustness of trained NNs (\cref{exp:nn-ntkgp,exp:small}). Furthermore, we use extensive real-world regression (\cref{exp:small}) and classification (\cref{exp:classification}) experiments to show that our \alg~criterion outperforms existing baselines in terms of both initialization robustness and predictive performances.
Of note, the performance advantage of our \alg~criterion is particularly pronounced when there is limited labelled data or the batch size is large (\cref{exp:classification}).
Lastly, we also empirically verify that our model selection method (\cref{sec:model-sel}) is able to further improve the performance of our \alg~criterion (\cref{subsec:exp:with:ms}).

\section{Problem Setup and Background}
\label{sec:lit}
In our problem setup, consistent with previous works on active learning, we have access to an unlabelled pool of input data $\cX_U$ whose corresponding labels are expensive to obtain. 
Also, suppose we have a set of testing input data $\cX_T$ (whose corresponding labels are unknown), which we want our NN to eventually make predictions on. 
In practice, when a set of testing input is not available, we can simply let $\cX_T = \cX_U$, which we have done in our experiments.
Additionally, we are given a set $\cL_0 = (\cX_{0}, \tensy_{0})$ of already-labelled data (which may be empty), and the algorithm proceeds to construct a set $\cL = (\cX_\cL, \tensy_\cL)$ where $\cX_\cL\setminus\cX_{0} \subset \cX_U$ and $|\cL \setminus \cL_0| \leq k$. The parameter $k$ can be considered the budget for active learning, which is a limit on the number of queries to the oracle. To construct $\cL$, in each round 
of active learning, the algorithm selects a set of $b$ unlabelled input data points $\cX_b \subset \cX_U$ and submits them to the oracle to obtain their corresponding labels $\tensy_b$.
It is desirable to have minimal training of the NN between different rounds of querying in order to avoid incurring excessive computational costs and long delays between each rounds.


Once the labelled training set $\cL$ is obtained after the active learning algorithm, we use it as the training set $\cD=\cL$ to train an NN, $f(\cdot; \theta)$, where $\theta$ denotes its parameters. The model parameters are firstly initialized, $\theta_0 \sim \text{init}(\theta)$, and following the common practice \cite{glorotUnderstandingDifficultyTraining2010}, we assume that each parameter of $\theta_0$ is independently drawn from a Gaussian distribution with zero mean and a known variance. After initialization, the model is trained with gradient-based optimization to minimize the regularized mean-squared error (MSE) loss:\vspace{-1.1mm}
\begin{equation}
	\loss(\cD; \theta) = \frac{1}{|\cD|} \sum_{(x, y)\in \cD} \half \norm{f(x; \theta)-y}^2 + \frac{\lambda}{2} \norm{\theta}^2\ . \vspace{-1.1mm}
    \label{eqn:loss}
\end{equation}
The loss function comprises of a mean-squared error term (first term) and a regularisation term (second term), in which $\lambda$ controls their trade-off.
Note that the regularized MSE loss \eqref{eqn:loss} is only required in our theoretical analysis and in practice, other loss functions such as the cross-entropy loss can also be used (\cref{exp:classification}).
The NN model is assumed to be trained till convergence, resulting in the parameters $\theta_\infty = \text{train}(\theta_0)$ where $\text{train}(\cdot)$ denotes the function of the training process. 
As discussed in \cref{sec:intro}, we would like the final model predictions $f(x;\text{train}(\theta_0))$ to achieve both small generalization errors and low output variances with respect to the randomness of $\theta_0$.\vspace{-2mm}

\subsection{Neural Tangent Kernels}
\label{lr:ntk}\vspace{-1mm}
When training a NN using a training set $\cD = (\cX, \tensy)$ with gradient descent (GD): $\theta_{t+1} \gets \theta_t - \eta \cdot \del_\theta \loss(\cD; \theta_t)$ where $\eta$ is the learning rate, it has been shown \cite{leeWideNeuralNetworks2020} that as long as $\eta$ is small enough, the training can be approximated by continuous GD.
As a result, the change in the predictive output $f(\cX; \theta_t)$ over time can be expressed as:\vspace{-1mm}
\begin{equation}
\begin{array}{c}
\displaystyle
\frac{\partial {f}(\cX; \theta_t)}{\partial t}
= -\eta \underbrace{\nabla_{\theta} f(\cX; \theta_t) \nabla_{\theta} f(\cX; \theta_t)\transpose}_{\triangleq \hat{\Theta}_t(\cX, \cX)} \nabla_{f} \loss(\cD; \theta)\ .\vspace{-4mm}
\end{array}
\label{eqn:ode_pred}\vspace{2mm}
\end{equation}
The term $ \hat{\Theta}_t(\cX, \cX') \triangleq \nabla_{\theta} f(\cX; \theta_t) \nabla_{\theta} f^\top(\cX'; \theta_t)$ is referred to as the (empirical) \emph{neural tangent kernel} (NTK) \cite{jacotNeuralTangentKernel2018, aroraExactComputationInfinitely2019}. 
While kernels are defined over a tuple of elements from the input space, we will overload the notation for all kernels and use $\Theta(\cX, \cX') = (\Theta(x, x'))_{x\in \cX, x'\in \cX'}$ to represent the matrix constructed using values of the kernel. 
We will sometimes use the shorthand notation $\Theta_\cX$ to represent $\Theta(\cX, \cX)$ and $\Theta_{\cX' \cX}$ to represent $\Theta(\cX', \cX)$ when the context is clear, which is consistent with previous works on NTK \cite{heBayesianDeepEnsembles2020}.

The work of \citet{leeWideNeuralNetworks2020} has shown that when the width of the NN approaches infinity, the output of the NN can be approximated by a linear model:\vspace{-0mm}
\begin{equation}
f(x; \theta) 
\approx f(x; \theta_0) + \innerprod{\nabla_\theta f(x; \theta_0) }{\theta - \theta_0}\ ,
\label{eq:linear:apprx}
\end{equation}
and the empirical NTK $\hat{\Theta}_t(\cdot, \cdot)$ stays constant throughout training and approaches a deterministic kernel $\Theta(\cdot, \cdot)$ regardless of the initialization \cite{leeWideNeuralNetworks2020}.
Using the linear approximation \eqref{eq:linear:apprx}, \citet{leeWideNeuralNetworks2020} have shown that if a randomly initialized NN is trained on data $(\cX, \tensy)$ with the mean-squared error loss (\eqref{eqn:loss} with $\lambda = 0$) till convergence, then the predictions of the converged model on a testing set $\cX_T$ follow a normal distribution: $f(\cX_T) \sim \normdist(\mu(\cX_T | \cX, \tensy), \Sigma_\text{NN}(\cX_T | \cX))$, where the output mean is
\begin{equation}
\label{eqn:krr-mean}
\mu_\text{NN}(\cX_T | \cX, \tensy)=\Theta_{\cX_T \cX} \Theta_\cX\inv \tensy\ ,
\end{equation}
and the output covariance is\vspace{-1mm}
\begin{multline}
\Sigma_\text{NN}(\cX_T | \cX)
=\cK_{\cX_T}
+\Theta_{\cX_T \cX} \Theta_\cX\inv \cK_\cX \Theta_\cX\inv \Theta_{\cX \cX_T} \\
\qquad\quad\  -\left(\Theta_{\cX_T \cX} \Theta_\cX\inv \cK_{\cX \cX_T} 
+ \cK_{\cX_T \cX} \Theta_\cX\inv \Theta_{\cX \cX_T}\right) . \label{eqn:krr-cov}
\end{multline}
When the testing set $\cX_T$ consists of only a single point $x$, we use $\sigma^2_\text{NN}(x|\cX)=\Sigma_\text{NN}(\cX_T | \cX)$ to denote the predictive variance at $x$.
The kernel $\cK$ in \eqref{eqn:krr-cov}, which is defined as 
$\cK(x, x') = \expected_{\theta_0 \sim \text{init}(\theta)} [f(x; \theta_0) \cdot f(x' ; \theta_0)]$, 
is the covariance of the NN output with respect to the random initialization. 
Similar to $\Theta$ above, we have used the shorthand notations $\cK_\cX$ and $\cK_{\cX' \cX}$ to represent $\cK(\cX, \cX)$ and $\cK(\cX', \cX)$, respectively.
We will also use $\mu_{\text{NN}, f}$ and $\sigma^2_{\text{NN}, f}$ to represent the predictive mean and variance calculated using a particular model architecture $f$.\vspace{-1mm}

\subsection{Gaussian Processes}
\label{sec:lit-gp}\vspace{-0.5mm}
To derive our active learning criterion (\cref{sec:method}), we will make use of tools from the literature of \emph{Gaussian processes} (GPs) \cite{rasmussenGaussianProcessesMachine2006}.
A GP is fully specified by a prior mean function $\mu(x)$ which is usually assumed to be $\mu(x) = 0\, \ \forall x$, and a covariance function $K(x,x')$ (also called kernel function).
Given some data $(\cX, \tensy)$, the GP posterior predictive distribution of the outputs $\tensy_T$ at a testing set $\cX_T$ (in the noiseless setting) is given by 
$\normdist(K_{\cX_T \cX} K_{\cX}\inv \tensy, K_{\cX_T} - K_{\cX_T \cX} K_{\cX}\inv K_{\cX \cX_T})$.
The principled uncertainty measures provided by GPs have been used for data selection in active learning \cite{krauseNonmyopicActiveLearning2007,krauseNearOptimalSensorPlacements2008,hoang2014active,hoang2014nonmyopic,ling2016gaussian,zhang2016near,nguyen2021information,xu2023fair}.
For example, \citet{krauseNonmyopicActiveLearning2007} have used the mutual information between the selected data and the unlabelled data, calculated using GPs, as the active learning criterion.\vspace{-1.7mm} 

\section{Initialization-Robust Active Learning}
\label{sec:method}\vspace{-0.5mm}

In this section, we firstly introduce a computationally efficient approximation of the output variance of an NN w.r.t.~random initializations, and derive a theoretical guarantee on the approximation quality (\cref{subsec:approximation:of:mu:sigma}).
Next, we show that the approximate output variance from \cref{subsec:approximation:of:mu:sigma}, is also, interestingly, an upper bound on the generalization error of the NN (\cref{subsec:connection:with:gene:error}).
Finally, we use the approximate output variance to design our active learning criterion (\cref{subsec:active:learning:criterion}).\vspace{-0.4mm}



\subsection{Approximation of Output Variance of Trained Neural Network}
\label{subsec:approximation:of:mu:sigma}



In the infinite-width regime, when an NN with initial parameters $\theta_0$ is trained till convergence to yield parameters $\text{train}(\theta_0)$, the output of the NN on a test set $\cX_T$ follows a normal distribution with mean $\mu_\text{NN}(\cX_T|\cX, \tensy)$ (\eqref{eqn:krr-mean}) and covariance $\Sigma_\text{NN}(\cX_T | \cX)$ \eqref{eqn:krr-cov}
\cite{leeWideNeuralNetworks2020}. The randomness in the output distribution results from the random initializations $\theta_0$.
Given this, we see that $\Sigma_\text{NN}$ can serve as a natural and principled measure of initialization robustness.

However, $\Sigma_\text{NN}$ presents a significant computational challenge because it requires computing two different kernels (i.e., $\Theta$ and $\cK$) and a number of matrix inversion and multiplication operations.
Therefore, by drawing inspiration from GPs \cite{rasmussenGaussianProcessesMachine2006}, we introduce an approximation of $\Sigma_\text{NN}$ which is both computationally efficient (\cref{subsubsec:sigma-ntk:comp:effi}) and equipped with a theoretical guarantee on the approximation quality (\cref{subsubsec:sigma-ntk:apprx:quality}).\vspace{-0.4mm}

\subsubsection{Computational Efficiency}
\label{subsubsec:sigma-ntk:comp:effi}
Given some data $(\cX, \tensy)$, performing GP regression with NTK as the covariance function (which we refer to as NTKGP following \citet{heBayesianDeepEnsembles2020}) leads to the output distribution (on a testing set $\cX_T$) of $f(\cX_T) \sim \normdist(\mu_\text{NTKGP}(\cX_T| \cX, \tensy), \Sigma_\text{NTKGP}(\cX_T |\cX))$, where $\mu_\text{NTKGP} = \mu_\text{NN}$ \eqref{eqn:krr-mean}
and\vspace{-0.2mm}
\begin{equation}
\Sigma_\text{NTKGP}(\cX_T |\cX)
= \Theta_{\cX_T} - \Theta_{\cX_T\cX} \Theta_\cX \inv \Theta_{\cX \cX_T}\ .
\label{eq:ntk-gp:sigma}
\end{equation}
When $\cX_T$ contains a single point $x$, we use $\sigma^2_\text{NTKGP}(x|\cX)=\Sigma_\text{NTKGP}(\cX_T | \cX)$ to denote the output variance at $x$.


Compared with $\Sigma_\text{NN}$, the covariance function $\Sigma_\text{NTKGP}$ is more efficient to compute because (a) it only requires computing one (instead of two) kernel $\Theta$ which can also be easily approximated using the inner product of the parameter gradients at initialization (\cref{lr:ntk}), and (b) the posterior distributions of GPs are well-studied, allowing us to adopt existing tools 
to further reduce the computational cost of $\Sigma_\text{NTKGP}$.
Even though $\Sigma_\text{NTKGP}$ incurs a cost of $\cO(n^3)$ with $n$ queried data points due to inversion of an $n$-by-$n$ matrix,
we are able to use results from linear algebra to perform low-rank updates to the GP covariance to reduce the running time to $\cO(n^2)$. 
Furthermore, we can adopt techniques from the abundant literature of sparse GPs to significantly reduce the dependency on $n$ \cite{quinonero-candelaUnifyingViewSparse2005, hoangUnifyingFrameworkAnytime2015}.
We discuss the various approximation methods for $\Sigma_\text{NTKGP}$ in \cref{appx:gp-approx}.\vspace{-0.4mm}


\subsubsection{Guaranteed Approximation Quality}
\label{subsubsec:sigma-ntk:apprx:quality}
To provide a theoretical justification for our approximation $\Sigma_\text{NTKGP}$ \eqref{eq:ntk-gp:sigma}, we need to theoretically bound the difference between $\Sigma_\text{NTKGP}$ and $\Sigma_\text{NN}$.
The work of \citet{heBayesianDeepEnsembles2020} has shown that $\Sigma_\text{NN}(\cX' | \cX) \preceq \Sigma_\text{NTKGP}(\cX' | \cX)$. In other words, for a single test point $x$, we have that $\sigma^2_\text{NN}(x | \cX) \leq \sigma^2_\text{NTKGP}(x | \cX)$, which suggests that the NTKGP approximation \eqref{eq:ntk-gp:sigma} is an overestimation of the true output variance \eqref{eqn:krr-cov}.
However, this result is not enough for guaranteeing a small approximation error.
Therefore, we prove here a stronger connection between $\sigma^2_\text{NN}(x | \cX)$ and $\sigma^2_\text{NTKGP}(x | \cX)$:
\begin{theorem}[Informal]
\label{thm:relu-bound-informal}
Let $f(\cdot; \theta)$ be an infinite-width NN with ReLU activation and $L \geq 2$ hidden layers whose NTK satisfies $|\Theta(x, x')| \leq B$ for all $x, x' \in \cX$. 
Then, there exist some constants $\alpha>0$ and $\beta = \cO \left( \textup{poly} (|\cX|, B, L,  \eigmin{\Theta_\cX}\inv ) \right)$ such that\vspace{-0.9mm}
\begin{equation*}
\left| \sigma_\textup{NN}^2(x | \cX) - \alpha \cdot \sigma^2_\textup{NTKGP}(x | \cX) \right|
\leq \beta\ .\vspace{-0.9mm}
\end{equation*}
\end{theorem}
\cref{thm:relu-bound-informal} shows that $\sigma_\text{NN}^2(x | \cX)$ and $\alpha \cdot  \sigma^2_\text{NTKGP}(x | \cX)$ have a bounded difference and are hence expected to behave similarly.
Of note, when using the active learning criterion based on $\sigma^2_\text{NTKGP}$ (\cref{subsec:active:learning:criterion}) to select data points to query, what affects the selected points is only the \emph{relative} ranking of the values of the criterion (at different inputs in the unlabelled pool), therefore, the presence of the constant $\alpha>0$ does not affect our active learning algorithm.
The approximation error in \cref{thm:relu-bound-informal} depends on a number of factors including the model architecture (which affects the difference between $\Theta$ and $\cK$), the number of points in $\cX$ (due to the accumulation of the approximation errors with more data points), and the eigenvalues of $\Theta_\cX$ (which affects how ``well-behaved" the matrix $\Theta_\cX$ is).


As a brief sketch of the proof, firstly, it can be verified that if $\cK(x, x') = \alpha \cdot \Theta(x, x')$, then $\sigma_\text{NN}^2(x | \cX) = \alpha \cdot \sigma^2_\text{NTKGP}(x | \cX)$. However, in general, $\cK(x, x') \neq \alpha \cdot \Theta(x, x')$. Instead, we have managed to show that the ratio between $\cK(x, x')$ and $\Theta(x, x')$ is bounded, i.e. there exists some constants $a_->0$ and $a_+>0$ (which depend on $L$) such that
\begin{equation}
\label{eqn:kernels-bound}
a_- \leq \cK(x, x') / \Theta(x, x') \leq a_+\ .
\end{equation}
As a result, \eqref{eqn:kernels-bound} allows us to bound $\big| \sigma_\text{NN}^2(x | \cX) - \alpha \cdot \sigma^2_\text{NTKGP}(x | \cX) \big|$.
The complete proof is presented in \cref{appx:relubound}.
We will also provide empirical justifications for \cref{thm:relu-bound-informal} in \cref{exp:nn-ntkgp} by showing that $\sigma^2_\text{NTKGP}$ is indeed highly correlated with the empirical output variance of the NN resulting from different model initializations, and is hence a reliable indicator of initialization robustness.\vspace{-1.5mm}



\subsection{Connection with Generalization Error}
\label{subsec:connection:with:gene:error}\vspace{-0.5mm}
In this section, we show that the approximate output variance $\sigma^2_\text{NTKGP}$ \eqref{eq:ntk-gp:sigma} is also an upper bound on the generalization error of the trained NN and hence a good indicator of its predictive performance.
To analyze the performance of the trained neural network, we make the following assumption about 
the groundtruth
function $f^*$ in a manner similar to \citet{vakiliUniformGeneralizationBounds2021}.
\begin{assumption}
	\label{assump:vak}
	Assume that the groundtruth function $f^* \in \cH_\Theta$.
        Specifically, $f^*$ lies in the reproducing kernel Hilbert space (RKHS) of the NTK $\Theta$, or equivalently, its RKHS norm is such that $\norm{f^*}_{\cH_\Theta} \leq B$ for some $B \in \real_{\geq 0}$. 
        Further assume that the function observation at any input $x_i$ is given by $y_i = f^*(x_i) + \xi_i$, in which every $\xi_i$ is i.i.d.~observation noise drawn from an $R$ sub-Gaussian distribution: $\expected[\exp ({\eta \xi_i})] \leq \exp(\eta^2 R^2 / 2)$ for all $\eta \in \real$.
\end{assumption}
Both assumptions in \cref{assump:vak} are commonly made in the analysis of kernelized and neural bandits \cite{chowdhuryKernelizedMultiarmedBandits2017,kassraieNeuralContextualBandits2022}.
They allow us to show the following theoretical guarantee on the generalization error:
\begin{theorem}[Informal]
\label{thm:err-informal}
Suppose we train an infinitely wide NN $f(\cdot; \theta)$ with training dataset $(\cX, \tensy)$ on mean-squared error loss function using gradient descent until convergence. Then, there exists a constant $\zeta = \cO\big( \text{\normalfont poly}(B, R) \big)$ such that for any $x \in \cX$, with high probability over the random observation noise $\boldsymbol{\eps}$ and network initialization $\theta_0$,
\begin{equation}
\big| {f^*(x) - f(x; \text{\normalfont train}(\theta_0))} \big|
\leq \zeta \cdot \sigma_\text{\normalfont NTKGP}(x | \cX). \label{eqn:err-inf}
\end{equation}
\end{theorem}
\cref{thm:err-informal} shows that the generalization error of a trained NN through gradient descent is proportional to $\sigma_\text{NTKGP}$ \eqref{eq:ntk-gp:sigma}, where the constant of proportionality $\zeta$ is \textit{independent} of $x$ and $\cX$.
As a result, minimizing $\sigma_\text{NTKGP}$ will not only (a) decrease the output variance (\cref{subsec:approximation:of:mu:sigma}) and hence \emph{improve initialization robustness}, but also (b) reduce the generalization error and hence \emph{enhance the predictive performance}.
The degree of correlation between $\sigma_\text{NTKGP}$ and the generalization performance, represented by the constant $\zeta$, depends on the parameters $B$ and $R$, such that the easier the function $f^*$ is to learn (i.e., a smaller $B$) or the less noisy the observations are (i.e., a smaller $R$), the better the degree of correlation.
\cref{thm:err-informal} is also consistent with the empirically observed characteristics of over-parameterized NNs, because NN models with lower variance are also observed to have higher predictive accuracy \cite{nealModernTakeBiasVariance2018}.
\cref{thm:err-informal} will be stated formally and proved in \cref{appx:loss-bound}.\vspace{-0.5mm}

\subsection{Active Learning Criterion}
\label{subsec:active:learning:criterion}
Since we have shown that minimizing $\sigma_\text{\normalfont NTKGP}(x | \cX)$ can improve both \emph{initialization robustness} (\cref{subsec:approximation:of:mu:sigma}) and \emph{generaliztion performance} (\cref{subsec:connection:with:gene:error}), we design our active learning criterion based on the minimization of $\sigma_\text{\normalfont NTKGP}(x | \cX)$ across all test input points $x \in \cX_T$.
Specifically, our \alg~criterion encourages the selection of input data points which result in small expected output variance $\sigma^2_\text{\normalfont NTKGP}(x | \cX)$ across all test inputs, and we estimate the expected variance by averaging $\sigma^2_\text{\normalfont NTKGP}(x | \cX)$ over the available test set $\cX_T$:
\begin{equation}
\begin{array}{c}
\displaystyle
\alpha_\text{EV}(\cX) = \frac{1}{|\cX_T|} \sum_{x \in \cX_T} \left[ \sigma^2_\text{NTKGP}(x|\emptyset) - \sigma^2_\text{NTKGP}(x|\cX) \right].\vspace{-3mm}
\end{array}
\label{eq:criterion}\vspace{1mm}
\end{equation}
We have added $\sigma^2_\text{NTKGP}(x|\emptyset)$ to the criterion so that $\alpha_\text{EV}(\cX) \geq 0$ and that our criterion is to be maximized during active learning.\footnote{This is to follow convention of other active learning methods. } 
We will sometimes also use $\alpha_\text{EV}(\cX; f)$ to indicate that the criterion uses the model architecture $f$.

Our $\alpha_\text{EV}$ criterion \eqref{eq:criterion} has multiple computational benefits. 
\emph{Firstly}, it is training-free, i.e., its calculation does not require any training of the NN and is hence able to sidestep significant computational costs resulting from model training.
\emph{Secondly}, it only requires calculating the variance at individual test points rather than the full covariance over the testing set. 
\emph{Thirdly}, it can make use of the approximation techniques based on sparse GPs discussed in \cref{subsubsec:sigma-ntk:comp:effi}, for which we simply need to replace $\sigma^2_\text{NTKGP}$ by its sparse GP counterparts in \eqref{eq:criterion}.
\emph{Fourthly}, it is monotone submodular, and therefore a greedy approach (i.e., select the point which gives the largest increase in the criterion in each selection round) is guaranteed to give a $(1 - \frac{1}{e})$-optimal solution \cite{nemhauserAnalysisApproximationsMaximizing1978}.
We adopt the greedy approach in our experiments for simplicity (with some 
techniques for speedups 
which we discuss in \cref{appx:alg}), and leave the use of other more sophisticated submodular optimization techniques to future works.
Furthermore, another advantage of our $\alpha_\text{EV}$ criterion is that it is label-independent, because the calculation of $\sigma^2_\text{NTKGP}$ \eqref{eq:ntk-gp:sigma} does not require the observations.
Therefore, our criterion does not need the heuristic of pseudo-labels which is required by previous active learning algorithms \cite{ashDeepBatchActive2020, mohamadiMakingLookAheadActive2022}.


In addition to our $\alpha_\text{EV}$ criterion \eqref{eq:criterion}, we can also use $\Sigma_\text{NTKGP}$ to construct alternative criteria with different characteristics. 
We introduce two additional criteria in \cref{appx:crit}, which are based on, respectively, mutual information (which requires computing the full covariance matrix) and the percentile variance (which is not submodular in general).






\section{Initialization-Robust Active Learning with Model Selection}
\label{sec:model-sel}

A common issue with existing neural active learning algorithms is that a model architecture has to be fixed in advance and then used for the data point selection.
However, in practice, especially when having no access to (labelled) data beforehand, it is infeasible to select the best model architecture prior to running the neural active learning algorithms.
To this end, by leveraging the output distributions of overparameterized NNs in a similar way to \cref{subsec:approximation:of:mu:sigma}, we extend our initialization-robust active learning algorithm (\cref{sec:method}) to simultaneously select the data points to query (\cref{subsec:active:learning:criterion}) \emph{and} optimize the model architecture in a training-free manner.

Given a model architecture $f$ 
and a training set $\cD = (\cX, \tensy)$, the expected squared error of the trained model on a testing set $\cD_T$ w.r.t.~random parameter initializations is given by (proof in \cref{appx:model-sel-proof})\vspace{-0mm}
\begin{equation}
\hspace{-3.1mm}
\begin{array}{l}
\displaystyle
\hat{\alpha}_{M, \cD_T}(f; \cD) 
\triangleq \underset{\theta_0 \sim \text{init}(\theta)}{\expected} \left[\ell\left(\cD_T, \text{train}(\theta_0)\right)\right]  \\
\displaystyle 
=\frac{1}{2} \sum_{(x, y) \in \cD_T} \big[ \underbrace{\left(y - \mu_{\text{NN}, f}( x | \cX)\right)^2}_{\text{\textcircled{1}}} + \underbrace{\sigma^2_{\text{NN}, f}(x | \cX)}_{\text{\textcircled{2}}} \big]. \vspace{-2mm}
\end{array}
\label{eqn:exp-loss-ms}
\end{equation}
The first term in \eqref{eqn:exp-loss-ms}, \textcircled{1}, characterizes how well the trained model is able to fit the underlying function, which is related to its generalization performance. 
The second term, \textcircled{2}, represents the predictive variance of the trained model, which is an indicator of the complexity of the model.
A good model architecture should be expressive enough to fit the underlying function well (i.e., have a small \textcircled{1}), while also not being too sensitive to the parameter initialization (i.e., have a small \textcircled{2}).

\eqref{eqn:exp-loss-ms} allows us to design our model selection method based on cross validation using the queried data $\cD$ during active learning.
Specifically, we adopt the cross validation method of bootstrapping \cite{kohaviStudyCrossvalidationBootstrap1995}: we select a random subset $\cD_T \subset \cD$ of size $\kappa$ as the testing set, and compute \eqref{eqn:exp-loss-ms} using $\cD \setminus \cD_T$ as the training set.
As a result, this leads to the following criterion for model selection:\footnote{The negative sign is added to convert to maximization.}
\begin{equation}
\alpha_M(f; \cD) = 
- \underset{\cD_T \subset \cD;\ |\cD_T| = \kappa}{\expected} \big[ \hat{\alpha}_{M, \cD_T}(f; \cD \setminus \cD_T) \big]. 
\label{eqn:ms-crit}
\end{equation}
In practice, the user chooses an appropriate $\kappa$ and computes the empirical mean to approximate the expectation in \eqref{eqn:ms-crit}. Since $\alpha_M$ is computed far fewer times than $\alpha_\text{EV}$ during the active learning process, it is reasonable to use $\sigma^2_\text{NN}$ directly rather than approximating it with $\sigma^2_\text{NTKGP}$.
As a result, in our algorithm for model selection, given a set of candidate model architectures $\cM = \{f_1, \ldots, f_m\}$, we evaluate \eqref{eqn:ms-crit} for every architecture $f\in\cM$ and subsequently choose the architecture which maximizes this criterion.

The full algorithm with both data and model selection, which we name \algms, is shown in \eqref{alg:al}. 
The algorithm alternates between two phases. The first phase uses a fixed model architecture to greedily maximize our $\alpha_\text{EV}$ criterion \eqref{eq:criterion}
for data selection, 
and the second phase utilizes the queried data so far to select the best model architecture using our criterion in \eqref{eqn:ms-crit}.
\vspace{-0.8mm}

\begin{algorithm}[tb]
\caption{\algms}
\label{alg:al}
\begin{algorithmic}
\STATE {\bfseries Input:} Initial labelled data $(\cX_0, \tensy_0)$, unlabelled pool $\cX_U$, candidate model architectures $\cM$, batch size $b$
\STATE $(\cX_\cL, \tensy_\cL) \gets (\cX_0, \tensy_0)$
\STATE Pick an initial model $f^* \in \cM$
\REPEAT
\STATE \textit{// Phase 1: Data selection}
\FOR{$b$ iterations}
\STATE $x^* \gets {\arg\max}_{x \in \cX_U \setminus \cX_\cL}\ \alpha_\text{EV}(\cX_\cL \cup \{ x \}; f^*)$
\STATE $\cX_\cL \gets \cX_\cL \cup \{ x^* \}$
\ENDFOR
\STATE Query the unlabelled points in $\cX_\cL$ for the labels $\tensy_\cL$
\STATE \textit{// Phase 2: Model selection}
\STATE $f^* \gets {\arg \max}_{f \in \cM}\  \alpha_M\big(f; (\cX_\cL, \tensy_\cL)\big)$
\UNTIL{budget exhausted}
\STATE {\bfseries return} $(\cX_\cL, \tensy_\cL), f^*$
\end{algorithmic}
\end{algorithm}
\section{Experiments}
\label{sec:exp}

When reporting the model output variance or average performance, we train an NN 50 times (regression) or 25 times (classification) with the same architecture but with different model initializations, and then calculate the empirical mean and variance. 
All experiments are repeated 5 times unless stated otherwise, and their mean and standard deviations are reported.
Although the theoretical properties of $\Sigma_\text{NN}$ and $\Sigma_\text{NTKGP}$ are applicable to infinite-width NNs, we follow the practice of previous works on NTKs \cite{heBayesianDeepEnsembles2020, mohamadiMakingLookAheadActive2022} and use finite-width NNs, because they are able to achieve good performances.
In our experiments, we test our algorithm using both the theoretical NTKs (computed using the \textsc{Jax}-based \cite{jax2018github} \texttt{Neural-Tangents} package \cite{novakNeuralTangentsFast2019}), and the empirical NTK (computed using PyTorch).
We will use \alg\textsc{-Emp} to denote instances when we use the empirical NTK for our algorithm. We discuss the computation of NTKs further in \cref{appx:ntk-comp}.
We adopt the MSE loss \eqref{eqn:loss} for regression experiments and the cross-entropy loss for classification experiments, which is consistent with previous works on NTK \cite{liuFindingTrainableSparse2020,shuNASILABELDATAAGNOSTIC2022}. We find that even though the NTK theory is developed based on MSE loss, prior works utilizing NTKs have shown that it is also effective in predicting behaviors of models trained under cross-entropy loss as well.
We compare our algorithm with previous baselines which also require minimal model training between different batches and do not incur significant extra computations. These benchmarks are described further in \cref{appx:exp-other-alg}.
In particular, we compare with random selection, \textsc{K-Means++} \cite{arthurKmeansAdvantagesCareful2007}, BADGE \cite{ashDeepBatchActive2020}, and MLMOC \cite{mohamadiMakingLookAheadActive2022}. The former two algorithms, like our algorithm, can select all the points in a single batch, whereas the latter two are not designed for such a setting but are applicable after modifications.
We omit results comparing our algorithm against \textsc{BatchBALD} \cite{kirschBatchBALDEfficientDiverse2019} in the main paper due to the method requiring a Bayesian neural network, although their results have been included in \ref{appx:exp-batchbald}.
We have chosen to focus on the cases with low query budgets since we find this is when the neural networks tend to show a larger difference in predictive performances, and so the active learning algorithm needs to be more careful in selecting which data points to query as they will have a larger impact on the final selected models. This setting studied in our experiments is realistic since it simulates the situations where there is little or no initial labeled data and querying any data incurs a large cost.
The code for the experiments can be found at \url{https://github.com/apivich-h/init-robust-al}.
Other experimental details are deferred to \cref{appx:exp-details} due to space limitation.

\subsection{Correlations Between $\sigma^2_\text{NTKGP}$ and Neural Network Output Variance}
\label{exp:nn-ntkgp}
Here we study whether our approximate output variance $\sigma^2_\text{NTKGP}$ (\cref{subsec:approximation:of:mu:sigma}) can accurately reflect the output variance of NNs (w.r.t.~the random initializations) and hence the initialization robustness.
\cref{fig:correlation-var} plots the \textit{individual} variance predicted by our NTKGP (i.e., $\sigma_\text{NTKGP}^2(x|\cX)$ for some $x$ and $\cX$) against the empirically observed output variance resulting from different random initializations. 
The figure verifies that $\sigma^2_\text{NTKGP}$ is highly correlated with the observed output variance of the NN and the variances are generally confined within some region, which provides an empirical corroboration for our \cref{thm:relu-bound-informal}.
This justifies our choice of using $\sigma^2_\text{NTKGP}$ to measure the output variance w.r.t.~the model initializations and hence the initialization robustness.
In addition, in \cref{appx:exp-sntkgp}, we present further experimental results to show that the output variance predicted using sparse GP approximations, which is more computationally efficient (\cref{subsubsec:sigma-ntk:comp:effi}), is also highly correlated with the observed output variance.


\begin{figure}[t]
\centering
\includegraphics[width=0.5\linewidth]{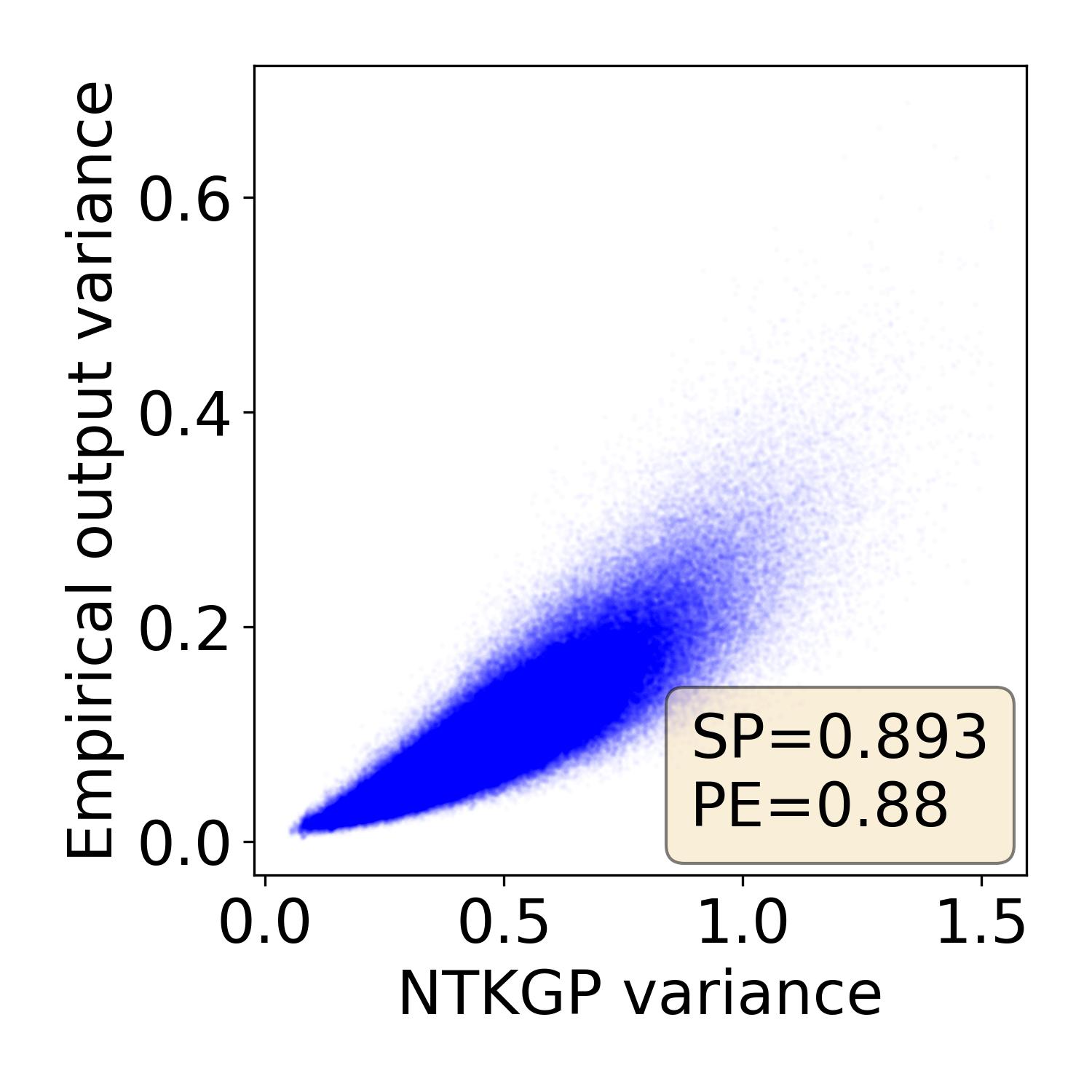}
\vspace{-6mm}
\caption{
Correlation between our approximate output variance $\sigma_\text{NTKGP}^2(x|\cX)$ and the empirical NN output variance. The full description of the graph is given in \cref{appx:fig-desc}.\vspace{-1mm}
}
\label{fig:correlation-var}
\end{figure}

\subsection{Experiments on Regression Tasks}
\label{exp:small}

\begin{figure}[t]
\centering
{\sffamily \scriptsize Protein}\\
\includegraphics[width=0.66\linewidth]{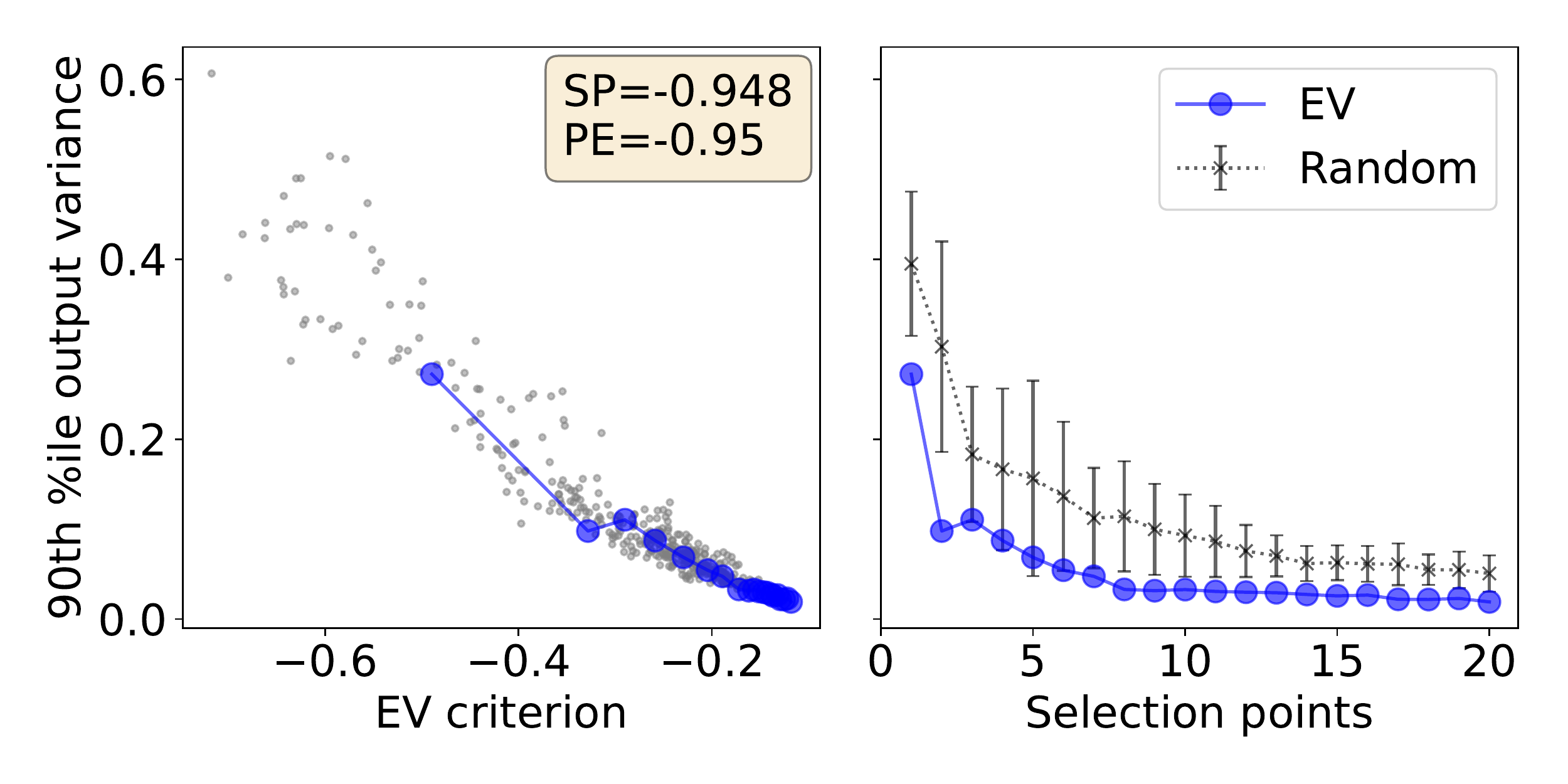}\hspace{-2mm}
\includegraphics[width=0.33\linewidth]{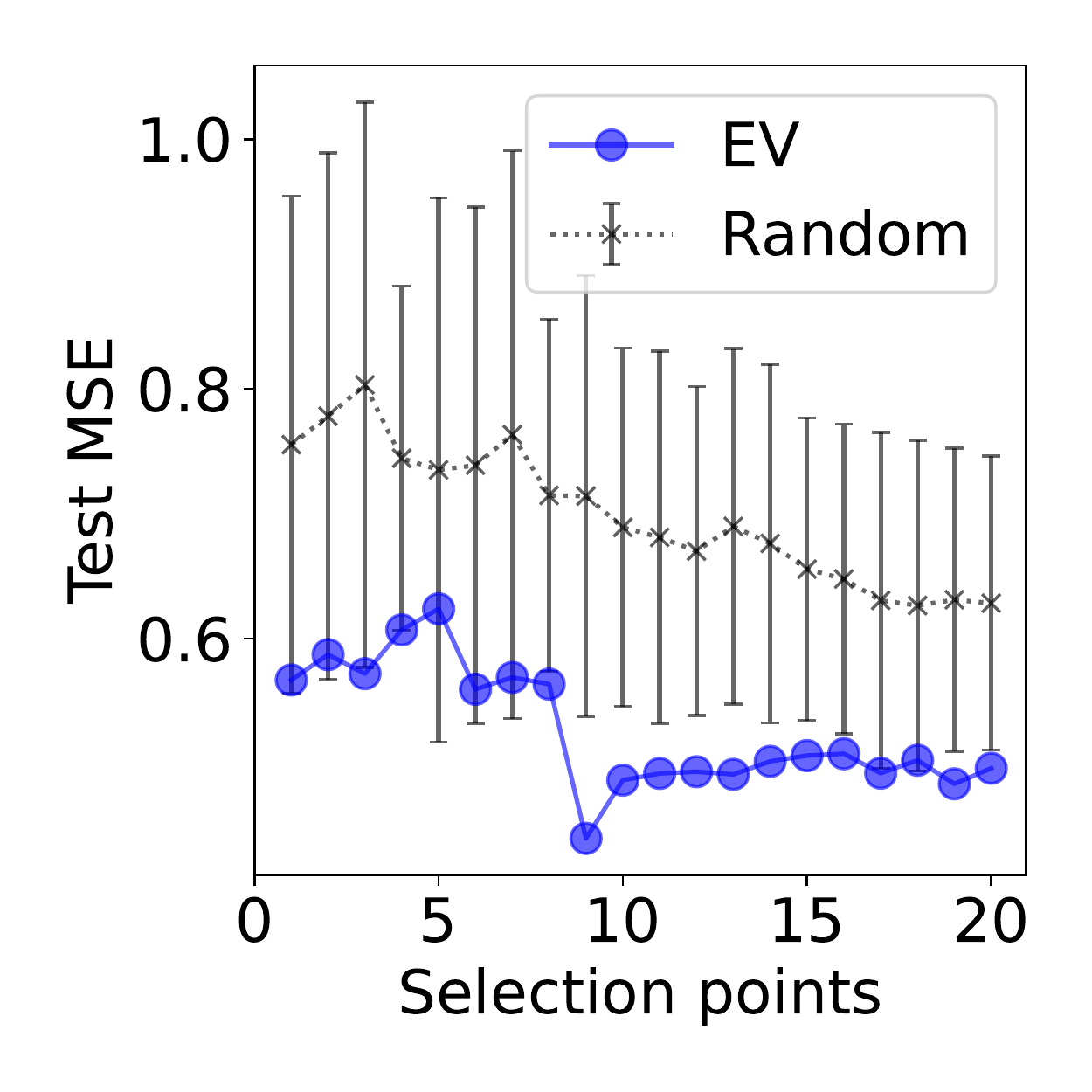}\hspace{-2mm}

\vspace{-5mm}
\caption{
Results of sequential data selection in regression tasks (discussed in \cref{exp:small} and detailed descriptions in \cref{appx:fig-desc}).
}
\label{fig:small-set}\vspace{-2.3mm}
\end{figure}

Here we evaluate our \alg~criterion\footnote{When reporting the value $\alpha_\text{EV}$, we will ignore the $\sigma^2_\text{NN}(x|\emptyset)$ terms and instead report the average of $- \sigma^2_\text{NN}(x|\cX)$.} (\cref{subsec:active:learning:criterion}) using regression tasks.
In the experiments here, each algorithm is given an unlabelled pool of data and no initial labelled data, and all methods use a 2-layer MLP with ReLU activation.

In \cref{fig:small-set}, our \alg~criterion is used to sequentially select the data points (i.e. the batch size is 1).
In the first column of \cref{fig:small-set}, we plot the 90th percentile output variance (i.e. 90\% of the test points have lower output variance than this value)
as the vertical axis, against the values of our \alg~criterion as the horizontal axis.
The gray dots show that our \alg~criterion is highly correlated with output variance, and the blue dots, which display the points selected by our \alg~criterion during active learning, demonstrate that our \alg~criterion is able to select points which lead to low output variance (since the selected points are mostly clustered in the bottom right corner).
This is also corroborated by 
the middle column of \cref{fig:small-set}, which shows that sequentially maximizing our \alg~criterion indeed leads to the selection points which progressively reduce the output variance, and our \alg~criterion consistently outperform random search.
The third column of \cref{fig:small-set} shows that the points selected by maximizing our \alg~criterion also sequentially reduce the test MSE and hence improve the predictive performance of the NN.
Therefore, the second and third columns of \cref{fig:small-set} combine to provide an empirical justification for our \cref{thm:err-informal}, which has theoretically shown that minimizing the approximate output variance $\sigma_\text{\normalfont NTKGP}$ (which is achieved by maximizing our \alg~criterion) also improves the generalization performance of overparameterized NNs.

We have also tested our \alg~criterion in the more practical active learning setting where a batch of points are selected in every round (with a batch size of 20).
\cref{fig:small-set-batch} shows that in the batch setting, our \alg~criterion is still able select batches of points which lead to both low output variance (first column) and small test error (second column), and outperforms the other baselines.
We include more experimental results for the regression tasks in \cref{appx:exp-regr}, which are consistent with those shown here (\cref{fig:small-set} and \cref{fig:small-set-batch}).
A particularly interesting additional result is \cref{fig:appx-regr-bias}, which shows that an easier regression task leads to a larger degree of correlation between the output variance and the test error. 
This, interestingly, is consistent with \cref{thm:err-informal}, because it has theoretically shown that an easier task (i.e., a simpler groundtruth function which is indicated by a smaller $B$) reduces the value of $\zeta$ on the right hand side of \cref{thm:err-informal}, which consequently increases the degree of correlation between the output variance (i.e., right hand side) and the generalization error (i.e., left hand side).



\begin{figure}[t]


\centering
{\sffamily \scriptsize Robot Kinematics}\\
\includegraphics[width=0.4\linewidth]{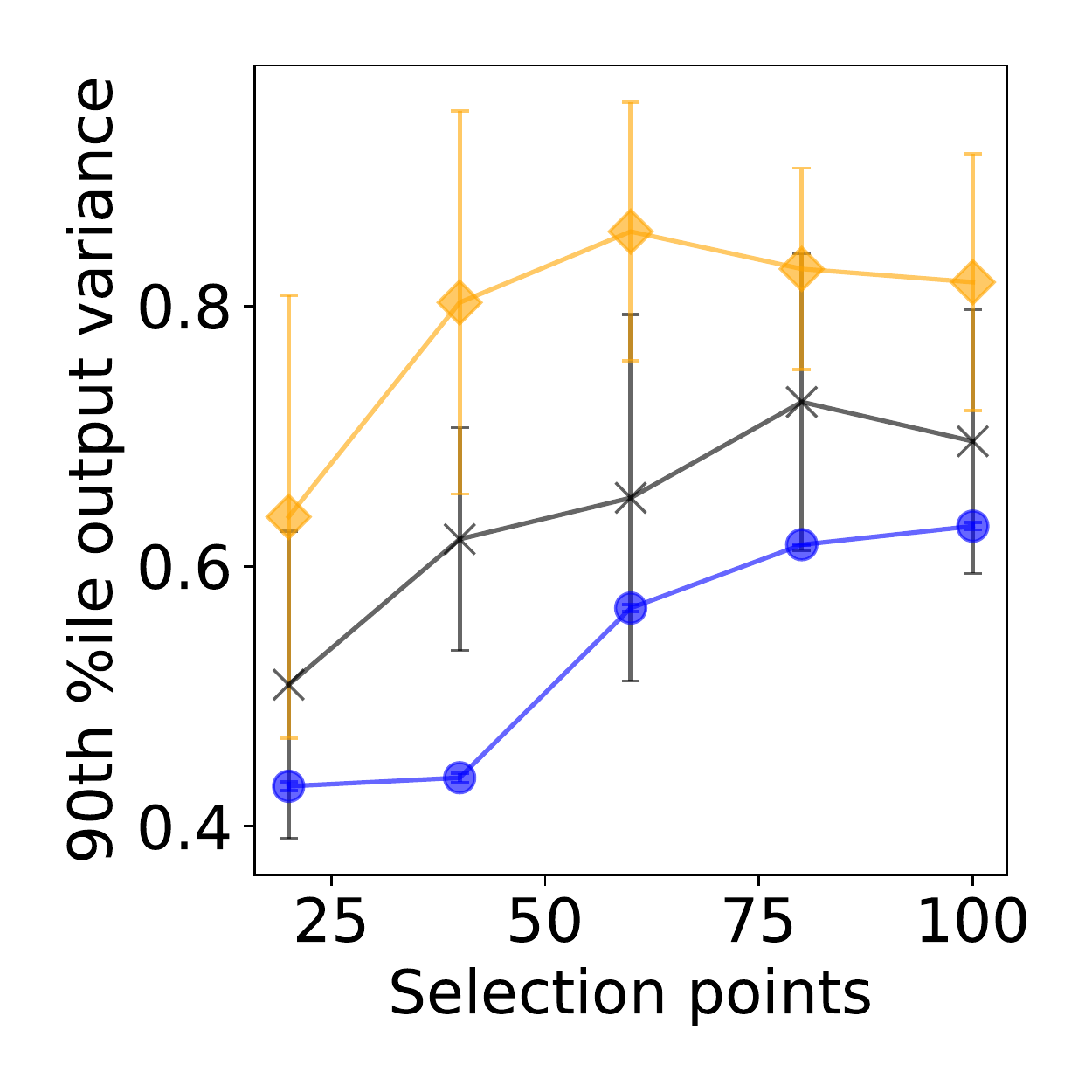}\hspace{-2mm}
\includegraphics[width=0.4\linewidth]{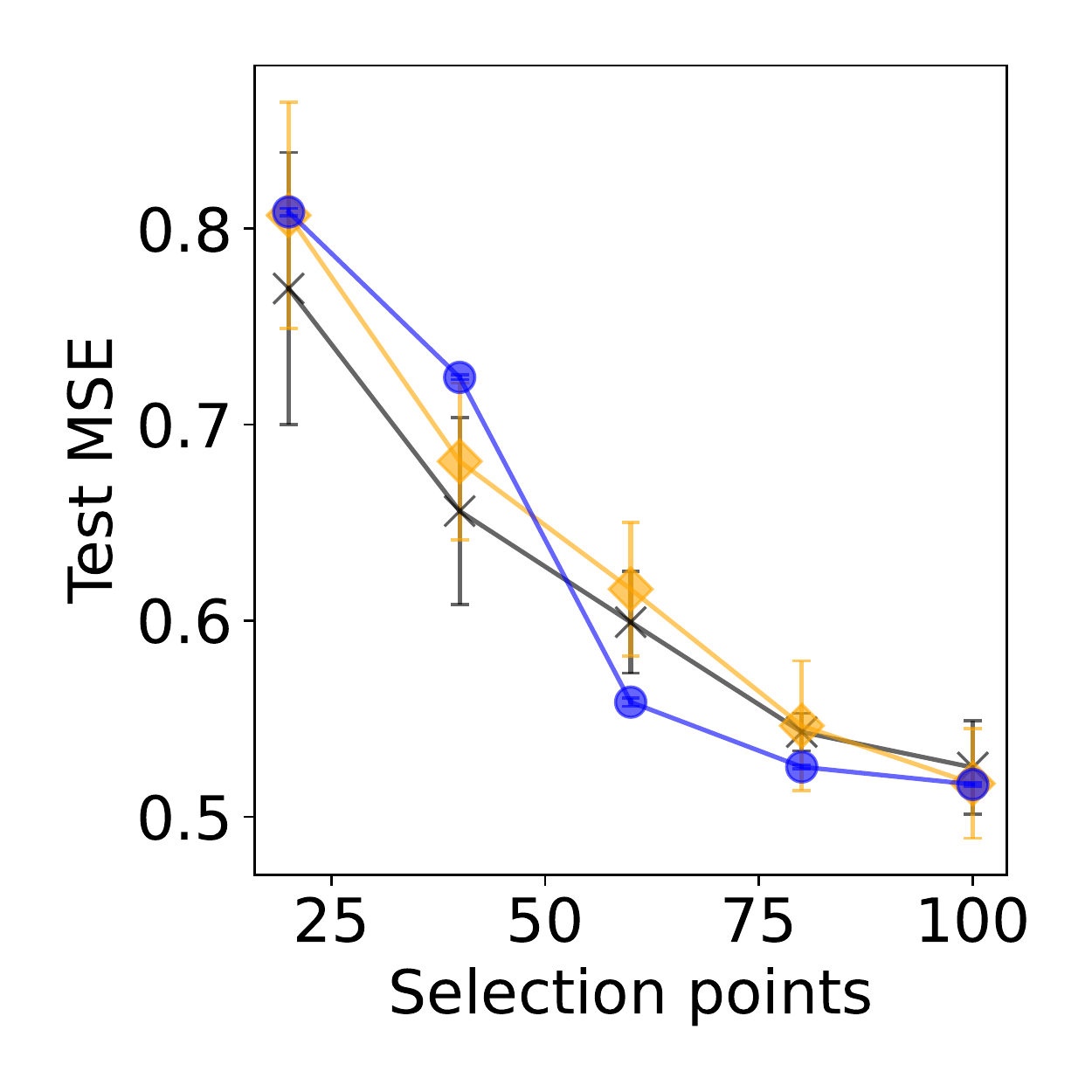}\hspace{-2mm}
\\
\vspace{-3mm}
{\sffamily \scriptsize Protein}\\
\includegraphics[width=0.4\linewidth]{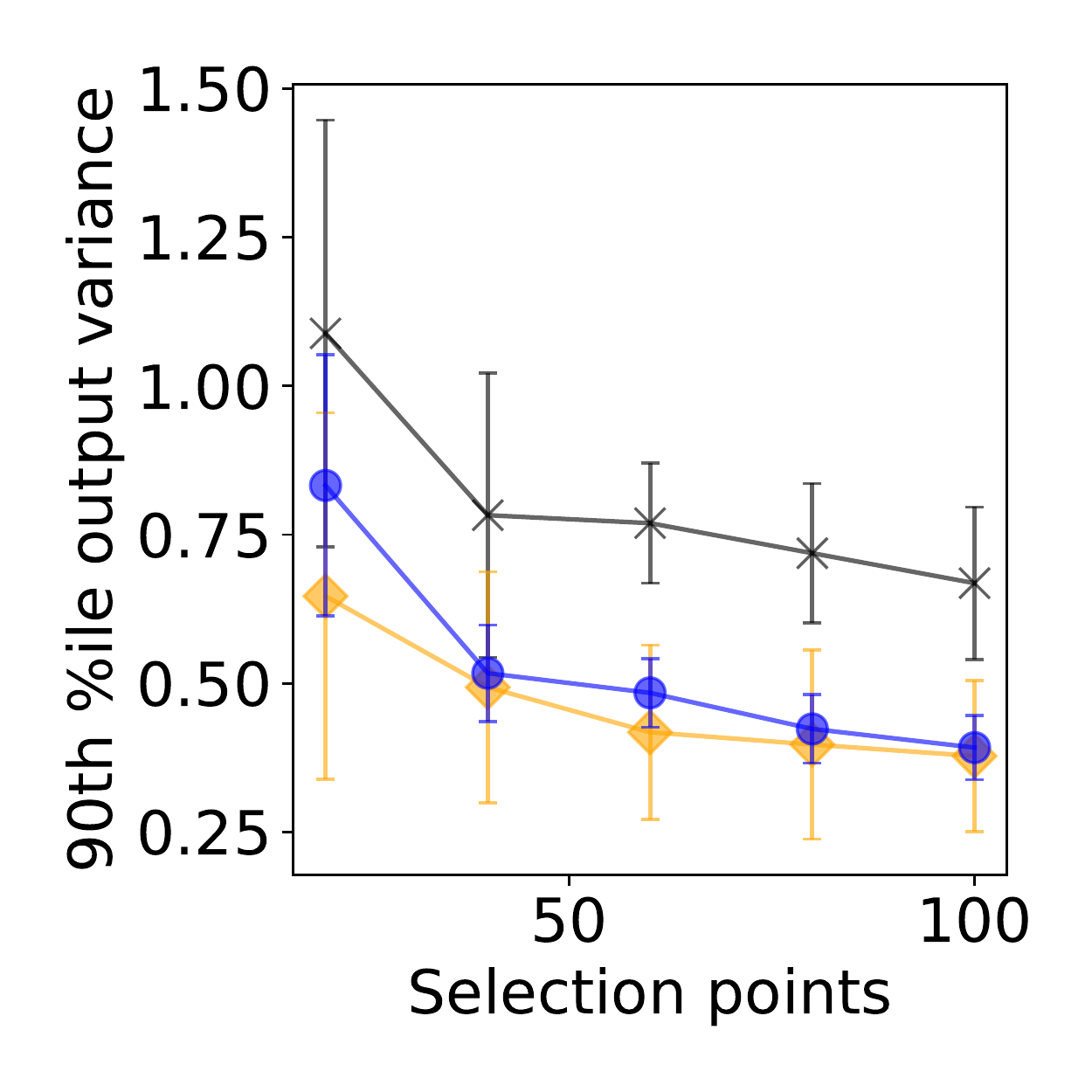}\hspace{-2mm}
\includegraphics[width=0.4\linewidth]{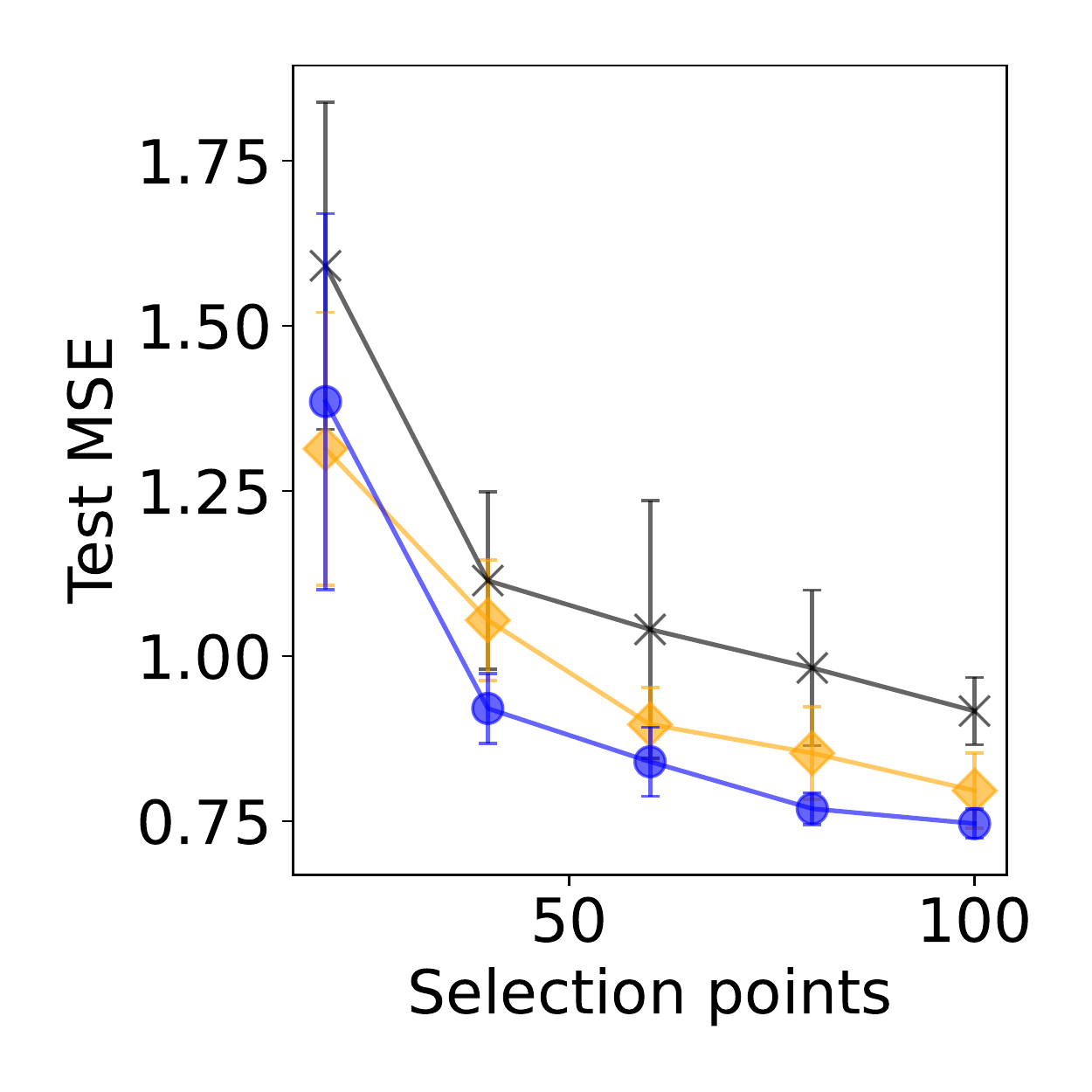}\\
\vspace{-2mm}
\includegraphics[width=0.7\linewidth]{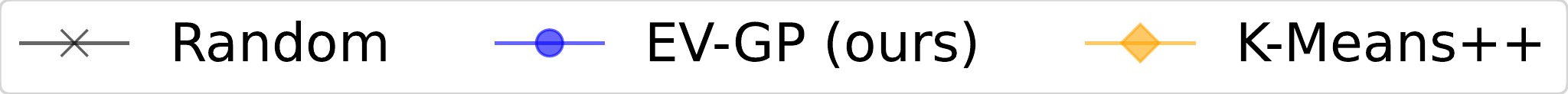}\\
\vspace{-2mm}
\caption{
Results on regression tasks. 
\textit{Left:} output variance of test predictions after training using the labelled active set. 
\textit{Right:} test MSE. 
The $x$-axis represents the size of the selected active set. More details about the metrics are in \cref{appx:exp-metrics}.
}
\label{fig:small-set-batch}
\end{figure}

\subsection{Experiments on Classification Tasks}
\label{exp:classification}

Here we apply our \alg~criterion to classification tasks.
We use a wider variety of NN architectures, including MLPs (ReLU activation), convolutional NNs (CNNs), and WideResNets \cite{zagoruykoWideResidualNetworks2017} which has also been used in experiments of \citet{mohamadiMakingLookAheadActive2022}.

\paragraph{Performance Comparison.}
\cref{fig:results-class-all}(a-b) presents the comparison of our \alg~criterion with other baselines, in which all methods use MLPs. 
The figures show that our \alg~criterion is indeed able to select points which lead to both initialization robustness (i.e., low output entropy plotted in the first column) and good generalization performances (i.e., high test accuracy shown in the second column).
The results are consistent with those for the regression tasks (\cref{exp:small}).
Moreover, our \alg~criterion outperforms the other baselines in \cref{fig:results-class-all}(a-b), especially in the earlier rounds when there is a small number of selected points.
\cref{fig:results-class-all}(c-d) plots the results using more sophisticated NN architectures (i.e., CNNs and WideResNets), in which our \alg~criterion also consistently outperforms the other baselines in terms of the test accuracy. 
Further experimental results on other dataset and model architectures (such as ResNet18) are also provided in \cref{appx:exp-class-more-res}.

\begin{figure*}[t]
\centering
\begin{tabular}{c|c|cc}

\multicolumn{2}{c|}{\includegraphics[width=0.3\linewidth]{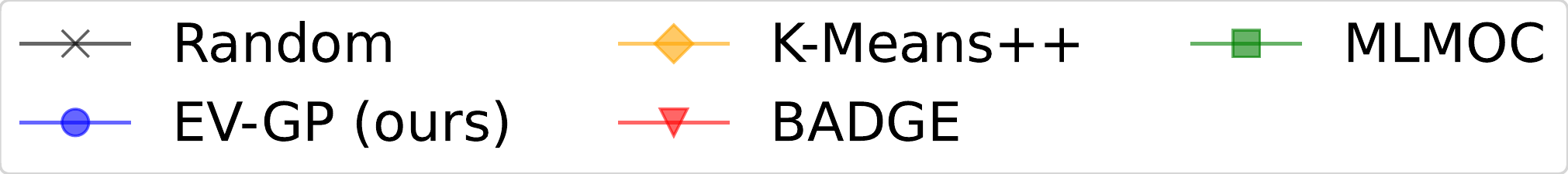}} &
\multicolumn{2}{c}{\includegraphics[width=0.3\linewidth]{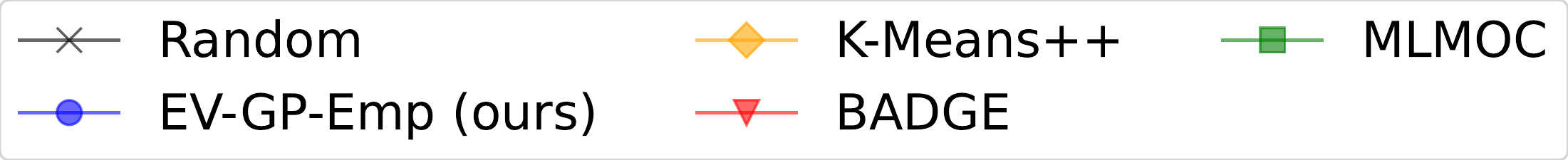}} \\

\includegraphics[width=0.15\linewidth]{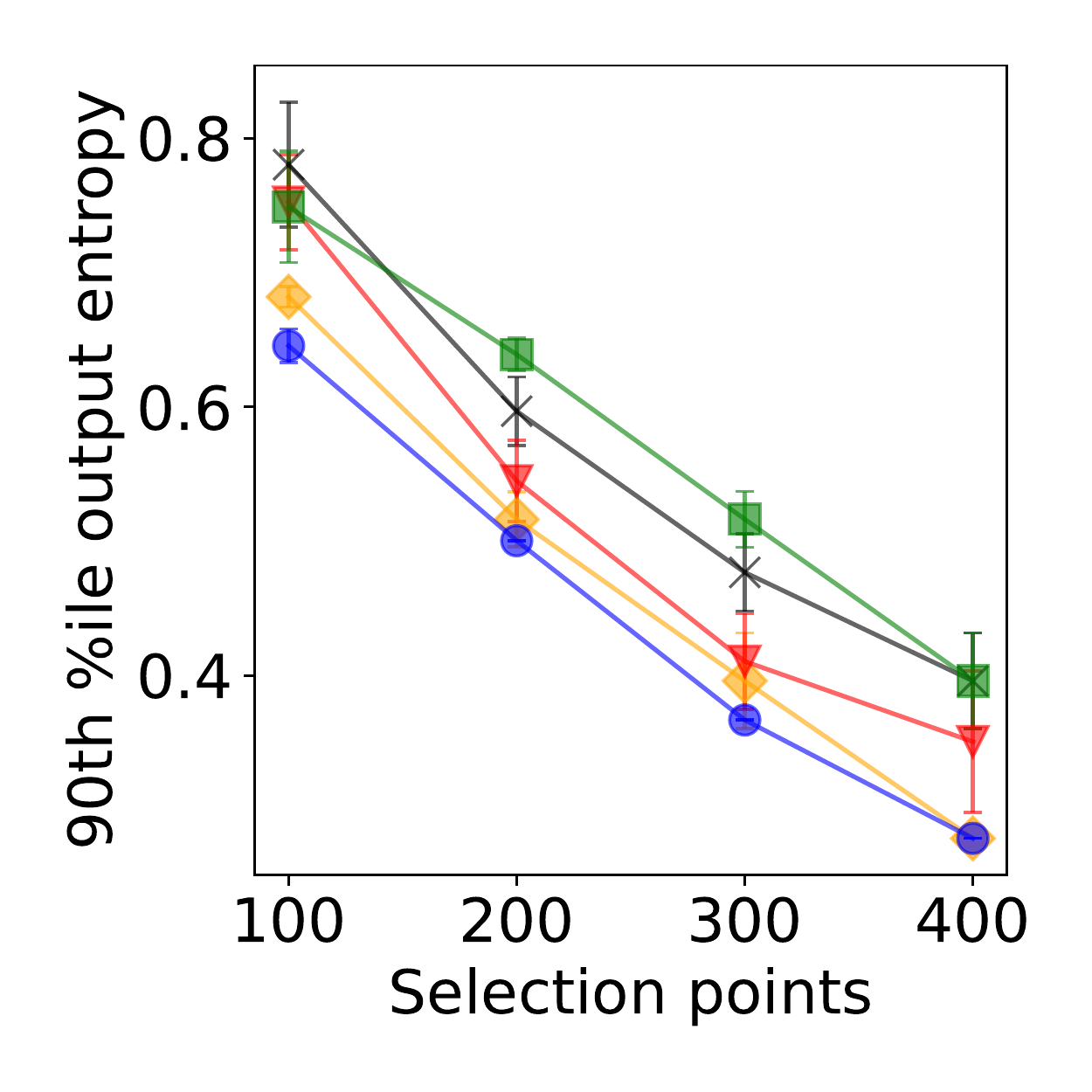}\hspace{-2mm}
\includegraphics[width=0.15\linewidth]{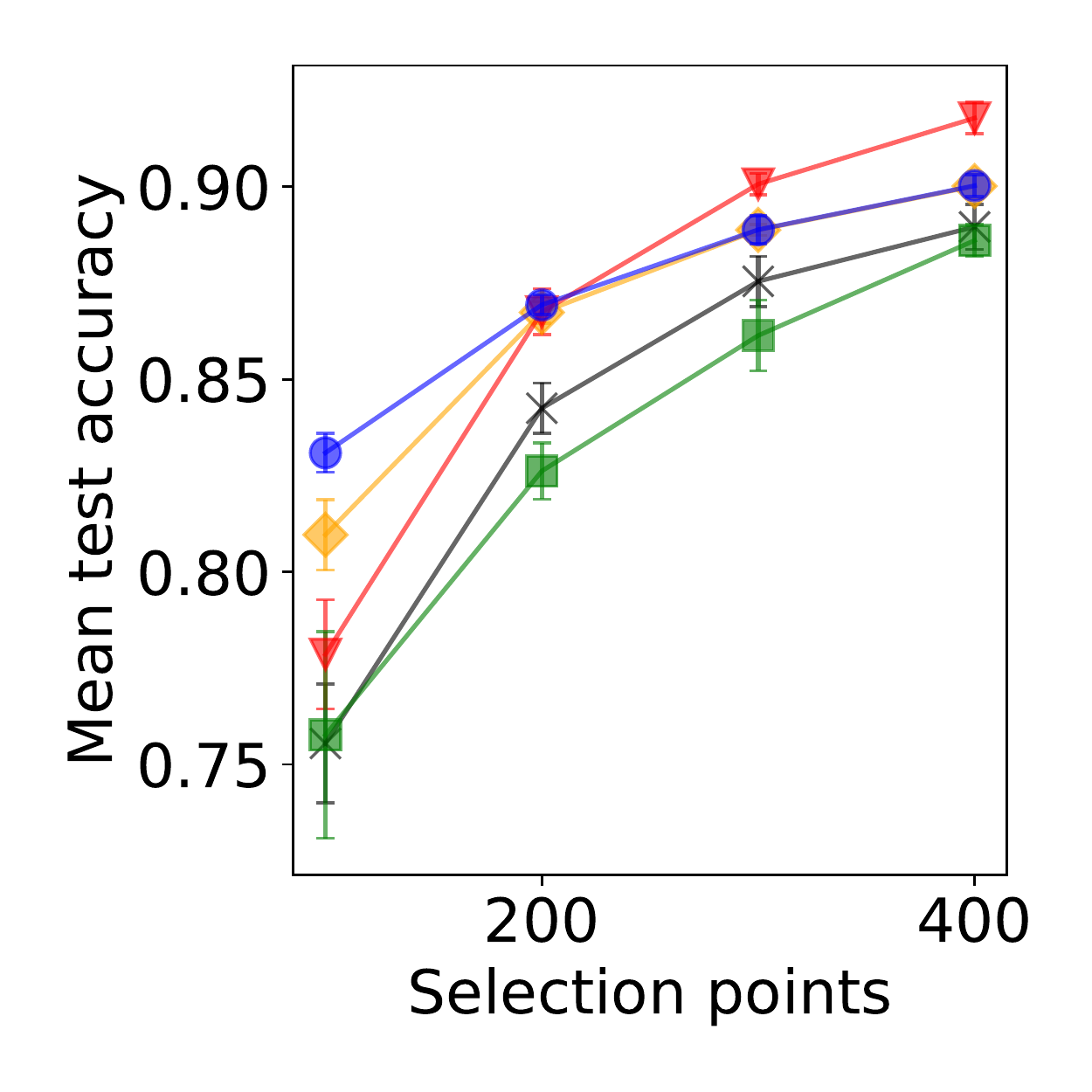}&

\includegraphics[width=0.15\linewidth]{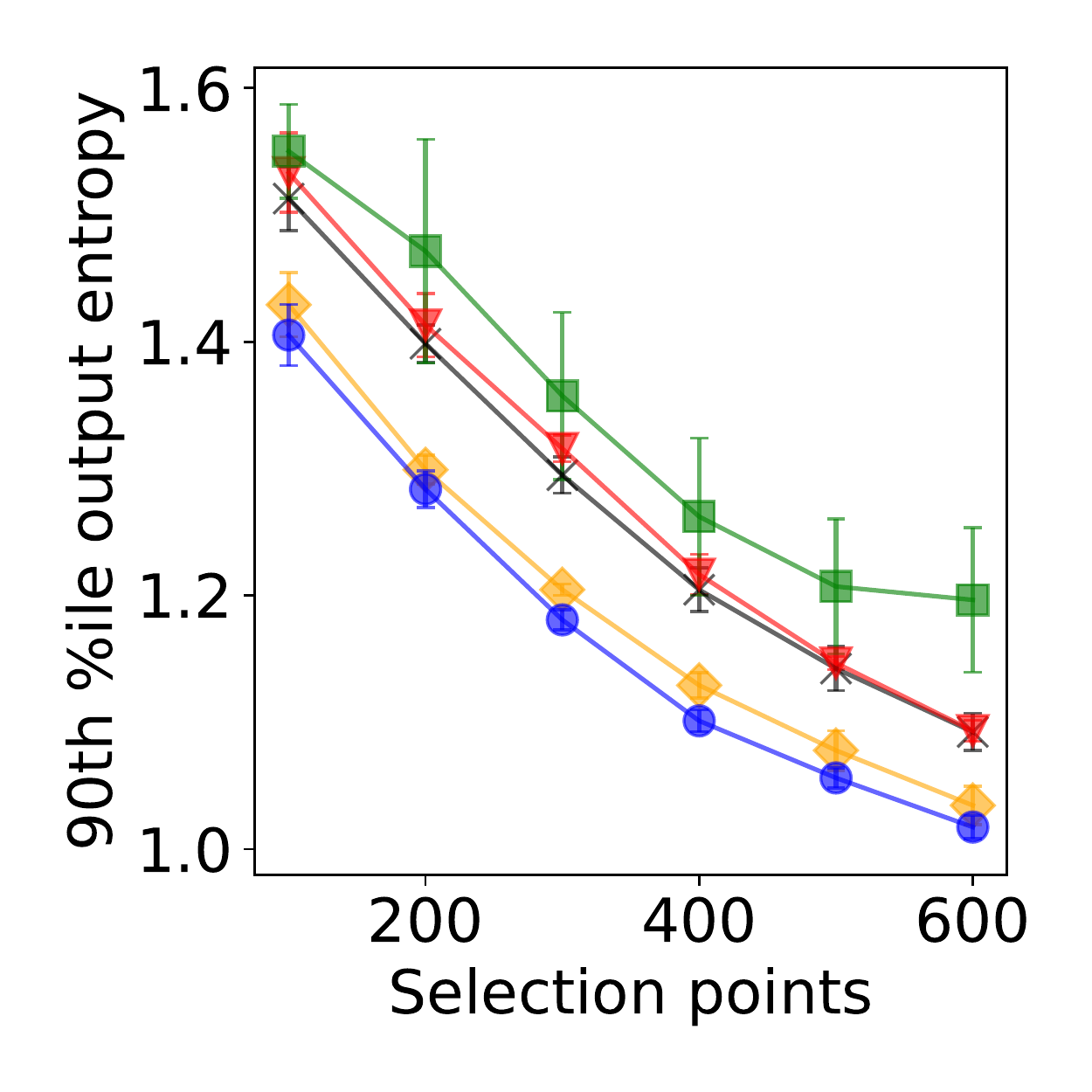}\hspace{-2mm}
\includegraphics[width=0.15\linewidth]{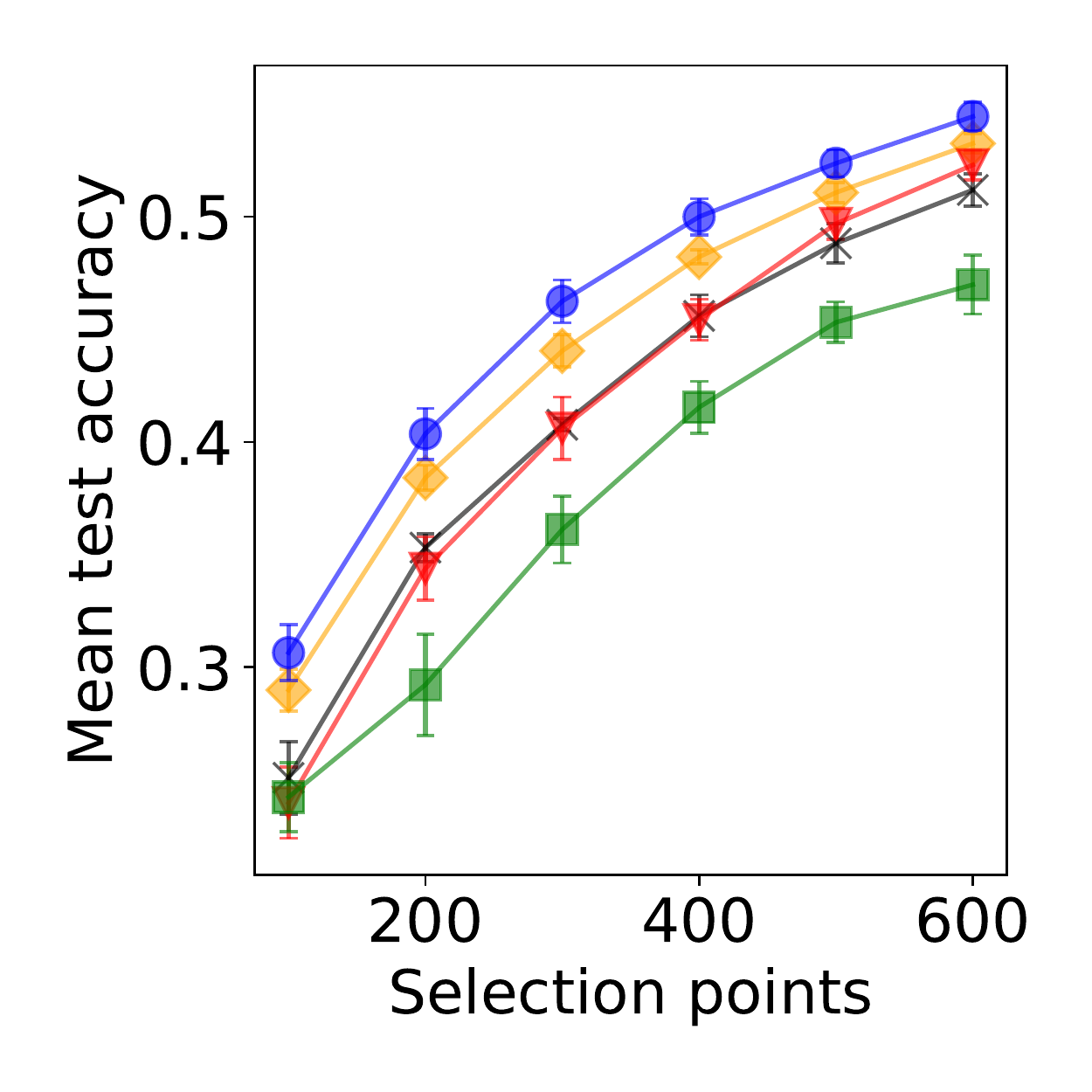}&

\includegraphics[width=0.15\linewidth]{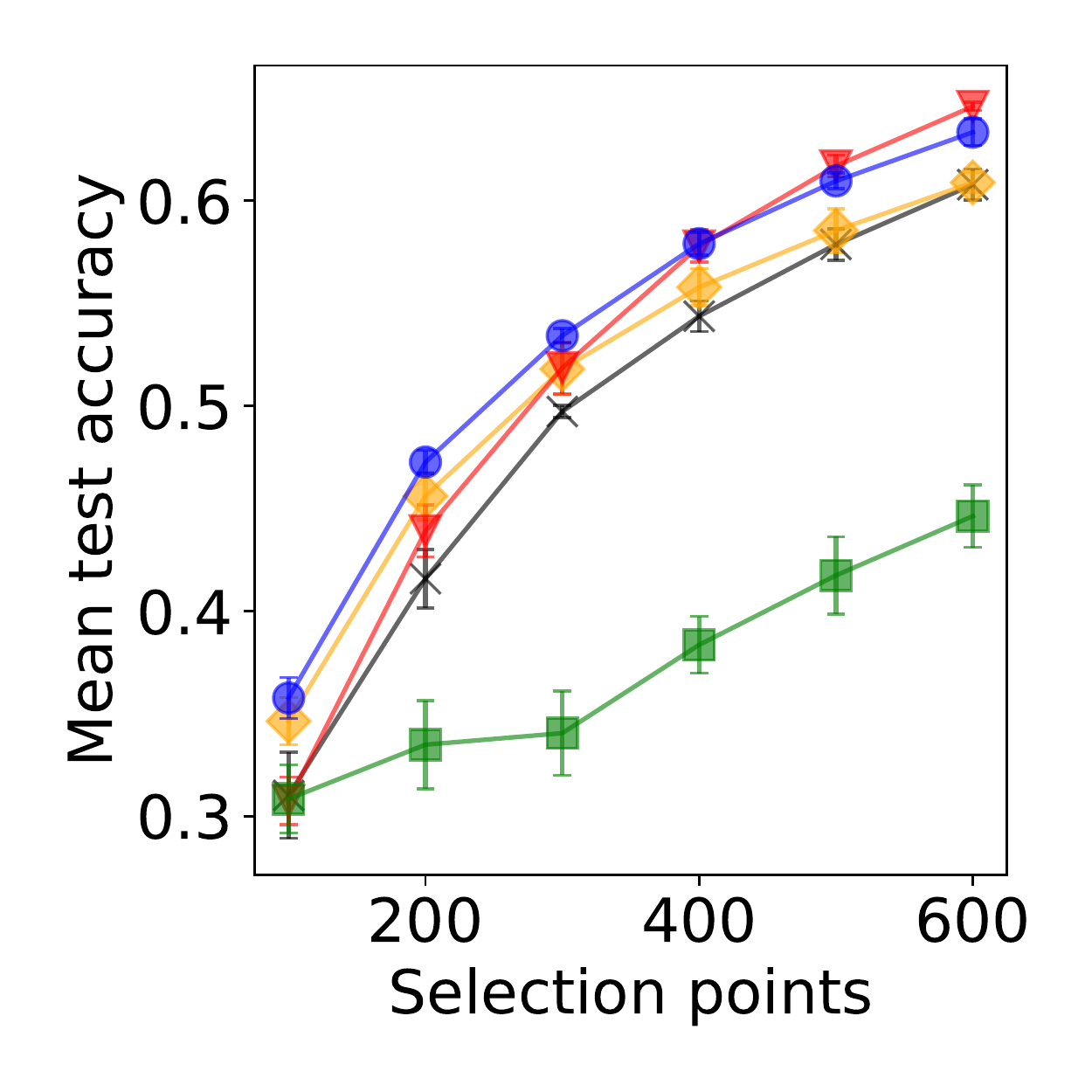} & 

\includegraphics[width=0.15\linewidth]{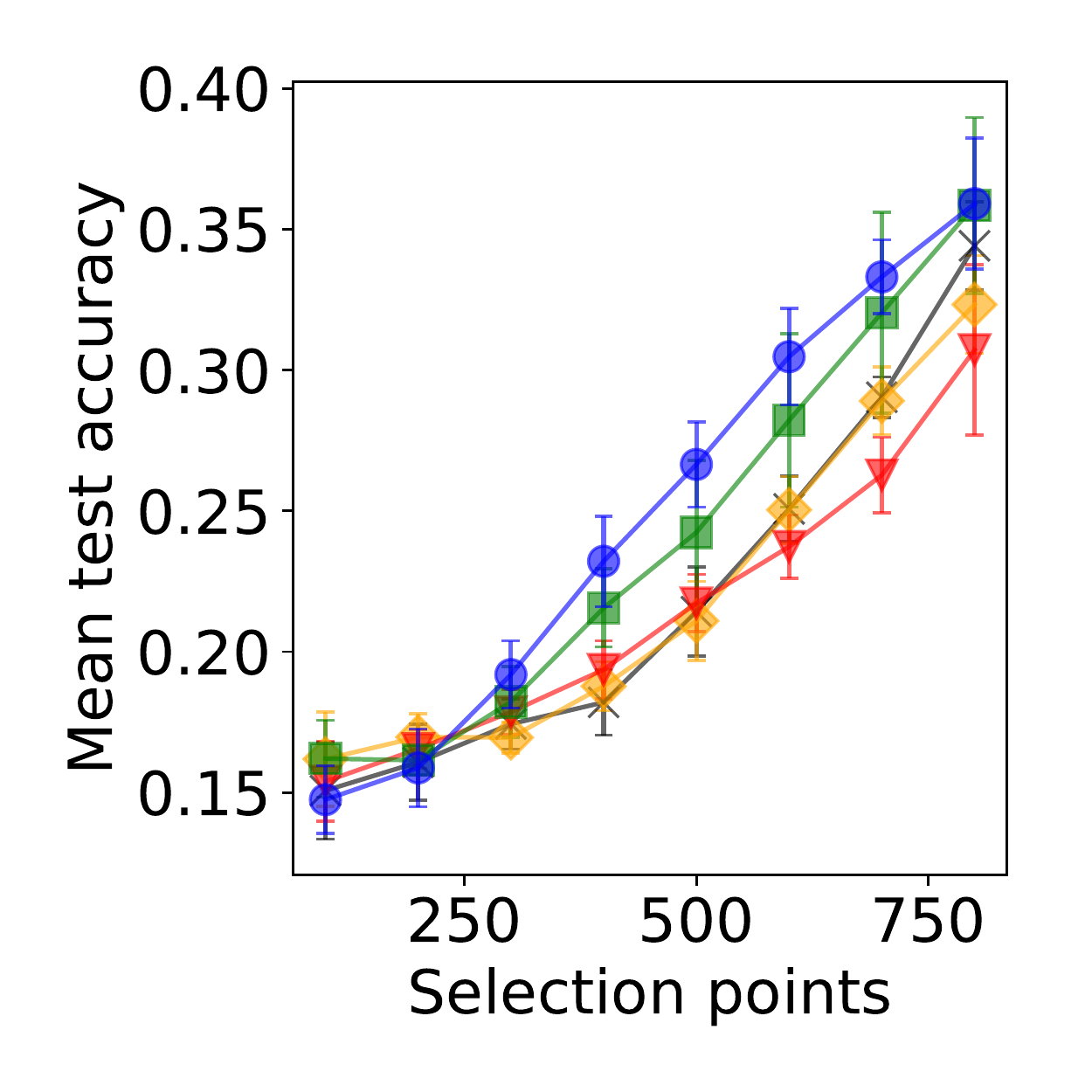} \vspace{-3mm}\\

{\sffamily \scriptsize (a) MNIST, 2-layer MLP} & 

{\sffamily \scriptsize (b) EMNIST, 3-layer MLP} & 

{\sffamily \scriptsize (c) EMNIST, 2-layer CNN}  & 

{\sffamily \scriptsize (d) SVHN, WideResNet}
\end{tabular}

\caption{
Results of active learning on classification tasks using NNs.
(a-b) the output entropy and mean test accuracy of test predictions for experiments involving MLPs. (c-d) the mean test accuracy for experiments involving more complex models with convolutions. More details about the metrics are provided in \cref{appx:exp-metrics}.
}
\label{fig:results-class-all}\vspace{-2mm}
\end{figure*}

\paragraph{Effects of the Batch Size.}
Here we examine the impact of the batch size on the performances of different active learning algorithms, by fixing the total query budget and varying the batch size.
The results in \cref{fig:results-vary-batch} show that \alg\textsc{-Emp} (which uses the empirical NTK) is minimally affected by the increasing batch size\footnote{We omit \textsc{Random}, \textsc{K-Means++} and \alg~ from the graph since the selection algorithm is independent of batch size.}.
This is reasonable because these algorithms do not require the labels, therefore, a smaller batch size (i.e., more frequent availability of the labels) has no impact on the performance.
In contrast, when the batch size is increased, MLMOC experiences a large drop in performance for both datasets and the performance of BADGE is significantly decreased for MNIST. This may be mainly attributed to their reliance on the labels, because a larger batch size reduces the frequency of the availability of, and hence their abilities to use, labels.
Moreover, another factor which causes the detrimental effect of a larger batch size on MLMOC is that a larger batch size is likely to reduce the diversity of the selected points \cite{mohamadiMakingLookAheadActive2022}.
We provide further discussions and additional results involving varying query batch sizes in \cref{appx:exp-batch-sz}.




\begin{figure}[t]
\centering

\begin{tabular}{cc}
{\sffamily \scriptsize MNIST (2-layer MLP)}  & 
{\sffamily \scriptsize SVHN (WideResNet)} \\ 
\includegraphics[width=0.43\linewidth]{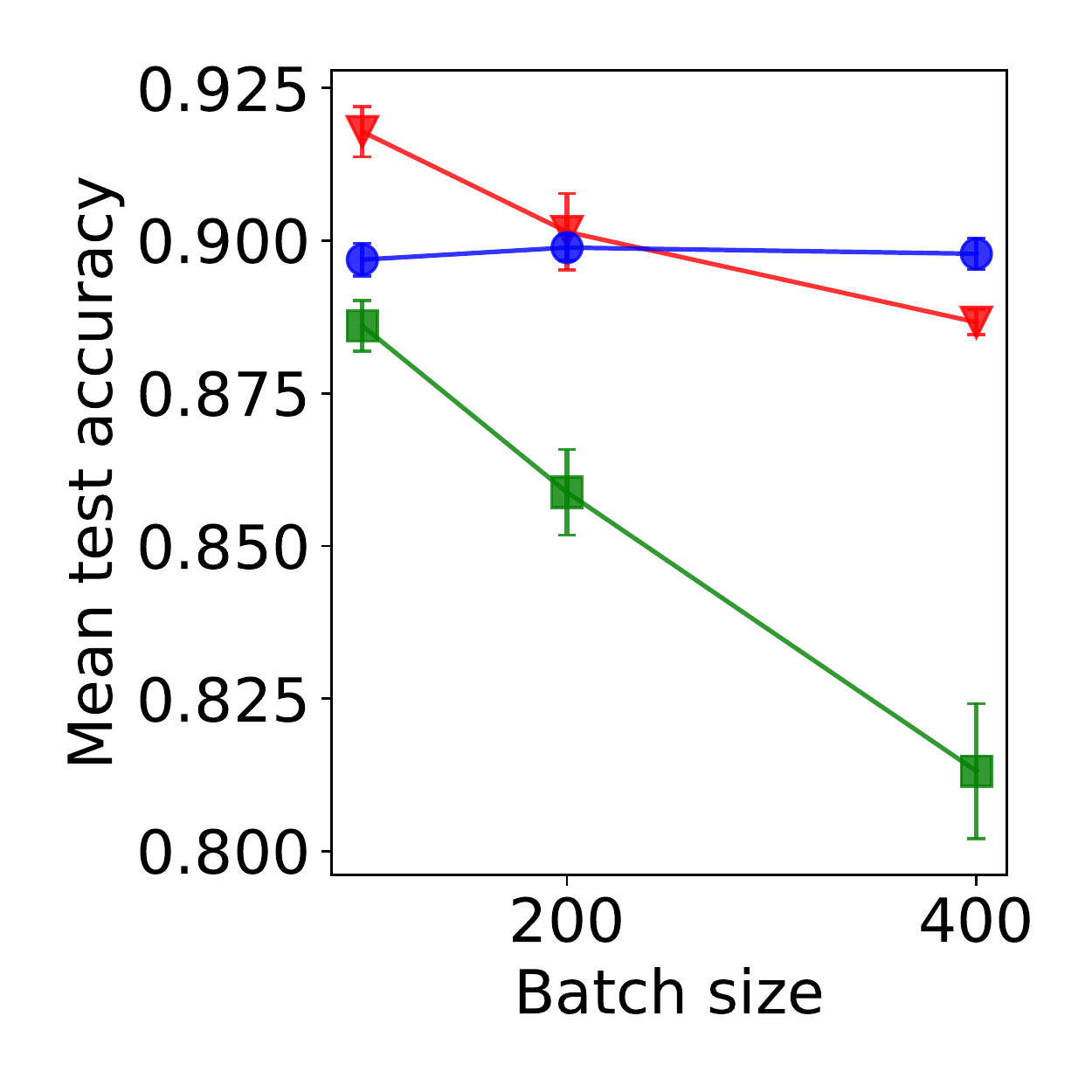}\hspace{-2mm}\vspace{-1mm} &
\includegraphics[width=0.43\linewidth]{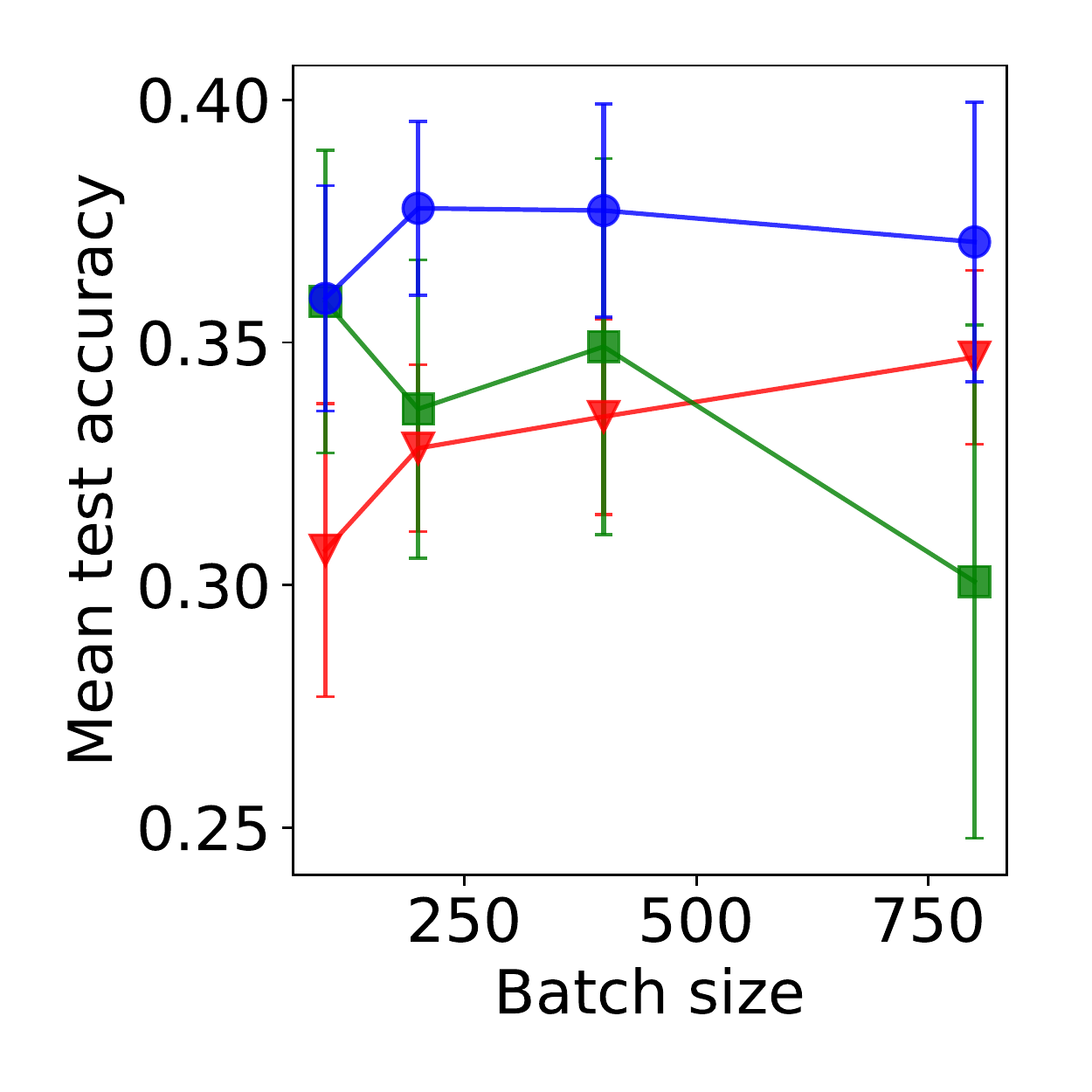}\vspace{-1mm}
\end{tabular}
\vspace{-4mm}
\includegraphics[width=0.8\linewidth]{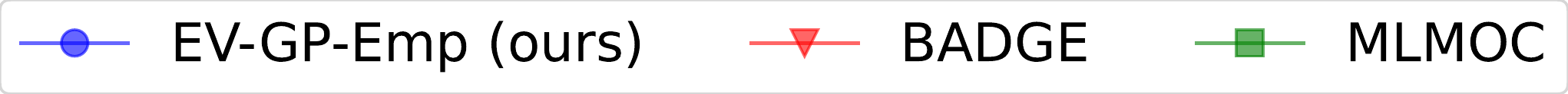}

\caption{
Results on classification with varying batch sizes.
}
\label{fig:results-vary-batch}
\end{figure}





\paragraph{Computation Costs of \alg.}
We find that the cost for only running \alg\xspace is not significantly higher than other active learning baseline algorithms, while still being able to outperform these baselines.
Nonetheless, in practical scenarios, the cost for querying labels will often dominate the cost for active data selection. Since \alg\xspace is training-free and minimally affected by query batch size, in practical scenarios \alg\xspace will be less affected by these high labelling costs and be more advantageous than other active learning algorithms. We discuss the results regarding the computation time of \alg\xspace further in \cref{appx:exp-running-time}.

\paragraph{Effects of the Network Width.}
We find that using a wider NN at training time improves both the model accuracy and initialization robustness, while increasing the width of the NN for active learning (i.e., data selection) only affects the initialization robustness yet has negligible effects on the resulting model accuracy. We explore these results further in \cref{appx:exp-model-width}.
\subsection{Active Learning with Model Selection}
\label{subsec:exp:with:ms}
Here we evaluate the effect of our model selection algorithm (\cref{sec:model-sel}), by comparing the performance of our \algms\xspace algorithm (with model selection) with that of our \alg\xspace algorithm using a fixed model architecture (2-layer MLP with ReLU activation) and with model selection using NASWOT \cite{mellorNeuralArchitectureSearch2021}.
The left figure in \cref{fig:results-ms} shows that for some datasets for which the fixed architecture already leads to a good performance, our \algms\xspace with model selection is able to perform on par with that of \alg\xspace with the fixed architecture.
For some other datasets where the fixed architecture is \textit{not} adequate (e.g. when a deeper model or a different activation function can better model the data), our \algms\xspace is able to discover a better model architecture and hence achieve a lower loss (the right figure in \cref{fig:results-ms}).
The performance of \algms\xspace can be attributed to the fact that it is coherent with our active learning criterion. Specifically, although it is possible to run other model selection algorithms together with \alg\xspace criterion, our model selection criterion \eqref{eqn:ms-crit} is developed based on the same theoretical foundation as our active learning criterion \eqref{eq:criterion}. Therefore, our model selection algorithm leads to a coherent framework for joint data and model selection, hence resulting in better empirical performances as shown.
We provide further discussions and additional results for \algms\xspace in \cref{appx:ms-res}.

\begin{figure}[t]
\centering
\begin{tabular}{c c}
\hspace{-2mm}{\sffamily \scriptsize Robot Kinematics}  & 
{\sffamily \scriptsize Naval}\vspace{-1mm}\\
\hspace{-2mm}\includegraphics[width=0.43\linewidth]{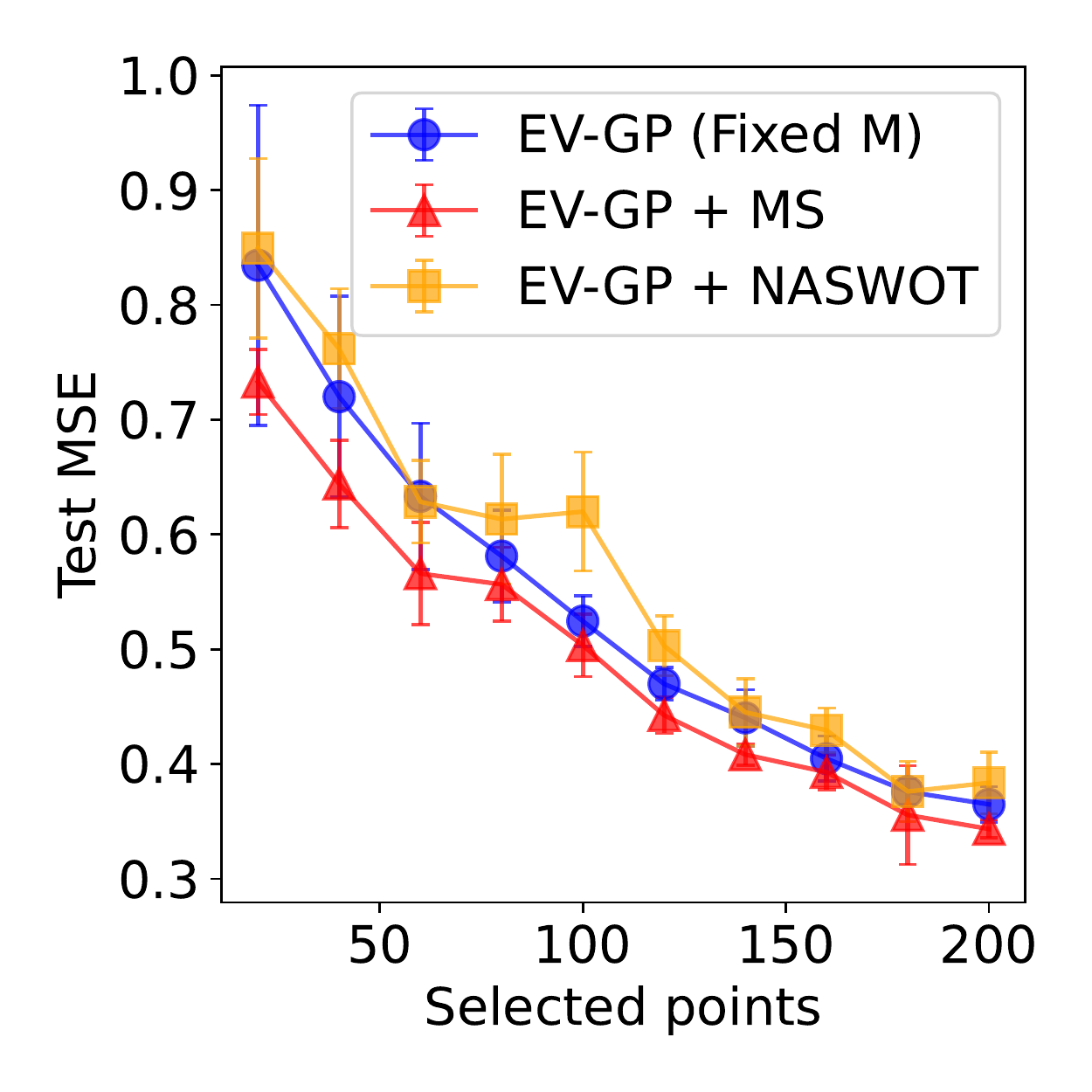} &\hspace{-5mm}
\includegraphics[width=0.43\linewidth]{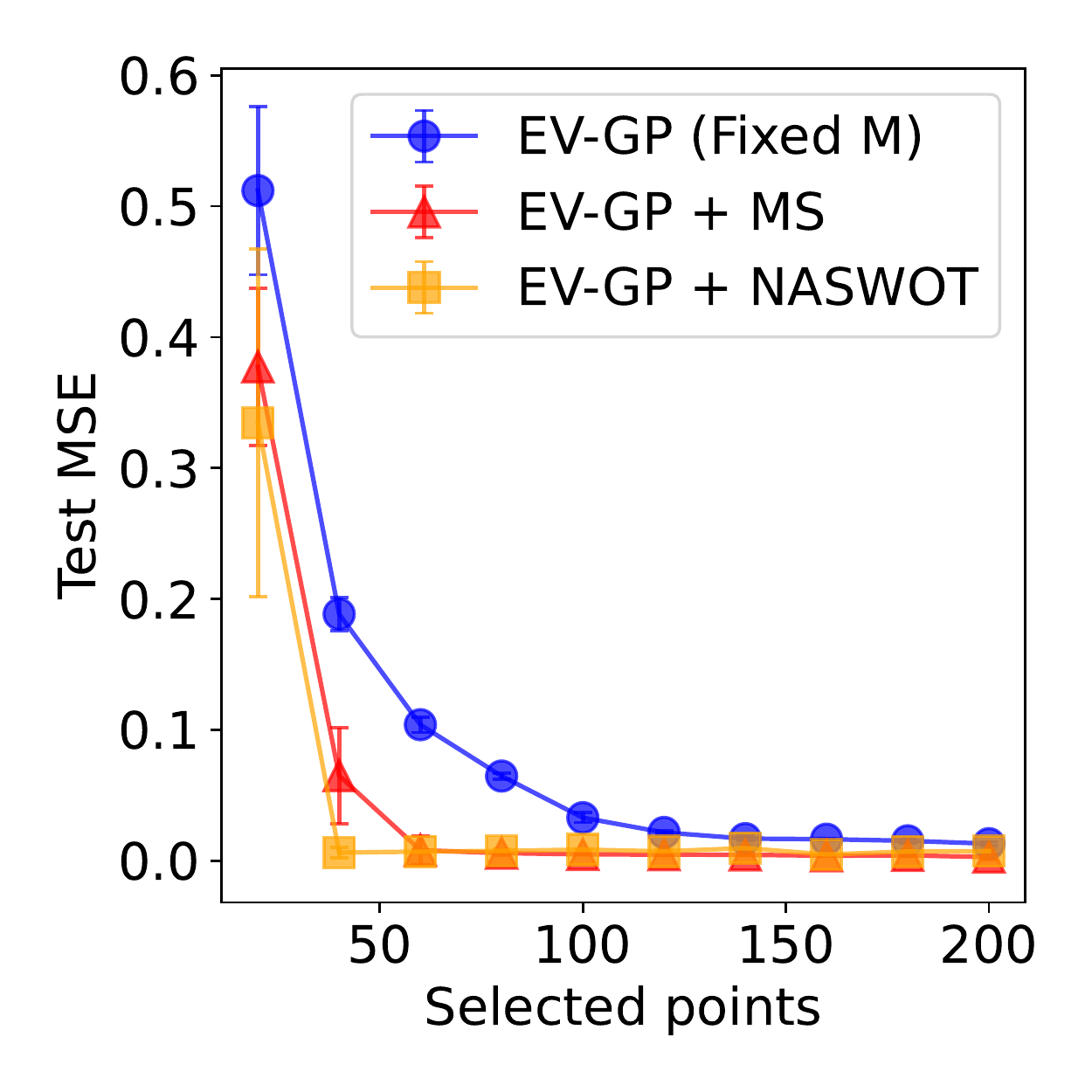}
\end{tabular}
\vspace{-5mm}
\caption{
Results of \algms\xspace compared to \alg\xspace on a fixed model architecture (2-layer MLP with ReLU activation).
}
\label{fig:results-ms}\vspace{-2mm}
\end{figure}

\section{Related Works}
\label{sec:related:works}
The existing works on active learning are usually based on either diversity or uncertainty.
Diversity-based active learning algorithms aim to select a diverse subset of data, in which the diversity of the data is measured 
based on a discriminator \cite{gissinDiscriminativeActiveLearning2019, sinhaVariationalAdversarialActive2019, kimTaskAwareVariationalAdversarial2021}, 
using some latent representation \cite{senerActiveLearningConvolutional2018, ashDeepBatchActive2020}, or
based on the degree to which they match with the other unlabelled data \cite{chattopadhyayBatchModeActive2012, shuiDeepActiveLearning2020}.
Uncertainty-based active learning algorithms select the unlabelled data based on how uncertain the trained model is about their predictions, in which the uncertainty can be captured through the softmax output \cite{ranganathanDeepActiveLearning2017, heBetterUncertaintySampling2019} or using a Bayesian NN \cite{galDeepBayesianActive2017, kirschBatchBALDEfficientDiverse2019, kirschTestDistributionAwareActive2021}.
In addition, some previous works on active learning have also combined both diversity and uncertainty for data selection \cite{ashDeepBatchActive2020, prabhuActiveDomainAdaptation2021a}.
Some works have improved the efficiency of neural active learning by estimating the performance of the NNs using computationally efficient proxies, such as an NN with a smaller size \cite{colemanSelectionProxyEfficient2020}. Another line of work is to approximate the neural network performance using the theory of NTKs \cite{wangNeuralActiveLearning2021,wangDeepActiveLearning2022,mohamadiMakingLookAheadActive2022}, which is able provide theoretical guarantees on how a model would behave without the need to perform expensive model training. 

In addition to active learning, NTKs also find use in other tasks such as data valuation \cite{wuDAVINZDataValuation2022}, neural architecture search \cite{shuNASILABELDATAAGNOSTIC2022, shuUnifyingBoostingGradientBased2022a}, multi-armed bandits \cite{dai2022federated}, Bayesian optimization \cite{dai2022sample}, and uncertainty quantification \cite{heBayesianDeepEnsembles2020}.
NTKs have also been useful in understanding model generalization loss \cite{caoGeneralizationBoundsStochastic2019} and also and in linking neural network training with gradient descent and kernel ridge regression \cite{aroraFineGrainedAnalysisOptimization2019, vakiliUniformGeneralizationBounds2021}.
Moreover, the problem of active learning is also closely related to Bayesian optimization (BO) which also utilizes the principled uncertainty measures provided by GPs to optimize a black-box function \cite{daxberger2017distributed,dmitrii20a,bala20,topk,TMES,metaBO}.
Specifically, active learning can be viewed as a variant of BO for pure exploration.

\section{Conclusion}

We have introduced a computationally efficient and theoretically grounded criterion for neural active learning, which can lead to the selection of points that result in both initialization robustness and good generalization performances.
Extensive empirical results have shown that our criterion is highly correlated to both the initialization robustness and generalization error, and that it consistently outperforms existing baselines.
An interesting future direction is to incorporate our algorithm to select the initial points for other neural active learning algorithms to further enhance their performances, because our algorithm has shown impressive performances in scenarios with limited initial labelled data.
Moreover, given the close connections between active learning and Bayesian optimization (BO) (Sec.~\ref{sec:related:works}), as well as the widespread real-world applications of BO, we will also explore extending our active learning algorithm to other real-world problem settings that have been considered by BO, such as problems with high-dimensional input spaces \cite{NghiaAAAI18}, multi-fidelity observations \cite{yehong17,ZhangUAI19,dai2019}, delayed feedback \cite{verma2022bayesian}, a necessity for rigorous privacy guarantees \cite{dmitrii20b} and a requirement for risk aversion \cite{nguyen2021conditional,nguyen2021value,SebICML22}, as well as problems that fall under the federated/collaborative setting \cite{dai2020federated,dai2021differentially,sim2021collaborative} and game-theoretic setting \cite{dai2020,tay2023no}.
Furthermore, it is also a promising future topic to extend our model selection method (Sec.~\ref{sec:model-sel}) by incorporating more advanced methods for neural architecture search \cite{shuNASILABELDATAAGNOSTIC2022,shuNeuralEnsembleSearch2022}.

\section*{Acknowledgements}
This research/project is supported by the National Research Foundation Singapore and DSO National Laboratories under the AI Singapore Programme (AISG Award No: AISG$2$-RP-$2020$-$018$) and by A*STAR under its RIE$2020$ Advanced Manufacturing and Engineering (AME) Programmatic Funds (Award A$20$H$6$b$0151$).
We would like to thank Shu Yao for his valuable inputs to our paper.
\newpage
\bibliography{ref}
\bibliographystyle{icml2023}

\newpage
\appendix
\onecolumn

\section{Proof of \cref{thm:relu-bound-informal}}
\label{appx:relubound}

In this section, we will show that the output variance of neural networks with ReLU activation can be represented using $\sigma_\text{NTKGP}$. For notation, we will let $\norm{\cdot}_p$ represent the $L_p$-norm of a vector or matrix. When $p$ is unspecified, we assume that we are referring to the $L_2$-norm.

\subsection{Dual Activations}

In this subsection, we first define the concept of dual activation functions \cite{leeWideNeuralNetworks2020}.

\begin{definition}
Let $\phi : \real \to \real$ be an activation function. Then, the dual activation function of $\phi$ is given by $\check{\phi} : \real^{2\times 2} \to \real$ where
$$\check{\phi}(\Lambda) = \expected_{(u, v) \sim \normdist(0, \Lambda)} [\phi(u) \phi(v)].$$

Similarly, if we let $\phi'(x) = \dfrac{d \phi(x)}{dx}$ be the derivative of the activation function, then its dual activation $\check{\phi '}$ is defined as
$$\check{\phi'}(\Lambda) = \expected_{(u, v) \sim \normdist(0, \Lambda)} [\phi'(u) \phi'(v)].$$
\end{definition}

\subsection{Neural Network Assumption}
We make an assumption about the parametrisation of our neural network. We use a common parametrisation from other NTK-related works, which can be found in e.g., \citet{leeWideNeuralNetworks2020}.
\begin{assumption}
	\label{assump:model-ntk-param}
	Assume $f(\cdot; \theta)$ is a multilayer perceptron with $L$ hidden layers each with width of $k_1, k_2, \ldots, k_L$. For simplicity we assume that all hidden layers have the same width. Let the input dimension of the neural network is $k_0$ and the output dimension is $k_{L+1}$. Let the neural network be parametrised as
	\begin{align*}
		h^{(0)}(x) &= x \\
		h^{(\ell)}(x) &= \frac{1}{\sqrt{k_\ell}}  \phi \big(W^{(\ell)} h^{(\ell - 1)}(x) + b^{(\ell)} \big) &\text{for } \ell \in [1,\ldots,L]\\
		f(x) &= W^{(L+1)} h^{(L)}(x) + b^{(L+1)}
	\end{align*}
	where $W^{(\ell)}_{ij}  \sim \normdist(0, \sigma_W^2)$ and $b^{(\ell)}_{ij} \sim \normdist(0, \sigma_b^2)$ are model weights and biases initialized randomly from a Gaussian distribution with variances $\sigma_W^2$ and $\sigma_b^2$ respectively, and $\phi$ is an activation function which is scaled so that $\check{\phi}\bigg( \bigg[\begin{matrix} 1 & 1 \\ 1 & 1 \end{matrix} \bigg]\bigg) = 1$. 
\end{assumption}

There are two main changes from the standard neural network assumption.
\begin{enumerate}
\item The difference in the parametrisation of the hidden layer, where a multiplicative factor of $\dfrac{1}{\sqrt{k_\ell}}$ has been added. This simplifies the computation of expectation values.

\item The scaling of the activation function. In our case, we scale the activation function $\phi$ such that its dual activation output is fixed under a specific input. This is an additional assumption not applied in \citet{leeWideNeuralNetworks2020}, and is done here to simplify the computation when we have to repeatedly apply the activation function.
\end{enumerate}

\subsection{Relationship between $\cK$ and $\Theta$}

We first present the relationship between $\cK$ and $\Theta$. The following lemma is also presented and proven in \citet{jacotNeuralTangentKernel2018}.
\begin{lemma}
For a neural network in the infinite-width regime, $\cK$ and $\Theta$ can be defined recursively as
\begin{align}
\cK^{(0)}(x, x') &= x\transpose x' + \sigma_b^2\\
\Theta^{(0)}(x, x') &= 0 \\
\cK^{(\ell)}(x, x') &= \check{\phi}\bigg( \bigg[\begin{matrix}
	\cK^{(\ell-1)}(x, x) & \cK^{(\ell-1)}(x, x') \\ 
	\cK^{(\ell-1)}(x', x) & \cK^{(\ell-1)}(x', x')
\end{matrix}\bigg] \bigg) + \sigma_b^2 & \ell \in [1, \cdots, L+1] \label{eqn:k-recursion}\\
\Theta^{(\ell)}(x, x') &= \cK^{(\ell)}(x, x') + \Theta^{(\ell - 1)}(x, x')\cdot \dot{\cK}^{(\ell-1)}(x, x') & \ell \in [1, \cdots, L+1] \label{eqn:ntk-recursion}
\end{align}
where we define
\begin{equation}
\label{eqn:k-dot}
\dot{\cK}^{(\ell)}(x, x') = \check{\phi}'\bigg( \bigg[\begin{matrix}
\cK^{(\ell)}(x, x) & \cK^{(\ell)}(x, x') \\ 
\cK^{(\ell)}(x', x) & \cK^{(\ell)}(x', x')
\end{matrix}\bigg] \bigg)
\end{equation}
\end{lemma}

Notice that \eqref{eqn:ntk-recursion} gives a recursive formula for computing $\Theta$. If we ``unroll" this formula, we will obtain
\begin{equation}
\Theta(x, x') = \sum_{\ell = 1}^{L+1} \cK^{(\ell)}(x, x') \prod_{\ell' = \ell+1}^{L+1} \dot{\cK}^{(\ell')}(x, x'). \label{eqn:ntk-unroll}
\end{equation}
Given \eqref{eqn:ntk-unroll}, it is now simple to bound $\Theta(x, x')$, since all that is required is to provide bounds for each $\cK^{(\ell)}$ and $\dot{\cK}^{(\ell)}$ individually. This can be done by inspecting the dual activation functions.

\subsection{Properties of ReLU Dual Activation Functions}

For the remainder of this section, we will consider the case where the activation function is the scaled ReLU function, defined as $\phi(x) = \sqrt{2} \cdot \max(x, 0)$. Based on \citet{leeWideNeuralNetworks2020}, it can be shown that for ReLU activations, if $\Lambda = \bigg[\begin{matrix}
x\transpose x& x\transpose y \\ y\transpose x& y\transpose y
\end{matrix}\bigg]$, then
$$\check{\phi}(\Lambda) = \frac{1}{\pi} \norm{x}\norm{y} \big( \sin \theta + (\pi - \theta) \cos \theta \big)$$
and
$$\check{\phi '}(\Lambda) = \frac{\pi - \theta}{\pi}$$
where $\theta = \arccos\big(\frac{x\transpose y}{\norm{x} \norm{y}}\big)$. Notice that the ReLU activation in this case has a $\sqrt{2}$ multiplicative factor, which scales the activation as required from above.

For convenience, define the function $\rho: [-1, 1] \to \real$ where
\begin{equation}
\rho(r) = \check{\phi}\bigg( \bigg[\begin{matrix}
1 & r \\ r & 1
\end{matrix}\bigg] \bigg)
= \frac{1}{\pi} \big( \sqrt{1 - r^2} + (\pi - \arccos r) \cdot r \big).
\end{equation}
This can be thought of as the re-parametrisation of the dual activation function $\phi$ in the case where $\norm{x} = \norm{y} = 1$, and we are only specifying the cosine distance $r = \cos \theta$ between $x$ and $y$. We can also define $\rho'$ to be a similar function but on the dual activation $\check{\phi}'$ instead,
\begin{equation}
\rho'(r) = \check{\phi}'\bigg( \bigg[\begin{matrix}
1 & r \\ r & 1
\end{matrix}\bigg] \bigg)
= \frac{\pi - \arccos r}{\pi}.
\end{equation}

It is simple to verify that for ReLU activations,
\begin{equation}
\check{\phi}\bigg( \bigg[\begin{matrix}
\norm{x}^2 & \norm{x}\norm{y} \cdot r \\ \norm{x}\norm{y} \cdot r & \norm{y}^2
\end{matrix}\bigg] \bigg) = \norm{x} \norm{y} \cdot \rho(r)
\end{equation}
and
\begin{equation}
\check{\phi}'\bigg( \bigg[\begin{matrix}
\norm{x}^2 & \norm{x}\norm{y} \cdot r \\ \norm{x}\norm{y} \cdot r & \norm{y}^2
\end{matrix}\bigg] \bigg) = \rho'(r).
\end{equation}



For any function $f$, we define a notation $f^m(x) = \underbrace{(f\circ f \circ \cdots \circ f)}_\text{$m$ times}(x)$. This is thought as repeatedly applying function $f$ to the input value $m$ times. For our dual activation function, we can show that repeating the function input will still result in a non-decreasing function.

\begin{lemma}
\label{claim:dual-inc}
For any $n\in\nat$, $\rho^n(r)$ is non-decreasing with respect to $r$.
\end{lemma}
\begin{proof}
We can prove this by induction. For the case that $n=1$, it is simple to see that
\begin{equation}
\frac{d\rho}{dr} 
= \frac{\pi - \arccos r}{\pi}.
\end{equation}
We can see that $\arccos r \leq \pi$, and so $\dfrac{d\rho}{dr} \geq 0$. This means that $\rho$ is non-decreasing.

Furthermore, we are able to see that $\rho(-1) = 0$ and $\rho(1) = 1$. Since the function is non-decreasing, this means for any $r \in [-1, 1]$, it is the case that $\rho(r) \in [0, 1] \subset [-1, 1]$.

Now, assume that $\rho^n$ is non-decreasing. Let $r_1, r_2 \in [-1, 1]$. If $r_1 \leq r_2$, then it is the case that $\rho^n(r_1) \leq \rho^n(r_2)$. Since we know $\rho^n(r_1), \rho^n(r_2) \in [-1, 1]$, we can therefore conclude that $\rho^{n+1}(r_1) \leq \rho^{n+1}(r_2)$. 
\end{proof}

Similarly, we can show a similar claim for the dual activation with respect to $\phi'$. Note that for $\phi'$, we do not need to show that repeated application of the function keeps the output non-decreasing.
\begin{lemma}
\label{claim:dual-prime-inc}
$\rho'$ is non-decreasing with respect to $r$.
\end{lemma}
\begin{proof}
It is simple to see that
\begin{equation}
\frac{d\rho'}{dr} 
= \frac{1}{\pi\sqrt{1 - r^2}}.
\end{equation}
Since square roots are always positive, it follows that $d\rho'/dr \geq 0$. This means that $\rho'$ is an increasing function. 
\end{proof}


\subsection{Bounding $\cK$ in terms of $\rho$}

We will first show the relationship between ${\cK}$ and $\rho$.
\begin{lemma}
\label{claim:k-form}
For $\ell \in [1, \ldots, L+1]$,
\begin{equation}
\cK^{(\ell)}(x, x') 
= \sqrt{u^{(\ell-1)}} \cdot \rho\big(r_u^{(\ell)}\big) + \sigma_b^2
\end{equation}
where $u^{(\ell)} = \big( \norm{x}^2 + \ell\sigma_b^2 \big)\big( \norm{x'}^2 + \ell\sigma_b^2 \big)$ and $r_u^{(\ell)} = \dfrac{\cK^{(\ell-1)}(x, x')}{\sqrt{u^{(\ell-1)}}}$.
\end{lemma}
\begin{proof}
We can prove this by induction. In the case that $\ell = 1$, we see that
\begin{align}
\cK^{(1)}(x, x') &= \check{\phi}\bigg( \bigg[\begin{matrix}
	\cK^{(0)}(x, x) & \cK^{(0)}(x, x') \\ 
	\cK^{(0)}(x', x) & \cK^{(0)}(x', x')
\end{matrix}\bigg] \bigg) + \sigma_b^2 \\
&= \check{\phi}\bigg( \bigg[\begin{matrix}
	\norm{x}^2 & x\transpose x' \\ 
	x\transpose x' & \norm{x'}^2
\end{matrix}\bigg] \bigg) + \sigma_b^2 \\
&= \norm{x}\norm{x'} \cdot \rho\Bigg(\frac{x\transpose x'}{\norm{x}\norm{x'}}\Bigg) + \sigma_b^2
\end{align}

For the inductive step, assume that $\cK^{(\ell)}(x, x') = \sqrt{u^{(\ell-1)}} \cdot \rho\big(r_u^{(\ell)}\big) + \sigma_b^2$ holds. Then, we see that
\begin{align}
\cK^{(\ell+1)}(x, x') &= \check{\phi}\bigg( \bigg[\begin{matrix}
	\cK^{(\ell)}(x, x) & \cK^{(\ell)}(x, x') \\ 
	\cK^{(\ell)}(x', x) & \cK^{(\ell)}(x', x')
\end{matrix}\bigg] \bigg) + \sigma_b^2 \\
&= \check{\phi}\bigg( \bigg[\begin{matrix}
	\big( \norm{x}^2 + (\ell - 1)\sigma_b^2 \big) \rho(1) + \sigma_b^2 & \cK^{(\ell)}(x, x') \\ 
	\cK^{(\ell)}(x', x) & \big( \norm{x'}^2 + (\ell-1)\sigma_b^2 \big) \rho(1) + \sigma_b^2
\end{matrix}\bigg] \bigg) + \sigma_b^2 \\
&= \check{\phi}\bigg( \bigg[\begin{matrix}
	\norm{x}^2 +\ell\sigma_b^2 & \cK^{(\ell)}(x, x') \\ 
	\cK^{(\ell)}(x', x) & \norm{x'}^2 + \ell\sigma_b^2
\end{matrix}\bigg] \bigg) + \sigma_b^2 \\
&= \sqrt{u^{(\ell)}} \cdot \rho \Bigg( \frac{\cK^{(\ell)}(x, x')}{\sqrt{u^{(\ell)}}} \Bigg) + \sigma_b^2
\end{align}
where we use the fact that $\cK^{(\ell)}(x, x') \leq \sqrt{\cK^{(\ell)}(x, x) \cdot \cK^{(\ell)}(x', x')}$. This proves the inductive step.
\end{proof}

The next proofs will attempt to bound the values of $\cK^{(\ell)}(x, x')$ based on $\rho$.
\begin{lemma}
\label{claim:rho1}
\begin{equation}
\tilde{\rho}^{(\ell)}_- 
\leq \cK^{(\ell)}(x, x') 
\leq \tilde{\rho}^{(\ell)}_+
\end{equation}
where
\begin{equation}
\tilde{\rho}^{(\ell)}_\pm = \begin{cases}
\pm\norm{x}\norm{x'} &\text{if }\ \ell = 0, \\
\sqrt{u^{(\ell - 1)}} \cdot \rho \bigg( \dfrac{\tilde{\rho}^{(\ell - 1)}_\pm}{\sqrt{u^{(\ell - 1)}}} \bigg) + \sigma_b^2 & \text{if }\ \ell \geq 1.
\end{cases}
\end{equation}
\end{lemma}
\begin{proof}
We can prove so by inspecting result from \cref{claim:k-form}, and using proof by induction. In the case that $\ell = 1$,
\begin{align}
\cK^{(1)}(x, x') 
&= \norm{x}\norm{x'} \cdot \rho\Bigg(\frac{x\transpose x'}{\norm{x}\norm{x'}}\Bigg) + \sigma_b^2 
\end{align}
then it is simple to show our claim in this case from the fact that $-{\norm{x}\norm{x'}}\leq {x\transpose x'} \leq {\norm{x}\norm{x'}}$ and that $\rho$ is monotone. 

For the inductive step, assume that the claim is true for $\ell$. We see that
\begin{align}
\cK^{(\ell + 1)}(x, x') 
&= \sqrt{u^{(\ell)}} \cdot \rho \Bigg( \frac{\cK^{(\ell)}(x, x')}{\sqrt{u^{(\ell)}}} \Bigg) + \sigma_b^2 \\
&\leq \sqrt{u^{(\ell)}} \cdot \rho \Bigg( \frac{\tilde{\rho}^{(\ell)}_+}{\sqrt{u^{(\ell)}}} \Bigg) + \sigma_b^2 
\end{align}
which uses the fact that $\rho$ is non-decreasing. A similar logic can be used to show the lower bound. This proves the inductive step and hence proves our lemma.
\end{proof}

\begin{corollary}
\label{claim:k-bound-final}
For all $\ell \in [1, \ldots, L+1]$,
\begin{equation}
\sqrt{u^{(\ell - 1)}} \cdot \rho \big(\hat{r}^{(\ell)}_- \big) + \sigma_b^2 
\leq \cK^{(\ell)}(x, x') 
\leq \sqrt{u^{(\ell - 1)}} \cdot \rho \big(\hat{r}^{(\ell)}_+ \big) + \sigma_b^2
\end{equation}
where 
\begin{equation}
\hat{r}^{(\ell)}_+ = 1
\end{equation}
and
\begin{equation}
\hat{r}^{(\ell)}_- = \begin{cases}
-1 &\ \text{if }\ \ell = 1, \\
\rho \big(\hat{r}^{(\ell - 1)}_-\big) \cdot \dfrac{\ell - 2}{\ell - 1} &\ \text{if }\ \ell > 1
\end{cases}
\end{equation}
\end{corollary}
\begin{proof} Notice that by \cref{claim:rho1} and due to the non-decreasing nature of $\rho$, proving the corollary above is equivalent to showing that $\dfrac{\tilde{\rho}^{(\ell - 1)}_+}{\sqrt{u^{(\ell - 1)}}} \leq r^{(\ell)}_+$ and $\dfrac{\tilde{\rho}^{(\ell - 1)}_-}{\sqrt{u^{(\ell - 1)}}} \geq r^{(\ell)}_-.$

We will start by showing that $\dfrac{\tilde{\rho}^{(\ell - 1)}_+}{\sqrt{u^{(\ell - 1)}}} \leq r^{(\ell)}_+ = 1$. It is simple to show that the statement is true for $\ell = 1$. 
In the case that $\ell > 1$, since $\rho(r) \leq 1$ for all $r$, we can show that
\begin{align}
\frac{\big(\tilde{\rho}^{(\ell - 1)}_+\big)^2}{u^{(\ell - 1)}}
&\leq \frac{\big( \sqrt{u^{(\ell - 2)}} + \sigma_b^2 \big)^2}{u^{(\ell - 1)}}\\
&= \frac{\big( \norm{x} + (\ell-2)\sigma_b^2 \big) \big( \norm{x'} + (\ell-2)\sigma_b^2 \big) + \sigma_b^2 \sqrt{\big( \norm{x} + (\ell-2)\sigma_b^2 \big) \big( \norm{x'} + (\ell-2)\sigma_b^2 \big)} + \sigma_b^4}
{\big( \norm{x'} + (\ell-2)\sigma_b^2 \big) + \sigma_b^2 \big[ {\big( \norm{x} + (\ell-2)\sigma_b^2 \big) + \big( \norm{x'} + (\ell-2)\sigma_b^2 \big)}\big] + \sigma_b^4} \label{eqn:us-lb}\\ 
&\leq 1
\end{align}
from the inequality $2 \sqrt{ab} \leq a + b$. This therefore means $\dfrac{\hat{\rho}^{(\ell - 1)}_+}{\sqrt{u^{(\ell - 1)}}} \leq 1$. This proves the first part of the corollary.

For the second part, we will prove so by induction. It is easy to show that the statement is true for $\ell = 1$. For the inductive step, consider some value of $\ell > 2$, and assume that the statement holds for $\ell - 1$, meaning that $\dfrac{\tilde{\rho}^{(\ell - 2)}_-}{\sqrt{u^{(\ell - 2)}}} \geq r^{(\ell - 1)}_-$. Then, we can see that
\begin{align}
\dfrac{\tilde{\rho}^{(\ell - 1)}_-}{\sqrt{u^{(\ell - 1)}}}
&= \frac{\sqrt{u^{(\ell - 2)}} \cdot \rho \bigg( \dfrac{\tilde{\rho}^{(\ell - 2)}_-}{\sqrt{u^{(\ell - 2)}}} \bigg) + \sigma_b^2}{\sqrt{u^{(\ell - 1)}}} \\
&\geq \frac{\sqrt{u^{(\ell - 2)}} \cdot \rho \big( r^{(\ell - 1)}_- \big) + \sigma_b^2}{\sqrt{u^{(\ell - 1)}}} \\
&\geq \rho \big( r^{(\ell - 1)}_- \big) \cdot \frac{\sqrt{u^{(\ell - 2)}}  + \sigma_b^2}{\sqrt{u^{(\ell - 1)}}} \\
&\geq \rho \big( r^{(\ell - 1)}_- \big) \cdot \frac{\ell - 2}{\ell - 1} \label{eqn:rbound-penult}\\
&= r^{(\ell)}_-
\end{align}
where in \eqref{eqn:rbound-penult} we use the fact that
\begin{equation}
\frac{\sqrt{u^{(\ell - 2)}}  + \sigma_b^2}{\sqrt{u^{(\ell - 1)}}} 
\geq \frac{\sqrt{u^{(\ell - 2)}}}{\sqrt{u^{(\ell - 1)}}}
\geq \sqrt{\frac{\norm{x} + (\ell-2) \sigma_b^2}{\norm{x} + (\ell-1) \sigma_b^2}} \cdot \sqrt{\frac{\norm{x'} + (\ell-2) \sigma_b^2}{\norm{x'} + (\ell-1) \sigma_b^2}} \geq \frac{\ell - 2}{\ell - 1}.
\label{eqn:u-ratio}
\end{equation}
This proves the second part of our corollary.
\end{proof}

\subsection{Bounding $\dot{\cK}$ in terms of $\rho$}

We will first show that we can express $\dot{\cK}$ in terms of $\rho'$.
\begin{lemma}
\label{claim:kot-form}
For $\ell \in [1, \ldots, L+1]$,
\begin{equation}
\dot{\cK}^{(\ell)}(x, x') 
= \rho'\big(r_u^{(\ell)}\big)
\end{equation}
where $r_u^{(\ell)}$ is defined in \cref{claim:k-form}.
\end{lemma}
\begin{proof}
We can show that
\begin{align}
\dot{\cK}^{(\ell)}(x, x') 
&= \check{\phi}'\bigg( \bigg[\begin{matrix}
	\cK^{(\ell-1)}(x, x) & \cK^{(\ell-1)}(x, x') \\ 
	\cK^{(\ell-1)}(x', x) & \cK^{(\ell-1)}(x', x')
\end{matrix}\bigg] \bigg) \\
&= \check{\phi}'\bigg( \bigg[\begin{matrix}
	\big( \norm{x}^2 + (\ell - 2)\sigma_b^2 \big) \rho(1) + \sigma_b^2 & \cK^{(\ell - 1)}(x, x') \\ 
	\cK^{(\ell - 1)}(x', x) & \big( \norm{x'}^2 + (\ell-2)\sigma_b^2 \big) \rho(1) + \sigma_b^2
\end{matrix}\bigg] \bigg) \\
&= \check{\phi}'\bigg( \bigg[\begin{matrix}
	\norm{x}^2 +(\ell-1)\sigma_b^2 & \cK^{(\ell - 1)}(x, x') \\ 
	\cK^{(\ell - 1)}(x', x) & \norm{x'}^2 + (\ell-1)\sigma_b^2
\end{matrix}\bigg] \bigg) \\
&= \rho' \Bigg( \frac{\cK^{(\ell - 1)}(x, x')}{\sqrt{u^{(\ell - 1)}}} \Bigg).
\end{align}
\end{proof}

Using the result above, we can now bound the values of  $\dot{\cK}^{(\ell)}(x, x')$.
\begin{corollary}
\begin{equation}
\rho' \big(\hat{r}^{(\ell)}_- \big) 
\leq \dot{\cK}^{(\ell)}(x, x') 
\leq \rho' \big(\hat{r}^{(\ell)}_+ \big)
\end{equation}
where $\hat{r}^{(\ell)}_-$ and $\hat{r}^{(\ell)}_+$ are as defined in \cref{claim:k-bound-final}.
\end{corollary}
\begin{proof}
From \cref{claim:k-bound-final}, it is the case that $\hat{r}^{(\ell)}_- \leq r^{(\ell)}_u \leq \hat{r}^{(\ell)}_+$. The corollary then follows from this fact.
\end{proof}

\subsection{Ratio between $\cK$ and $\Theta$}

We will now prove the bound of the ratio between $\cK$ and $\Theta$.
\begin{theorem}
\label{claim:relu-bound-bias}
For a neural network with $L \geq 2$, if $\max\big\{\norm{x}, \norm{x'} \big\}\leq B$, then
\begin{equation}
1 + \sum_{\ell = 1}^{L} \rho \big(\hat{r}^{(\ell)}_-\big) \cdot \frac{\ell - 1}{L + 1} \prod_{\ell' = \ell+1}^{L+1} \rho'\big( \hat{r}^{(\ell')}_- \big)
\leq \frac{\Theta(x, x')}{\cK(x, x')}
\leq 1 + \frac{L}{\rho \big(\hat{r}^{(L+1)}_- \big)}
\end{equation}
\end{theorem}
\begin{proof}
We will first prove the right hand inequality. We see that
\begin{align}
\frac{\Theta(x, x')}{\cK(x, x')}
&=\frac{1}{\cK^{(L+1)}(x, x')} \cdot \sum_{\ell = 1}^{L+1} {\cK^{(\ell)}(x, x')} \prod_{\ell' = \ell+1}^{L+1} \dot{\cK}^{(\ell')}(x, x') \\
&\leq 1 + \sum_{\ell = 1}^{L} \frac{\sqrt{u^{(\ell - 1)}} \cdot \rho \big(\hat{r}^{(\ell)}_+ \big) + \sigma_b^2}{\sqrt{u^{(L)}} \cdot \rho \big(\hat{r}^{(L+1)}_- \big) + \sigma_b^2} \prod_{\ell' \leq \ell+1}^{L+1} \rho'\big( \hat{r}^{(\ell')}_+ \big) \\
&= 1 + \sum_{\ell = 1}^{L} \frac{1}{\rho \big(\hat{r}^{(L+1)}_- \big)} \cdot \frac{\sqrt{u^{(\ell - 1)}} + \sigma_b^2}{\sqrt{u^{(L)}}} \\
&\leq 1 + \frac{L}{\rho \big(\hat{r}^{(L+1)}_- \big)}
\end{align}
where we use the fact that $\dfrac{\sqrt{u^{(\ell - 1)}} + \sigma_b^2}{\sqrt{u^{(L)}}} \leq 1$ based on a similar argument used in \eqref{eqn:us-lb}.

Similarly for the left hand inequality,
\begin{align}
\frac{\Theta(x, x')}{\cK(x, x')}
&\geq 1 + \sum_{\ell = 1}^{L} \frac{\sqrt{u^{(\ell - 1)}} \cdot \rho \big(\hat{r}^{(\ell)}_- \big) + \sigma_b^2}{\sqrt{u^{(L)}} \cdot \rho \big(\hat{r}^{(L+1)}_+ \big) + \sigma_b^2} \prod_{\ell' = \ell+1}^{L+1} \rho'\big( \hat{r}^{(\ell')}_- \big) \\
&\geq 1 + \sum_{\ell = 1}^{L} \rho \big(\hat{r}^{(\ell)}_-\big) \cdot \frac{\sqrt{u^{(\ell - 1)}}}{\sqrt{u^{(L)}} + \sigma_b^2} \prod_{\ell' = \ell+1}^{L+1} \rho'\big( \hat{r}^{(\ell')}_- \big) \\
&\geq 1 + \sum_{\ell = 1}^{L} \rho \big(\hat{r}^{(\ell)}_-\big) \cdot \frac{\ell - 1}{L + 1} \prod_{\ell' = \ell+1}^{L+1} \rho'\big( \hat{r}^{(\ell')}_- \big)
\end{align}
where we use the fact that $\dfrac{\sqrt{u^{(\ell - 1)}}}{\sqrt{u^{(L)}} + \sigma_b^2} \geq \dfrac{\sqrt{u^{(\ell - 1)}}}{\sqrt{u^{(L + 1)}}} \geq \dfrac{\ell - 1}{L + 1}$ based on a similar reasoning as \eqref{eqn:u-ratio}. This proves our theorem.
\end{proof}
From \cref{claim:relu-bound}, we are therefore able to give a bound for the ratio between $\Theta$ and $\cK$. While the constant is defined recursively based on the function $\rho$, we claim that this is still useful since the constants are expressed in a form which allows it to be computed directly, and more importantly, we are able to see that such a constant will exist.

\subsection{Ratio between $\cK$ and $\Theta$ when $\sigma_b = 0$}

In the case that there is no bias, we can improve the bound on the ratio between $\cK$ and $\Theta$. The key fact is that in the case without bias, $u^{(\ell)}$ is equal for all values of $\ell$, and we can simplify $\hat{\rho}^{(\ell)}_\pm = \norm{x}\norm{x'} \cdot \rho^\ell(\pm 1)$. During the computation, the constants $\norm{x}\norm{x'}$ will then cancel each other out.
\begin{theorem}
\label{claim:relu-bound}
For a neural network with $L \geq 2$ and $\sigma_b = 0$,
\begin{equation}
\sum_{\ell = 1}^{L+1} \frac{\rho^\ell (-1)}{\rho^{L+1}(1) } \prod_{\ell' = \ell+1}^{L+1} (\rho' \circ \rho^{\ell'})(-1)
\leq \frac{\Theta(x, x')}{\cK(x, x')}
\leq \sum_{\ell = 1}^{L+1} \frac{\rho^\ell (1)}{\rho^{L+1}(-1) } \prod_{\ell' = \ell+1}^{L+1} (\rho' \circ \rho^{\ell'})(1)
\end{equation}
\end{theorem}
\begin{proof}
It is easy to see from the recursive form in \eqref{eqn:k-recursion} that if $\sigma_b = 0$, then $\cK^{(\ell)}(x, x') = \norm{x}\norm{x'} \cdot \rho^\ell\bigg( \dfrac{x\transpose x'}{\norm{x}\norm{x'}} \bigg)$. Since we know $\rho^\ell$ is non-decreasing from \cref{claim:dual-inc}, we are able to bound
\begin{equation}
\norm{x}\norm{x'} \cdot \rho^\ell (-1) \leq \cK^{(\ell)}(x, x') \leq \norm{x}\norm{x'} \cdot \rho^\ell (1).\label{eqn:k-bound-nobias}
\end{equation}

Similarly, we can also see from \eqref{eqn:k-dot} that $\dot{\cK}^{(\ell)}(x, x') = (\rho' \circ \rho^\ell)\bigg( \dfrac{x\transpose x'}{\norm{x}\norm{x'}} \bigg)$. Since $\rho'$ is non-decreasing based on \cref{claim:dual-prime-inc}, we are able to bound 
\begin{equation}
(\rho' \circ\rho^\ell) (-1) \leq \dot{\cK}^{(\ell)}(x, x') \leq (\rho' \circ\rho^\ell) (1). \label{eqn:kdot-bound-nobias}
\end{equation}

Given that we can write $\Theta(x, x')$ with the recursive form given by \eqref{eqn:ntk-unroll}, it is then simple to use \eqref{eqn:k-bound-nobias} and \eqref{eqn:kdot-bound-nobias} to bound $\dfrac{\Theta(x, x')}{\cK(x, x')}$.
\end{proof}

\subsection{Bounding the Difference Between $\sigma_\text{NN}$ and $\sigma_\text{NTKGP}$}

From above, we are able to see that it is possible to bound the ratio 
\begin{equation}
\label{eqn:bound_a}
a_- \leq \frac{\cK(x, x')}{\Theta(x, x')} \leq a_+
\end{equation}
for some appropriately set $a_-$ and $a_+$ according to either \cref{claim:relu-bound-bias} or \cref{claim:relu-bound} depending on the value of $\sigma_b$ (note that in the two theorems above we show the bounds for the reciprocal of what is stated in \eqref{eqn:bound_a}). Given this bound, we are able to show the following main result.
\begin{theorem}[Formal Version of \cref{thm:relu-bound-informal}]
For a neural network with ReLU activation and $L\geq 2$ hidden layers, if $\Theta_\cX \succeq 0$, then
\begin{equation}
\big| \sigma_\text{NN}^2(x | \cX) - \alpha \cdot \sigma^2_\text{NTKGP}(x | \cX) \big| 
\leq \beta
\end{equation}
where $n$ is the upper limit on the size of the training set, $B \geq |\Theta(x, x')|$, $\alpha \in [a_-, a_+]$,  $\gamma = B \cdot \max \big\{\alpha - a_-, a_+ - \alpha \big\}$, and $\beta = \gamma + \dfrac{n \gamma B^2}{\eigmin{\Theta_\cX}^2} + \dfrac{2 n\gamma B}{\eigmin{\Theta_\cX}}$.
\end{theorem}
\begin{proof}
First, we know that we have $\dfrac{\cK(x, x')}{\Theta(x, x')} \in [a_-, a_+]$ by assumption. In the case that $\Theta(x, x') \geq 0$, we can convert the multiplicative bound into an additive bound as
\begin{align}
\cK(x, x') &\geq a_- \cdot \Theta(x, x') \\
\alpha \cdot \Theta(x, x') - \cK(x, x') &\leq (\alpha - a_-)\cdot \Theta(x, x') \\
&\leq (\alpha - a_-)\cdot B
\end{align}
and
\begin{align}
\cK(x, x') &\leq a_+ \Theta(x, x') \\
\cK(x, x') - \alpha \cdot \Theta(x, x') &\leq (\alpha - a_+)\cdot \Theta(x, x') \\
&\leq (\alpha - a_+)\cdot B \\
\end{align}
which can be combined to give $\big| \cK(x, x') - \alpha \cdot \Theta(x, x') \big| \leq \gamma$ for $\gamma$ as defined earlier. The same is the case when $\Theta(x, x') < 0$.

Given this additive bound, we are then able to bound each term which appears in $\sigma_\text{NN}$ individually. We can see that
\begin{align}
\big| \cK_{x\cX}\Theta_\cX\inv \Theta_{\cX x} - \alpha \cdot \Theta_{x\cX}\Theta_\cX\inv \Theta_{\cX x}\big|
&=\big| \big( \cK_{x\cX} - \alpha \cdot \Theta_{x\cX}\big) \Theta_\cX\inv \Theta_{\cX x}\big| \\
&\leq \eigmax{\Theta_\cX\inv} \norm{\cK_{\cX x} - \alpha \cdot \Theta_{\cX x}} \norm{\Theta_{\cX x}} \\
&\leq \frac{1}{\eigmin{\Theta_\cX}} \cdot \gamma \sqrt{n} \cdot B \sqrt{n} \\
&=\frac{n\gamma B}{\eigmin{\Theta_\cX}}.
\end{align}

Similarly,
\begin{align}
\big| \Theta_{x\cX}\Theta_\cX\inv\cK_\cX\Theta_\cX\inv \Theta_{\cX x} - \alpha \cdot \Theta_{x\cX}\Theta_\cX\inv \Theta_{\cX x}\big|
&= \big| \Theta_{x\cX}\Theta_\cX\inv (\cK_\cX - \alpha \cdot  \Theta_\cX) \Theta_\cX\inv \Theta_{\cX x}\big| \\
&\leq \eigmax{\cK_\cX - \alpha \cdot  \Theta_\cX} \norm{\Theta_\cX\inv \Theta_{\cX x}}^2 \\
&= \eigmax{\cK_\cX - \alpha \cdot  \Theta_\cX} \cdot \eigmax{\Theta_\cX\inv}^2 \cdot \norm{ \Theta_{\cX x}}^2 \\
&\leq \frac{n\gamma B^2}{\eigmin{\Theta_\cX}^2}.
\end{align}

Combining these results together, we obtain
\begin{align}
\big| \sigma_\text{NN}^2(x | \cX) - \alpha \cdot \sigma^2_\text{NTKGP}(x | \cX) \big| 
&= \big| \big( \cK_x  + 
\Theta_{x\cX}\Theta_\cX\inv\cK_\cX\Theta_\cX\inv \Theta_{\cX x}
- \cK_{x\cX}\Theta_\cX\inv \Theta_{\cX x} - \Theta_{x\cX}\Theta_\cX\inv \cK_{\cX x} \big) +  \nonumber
\\ & \quad\quad\quad  - \alpha \cdot \big( \Theta_x - \Theta_{x\cX}\Theta_\cX\inv \Theta_{\cX x} \big) \big| \\
&\leq \big|\cK_x - \alpha \Theta_x \big|
+ \big| \Theta_{x\cX}\Theta_\cX\inv\cK_\cX\Theta_\cX\inv \Theta_{\cX x} - \alpha \Theta_{x\cX}\Theta_\cX\inv \Theta_{\cX x}\big| \nonumber \\
& \quad\quad\quad + \big| \cK_{x\cX}\Theta_\cX\inv \Theta_{\cX x} - \alpha \Theta_{x\cX}\Theta_\cX\inv \Theta_{\cX x}\big|
+ \big| \Theta_{x\cX}\Theta_\cX\inv \cK_{\cX x} - \alpha \Theta_{x\cX}\Theta_\cX\inv \Theta_{\cX x}\big|\\
&\leq \gamma + \frac{n \gamma B^2}{\eigmin{\Theta_\cX}^2} + \frac{2 n\gamma B}{\eigmin{\Theta_\cX}}
\end{align}
which proves our claim.
\end{proof}

\subsection{Bounding the Ratio Between $\alpha_\text{EV}$ using $\sigma_\text{NN}$ and $\sigma_\text{NTKGP}$}

It turns out that we are able to get a bound on the ratio of the EV criterion directly when we use $\sigma_\text{NN}$ versus when we use $\sigma_\text{NTKGP}$. Considering the case of a single test point, we can see that using $\sigma_\text{NN}$, the EV criterion gives
\begin{equation}
\alpha_\text{EV,NN}(x | \cX) = \sigma^2_\text{NN}(x|\emptyset) - \sigma^2_\text{NN}(x|\cX) = 2 \cdot \Theta_{x\cX} \Theta_\cX\inv \cK_{\cX x} - \Theta_{x\cX} \Theta_\cX\inv \cK_\cX \Theta_\cX\inv \Theta_{\cX x} \geq 0,
\end{equation}
and using $\sigma_\text{NTKGP}$,
\begin{equation}
\alpha_\text{EV,NTKGP}(x | \cX) = \sigma^2_\text{NTKGP}(x|\emptyset) - \sigma^2_\text{NTKGP}(x|\cX) = \Theta_{x\cX} \Theta_\cX\inv \Theta_{\cX x} \geq 0.
\end{equation}
Based on this, we are able to show the following bound.

\begin{lemma}
Assume $\cX$ is such that $\Theta_\cX , \cK_\cX \succeq 0$. Let $\norm{\Theta_\cX}_\infty, \norm{\Theta_{\cX x}}_\infty \leq B$ and ${a}_- \leq \dfrac{\cK(x, x')}{\Theta(x, x')} \leq {a}_+$. Then,
\begin{equation}
\dfrac{2a_- \cdot \eigmin{\Theta_\cX}}{B} - \dfrac{a_+ B}{\eigmin{\Theta_\cX}} 
\leq \dfrac{\alpha_\text{\normalfont EV,NN}(x | \cX)}{\alpha_\text{\normalfont EV,NTKGP}(x | \cX)}
\leq \frac{2 {a}_+ B}{\eigmin{\Theta_\cX}}.
\end{equation}
\end{lemma}
\begin{proof}
For the right hand inequality, we can see that
\begin{align}
\dfrac{\alpha_\text{EV,NN}(x | \cX)}{\alpha_\text{EV,NTKGP}(x | \cX)}
&\leq \frac{2 \cdot \Theta_{x\cX} \Theta_\cX\inv \cK_{\cX x}}{\Theta_{x\cX} \Theta_\cX\inv \Theta_{\cX x}} \\
&\leq \dfrac{2\cdot \eigmax{\Theta_\cX\inv} \norm{\Theta_{\cX x}} \norm{\cK_{\cX x}}}{\eigmin{\Theta_\cX\inv} \norm{\Theta_{\cX x}}^2} \\
&= \dfrac{2\cdot \eigmax{\Theta_\cX} \norm{\cK_{\cX x}}}{\eigmin{\Theta_\cX} \norm{\Theta_{\cX x}}}\\
&\leq \frac{2 {a}_+ B}{\eigmin{\Theta_\cX}}.
\end{align}
Meanwhile, for the left hand inequality,
\begin{align}
\dfrac{\alpha_\text{EV,NN}(x | \cX)}{\alpha_\text{EV,NTKGP}(x | \cX)}
&= \frac{2 \cdot \Theta_{x\cX} \Theta_\cX\inv \cK_{\cX x}}{\Theta_{x\cX} \Theta_\cX\inv \Theta_{\cX x}} - \frac{\Theta_{x\cX} \Theta_\cX\inv \cK_\cX \Theta_\cX\inv \Theta_{\cX x}}{\Theta_{x\cX} \Theta_\cX\inv \Theta_{\cX x}} \\
&\geq \dfrac{2\cdot \eigmin{\Theta_\cX\inv} \norm{\Theta_{\cX x}} \norm{\cK_{\cX x}}}{\eigmax{\Theta_\cX\inv} \norm{\Theta_{\cX x}}^2} - \dfrac{\eigmax{\cK_\cX} \norm{\Theta_\cX\inv \Theta_{\cX x}}^2}{\eigmin{\Theta_\cX\inv} \norm{\Theta_\cX\inv \Theta_{\cX x}}^2} \\
&\geq \dfrac{2\cdot \eigmin{\Theta_\cX\inv} \norm{\cK_{\cX x}}}{\eigmax{\Theta_\cX\inv} \norm{\Theta_{\cX x}}} - 
\dfrac{\eigmax{\cK_\cX}}{\eigmin{\Theta_\cX}} \\
&\geq \dfrac{2\cdot \eigmin{\Theta_\cX} \norm{\cK_{\cX x}}}{\eigmax{\Theta_\cX} \norm{\Theta_{\cX x}}} - 
\dfrac{\eigmax{\cK_\cX}}{\eigmin{\Theta_\cX}} \\
&\geq \dfrac{2\cdot \eigmin{\Theta_\cX} \cdot a_-}{B} - 
\dfrac{a_+ B}{\eigmin{\Theta_\cX}} 
\end{align}
which provides a lower bound. 
\end{proof}
This provides another method of showing the agreement when using $\sigma^2_\text{NTKGP}$ variance function for our criterion compared to using $\sigma^2_\text{NN}$ variance function.

\section{Proof of \cref{thm:err-informal}}
\label{appx:loss-bound}


In this section we will state \cref{thm:err} more formally and proceed to prove it. 
\begin{theorem}[Formal version of \cref{thm:err-informal}]
\label{thm:err}
Let the function $f^* \in \cH_\Theta$ and training set $(\cX, \tensy)$ follow \cref{assump:vak}. Suppose we train an infinitely wide neural network $f(\cdot; \theta)$ with training dataset $(\cX, \tensy)$ and mean-squared error loss function as given by \eqref{eqn:loss}, using gradient descent until convergence. Then, for any $x \in \cX$, with probability at least $1-2\delta$ over the random observation noise and network initialization, 
\begin{equation}
\big| {f^*(x) - f(x; \text{\normalfont train}(\theta_0))} \big|
\leq \bigg[B + \bigg(\frac{R}{\lambda} + 1\bigg)\sqrt{2 \log\delta^{-1}}\bigg] \sigma_\text{\normalfont NTKGP}(x | \cX) 
\end{equation}
where $\lambda$ is the regularization of the loss function.
\end{theorem}
\begin{proof}
By the triangle inequality, we can show that
\begin{align}
\big| {f^*(x) - f(x; \text{\normalfont train}(\theta_0))} \big|
&\leq \underbrace{ \big| {f^*(x) - \mu(x | \cX)} \big| }_\text{\textcircled{1}} + \underbrace{\big| {\mu(x | \cX) - f(x; \theta_\infty)} \big|}_\text{\textcircled{2}} \label{thm:err_1}\\
& \leq \bigg( B + \frac{R}{\lambda}\sqrt{2 \log\delta^{-1}} \bigg) \sigma_\text{\normalfont NTKGP}(x | \cX) +  \big| {\mu(x | \cX) - f(x; \theta_\infty)} \big| \label{thm:err_2}\\
& \leq \bigg( B + \frac{R}{\lambda}\sqrt{2 \log\delta^{-1}} \bigg) \sigma_\text{\normalfont NTKGP}(x | \cX) +  \sqrt{2 \log \delta\inv} \cdot \sigma_\text{\normalfont NN}(x | \cX) \label{thm:err_3} \\
& \leq \bigg[B + \bigg(\frac{R}{\lambda} + 1 \bigg)\sqrt{2 \log\delta^{-1}}\bigg] \sigma_\text{\normalfont NTKGP}(x | \cX) \label{thm:err_4}
\end{align}
where 
\begin{itemize}
	\item \eqref{thm:err_1} is true due to the triangle inequality, 
	
	\item \eqref{thm:err_2} is true with probability at least $1-\delta$ based on Theorem 1 of \cite{vakiliOptimalOrderSimple2021}, 
	
	\item \eqref{thm:err_3} is true with probability at least $1 - \delta$ due to Hoeffding inequality for sub-Gaussian random variables, and 
	
	\item \eqref{thm:err_4} is true since we know that  $\sigma_\text{\normalfont NTKGP}(x | \cX) \geq \sigma_\text{\normalfont NN}(x | \cX)$ from  \citet{heBayesianDeepEnsembles2020}.
\end{itemize}
By union bound, the statement above is true with probability at least $1 - 2\delta$. 
\end{proof}

The loss as decomposed in \eqref{thm:err_1} can also be thought of as the sum of the model bias (i.e. how well our neural network architecture is able to fit the given underlying function), given by \textcircled{1}, and the variance of prediction due to the random network initialization, given by \textcircled{2}. By \cref{thm:err}, we can show that we are minimising the predictive variance \textcircled{2}, however our ability to do so will also depend on the bias term \textcircled{1}.
\section{Efficient Computations of the NTKGP}
\label{appx:gp-approx}

In the practical deployment of active learning, the cost of data label querying is usually significantly costly (in terms of time and computation costs), meaning the time cost of the active set selection can be easily be dominated by the querying costs. Note that our algorithm incurs less cost of data label querying because it can run in large batches while maintaining competitive performance. Furthermore, another factor contributing significantly to the cost of active learning is the training of the neural networks, which our algorithm mitigates due to its training-free nature. Nonetheless, we preent some possible speedups that our algorithm can utilize, particularly in computing the value of $\sigma^2_\text{NTKGP}$.

\subsection{Incremental Computation of $\sigma^2_\text{NTKGP}$}
\label{appx:ntkgp-comp}

Computing $\sigma^2_\text{NTKGP}( \cdot |\cX)$ for the active learning criteria is computationally expensive. Fortunately, we know that in our active learning algorithm the labelled set $\cX$ is always incremented by one each time, and so when we would like to compute the criterion $\alpha(\cX \cup \{x'\})$, rather than recomputing the $\sigma^2_\text{NTKGP}( \cdot |\cX \cup \{x'\})$ again each time, we can just get the updated value of $\sigma^2_\text{NTKGP}( \cdot |\cX \cup \{x'\})$ based on the already-known value of $\sigma^2_\text{NTKGP}( \cdot |\cX)$.

Suppose we let $\cX' = \cX \cup \{x'\}$. We can utilize the formula for inversion of block matrices (e.g. Eqns. A.11 and A.12 of of \citet{rasmussenGaussianProcessesMachine2006})
\begin{equation}
\left[\begin{array}{ll}
    \mathbf{A} & \mathbf{B} \\
    \mathbf{C} & \mathbf{D}
\end{array}\right]\inv 
= 
\left[\begin{array}{cc}
    \left(\mathbf{A}-\mathbf{B D}^{-1} \mathbf{C}\right)^{-1} & -\mathbf{A}^{-1} \mathbf{B}\left(\mathbf{D}-\mathbf{C A}^{-1} \mathbf{B}\right)^{-1} \\
    -\mathbf{D}^{-1} \mathbf{C}\left(\mathbf{A}-\mathbf{B D}^{-1} \mathbf{C}\right)^{-1} & \left(\mathbf{D}-\mathbf{C A}^{-1} \mathbf{B}\right)^{-1}
\end{array}\right] 
\end{equation}
and the matrix inversion lemma (e.g. Eqn. A.9 of \citet{rasmussenGaussianProcessesMachine2006})
\begin{equation}
\label{eqn:matinv}
(\tens{A} + \tens{B}\tens{C}\tens{D})\inv = \tens{A}\inv - \tens{A}\inv \tens{B} (\tens{C}\inv + \tens{D}\tens{A}\inv \tens{B})\inv \tens{D} \tens{A}\inv
\end{equation}
to see that
\begin{equation}
\label{eqn:blockmat}
\Theta_{\cX'}\inv = \left[\begin{array}{ll}
    \Theta_{\cX} & \Theta_{\cX x'} \\
    \Theta_{x' \cX } & \Theta_{x'}
\end{array}\right]\inv 
= 
\left[\begin{array}{cc}
    \left(\Theta_{\cX}-\Theta_{\cX x'}\Theta_{x'}\inv \Theta_{x' \cX }\right)^{-1} & -\Theta_{\cX}^{-1} \Theta_{\cX x'}\left(\Theta_{x'}-\Theta_{x' \cX }\Theta_{\cX}\inv \Theta_{\cX x'}\right)^{-1} \\
    -\Theta_{x'}^{-1} \Theta_{x' \cX }\left(\Theta_{\cX}-\Theta_{\cX x'}\Theta_{x'}\inv \Theta_{x' \cX }\right)^{-1} & \left(\Theta_{x'}-\Theta_{x' \cX }\Theta_{\cX}\inv \Theta_{\cX x'}\right)^{-1}
\end{array}\right] 
\end{equation}
and 
\begin{align}
\sigma^2_\text{NTKGP}(x | \cX')
&= \Theta_x - \Theta_{x \cX'} \Theta_{\cX'}\inv \Theta_{\cX' x} \\
&= \sigma^2_\text{NTKGP}(x | \cX) 
- \Theta_{x\cX} \Theta_\cX\inv \Theta_{\cX x'}\left(\Theta_{x'}-\Theta_{x' \cX }\Theta_{\cX}\inv \Theta_{\cX x'}\right)^{-1} \Theta_{x' \cX} \Theta_\cX\inv \Theta_{\cX x}
\nonumber \\
&\hphantom{=} \quad + 
2 \cdot \Theta_{x\cX}\Theta_{\cX}^{-1} \Theta_{\cX x'}\left(\Theta_{x'}-\Theta_{x' \cX }\Theta_{\cX}\inv \Theta_{\cX x'}\right)^{-1} \Theta_{x'x} \nonumber \\
&\hphantom{=} \quad - \Theta_{xx'} \left(\Theta_{x'}-\Theta_{x' \cX }\Theta_{\cX}\inv \Theta_{\cX x'}\right)^{-1} \Theta_{x'x}.
\end{align}
Given that $\Theta_\cX\inv$ can also be reused from the previous round and updated iteratively using \cref{eqn:blockmat}, this means the variance can be computed more efficiently.

\subsection{Approximating $\mu_\text{NTKGP}$ and $\Sigma_\text{NTKGP}$ Using Sparse GPs}
\label{appx:sgp}

A benefit for using $\Sigma_\text{NTKGP}$ for our criterion is that unlike $\Sigma_\text{NN}$, the covariance matrices in the form of $\Sigma_\text{NTKGP}$ is well-studied in Gaussian process literature, and methods of sparsifying the kernel matrix has been proposed. Below, we discuss how such approximation techniques can be used.

Given we have a set of labelled training data $\cX$, we would like to compute the posterior variance $\Sigma_\text{NTKGP}(\cX_T | \cX)$. Since it is expensive to compute the posterior directly, we will use techniques from sparse GPs to compute the approximate posterior. In particular, we will use the approximation framework as proposed by \citet{hoangUnifyingFrameworkAnytime2015}.

Let $\cU = (\cX_\cU, \tensy_\cU)$ be a set of inducing points which is a representative for $\cX_T$. Note that in general, $\cU$ does not have to be a subset of $\cX$, and therefore is not necessarily a valid candidate for the active set. For simplicity, we will attempt to apply the Fully-Independent Conditional (FIC) approximation, where we assume that $\tensy_T$ and $\tensy$ are conditionally independent given $\tensy_\cU$, i.e.
\begin{equation}
p(\tensy_T | \tensy) = \int p(\tensy_T | \tensy_\cU, \tensy) p(\tensy_\cU | \tensy) d\tensy_\cU \approx \int p(\tensy_T | \tensy_\cU) p(\tensy_\cU | \tensy) d\tensy_\cU.
\end{equation}
We can assume that $p(\tensy_\cU | \tensy) \approx q(\tensy_\cU)$ and the goal is to compute the distribution of $q(\tensy_\cU)$ such that $\kldiv{q(\tensy_\cU)}{p(\tensy_\cU|\tensy)}$ is minimised. 

Assume that $q(\tensy_\cU)$ is a normal distribution, i.e. assume that $q(\tensy_\cU) \triangleq \normdist(\mu_\cU, \Sigma_\cU)$ where the mean and the covariance function will be dependent on our choice of set $\cX$. According to \citet{hoangUnifyingFrameworkAnytime2015}, it can be shown that the $\mu_\cU$ and $\Sigma_\cU$ which minimises the KL divergence will also maximise the ELBO as given by (Equation 6 of \cite{hoangUnifyingFrameworkAnytime2015})
\begin{equation}
\label{eqn:elbo}
L(q) = \int q(\tensf_\cU) \bigg[ 
\int q(\tensf | \tensf_\cU)  \log \frac{ p(\tensf, \tensy | \tensy_\cU)}{q(\tensf | \tensf_\cU)} d\tensf
\bigg] d\tensf_\cU - \kldiv{q(\tensf_\cU)}{p(\tensf_\cU)}
\end{equation}
where $\tensf$ is the true model function (i.e. the noiseless version of $\tensy$).
In order to compute the distribution $q(\tensy_\cU)$ which maximises $L$, we can compute its gradient. In this case the gradient has a nice closed form, and is given by (Equation 17 of \citet{hoangUnifyingFrameworkAnytime2015})
\begin{equation}
	\label{eqn:fic-update-mu}
	\partialdiff{L}{\mu_\cU} =\Theta_\cU\inv \mu_\cU  - \sum_{(x, y)\in \cD}\big[ G(\cU, x) - F(\cU, x)\mu_\cU \big]
\end{equation}
and (Equation 18 of \citet{hoangUnifyingFrameworkAnytime2015})
\begin{equation}
\label{eqn:fic-update-cov}
\partialdiff{L}{\Sigma_\cU} = \Sigma_\cU\inv  - \frac{1}{2}\Theta_\cU\inv - \frac{1}{2}\sum_{x\in \cX} F(\cU, x)\end{equation}
where $G(\cU, x) =  \Theta_\cU\inv \Theta_{\cU x} \Gamma_{x}\inv y$, $F(\cU, x) = \Theta_\cU\inv  \Theta_{\cU x} \Gamma_{x}\inv \Theta_{x \cU} \Theta_\cU\inv\mu_\cU$ and $\Gamma_x = \Theta_x - \Theta_{x\cU} \Theta_{\cU}\inv \Theta_{\cU x}$. We will also let $\Theta_{\cU} = \Theta(\cX_\cU, \cX_\cU)$ and $\Theta_{x\cU} = \Theta(x, \cX_\cU)$. Using the closed form expression of the gradient, the optimal $\mu_\cU$ and $\Sigma_\cU$ can be computed through gradient ascent or by directly solving equations $\partialdiffil{L}{\mu_\cU} = 0$ and $\partialdiffil{L}{\Sigma_\cU} = 0$. In our experiments, we will use the latter method. 

Given the optimal $\mu_\cU$ and $\Sigma_\cU$, the mean of the sparse NTKGP approximation can be computed as
\begin{equation}
	\label{eqn:ntkgp-approx-mean}
	\mu_\text{sNTKGP}(\tensy_T | \tensy) = \Theta_{T\cU}\Theta_\cU\inv \mu_\cU^*,
\end{equation}
and the covariance by
\begin{equation}
\label{eqn:ntkgp-approx-cov}
\Sigma_\text{sNTKGP}(\tensy_T | \tensy) = \Theta_T - \Theta_{T\cU} \Theta_\cU\inv \Theta_{\cU T} +  \Theta_{T\cU} \Theta_\cU\inv \Sigma_\cU^* \Theta_\cU\inv \Theta_{\cU T}.
\end{equation}
The approximate posterior can now be computed with time cubic with number of inducing points, rather than cubic in number of training points. The equation can also be iteratively computed in quadratic time if \cref{eqn:matinv} is used for matrix inversion.

From the decomposition in \cref{eqn:fic-update-mu,eqn:fic-update-cov}, we see that the FIC approximation is nice to use since it is possible to decompose the summation above based on terms $F(\cU, x)$ and $G(\cU, x)$ where each terms only depends on one element, and independent of the other elements in $\cD$. This is convenient when trying to add one element and see what effect it has on the inducing points distribution. Using this fact, it is also possible to apply a further linear approximation on the criterion (as we will present in the next subsection) which allows for even more efficient computation.

\subsection{Approximating incremental change of $\mu_\text{NTKGP}$ and $\Sigma_\text{NTKGP}$}

Given the inducing point prior $\Sigma_\cU(\cX)$ as computed by the method above, we could theoretically compute the approximate posterior $\Sigma_\text{sNTKGP}(\cX_T | \cX \cup \{x'\})$ when one point is added directly using the same method. However, this can still be expensive to perform. Therefore, we propose a simple technique to approximate the updated value of $\Sigma_\cU(\cX \cup \{x'\})$, and therefore $\Sigma_\text{sNTKGP}(\cX_T | \cX \cup \{x'\})$, when we add a small number of points to the training set. The approximation technique is akin to the influence function as proposed by \citet{kohUnderstandingBlackboxPredictions2017}.

When we add a single point to the training set, the ELBO (\cref{eqn:elbo}) will only change by a small amount. Let $L' = L + \Delta L$ be the new loss function after adding $x'$ into the set $\cX$, and $\Delta L$ be the contribution specifically from the new element $x'$. Since $L$ only changes by an amount $\Delta L$, we also do not expect $\Sigma_\cU$ to change by much. Let $\Sigma_\cU(\cX \cup \{x'\}) = \Sigma_\cU(\cX) + \Delta\Sigma_\cU(x'; \cX)$ where $\Delta\Sigma_\cU(x'; \cX)$ is the change of the inducing prior after adding the training point $x'$. We see that $\Sigma_\cU(\cX \cup \{x'\})$ and $\Sigma_\cU(\cX)$ would be slightly different from each other.

We find that the change in the inducing points posterior can be instead given by
\begin{equation}
\label{eqn:linapprox-cov}
\Delta \Sigma_\cU(x'; \cX) \approx 
- \bigg( \frac{\partial^2 L}{\partial \Sigma^2}\bigg|_{\Sigma = \Sigma_\cU(\cX) } \bigg)\inv 
 \bigg( \frac{\partial (\Delta L)}{\partial \Sigma}\bigg|_{\Sigma = \Sigma_\cU(\cX) } \bigg)
\end{equation}

Conveniently, the FIC approximation gives rise to a loss function whose gradient is easily decomposable. We can write the derivative of the loss function contribution from $x'$ as
\begin{equation}
\label{eqn:linapprox-cov-a1}
\partialdiff{(\Delta L)}{\Sigma} = -\frac{1}{2}  F(\cU, x).
\end{equation}

We also see that the Hessian of the original loss is equal to the Jacobian of the matrix inverse, i.e.
\begin{equation}
\label{eqn:linapprox-cov-a2}
\frac{\partial^2 L}{\partial \Sigma^2} = \frac{\partial}{\partial \Sigma} \bigg(\frac{1}{2} \Sigma\inv\bigg)
\end{equation}
which can easily be computed in closed form and also using any auto-differentiation package. Expressions from \ref{eqn:linapprox-cov-a1} and \ref{eqn:linapprox-cov-a2} can be plugged back into \ref{eqn:linapprox-cov} to obtain the change of $\Sigma_\cU$ when one new sample is added. 

Therefore, any criterion $\alpha$ we have which would depend on the approximation of the sNTKGP can be approximated as
\begin{equation}
\Delta \alpha(x'; \cX) \approx 
\bigg\langle{\partialdiff{\alpha}{\Sigma_\cU}}, {\Delta \Sigma_\cU(x'; \cX)}\bigg\rangle 
.
\end{equation}
This expression can be used to compute the change in the active learning criterion when an additional data point is added. Because this term is based on the inner product of some matrix quantities, it is parallelisable and in practice speeds up the algorithm by a large amount.

Experimentally, the Hessian can be expensive to compute. In order to speed up the computation, we can instead use an alternate parametrisation of the sparse GP called the \textit{natural parameters} which defines parameters to match more naturally to the terms that appears in the Gaussian distribution. We refer the readers to the original paper \cite{hoangUnifyingFrameworkAnytime2015} for further details on this.

\section{Discussions On Active Learning Criteria}
\label{appx:crit}

In this section, we discuss the active learning criteria based on the NTKGP which we use. We first provide further discussion on the expected value criterion, then we discuss two other possible criteria omitted from the main text based on variance percentile and on mutual information. The experiments comparing the performances of the other active learning criteria can be found in \cref{appx:al-other-crit-appx}.

\subsection{Expected Variance Criterion}




The expected variance criterion has been introduced in the main paper in \eqref{eq:criterion}. In this section, we would like to further explore how the active set is selected based on this criterion. In \cref{fig:tsneeoloptdigits}, we embed our data onto a space using t-SNE \cite{maatenVisualizingDataUsing2008}, then plot out the predictive variance at model initialization of different regions. An interesting observation from the plot is that the output variance with respect to model initialization is actually heteroscedastic. This means there will be some inputs which will have more varied model prediction than others at initialization. From the plot, we see that our algorithm does not necessarily pick the samples which evenly covers the whole input  space, unlike what coverset-based algorithms may choose to do. Rather, our algorithm tends to also prioritise querying points from regions where the predictive variance is high, and balance it with querying from regions with a larger number of points.

\begin{figure}[h]
	\centering
	\includegraphics[width=0.5\linewidth]{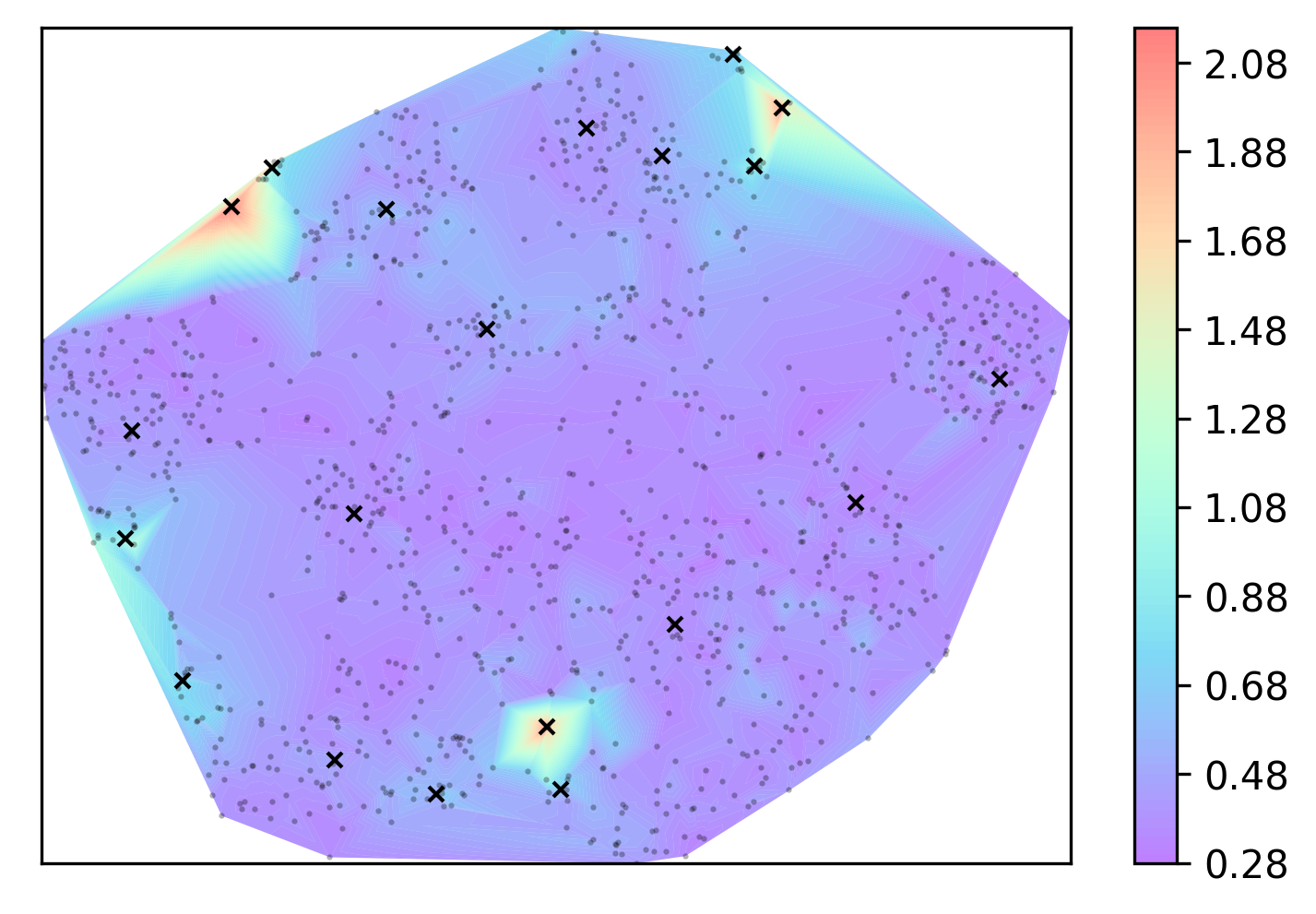}
	\caption{Visualisation of which points are selected by $\alpha_\text{EV}$ criterion. The black points represent the unlabelled set of data from the {\sffamily Handwritten Digits} (from the UCI repository), mapped onto a 2D plane using t-SNE based on the predictive correlation (i.e. if the model output is more correlated, then they will be mapped closer in the above representation). The contours represent the predictive standard deviation of each input point, computed using the empirically with the model output at initialization. The crosses represent the points which are selected by our active learning algorithm using the $\alpha_\text{EV}$ criterion.}
	\label{fig:tsneeoloptdigits}
\end{figure}

\subsubsection{Running Time}

The computation of the criterion $\alpha_\text{EV}(\cX)$ has a cost of $\cO(|\cX_T|\cdot |\cX|^2)$, where $\cX_T$ is the reference test set. This is due to the fact that we have to compute the GP posterior variance for each of the element in $\cX_T$. Computing the GP posterior is cubic with respect to $\cX|$, however can be made quadratic due to the trick of incrementally computing the NTKGP which we will discuss in \cref{appx:ntkgp-comp}. Note that since the NTK is assumed to be constant from initialization, it can be precomputed at the beginning and reused throughout the active learning process, incurring little cost during the actual active learning process.

\subsubsection{Practical Speedups}

In additional to the speedups in computations as discussed in \cref{appx:sgp}, we also make further approximations in order to compute $\alpha_\text{EV}$. Namely, rather than computing the empirical mean of output variance on the whole test set $\cX_T$, we first select a subset of the test set using simple diversification methods (e.g. \textsc{K-Means++}), then only compute the criterion on those data. We find that not only does this speed up the criterion computation, this technique also remove some bias in the test data distribution and give a more representative of the overall variance across the whole space. We also optimize our computation of $\alpha_\text{EV}$ further by only computing the variance $\sigma^2_\text{NTKGP}$ and not the whole covariance $\Sigma_\text{NTKGP}$.

\subsection{Variance Percentile} 

Another possible idea is to ensure that some proportion of test points has a low variance. This can be done by considering the $r$th percentile of variance, i.e.
\begin{equation}
\alpha_\text{rV}(\cX) = -\  \text{percentile} \big( \{ \sigma^2_\text{NTKGP}(x|\cX) | x \in \cX_T \}, r \big).
\end{equation}
In this case, letting $r=50$ is equivalent to considering the median of test output variance, and letting $r=100$ is equivalent to considering the maximum test output variance.

Unfortunately, the maximum function is not submodular in general, and therefore has no theoretical guarantees when points are selected using the greedy algorithm. Nonetheless, we will conduct experiments to select points based on this criterion for the cases where $r=90$ and $r=100$.

\subsection{Mutual Information Criterion}

The previous criteria we have discussed often does not take into account the distribution of $\cX_T$. Through the lens of information theory, we can view the active learning problem as attempting to select points $\cX_\cL$ with outputs $\tensy_\cL$ such that the most information can be obtained about $\tensy_T$, and that the dependence between each datapoint is also taken into account. Similar to \citet{krauseNearOptimalSensorPlacements2008}, we can attempt to maximise the mutual information
\begin{align}
\alpha_\text{MI}(\cX_\cL) &= \mutualinf[\tensy_\cL ; \tensy_{T \setminus \cL}] \\
&= \entropy[\tensy_{T \setminus \cL}] - \entropy[\tensy_{T \setminus \cL} | \tensy_\cL] \\
&= \frac{1}{2}\log \det \Sigma_\text{NTKGP}(\cX_{T \setminus \cL})
-\frac{1}{2}\log \det \Sigma_\text{NTKGP}(\cX_{T \setminus \cL} | \cX_\cL) 
+ \text{constant,}
\end{align}
where we use the shorthand $\cX_{T \setminus \cL} = \cX_{T} \setminus \cX_\cL$. We are able to arrive at the final line since we know the predictive covariance follows a multivariate Gaussian distribution.

In the case that the test set and the unlabelled pool are disjoint, we have $\tensy_{T \setminus \cL} = \tensy_T$, which means $\entropy[\tensy_{T \setminus \cL}]$ is a constant regardless of which active set we choose. In this case, maximising $\alpha_\text{MI}$ is equivalent to minimising $\entropy[\tensy_T | \tensy_\cL]$. Similar to the case of $\alpha_\text{EV}$, in experiments where the test set is separate from the unlabelled pool, we will simply report the values of $-\entropy[\tensy_T | \tensy_\cL]$. 

\subsubsection{Numerical Stability}

A particular issue with using the mutual information is the numerical stability when computing the entropy values. In particular, $\alpha_\text{MI}$ involves computing the entropy $\entropy[ \tensy_{T\setminus \cL} | \tensy_\cL]$, which involves computing the determinant $\det \Sigma_\text{NN}(\cX_{T\setminus \cL} | \cX_\cL)$. However, when the test set $\cX_{T \setminus \cL}$ have points which are highly correlated with each other, the matrix $\Sigma_\text{NN}(\cX_{T\setminus \cL} | \cX_\cL)$ may be singular and cause the determinant to be undefined. 

Notice that the decision of using $\cX_{T \setminus \cL}$ instead of $\cX_{T}$ will already alleviate some of the issues regarding singular matrices. However, in order to prevent further issues, in our experiments, we decide to pre-filter the points that are used for the test set. In particular, instead of using all points from the test set, we select only a subset of training points such that $\entropy[ \tensy_{T}]$ is maximised. This means that the subset selected will be as independent from each other as possible. Furthermore, using a subset of points instead of the full test set also reduces the matrix size, which speeds up our computation as well. The subset of points that are selected are done so using the \textsc{K-Means++} initialization method.


Furthermore, to avoid computing the determinant of a singular matrix, we add in a diagonal noise term in order to compute the determinant of $\Sigma_\text{NN}(\cX_{T\setminus \cL} | \cX_\cL) + \sigma^2_n I$ instead. This corresponds to the case where the observation has some added noise with variance $\sigma_n^2$.



\section{Further Details about the Greedy Algorithm}
\label{appx:alg}

As denoted in \cref{alg:al}, for the data selection stage, we utilise the simple greedy algorithm to find the set of suitable points. However, the greedy algorithm requires $\cO(nk)$ criterion computations, which can be slow. 

There is a large literature on efficient submodular maximisation algorithms with cardinality constraints. This, however is not the focus of this paper. Furthermore, since the data selection stage is independent of the true labels, our algorithm can  Nonetheless, for practical purposes, we choose to apply two simple algorithm optimization techniques.

The first optimization performed based on the \textsc{Accelerated-Greedy} algorithm \cite{minouxAcceleratedGreedyAlgorithms1978}. For each element in the unlabelled data pool, we store the marginal gain value $\Delta_x(\cX) \triangleq \alpha_\text{EV}(\cX \cup \{ x \}) -\alpha_\text{EV}(\cX)$ for some $\cX$ that was previously computed. Then, as the greedy algorithm progresses, we will continue to grow the selected active set into some $\cX' \supset \cX$. From submodularity, we know that $\Delta_x(\cX) \geq \Delta_x(\cX')$. This means that in one particular round with set $\cX'$, if we already have checked another point whose marginal gain in that round is such that $\Delta_{x'}(\cX') \geq \Delta_x(\cX)$, then there is no point in checking $x'$ since we know the criterion value of $\cX' \cup \{x'\}$ will definitely be higher than that obtained from $\cX' \cup \{x'\}$. In our algorithm, when the empirical NTK is used, we reinitialize the random neural network after each batch of data (which changes the empirical NTK), and as a result also reset the stored values of $\Delta_x$. 

Another optimization used is that in each greedy round, we do not compute the criterion on all elements, but only on a subset of $k'$ elements per round. In each greedy round, the maximum value amongst the $k'$ elements tested is added to the active set. This is the technique used in the \textsc{Stochastic-Greedy} algorithm \cite{mirzasoleimanLazierLazyGreedy2015}, and is able to achieve accuracy arbitrarily close to $1 - 1/e$ in $\cO(n)$ time (independent of $k$). In our experiments, we either set $k' = 1000$ for experiments with smaller budgets and $k' = 250$ for experiments with larger budgets.

\section{Proof of \cref{eqn:exp-loss-ms}}
\label{appx:model-sel-proof}

We verify \eqref{eqn:exp-loss-ms} for MSE loss. For a single test point $(x, y)$, notice that we can write
\begin{align*}
\expected_{\theta_0}\big[\ell\big((x, y), \text{train}(\theta_0)\big)\big]
&=\expected_{y' \sim \normdist(\mu_\text{NN}(x | \cX), \sigma^2_\text{NN}(x | \cX))} \big[ (y' - y)^2 \big] \\
&= \expected_{y'}[(y')^2] - 2y \expected_{y'}[y'] + y^2 \\
&= \mu_\text{NN}(x|\cX)^2 + \sigma^2_\text{NN}(x|\cX) - 2y\mu_\text{NN}(x|\cX) - y^2 \\
&= (y - \mu_\text{NN}(x|\cX))^2 + \sigma^2_\text{NN}(x|\cX) 
\end{align*}
where the first equality is based on the predictive distribution of the converged model.


\section{Further Details on Experimental Setup}
\label{appx:exp-details}

In this section, we provide further details of the experiments conducted. All the code for the experiments are also provided in the supplementary material. 

\subsection{Training Setup for Each Experiments}

\subsubsection{Regression Experiments}

We will use a combination of generated dataset and real-life data. For this paper, randomly generated data ({\sffamily Random Model}) is constructed from random points from a ball, with their output values generated from a random model initialization (which ensures that the neural network will be adequate to describe the data used). 

Meanwhile the real-life training data are taken from the UCI Machine Learning Repository \cite{duaUCIMachineLearning}. For all dataset, we split the whole data in half, and let one half be the pool of data and the other be the test data which the algorithm has no access to. All the dataset used are regression datasets, with the exception of the {\sffamily Handwritten Digits} dataset which is a classification dataset where we perform regression on the label value (i.e. for the inputs corresponding to the digit 1, we assign $y=1$). All datasets are also normalized such that they have mean of zero and variance of 1.

In all of the regression experiments, the model used is a two-layer multi-layer perceptron with width of 512 and with bias. We set $\sigma_W = 1.$ and $\sigma_b = 0.1$. The NNs are optimized using gradient descent with step size 0.01.

\subsubsection{Experiments on the Classification Dataset}

In the classification datasets, in order to reduce the training time, we restrict the unlabelled data pool to a random subset of the whole training set. For the MNIST dataset, we randomly select 10k points and use it as our unlabelled pool, while for the remaining classification experiments we randomly select 20k points for the unlabelled pool. All the inputs are also rescaled such that the distribution of input values are between $[-1, 1]$. For all the models, we train the models using stochastic gradient descent with learning rate of 0.1 and weight decay of 0.005. The models are trained with training batch size of 32 and are trained for 100 epochs. Below, we provide a brief description of the datasets used and the model architectures used for training on the correpsonding datasets.

\begin{itemize}
\item \textbf{MNIST} \cite{MNISTDatabaseHandwritten}. MNIST is a database whose inputs are single-channel images of size $28\times 28$ that corresponds to handwriting of different digits. For MNIST, we use a multi-layer perceptron with two hidden layers each of width 512 as we have done for the regression experiments. We find that using such a simple model without convolution is sufficient for the MNIST dataset. In the experiments where we vary the network width at active learning stage and training stage (\cref{fig:results-fixed-train-width}), we used a 3-layer MLP and experiment with network widths 128, 256, 512 and 1024. The total query budget was fixed at 600 and the query size at 200. 

\item \textbf{EMNIST} \cite{cohenEMNISTExtensionMNIST2017}. The EMNIST dataset is an extension to the MNIST dataset, where the input contains handwritten characters as well as handwritten digits. There are a total of 47 output classes (since some of the characters that have a similar shape are grouped together in the same group).
For EMNIST, we either use a multi-layer perceptron with three hidden layers each of width 512, or a simple convolutional neural network with 2 blocks of [\textsc{Conv - MaxPool - ReLU}], followed by a fully-connected layer of width 512.

\item \textbf{SVHN} \cite{netzerReadingDigitsNatural2011}. SVHN consists of RGB images of size $32\times 32$ which corresponds to photos different digits. Each image in the SVHN dataset corresponds to one of the 10 classes, each representing a digit. For SVHN, we use a 1-layer WideResNet \cite{zagoruykoWideResidualNetworks2017}. We modify the network by removing the batch normalization layer in order to make the module compatible with \textsc{FuncTools} on Pytorch (see \cref{appx:ntk-comp}) such that the NTK is computable by gradient inner product, and also to ensure positive-semidefiniteness of the resulting matrix. While this may result in a lower model accuracy, this is not a main concern of this work regardless since we are more interested in the \textit{relative} accuracy between each active learning methods. 

\item \textbf{CIFAR100} \cite{krizhevskyLearningMultipleLayers2009}. CIFAR100 is a dataset consisting of RGB images of size $32\times 32$ which corresponds to one of 100 different image categories. For CIFAR100, we use a 2-layer WideResNet with similar modifications as we have done for SVHN.

\end{itemize}

\subsubsection{Model Selection Algorithm}

For the model selection experiments, the data used are from the UCI dataset (as we have done for the regression experiments). 
We also construct an artificial dataset {\sffamily Sinusoidal} whose output is generated from the sine function $y = \sin(\omega\transpose x + b)$ for some random $\omega$ and $b$.

The pool of models $\cM$ consists of MLPs with depths 1, 2, 3 and 4, and with ReLU, GeLU, leaky ReLU ($a=0.1$), and Erf. This makes up a total of 16 candidate models. 
For the {\sffamily Sinusoidal} dataset, we also added an additional model architecture with sinusoidal activation, making a total of 20 candidate models in this case.
All the models are implemented on \textsc{Jax} (in order to match the theoretical NTK which is also computed in \textsc{Jax}). For the model selection algorithm, in order to approximate $\alpha_M$, we choose $\kappa = 0.3 |\cL|$ for the labelled set $\cL$ at that round, and compute the empirical mean of $\alpha_M$ by sampling the random subset 100 times. The batch size for active learning is 10 or 20 depending on the size of the dataset.

The models are trained using the Adam optimizer with learning rate 0.01. The reason the Adam optimizer is used is that some of the activation functions suffer from diminishing gradients, and therefore does not behave well when using gradient descent. The training is all done in a single batch, and is done for 500 epochs.

\subsection{Computation of the Neural Tangent Kernel}
\label{appx:ntk-comp}

\subsubsection{Theoretical NTK using \textsc{Neural-Tangents}}

The theoretical NTK are computed using the \textsc{Neural-Tangents} package \cite{novakNeuralTangentsFast2019}. Even though the final trained model will not have exactly the same network parameters as that assumed to compute the theoretical NTK (due to different packages used during active learning and true model training), we still find that the kernel itself shows enough agreements and are still useful for predicting the model behaviour.

The main disadvantage of using the theoretical NTK is that the theoretical NTK is not available for all model architectures. In particular, \textsc{Neural-Tangents} does not work for MaxPooling layers (only AvgPooling is supported), and also does not work for BatchNorm layers or Dropout layers placed arbitrarily. For this reason, experiments involving the theoretical NTK are only restricted to ones which uses a MLP. We also only use the theoretical NTK for experiments involving our algorithms, and do not use it for MLMOC which uses the empirical kernel in order to predict how a specific model initialization will behave during training.

\subsubsection{Empirical NTK Using \textsc{PyTorch}}

An alternative method of computing the NTK is to do so empirically. In order to compute the NTK empirically, we require taking the inner product of the model output gradient $\del_\theta f(x, \cX)$ with respect to its parameters. This requires us to compute the Jacobian $\del_\theta f(x, \cX)$, which is extremely expensive in terms of required memory and cost for performing the matrix multiplication. In particular, for a model with $p$ parameters and $o$ outputs, computing the NTK on an input of size $n$ requires $\cO(npo)$ space and $\cO(n^2po)$ running time for the matrix multiplication alone. In order to perform this operation efficiently, we follow the method as used in PyTorch's \textsc{FuncTools} tutorial\footnote{See \url{https://pytorch.org/functorch/stable/notebooks/neural_tangent_kernels.html}.}. Note that \textsc{FuncTools} is not compatible with all neural network components (e.g. {BatchNorm}), and therefore in our experiments we do not use those components in the models. Even though \texttt{Neural-Tangents} is able to compute the empirical NTK as well, we chose to compute the empirical NTK on PyTorch since the models and (most of) the model training are done on PyTorch anyway.

For classification instances where the model has multiple outputs, we also perform a further approximation of only using the gradient which contributes to a single model output (which is chosen at random). This is possible since the gradient inner product with respect to each outputs are independent of each other and has the same value in the limit. This is a similar trick which is also used by \citet{mohamadiMakingLookAheadActive2022}.

\subsection{Reported Metrics}
\label{appx:exp-metrics}
To quantify the model performances, we use either the test mean-squared error (MSE) for regression problems, or the test accuracy for classification problems.

To quantify the initialization robustness of the models, we use two different metrics as described below.
\begin{itemize}
\item \textbf{Output variance}. This is used as an initialization-robustness measure for regression tasks, and is the empirical variance of output of trained neural networks with respect to the different model initializations. In our experiments, we report the 90th percentile output variance, which is defined as the 90th percentile output variance test data. The 90th percentile is used instead of the 100th percentile (i.e. the maximum output variance) since using the 100th percentile tends to focus on some outlier rather than reflecting the output variance of the majority of the test data. We also chose to report the 90th percentile value instead of the average value since it is able to indicate the output variance of the worst-case inputs and is therefore a better measure of initialization robustness on the overall space.

\item \textbf{Output entropy}. This is used as an initialization-robustness measure for classification tasks, and is defined as the empirical entropy of the predicted label. For a neural network, the predicted label for some input $x$ is given by $\hat{y} = \arg\max_i f(x ; \theta)_i$. Given multiple trained models with different parameters $\theta_1, \ldots, \theta_k$, we will obtain different predictions $\hat{y}_1, \ldots, \hat{y}_k$. The output entropy is then defined as
$-\sum_{i=1}^{c} v_i \log v_i $ where $v_i = \frac{1}{c} \cdot \sum_{j=1}^{k} \indic{\hat{y}_j = i}$. A lower output entropy corresponds to models whose predictions are more consistent with each other.
\end{itemize}

\subsection{Description of Other Active Learning Algorithms}
\label{appx:exp-other-alg}

Below, we provide a brief description of the other active learning algorithms which we compare our method against.
\begin{itemize}
\item \textsc{Random}. The batch of points are selected uniformly at random without replacement.

\item \textsc{K-Means++}. The batch of points are selected based on the initialization method from \citet{arthurKmeansAdvantagesCareful2007}. When using this algorithm, all the points are selected right from the start (and order in selection by the algorithm is kept). When the user queries for a batch of size $b$, the algorithm returns the next $b$ elements that it has chosen from the \textsc{K-Means++} initialization algorithm.


\item \textsc{BADGE} \cite{ashDeepBatchActive2020}. BADGE utilizes the hallucinated gradient space, or the gradient of the loss function with respect to the final output layer,
$$h(x) = \partialdiff{}{\theta^{(-1)}} \loss((x, \tilde{y}), \theta)$$
where $\theta^{(-1)}$ are the model parameters of the final output layer and $\tilde{y}$ is the pseudo-prediction from the model based on the model output (for example, in classification problems, $\tilde{y}$ would be the one-hot vector representing the class prediction given by the model output). Pseudo-labels are used since in active learning, the algorithm would not have access to the true labels. Given the embedding $h(x)$, BADGE then performs diversification on this embedded input. Theoretically, BADGE is able to provide a balance between selecting points which are diverse and selecting points which the model is uncertain about.

Unlike in the original paper, for experiments involving \textsc{BADGE}, we do not provide the algorithm with an initial training set. Despite this, we find that \textsc{BADGE} is still able to select a good batch of points from the start, and is even able to beat our algorithm in some experiments. 

\item \textsc{MLMOC} \cite{mohamadiMakingLookAheadActive2022}. MLMOC uses the NTK to provide a prediction of how a neural network would behave when an additional point is added to the training set. Given this, the algorithm selects an active set which when trained, will cause the largest change in the model output.

In our experiments, \textsc{MLMOC} will randomly select 100 initial points for classification experiments. The random subset is counted as a part of the budget used up by MLMOC. For SVHN and CIFAR100 experiments, this is fewer points than provided in the original paper. However, this setup is chosen in order to test how the algorithm performs when there is little initial data available. We found that MLMOC performs poorer when there are fewer initial data points available to it. 

\end{itemize}

\section{Further Details on Experimental Results}

\subsection{Description of Figures}
\label{appx:fig-desc}

In this section, we give a more detailed description of the figures plotted in the main paper.

\begin{itemize}
\item \textbf{\cref{fig:correlation-var}}. This figure shows the correlation between variance predicted by the NTKGP, $\sigma_\text{NTKGP}^2(x|\cX)$, and the (empirical) model output variance. The horizontal represents the variance based on $\sigma_\text{NTKGP}^2(x|\cX)$ for a certain element $x$ from the validation set and a certain subset $\cX$ of the unlabelled data pool, while the vertical represents the empirical output variance with respect to different model initialization. In this experiment the {\sffamily Robot Kinematics} dataset was used.

\item \textbf{\cref{fig:small-set}}. This figure shows the relationship between the different criteria value, the 90th percentile output variance, and the test MSE. In each row, from left to right, the graph shows: the relationship between the 90th percentile variance and $\alpha_\text{EV}$; the 90th percentile output variance of models when trained on active sets of increasing sizes; the test MSE when trained on active sets of increasing sizes. The blue line represents the active set which has been sequentially constructed based on the $\alpha_\text{EV}$ criterion. Each row represents trials conducted on one dataset. SP and PR represents the Spearman Rank correlation and Pearson correlation coefficients respectively. The metrics reported on the $y$-axis are described further in \cref{appx:exp-metrics}.
\end{itemize}

\subsection{Relationship Between $\sigma_\text{NTKGP}$, $\sigma_\text{sNTKGP}$ and Output Variance}
\label{appx:exp-sntkgp}

In \cref{fig:appx-regr-sntkgp}, we provide further results comparing the variance obtained from the $\sigma_\text{NTKGP}$, $\sigma_\text{sNTKGP}$ approximations and the true output variance. We find that $\sigma_\text{NTKGP}$ provides good agreement with the output variance, while $\sigma_\text{sNTKGP}$ provides a decent approximation of the true output variance. Furthermore, note that the graph shows the \textit{individual} variance and not the sum of variance (as $\alpha_\text{EV}$ requires) so when computing the criterion, the discrepancy between using the NTKGP or sNTKGP approximation and using true output variance will be lower (see \cref{fig:appx-small-set}).


\begin{figure}[ht]
\centering
\resizebox{\linewidth}{!}{
\begin{tabular}{c|c}
{\sffamily \small Robot Kinematics}  & 
{\sffamily \small Protein}  \\ 
\includegraphics[width=0.4\linewidth]{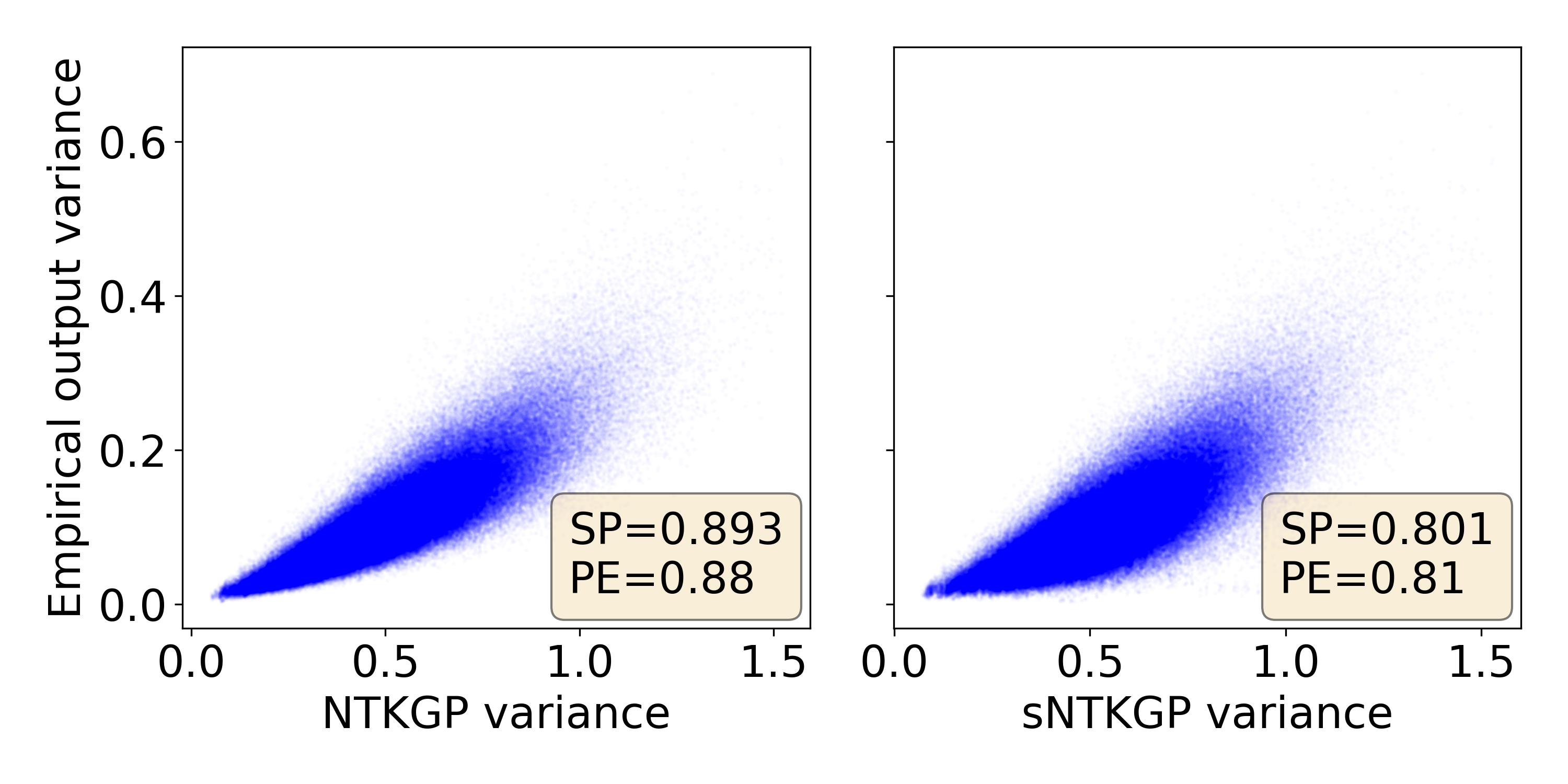} &
\includegraphics[width=0.4\linewidth]{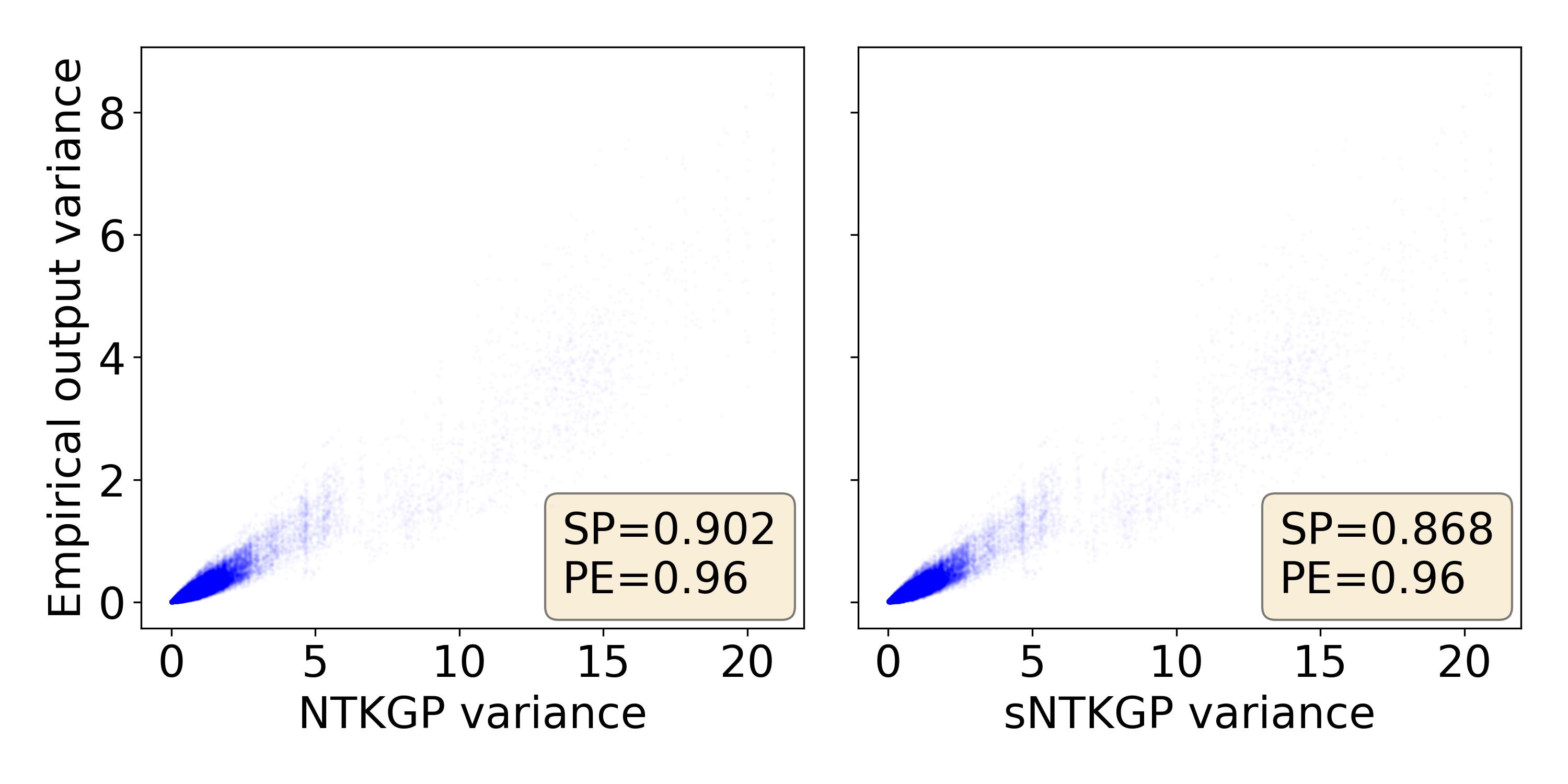} \\

{\sffamily \small Naval}  & 
{\sffamily \small Handwritten Digits (Regression)}  \\ 
\includegraphics[width=0.4\linewidth]{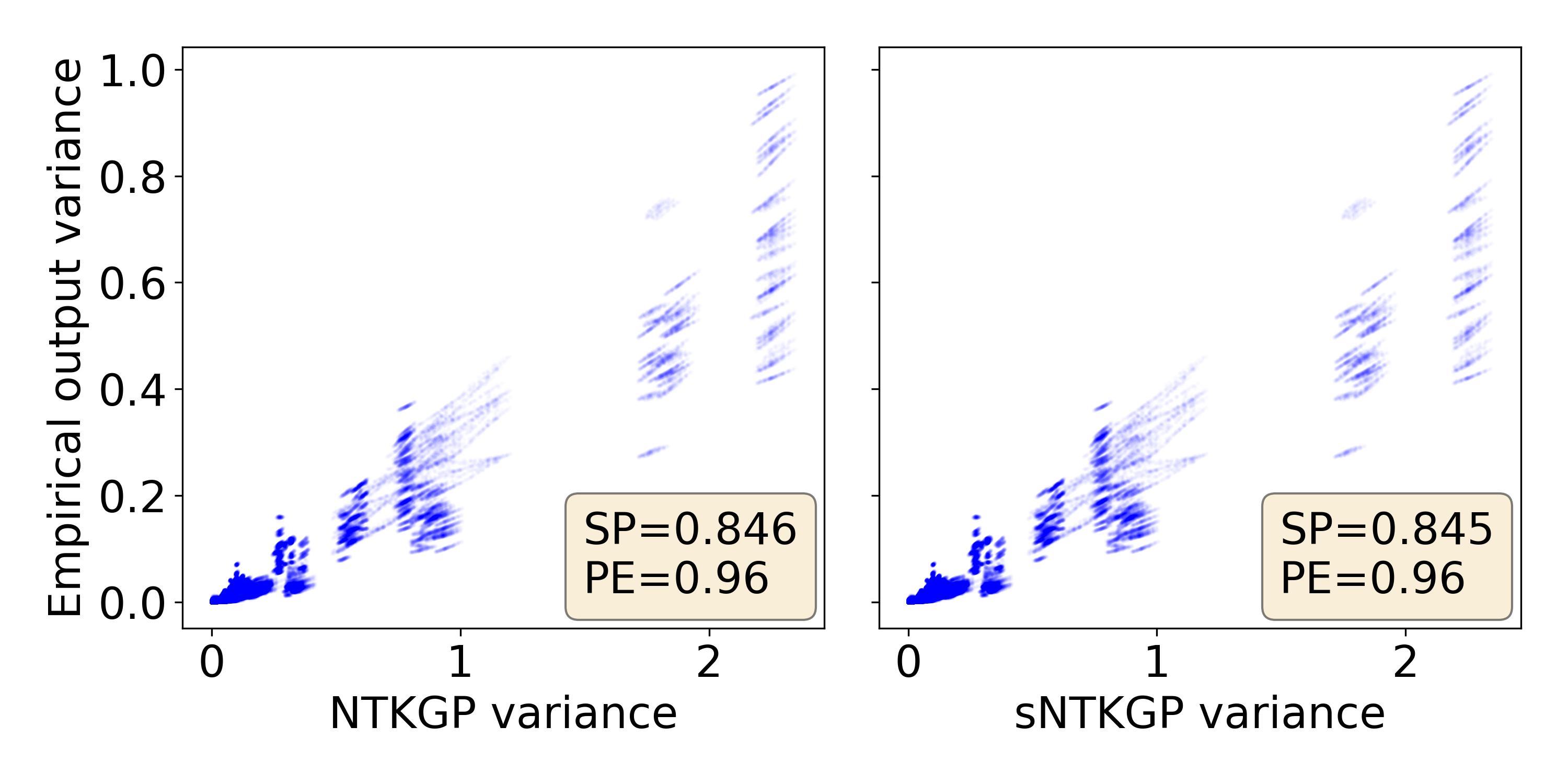} &
\includegraphics[width=0.4\linewidth]{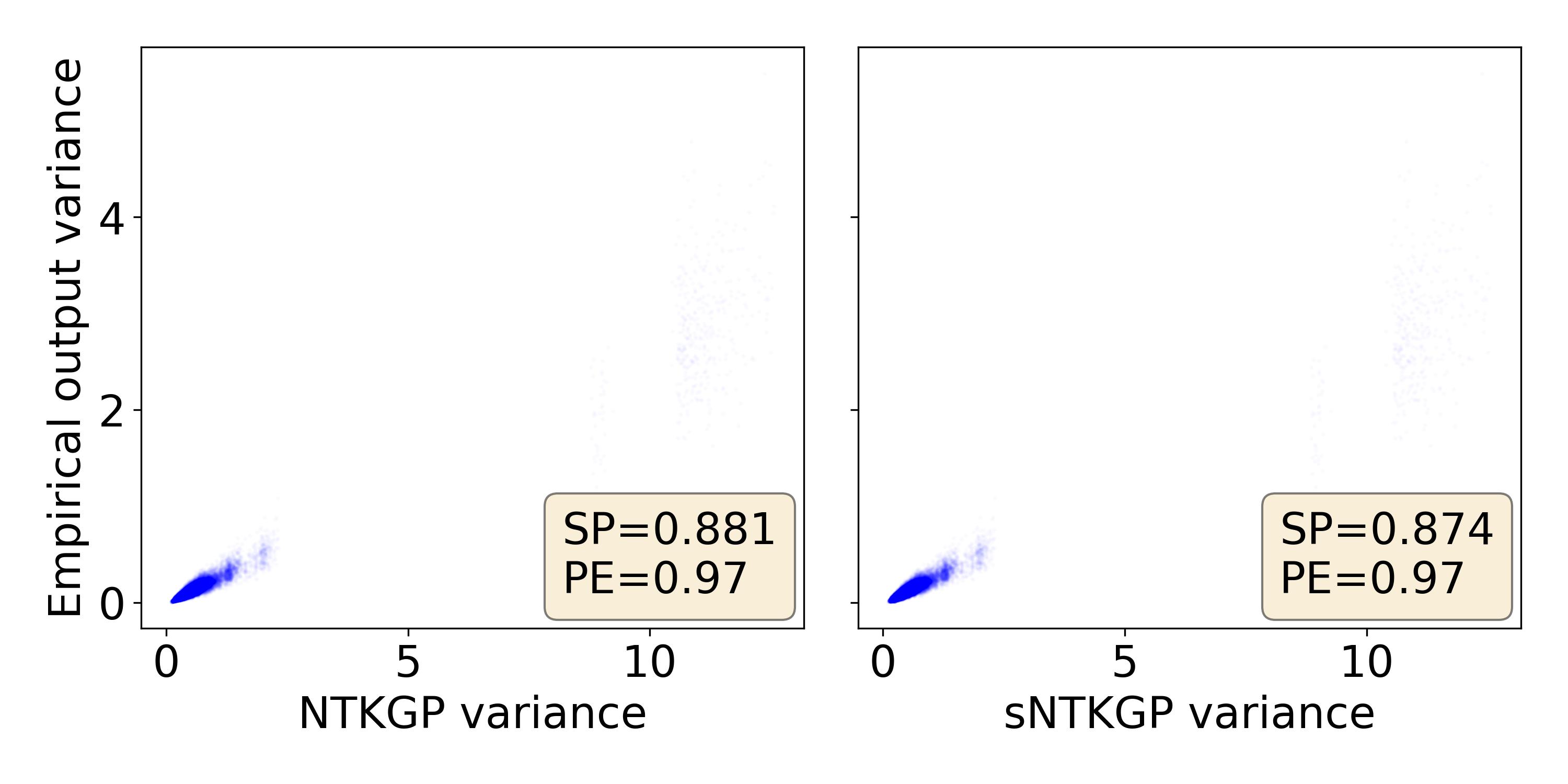} 
\end{tabular}
}
\caption{The correlation between variance predicted by the NTKGP and sNTKGP, and the (empirical) model output variance. The $x$-axis represents the variance based on the NTKGP or the sNTKGP for a certain element $x$ from the validation set and a certain subset $\cX$ of the unlabelled data pool, while the $y$-axis represents the empirical output variance with respect to different model initialization.}
\label{fig:appx-regr-sntkgp}
\end{figure}

\subsection{Additional Results From Regression Experiments}
\label{appx:exp-regr}

\subsubsection{Experiments on Other Regression Datasets}

In \cref{fig:appx-small-set} we show further results for other regression datasets in the sequential selection process, and in \cref{fig:appx-small-set-batch}, we show the results in the larger-scale, batch setting. 

We see that \alg~criterion is able to outperform both the Random selection and \textsc{K-Means++} selection method. The latter result is true likely due to the fact that \textsc{K-Means++} aims at selecting points which are as diverse as possible from each other in the input space. This, however, makes \textsc{K-Means++} more prone to selecting outliers which have large differences with other points in the input space yet have little correlation with the rest of the dataset. Therefore, these outliers would provide little information for making predictions on the more "common" data points in the test set. Meanwhile, \alg~criterion directly incorporates the distribution of the dataset in the point selection process, and therefore is able to select more correlated points based on the existing dataset.

\begin{figure*}[ht]
\centering

\begin{tabular}{c|c}
{\sffamily \small Boston}  & 
{\sffamily \small Naval}  \\ 
\includegraphics[width=0.30\linewidth]{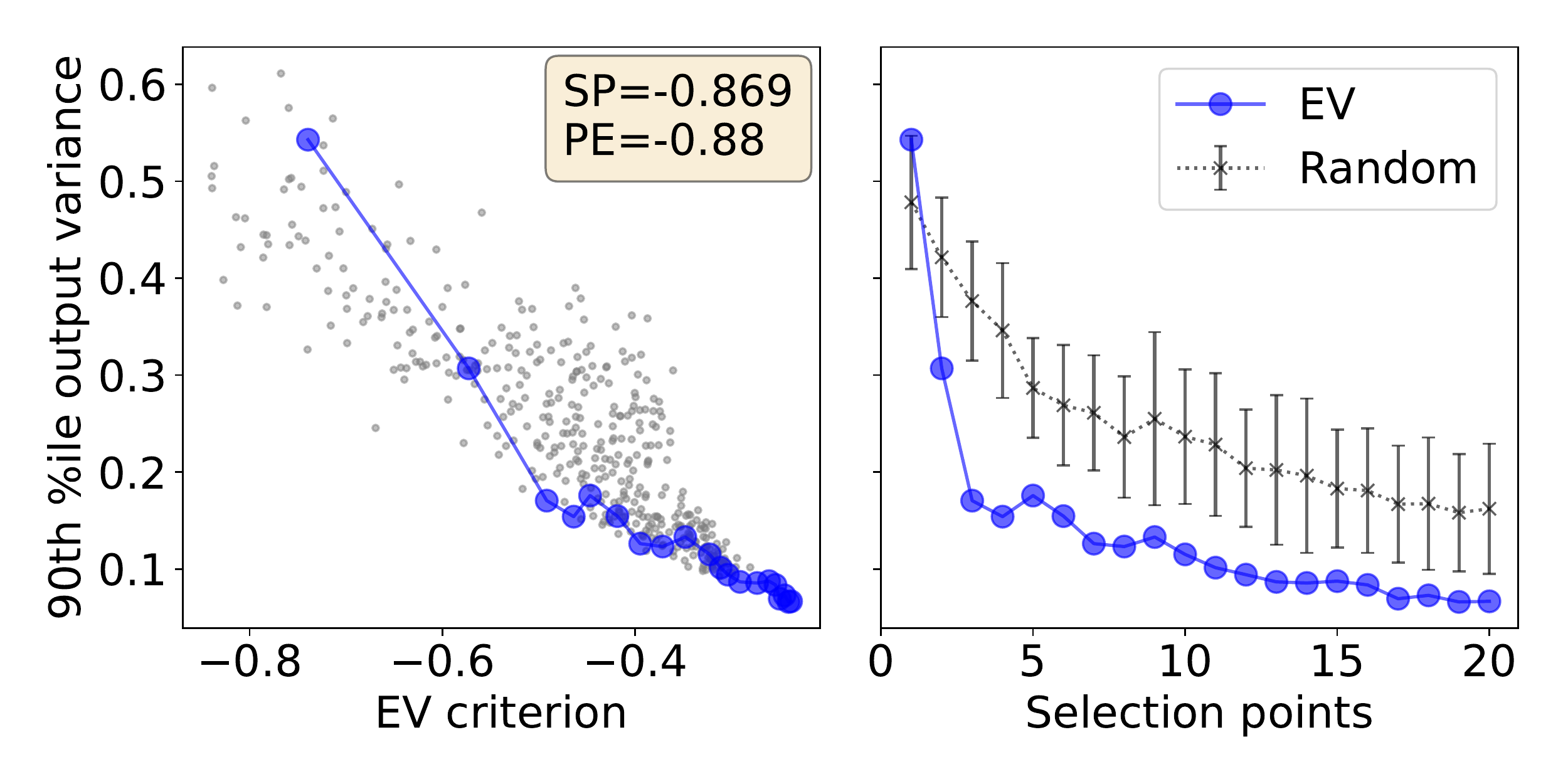}
\includegraphics[width=0.15\linewidth]{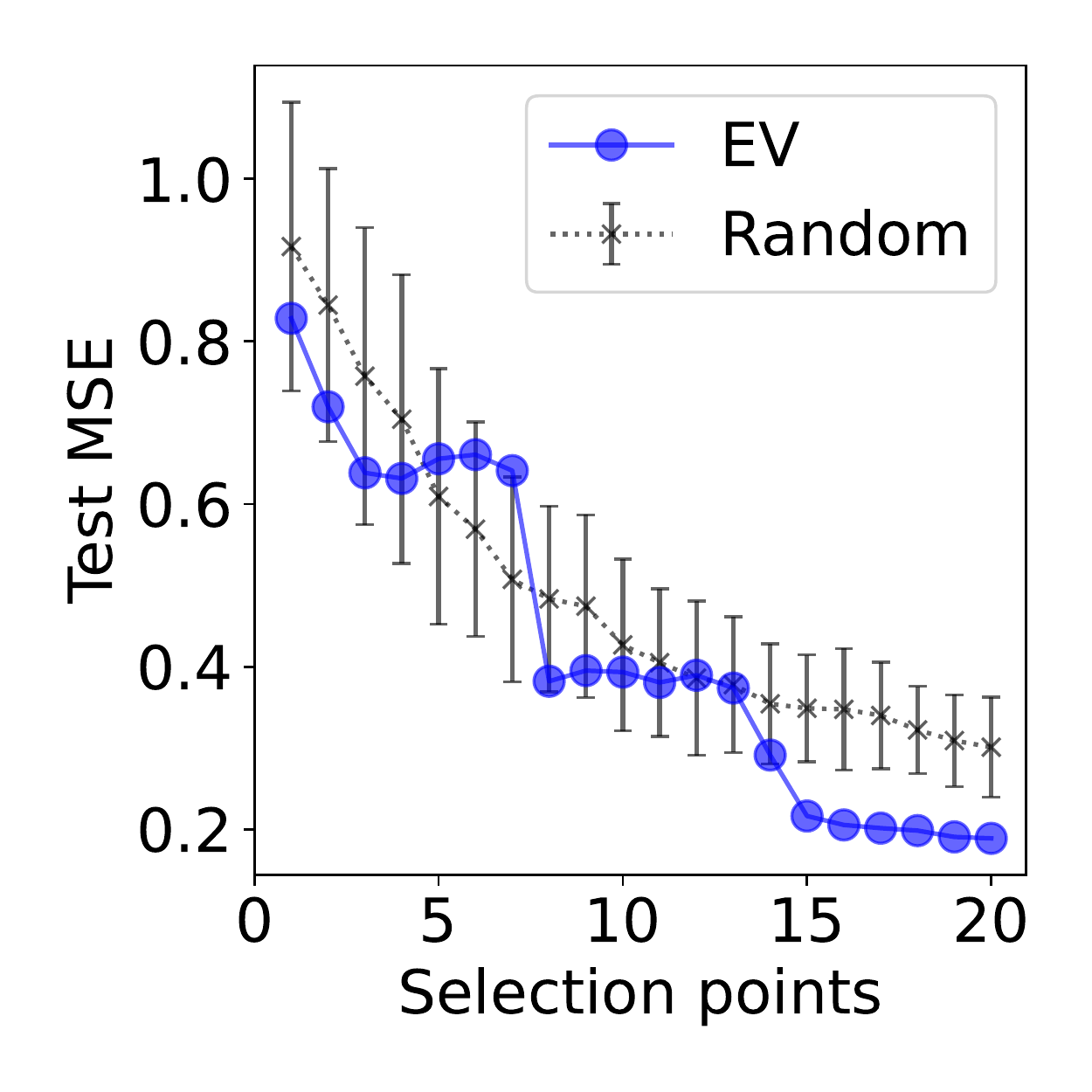} &
\includegraphics[width=0.30\linewidth]{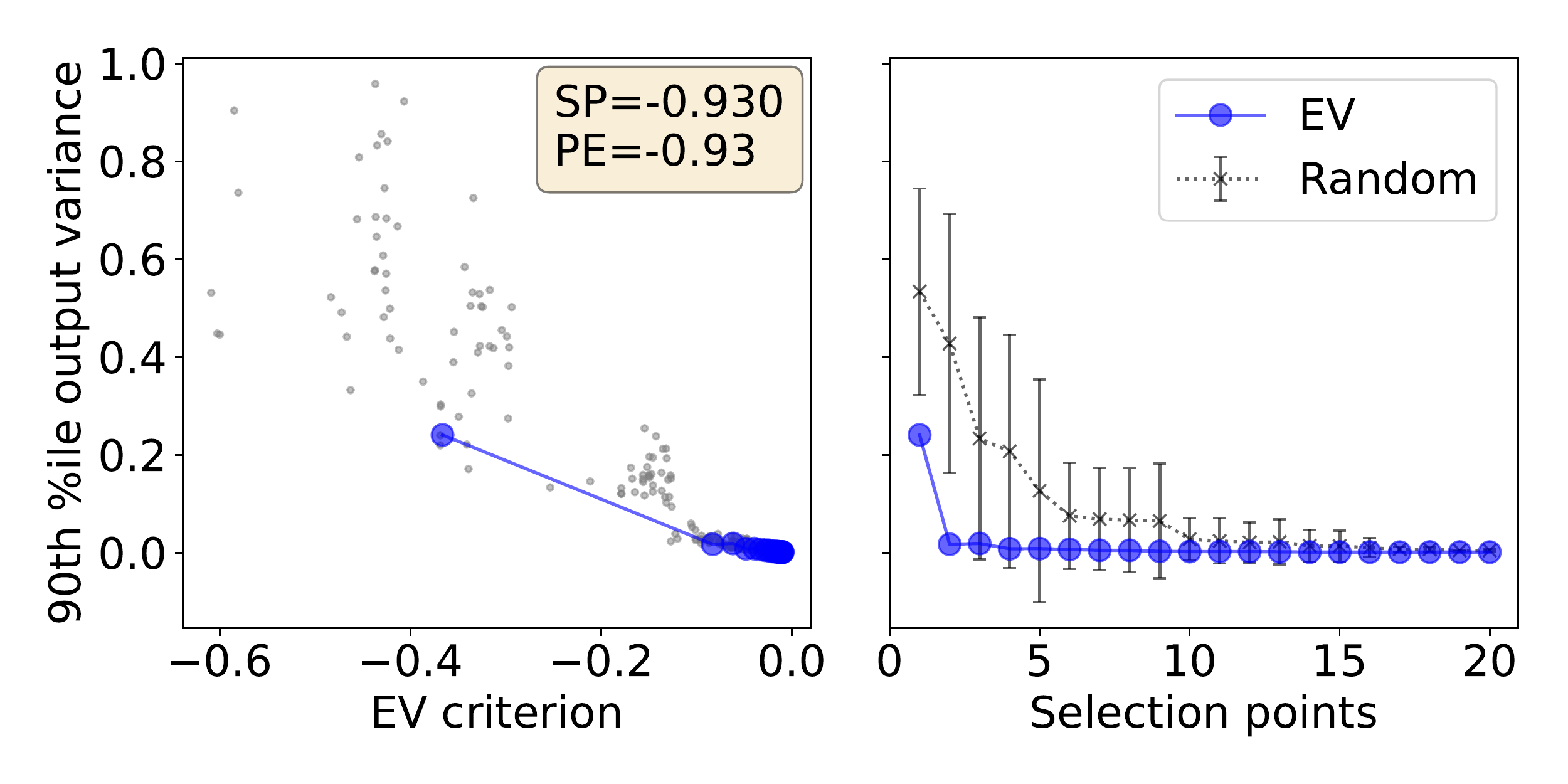}
\includegraphics[width=0.15\linewidth]{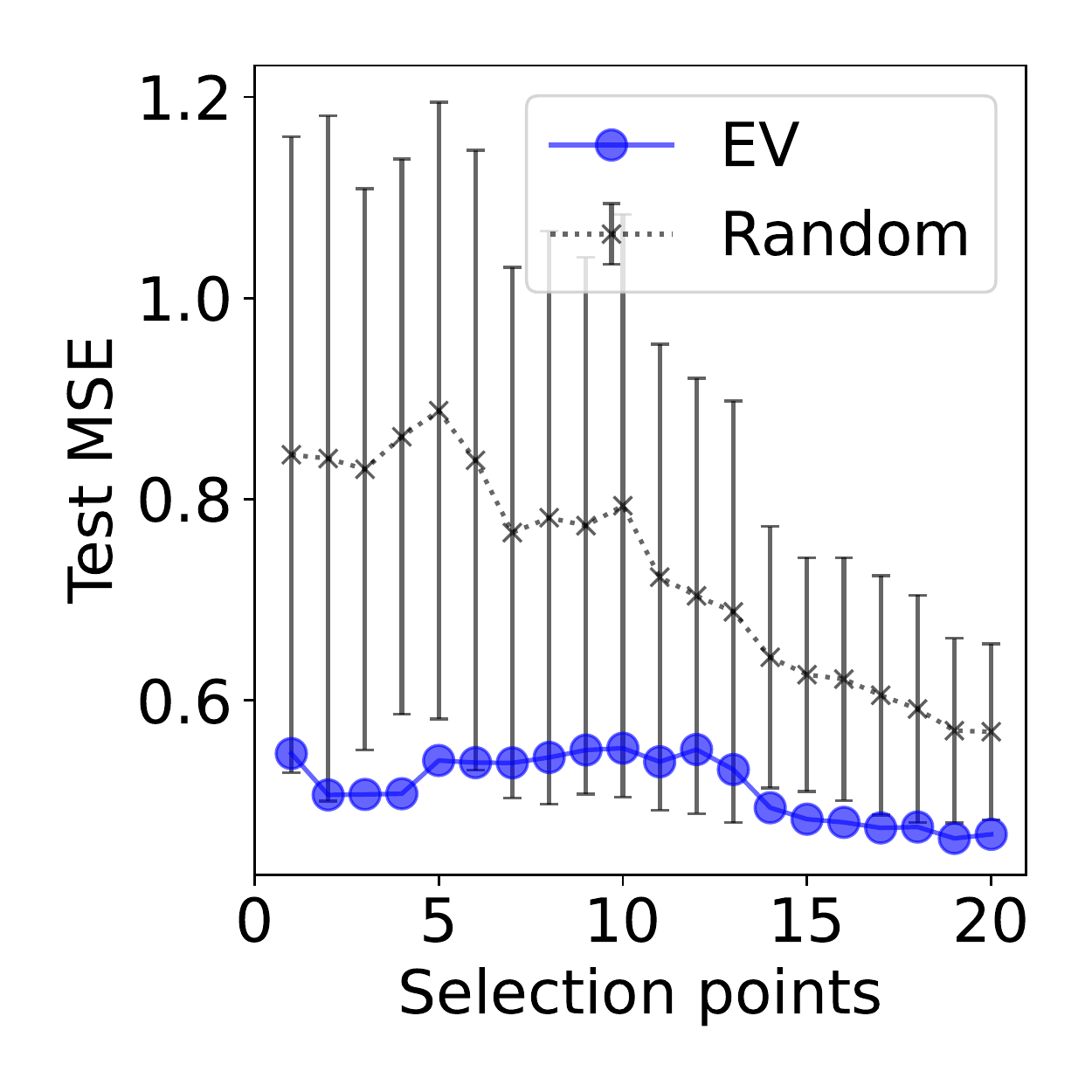} \\

{\sffamily \small Handwritten Digits (Regression)}  & 
{\sffamily \small Robot Kinematics}  \\ 
\includegraphics[width=0.30\linewidth]{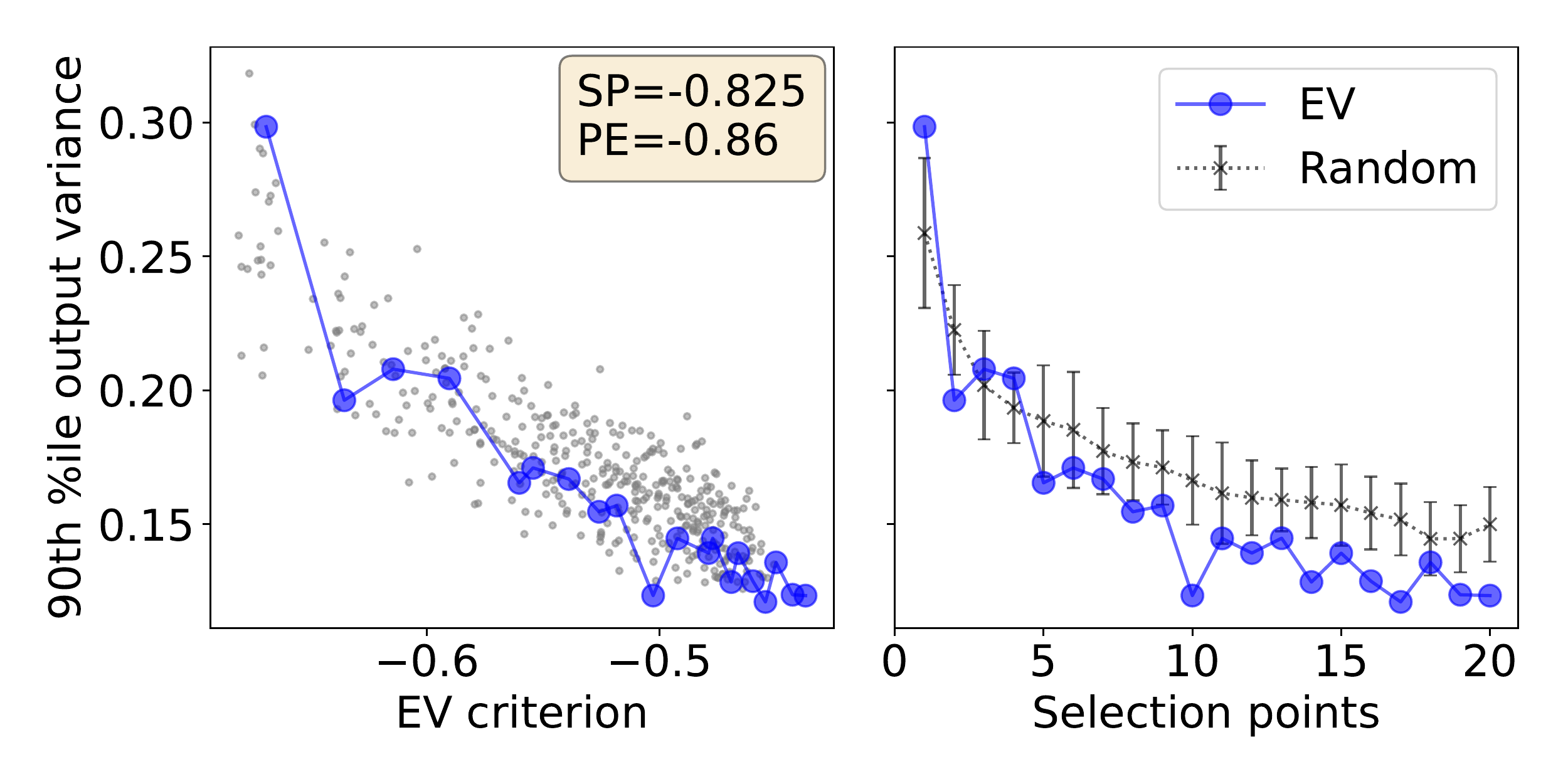}
\includegraphics[width=0.15\linewidth]{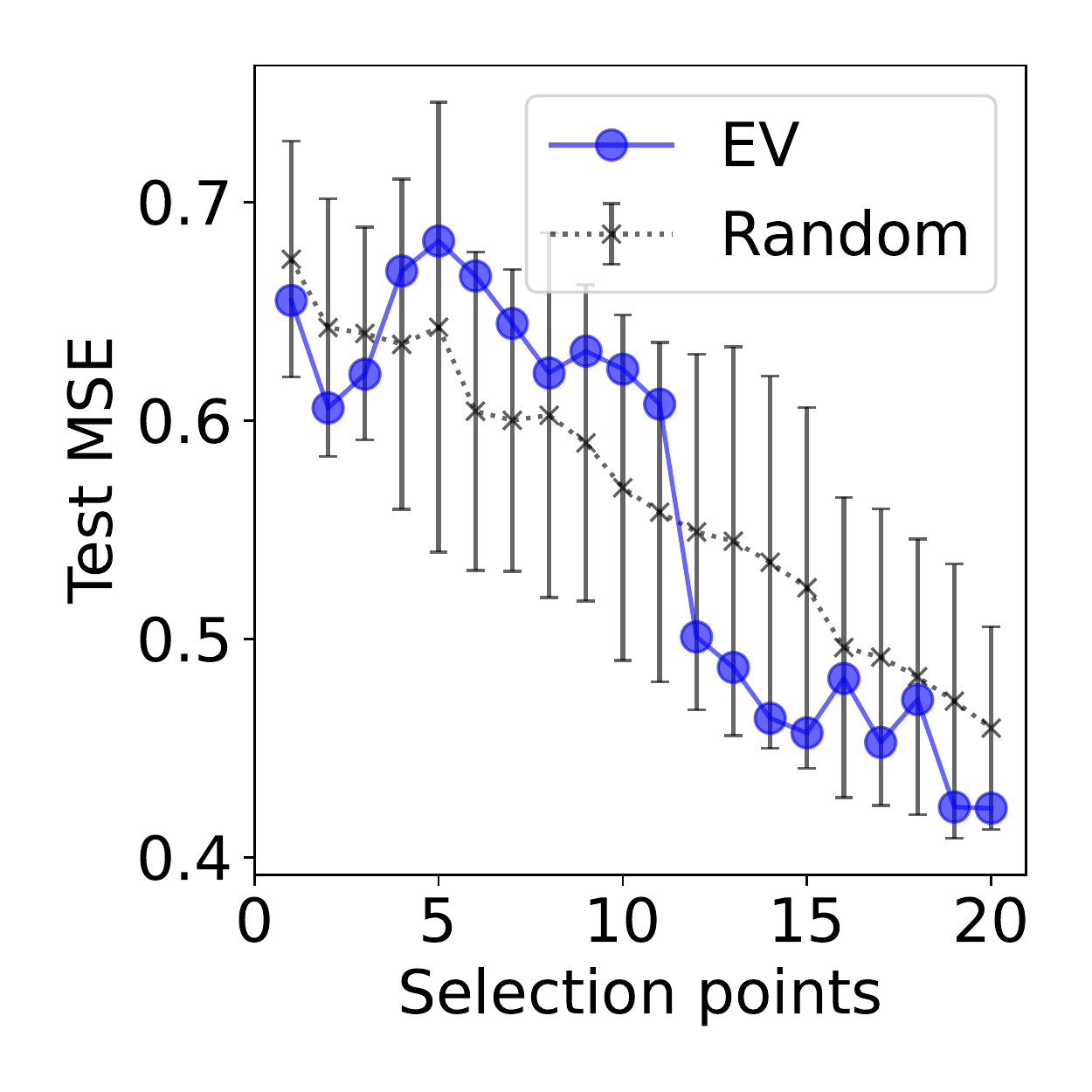} &
\includegraphics[width=0.30\linewidth]{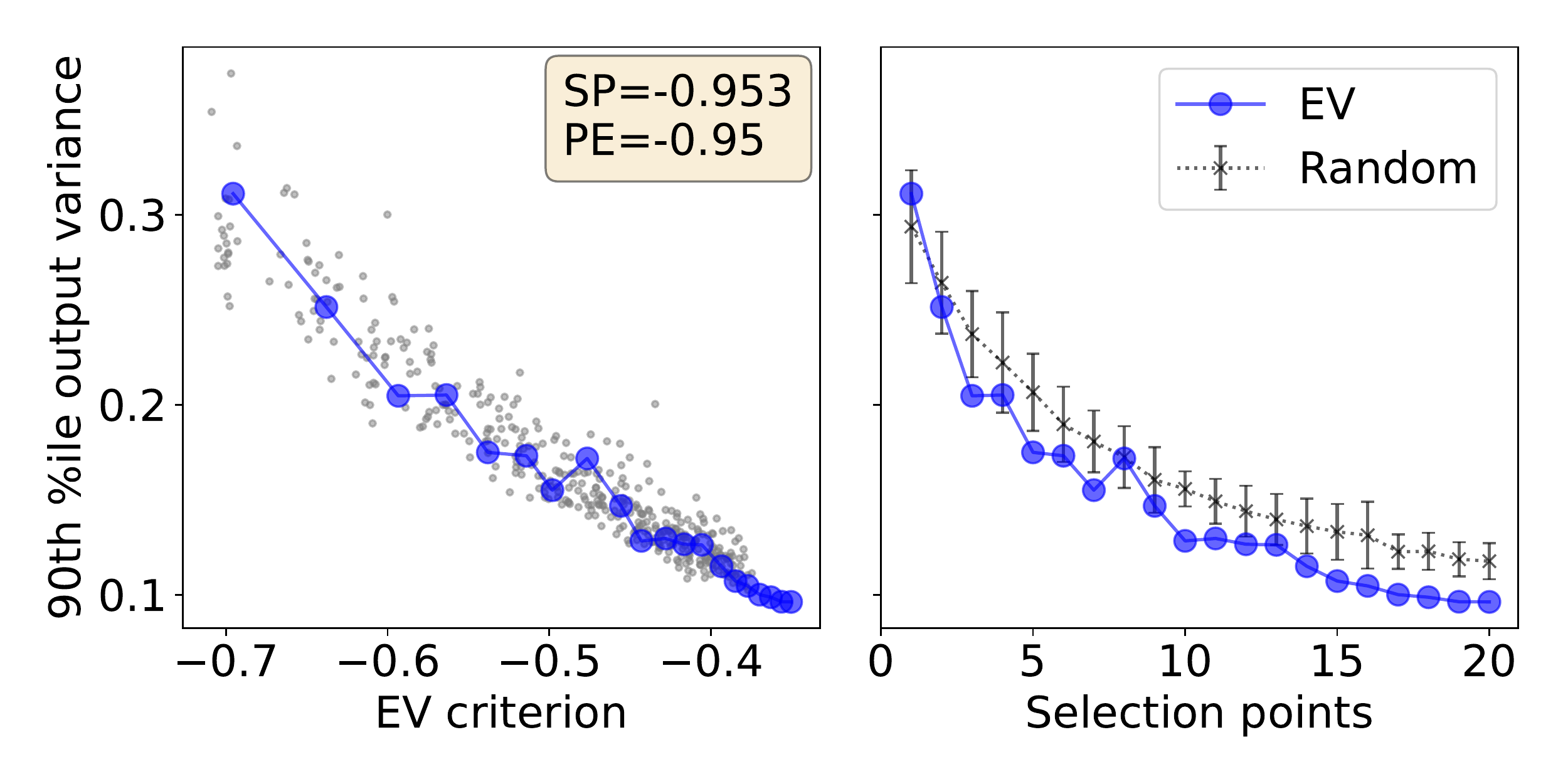}
\includegraphics[width=0.15\linewidth]{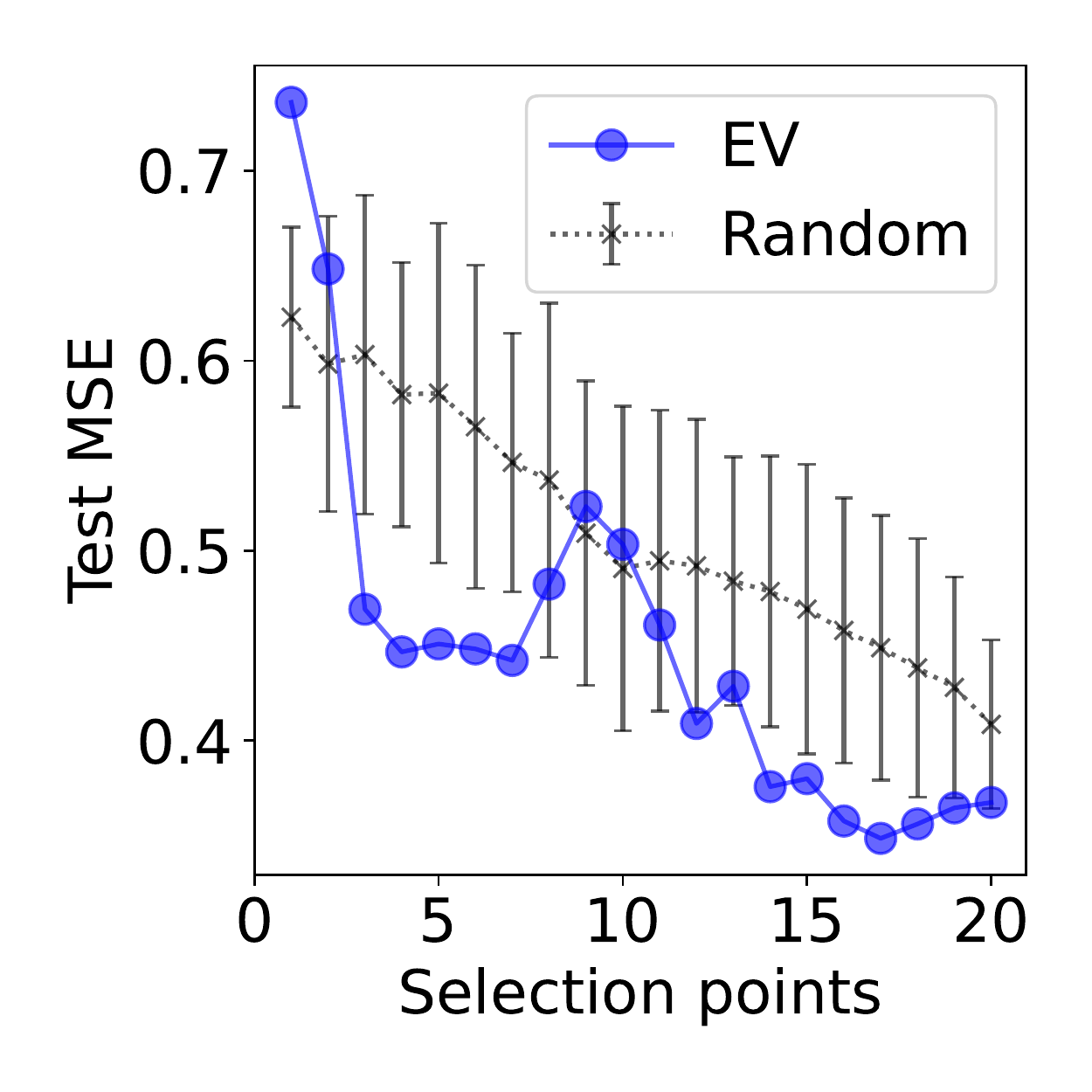}

\end{tabular}

\caption{Relationship between the different criteria value, the 90th percentile output variance, and the test MSE. The information shown is presented in the same manner as \cref{fig:small-set} (see \cref{appx:fig-desc}).}
\label{fig:appx-small-set}
\end{figure*}

\begin{figure}[ht]

\centering
\begin{tabular}{c|c}
{\sffamily \small Boston} &
{\sffamily \small Naval}  \\ 
\includegraphics[width=0.24\linewidth]{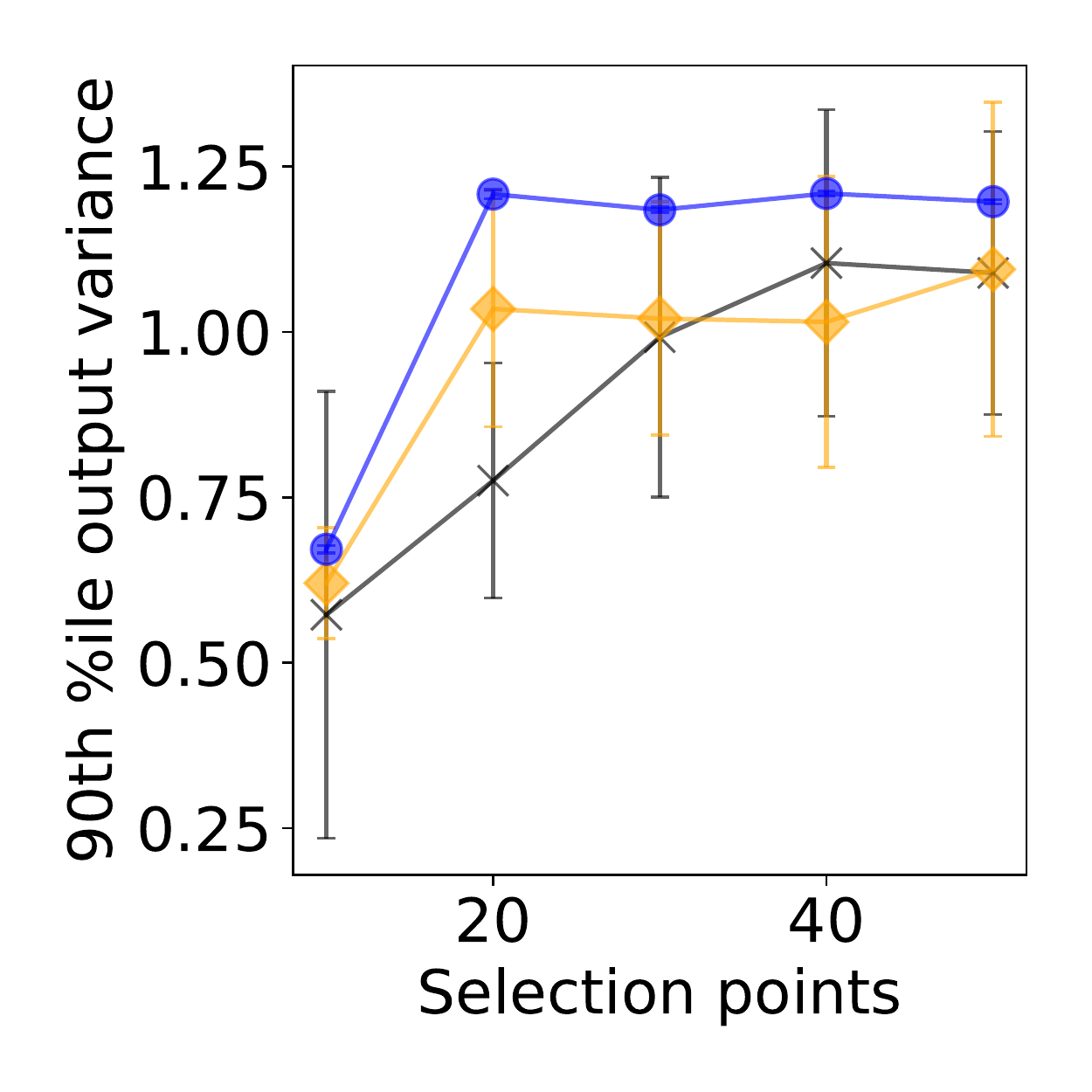}
\includegraphics[width=0.24\linewidth]{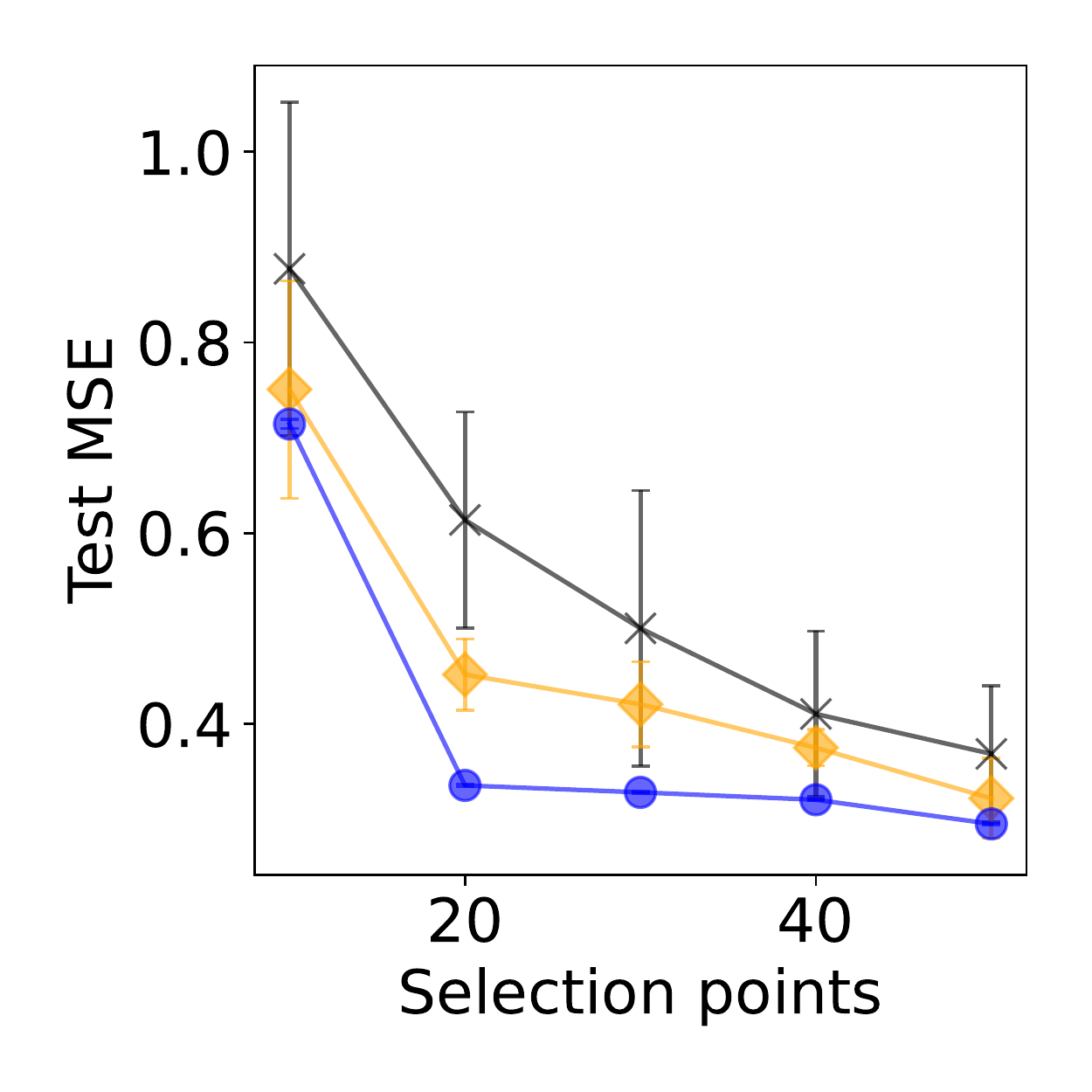} 
&
\includegraphics[width=0.24\linewidth]{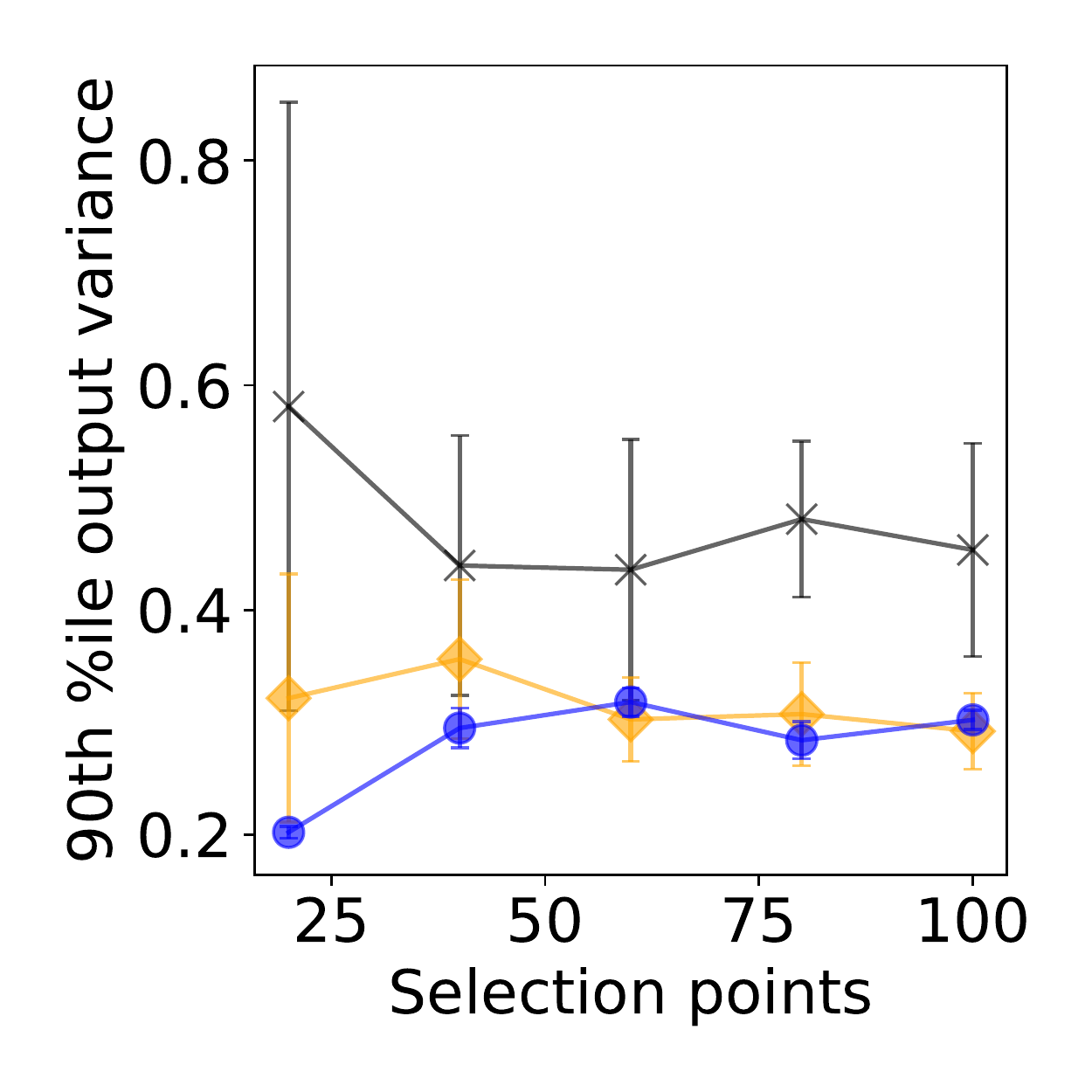}
\includegraphics[width=0.24\linewidth]{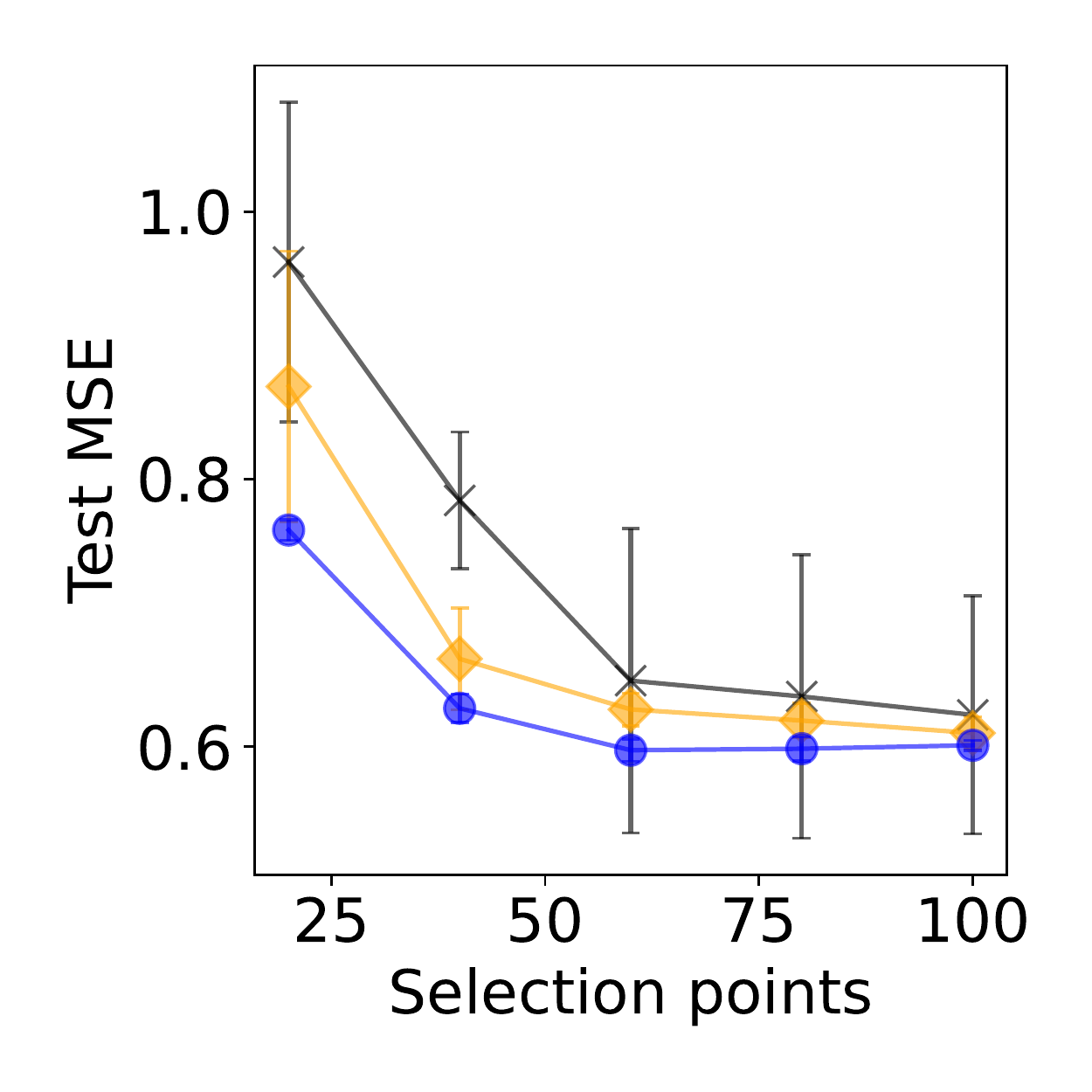} 
\end{tabular}
\includegraphics[width=0.3\linewidth]{fig/legends/legend_regr.pdf}

\caption{Active learning on regression datasets. The information shown is presented in the same manner as \cref{fig:small-set-batch} (see \cref{appx:fig-desc}).}
\label{fig:appx-small-set-batch}
\end{figure}

\subsubsection{Suitability of $\alpha_\text{EV}$ and MSE Loss}

In \cref{fig:appx-regr-bias}, we study how the ability of $\alpha_\text{EV}$ to represent the test loss is able to reflect how suitable the model is for the given dataset. We find that when $\alpha_\text{EV}$ shows a higher correlation to the test loss, the achievable test loss (when the model is trained on the entire set) is also lower. This shows that if we have a model which matches closely to the true data, then optimizing $\alpha_\text{EV}$ is more likely to result in a low test loss. This fact can be related back to \cref{thm:err}, where we have the constant $B$ which represents how well the underlying function fits with the neural network model. The lower value of $B$ means that optimizing $\alpha_\text{EV}$ can provide a tighter bound on the difference between the trained model and the true function. We also note that despite the results showing low correlations between $\alpha_\text{EV}$ and MSE loss in some of the datasets, we are still able to minimize the loss simply by maximising $\alpha_\text{EV}$. 

\begin{figure}[ht]
\centering
\includegraphics[width=0.4\linewidth]{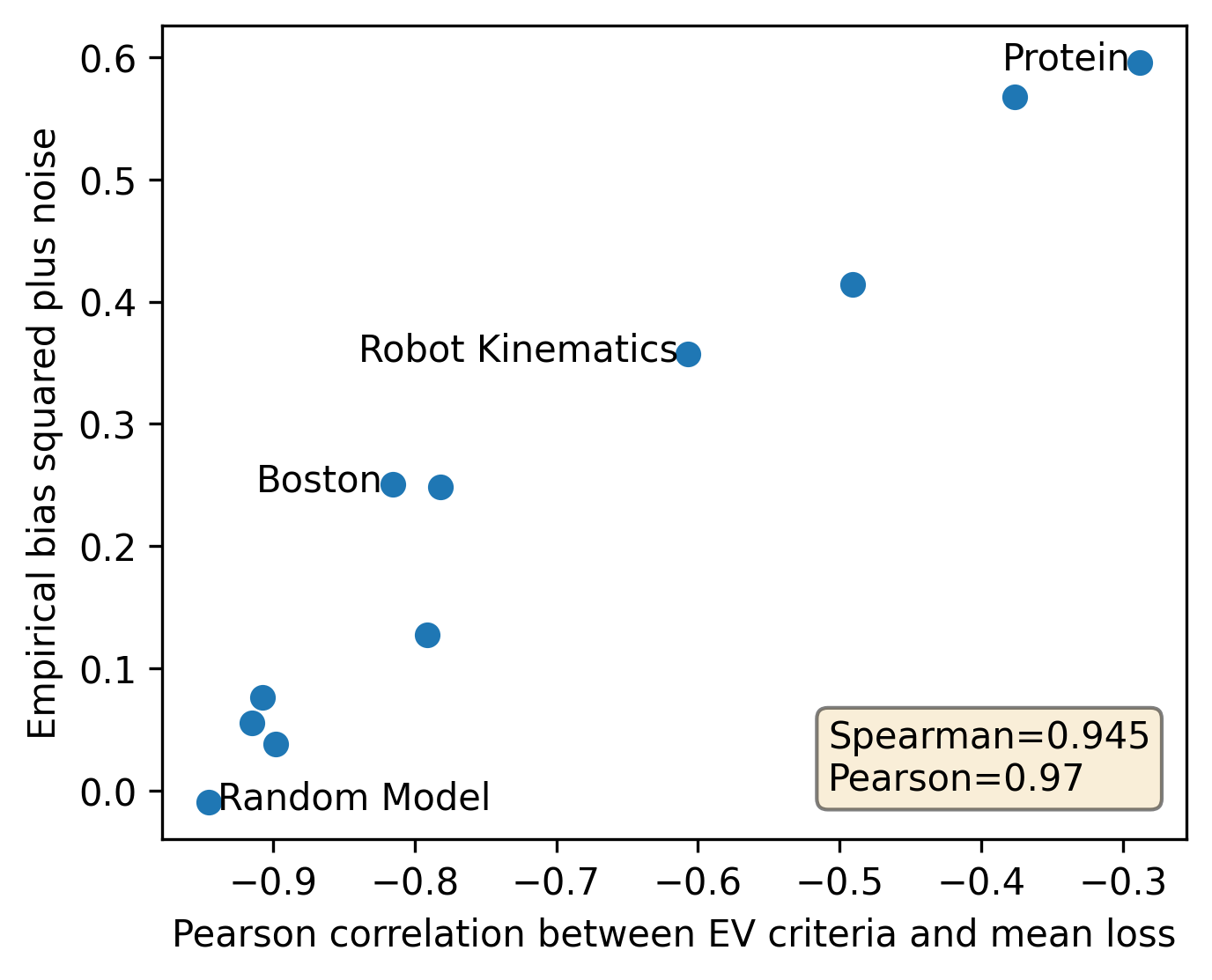}

\caption{Plot between the correlation of $\alpha_\text{EV}(\cX)$ and the MSE loss when trained on $\cX$, and the empirical bias (defined as the average MSE loss minus the output variance), on different regression datasets. Note that $\alpha_\text{EV}$ is higher when the output variance and the MSE loss is \textit{lower}, meaning the correlation values on the $x$-axis are all negative. The {\sffamily Random Model} dataset is a dataset whose outputs are generated from a randomly initialized 2-layer MLP (which the neural network is expected to fit well on). The results for active learning on the datasets labelled in the graph are shown in \cref{fig:appx-regr-bias-examples}.}
\label{fig:appx-regr-bias}
\end{figure}

\begin{figure}[ht]
\centering
\begin{tabular}{cc}
{\sffamily \small Random Model}  &
{\sffamily \small Boston}  \\ 
\includegraphics[width=0.4\linewidth]{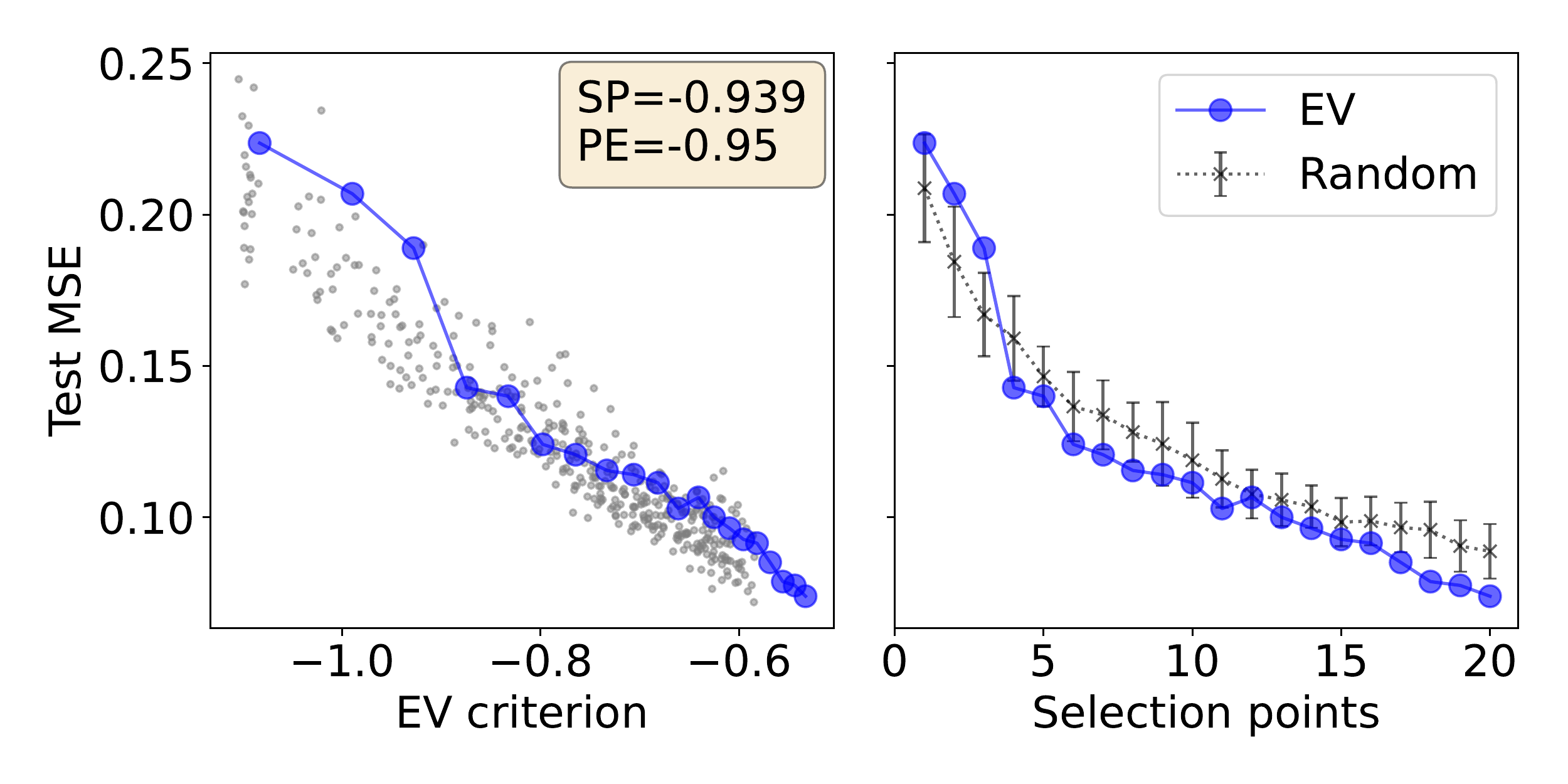} &
\includegraphics[width=0.4\linewidth]{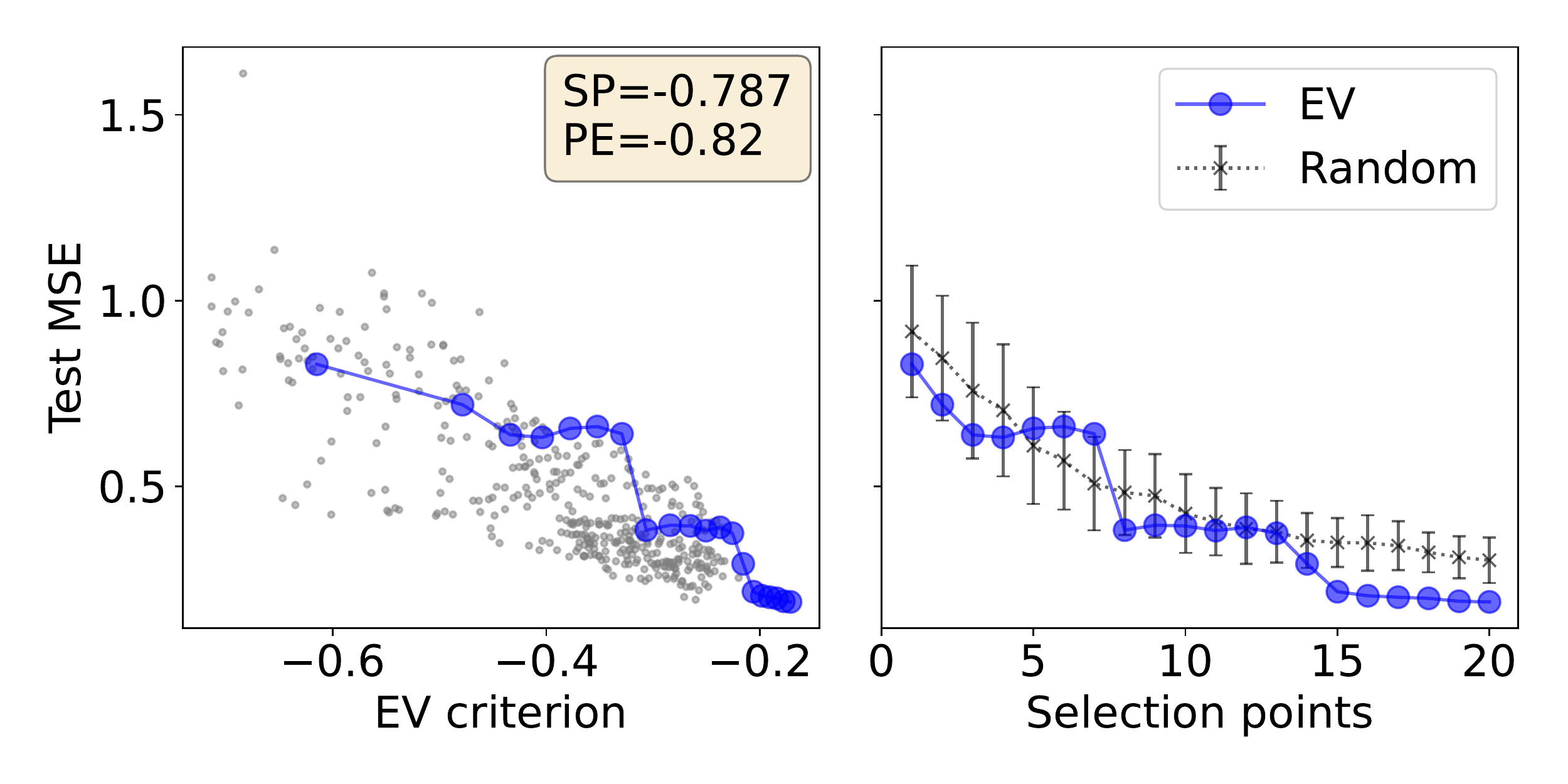} \\

{\sffamily \small Robot Kinematics}  &
{\sffamily \small Protein} \\
\includegraphics[width=0.4\linewidth]{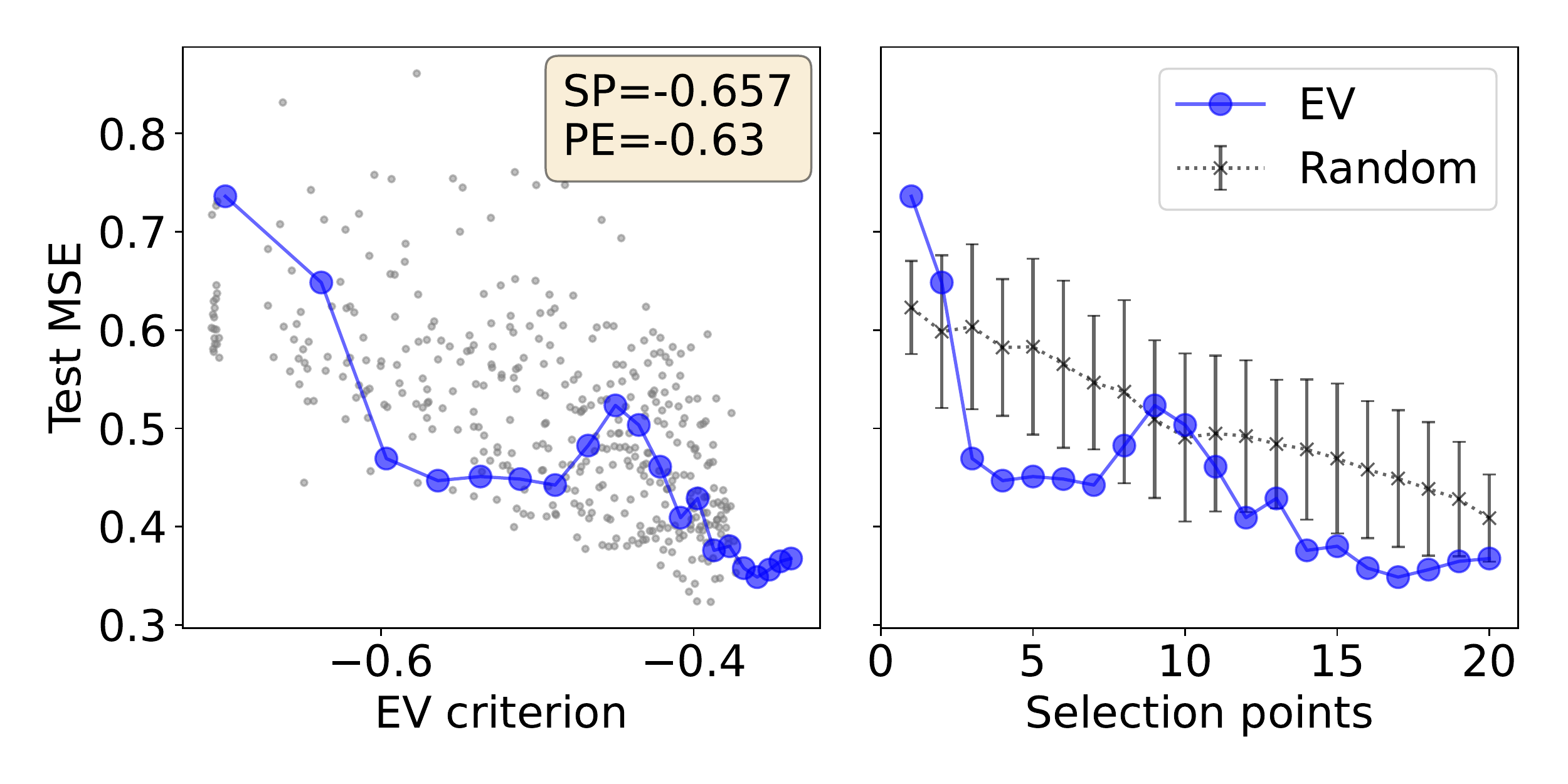} &
\includegraphics[width=0.4\linewidth]{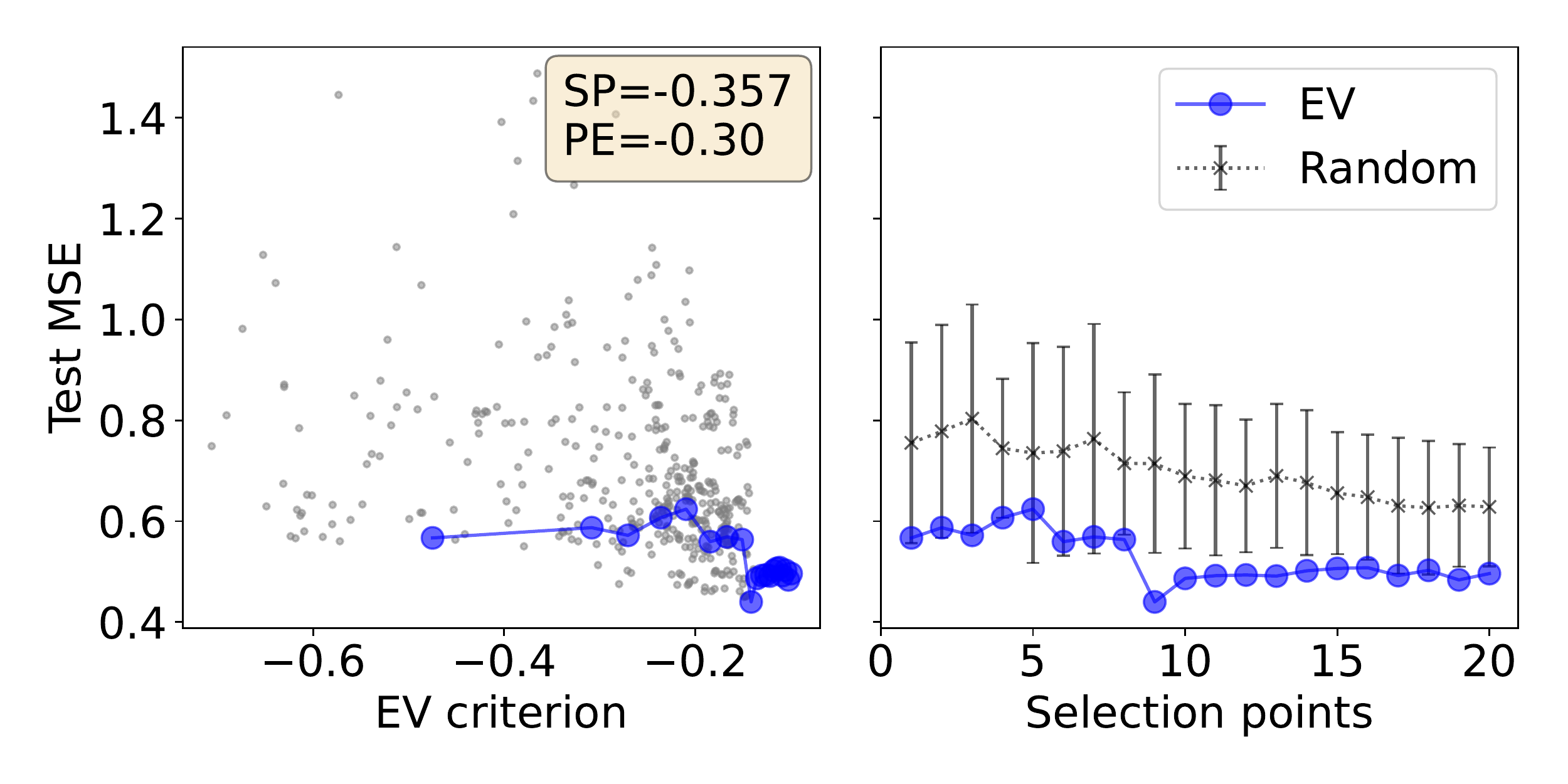} \\
\end{tabular}

\caption{Plot comparing the values of $\alpha_\text{EV}(\cX)$ and the MSE loss when trained on $\cX$ on different regression datasets. For each dataset, the left plot shows $\alpha_\text{EV}(\cX)$ and test MSE when trained on different subsets of $\cX$, and the right plot shows the test MSE when trained on subsets of different sizes. The blue line shows the subset selected by our greedy algorithm to maximize $\alpha_\text{EV}$. The datasets in the top row shows higher correlation between the criterion and the corresponding MSE loss, whereas the bottom row shows the opposite case. This correlation also reflects how well the chosen model architecture fits with the dataset.}
\label{fig:appx-regr-bias-examples}
\end{figure}

\subsubsection{Results for Other Active Learning Criteria}
\label{appx:al-other-crit-appx}

We present some results for other active learning criteria presented in \cref{appx:crit}. In \cref{fig:appx-small-set-other-crit} we show the results in the sequential setting, and in \cref{fig:appx-small-set-batch-crits} we show the results in the batch setting.

We find that $\alpha_\text{100V}$ tends to perform poorly compared to the other criteria, and sometimes even worse than \textsc{Random} in some instances (in \cref{fig:appx-small-set-batch-crits} we did not show some of the results from $\alpha_\text{100V}$ since it gave a training set which was unstable during training). We believe that this is due to the fact that $\alpha_\text{100V}$ will sometimes tend to be too focused on reducing the variance of outliers that it does not provide a diverse subset of points. Reducing this threshold to the 90th percentile (and using $\alpha_\text{90V}$) is able to give a better subset however still is not as effective as $\alpha_\text{EV}$.

Meanwhile, $\alpha_\text{MI}$ often results in good performances, and sometimes even outperforming $\alpha_\text{EV}$. However, we find it is still less preferable since its advantage over $\alpha_\text{EV}$ is not consistent and also due to the aforementioned issue of computational efficiency.

\begin{figure*}[t]
\centering
{\sffamily \small Protein}\\
\includegraphics[width=0.72\linewidth]{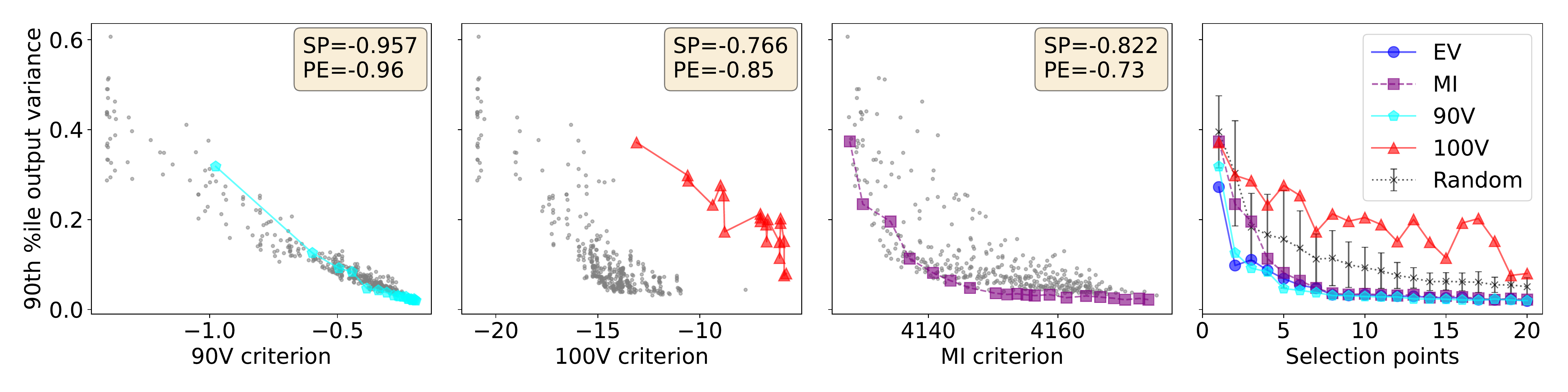}
\includegraphics[width=0.18\linewidth]{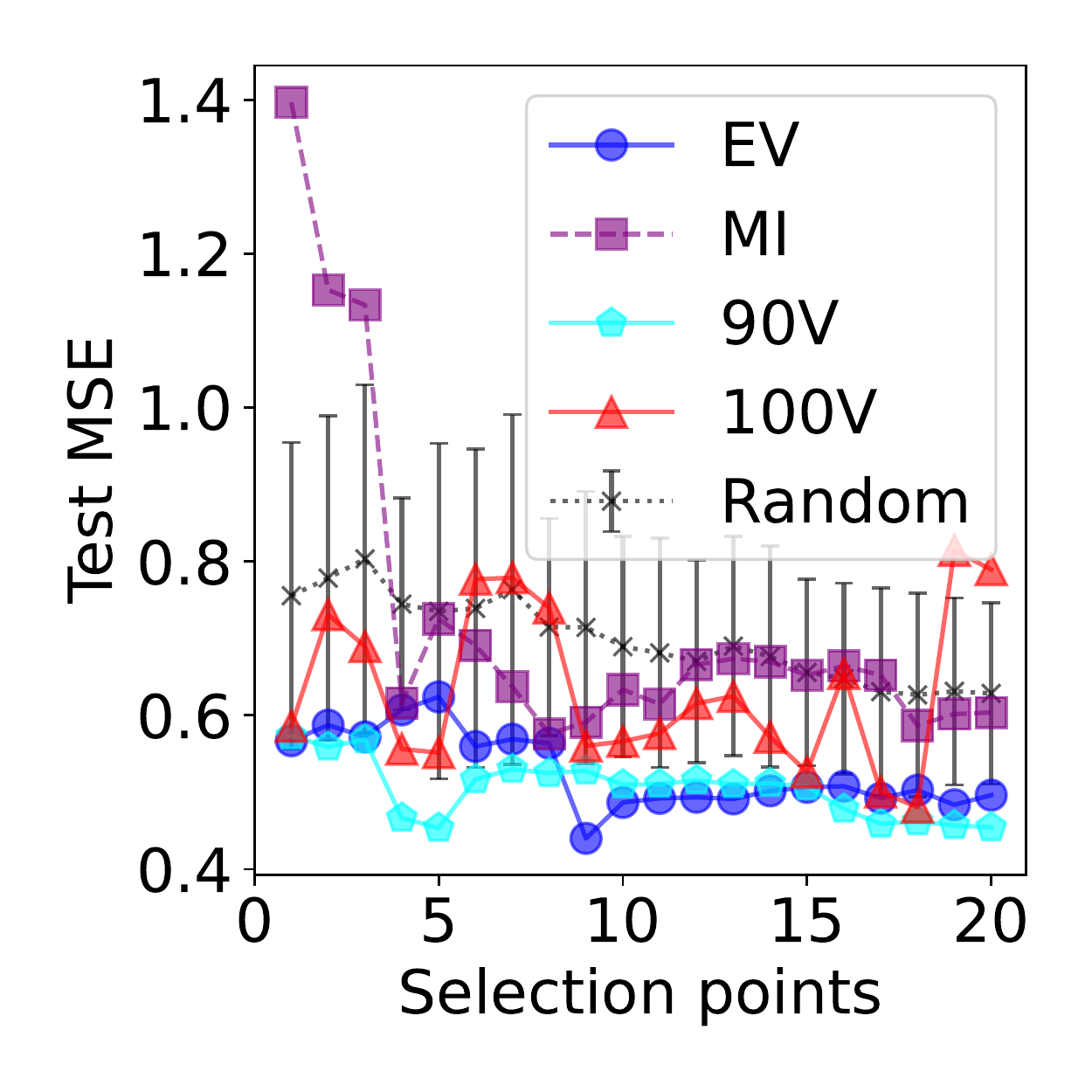}
\\
{\sffamily \small Robot Kinematics}\\
\includegraphics[width=0.72\linewidth]{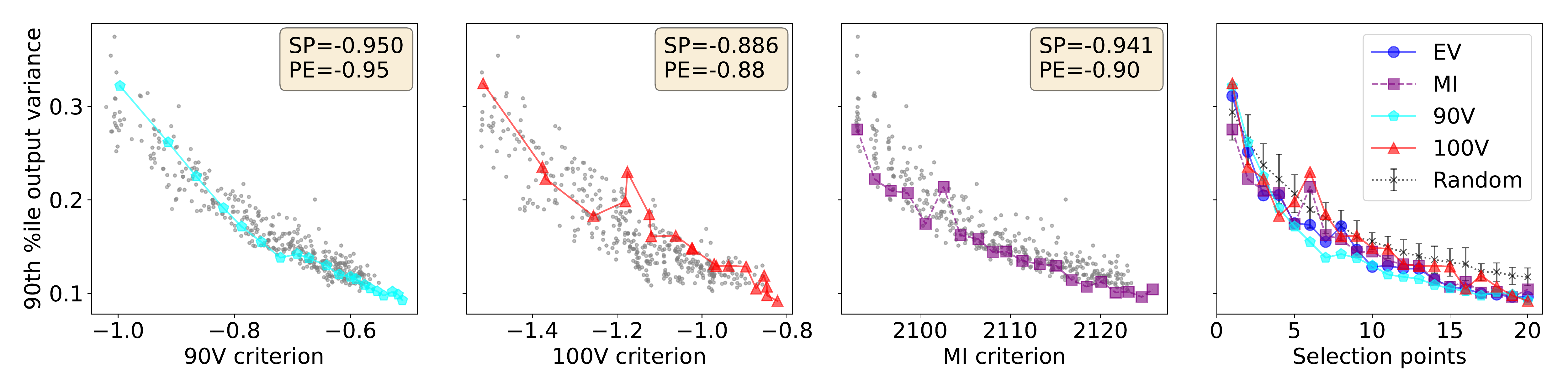}
\includegraphics[width=0.18\linewidth]{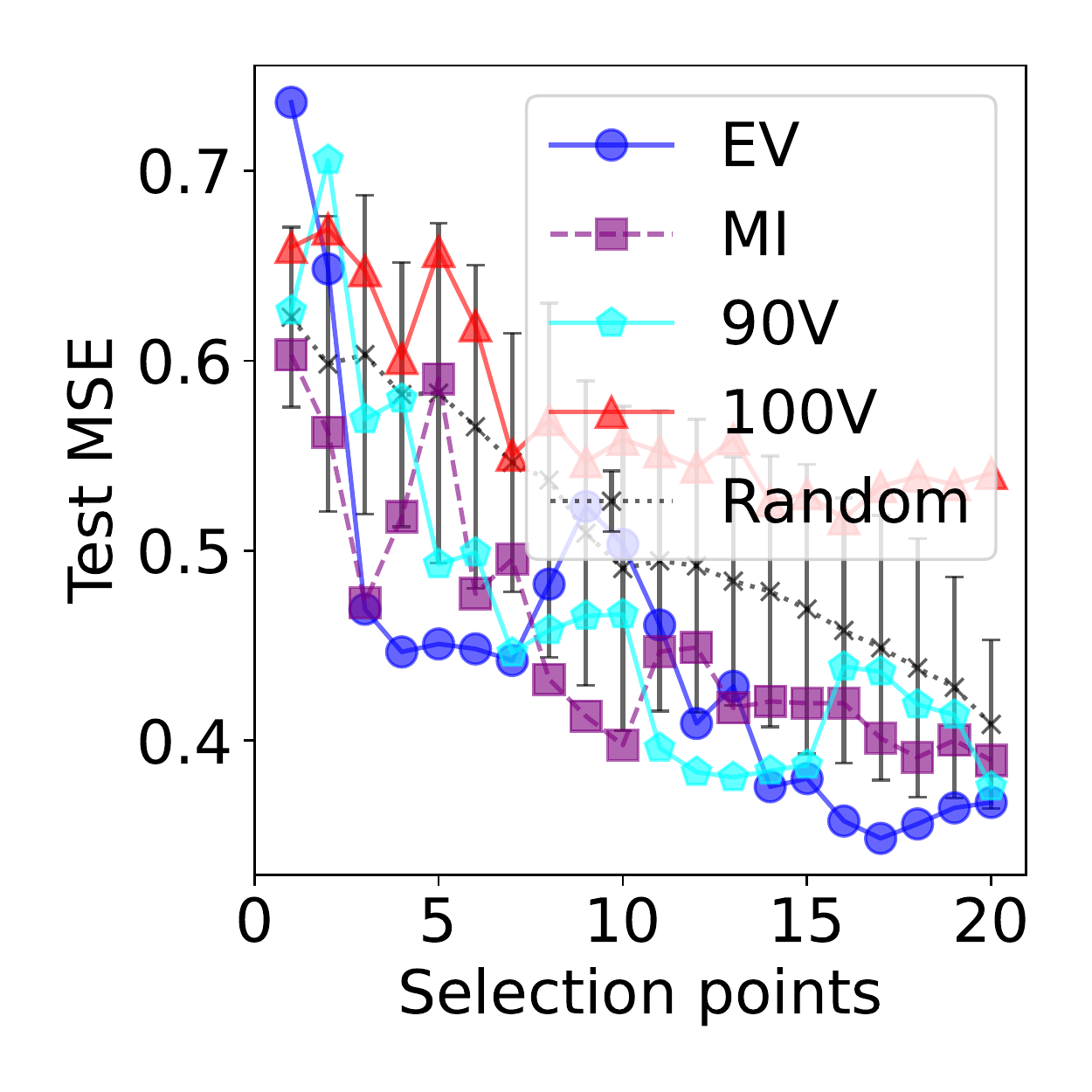}
\\
{\sffamily \small Boston}\\
\includegraphics[width=0.72\linewidth]{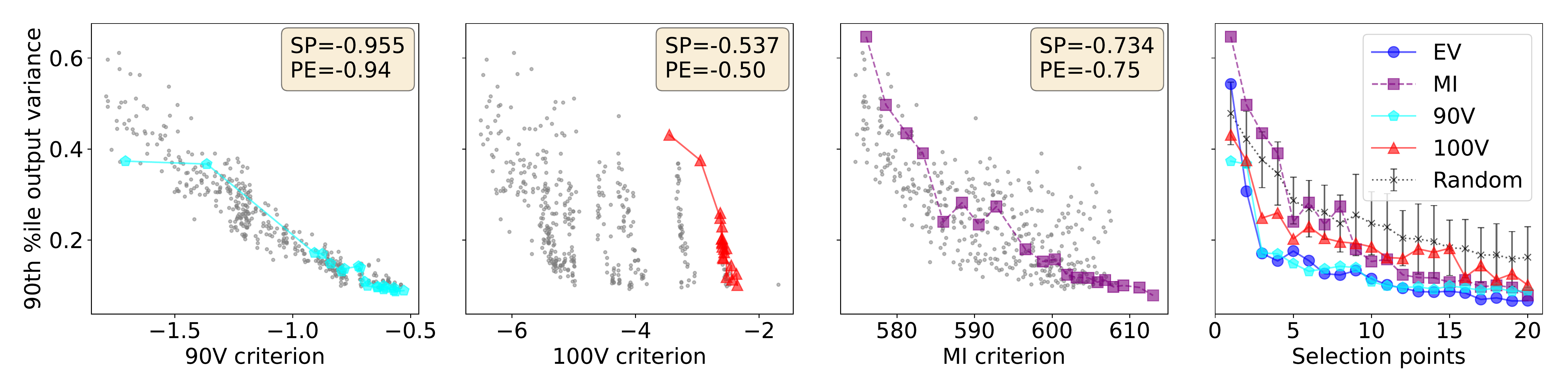}
\includegraphics[width=0.18\linewidth]{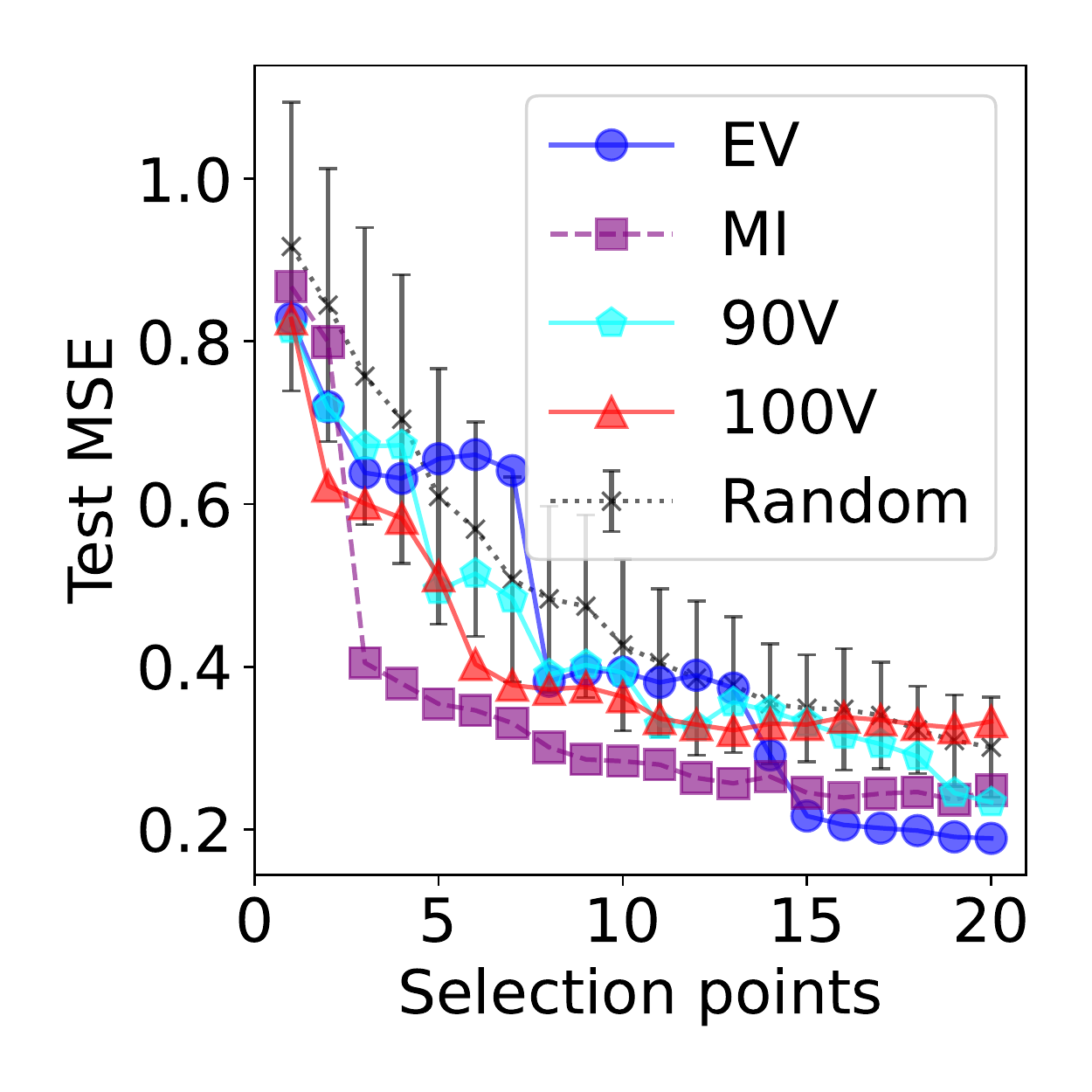}
\\
{\sffamily \small Naval}\\
\includegraphics[width=0.72\linewidth]{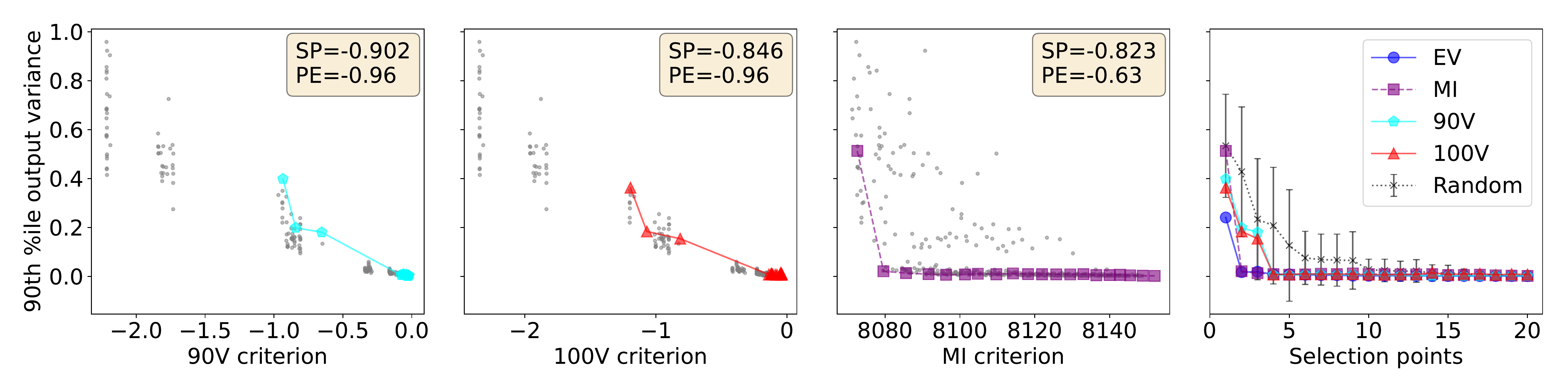}
\includegraphics[width=0.18\linewidth]{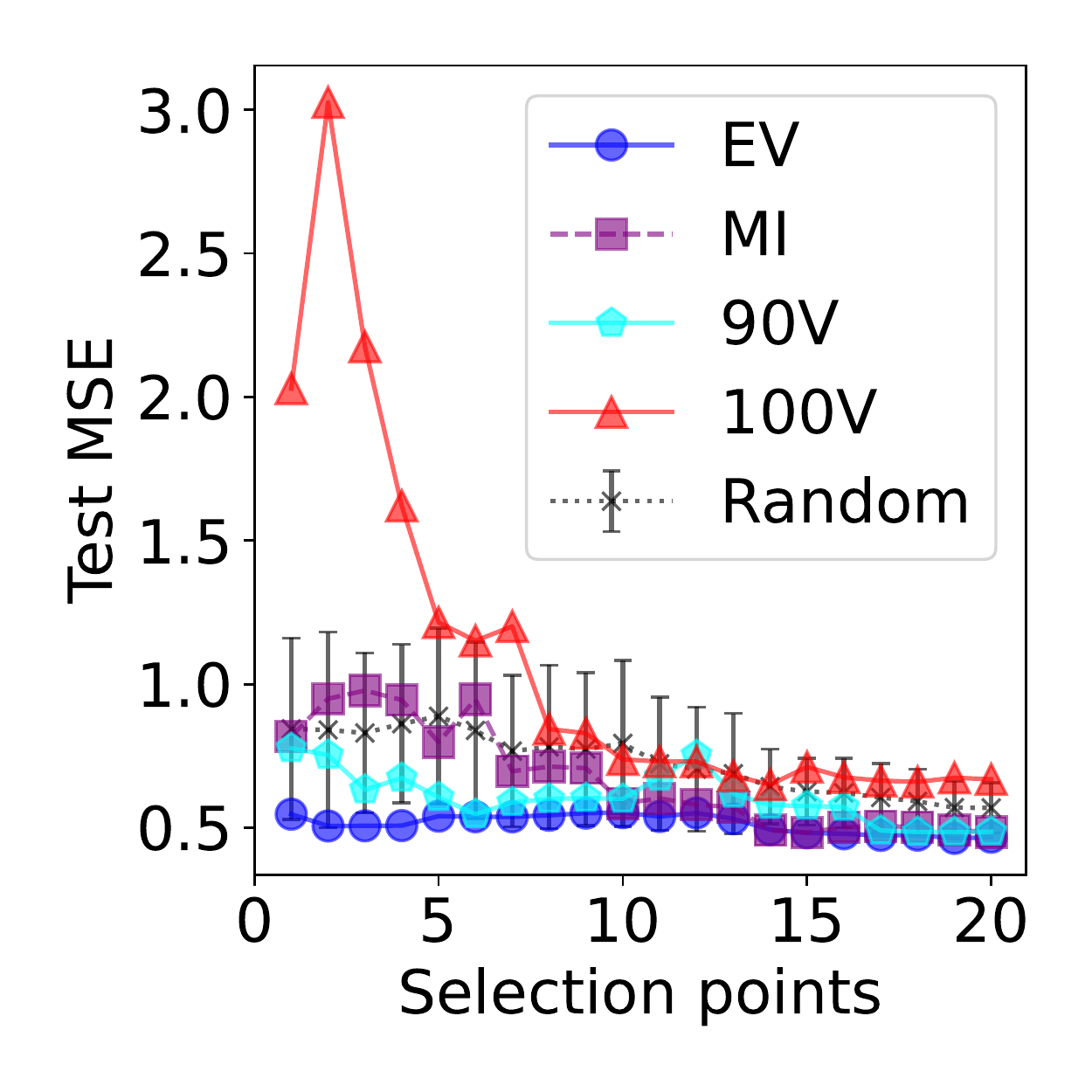}
\\
{\sffamily \small Handwritten Digits}\\
\includegraphics[width=0.72\linewidth]{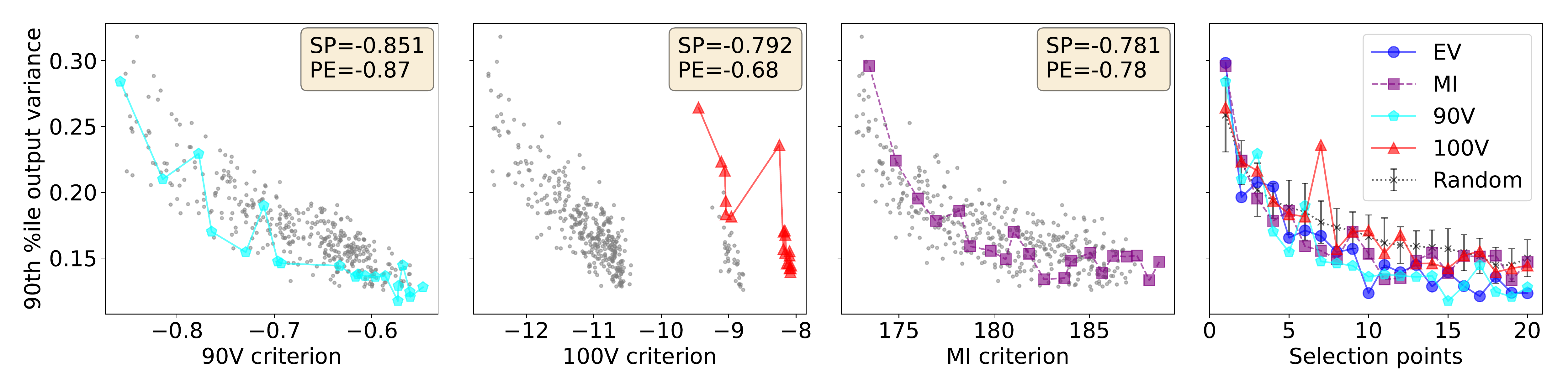}
\includegraphics[width=0.18\linewidth]{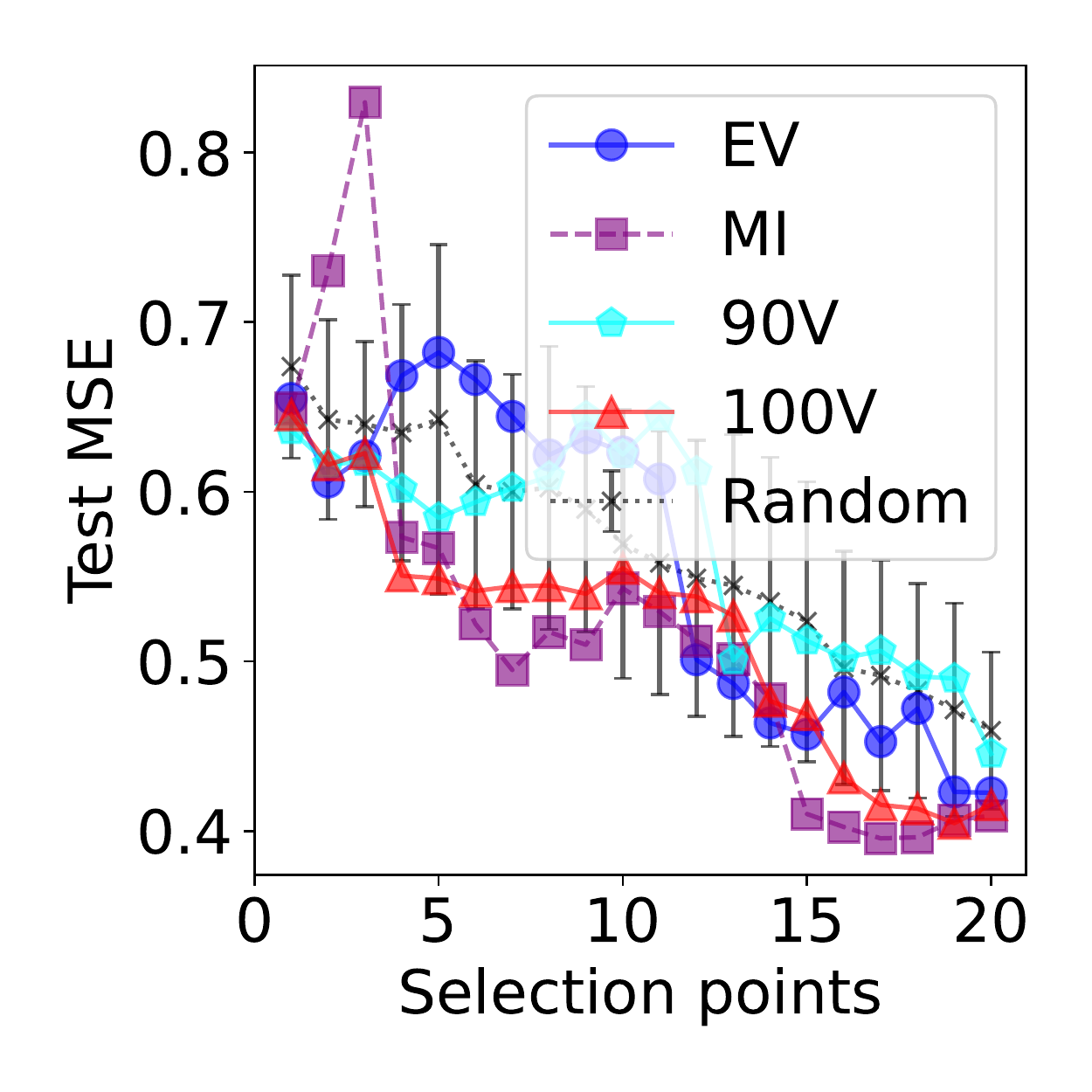}
\\
{\sffamily \small Mock Data}\\
\includegraphics[width=0.72\linewidth]{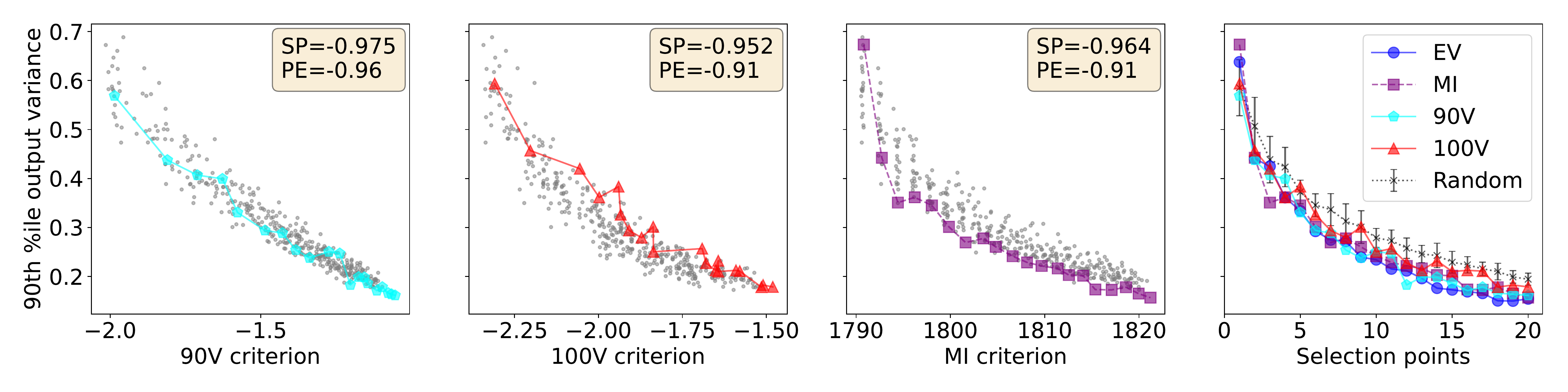}
\includegraphics[width=0.18\linewidth]{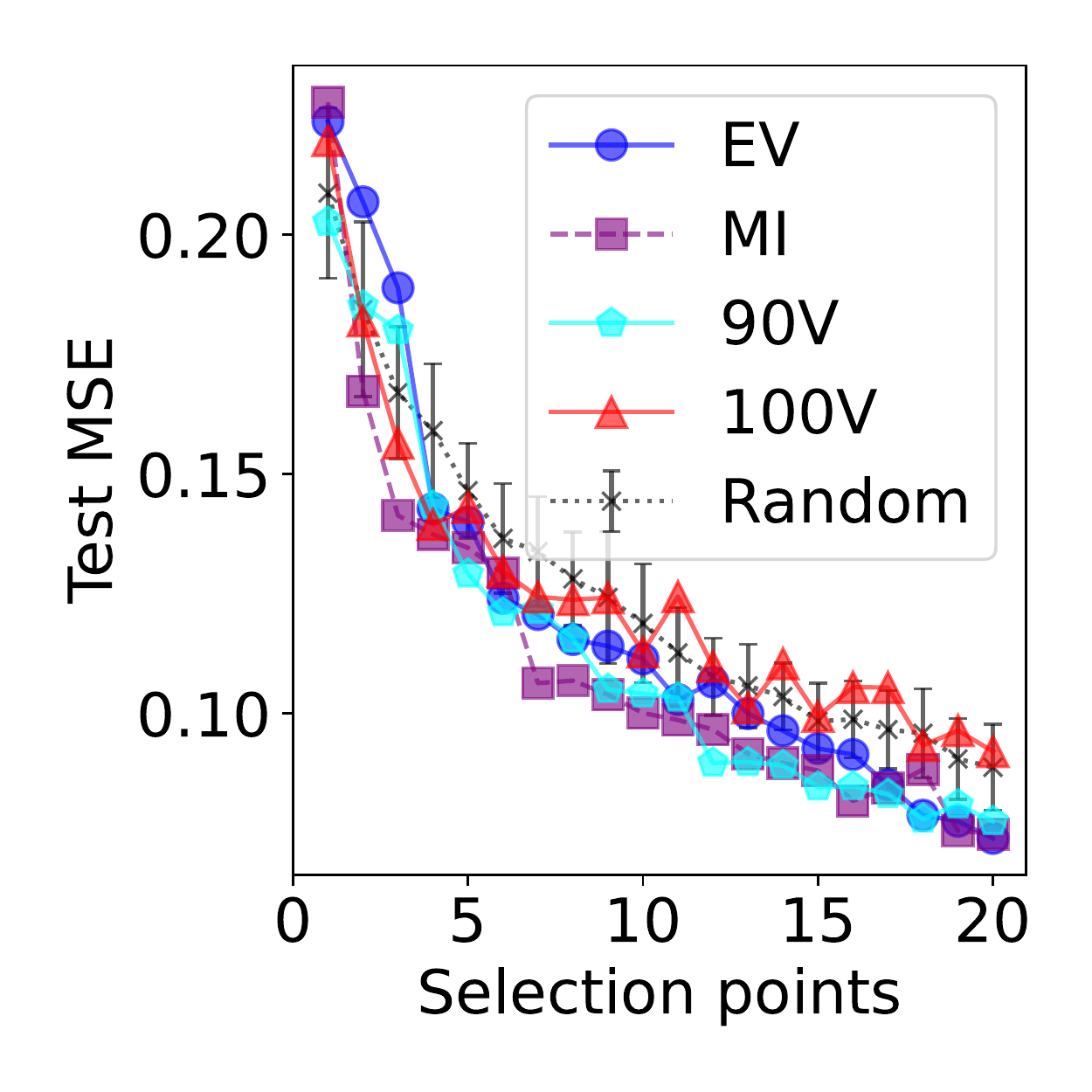}
\\
\caption{Relationship between the different criteria value, the 90th percentile predictive variance, and the test MSE. Each grey point represents one instance of a randomly-selected active set. The coloured line represents the subset greedily selected to optimize the corresponding criteria. Each row represents trials conducted on one dataset. SP and PR represents the Spearman Rank correlation and Pearson correlation coefficients respectively. The metrics reported on the $y$-axis are described further in \cref{appx:exp-metrics}.}
\label{fig:appx-small-set-other-crit}
\end{figure*}

\begin{figure}[ht]

\centering
\begin{tabular}{c|c}
{\sffamily \small Boston} &
{\sffamily \small Naval}  \\ 
\includegraphics[width=0.2\linewidth]{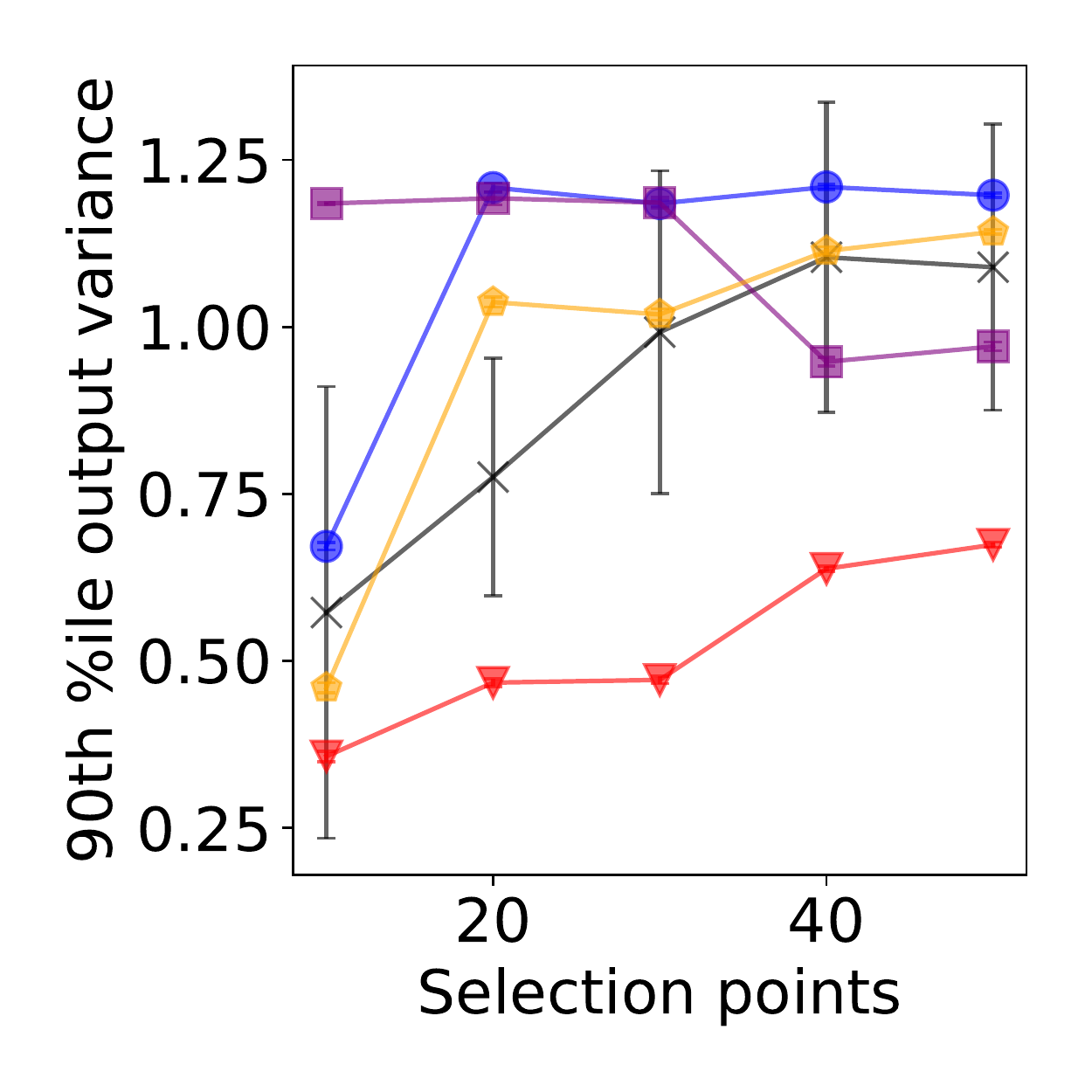}
\includegraphics[width=0.2\linewidth]{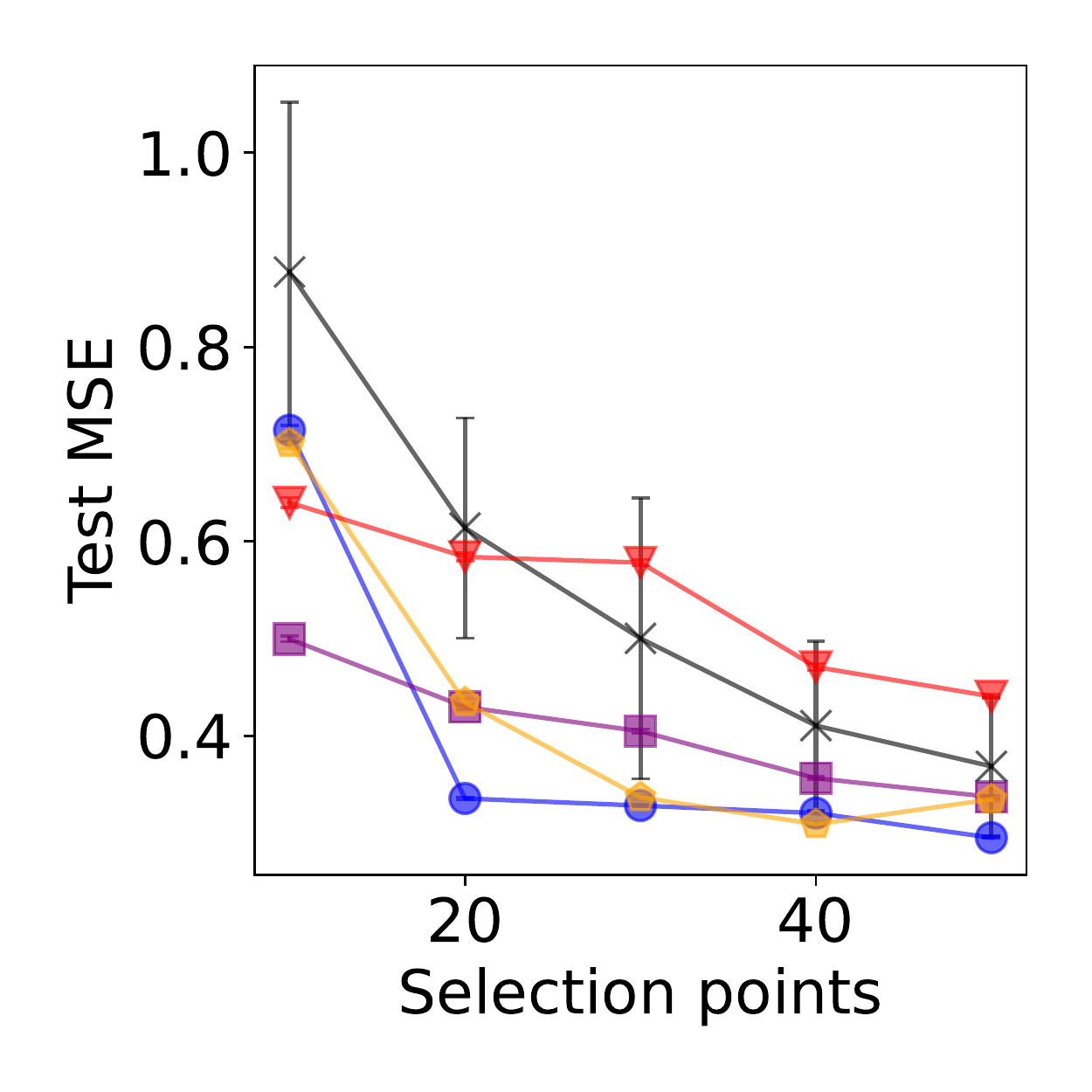} 
&
\includegraphics[width=0.2\linewidth]{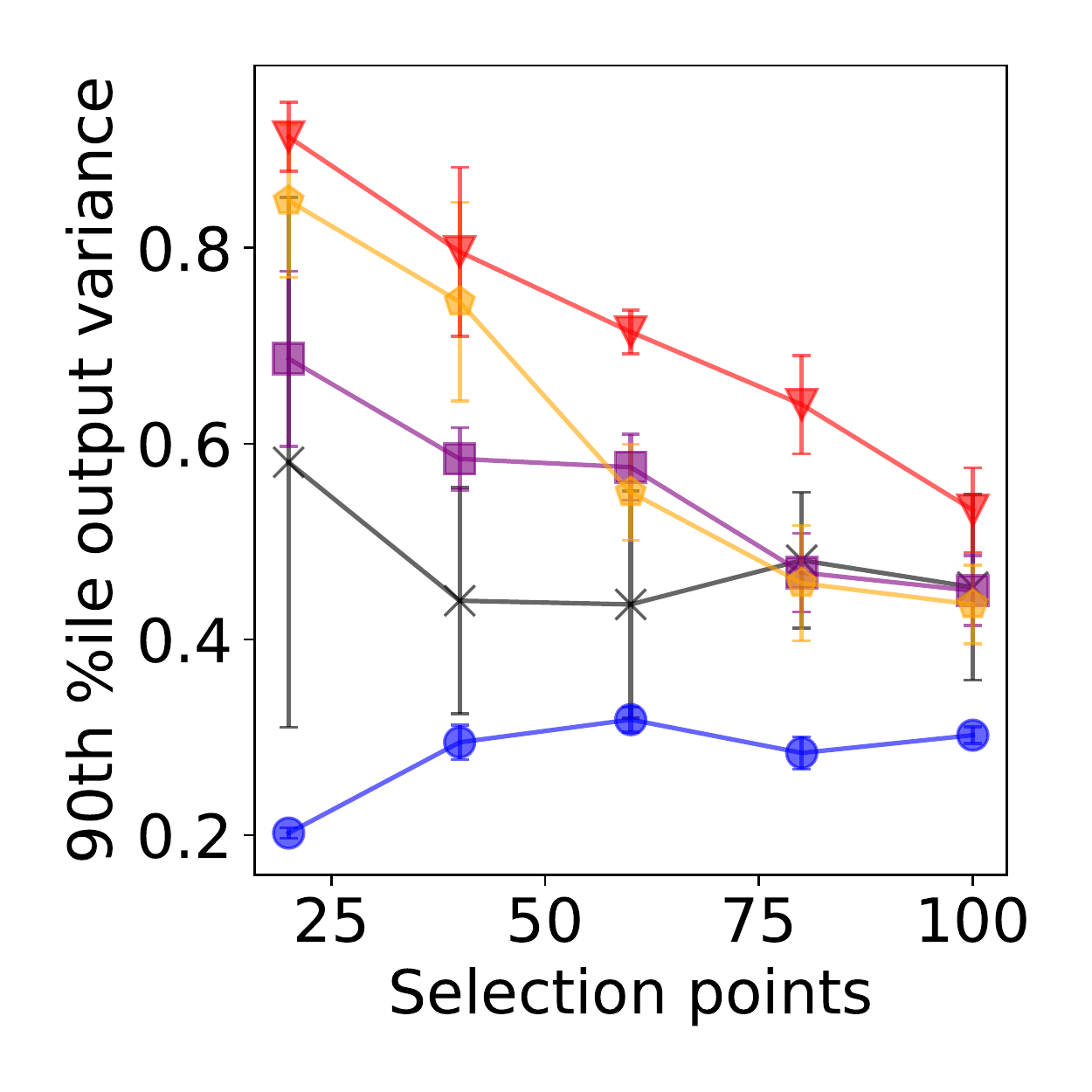}
\includegraphics[width=0.2\linewidth]{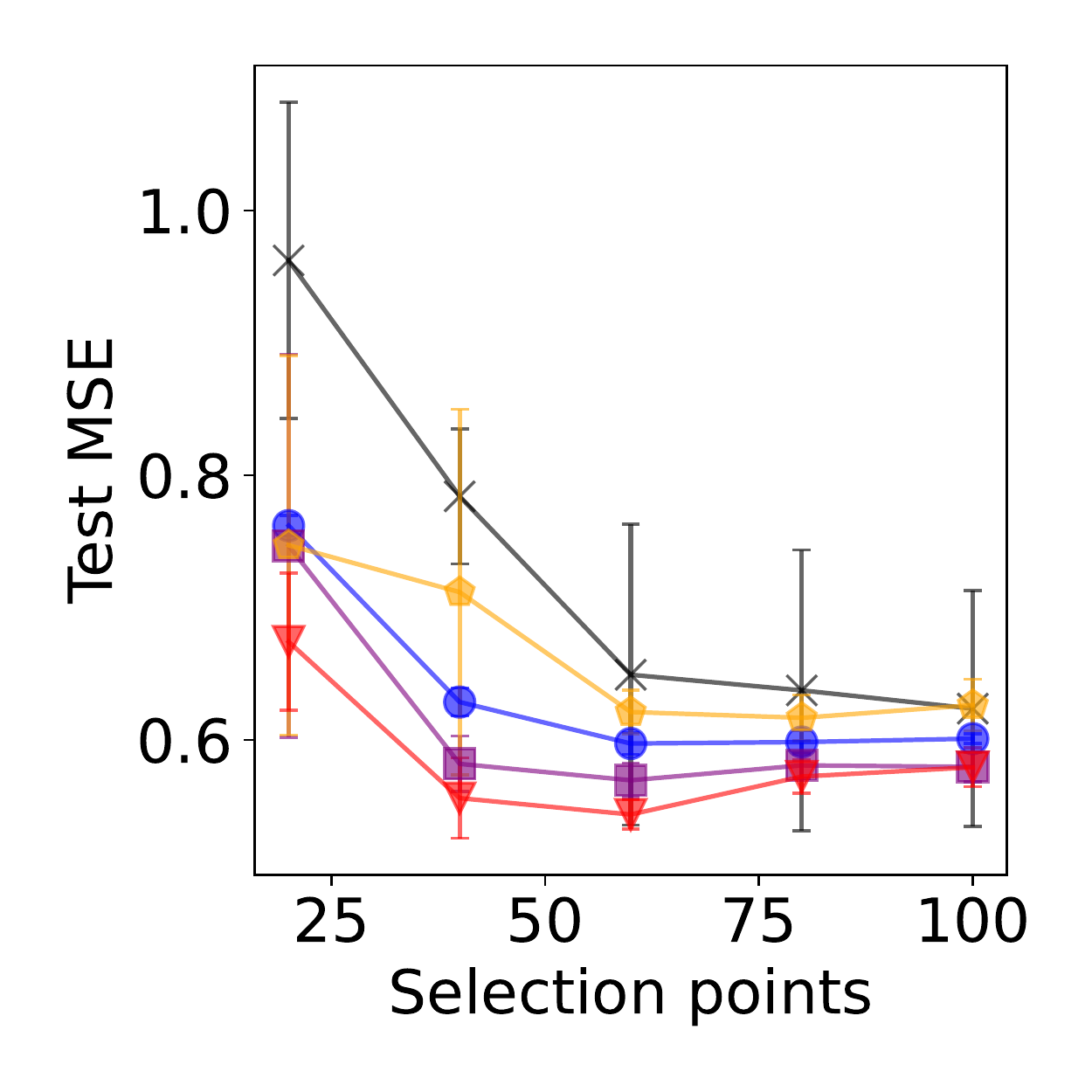} \\
{\sffamily \small Robot Kinematics} &
{\sffamily \small Protein}  \\ 
\includegraphics[width=0.2\linewidth]{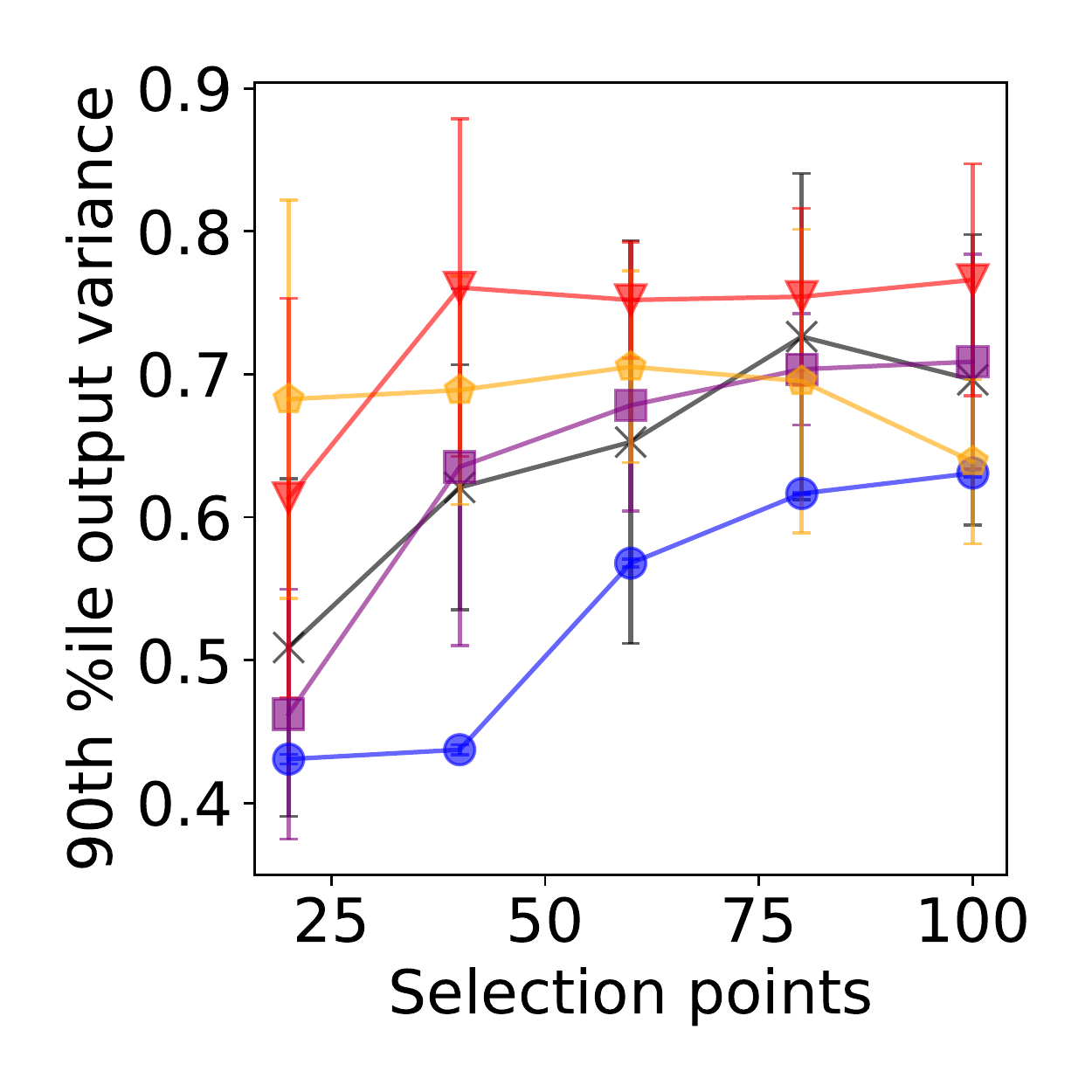}
\includegraphics[width=0.2\linewidth]{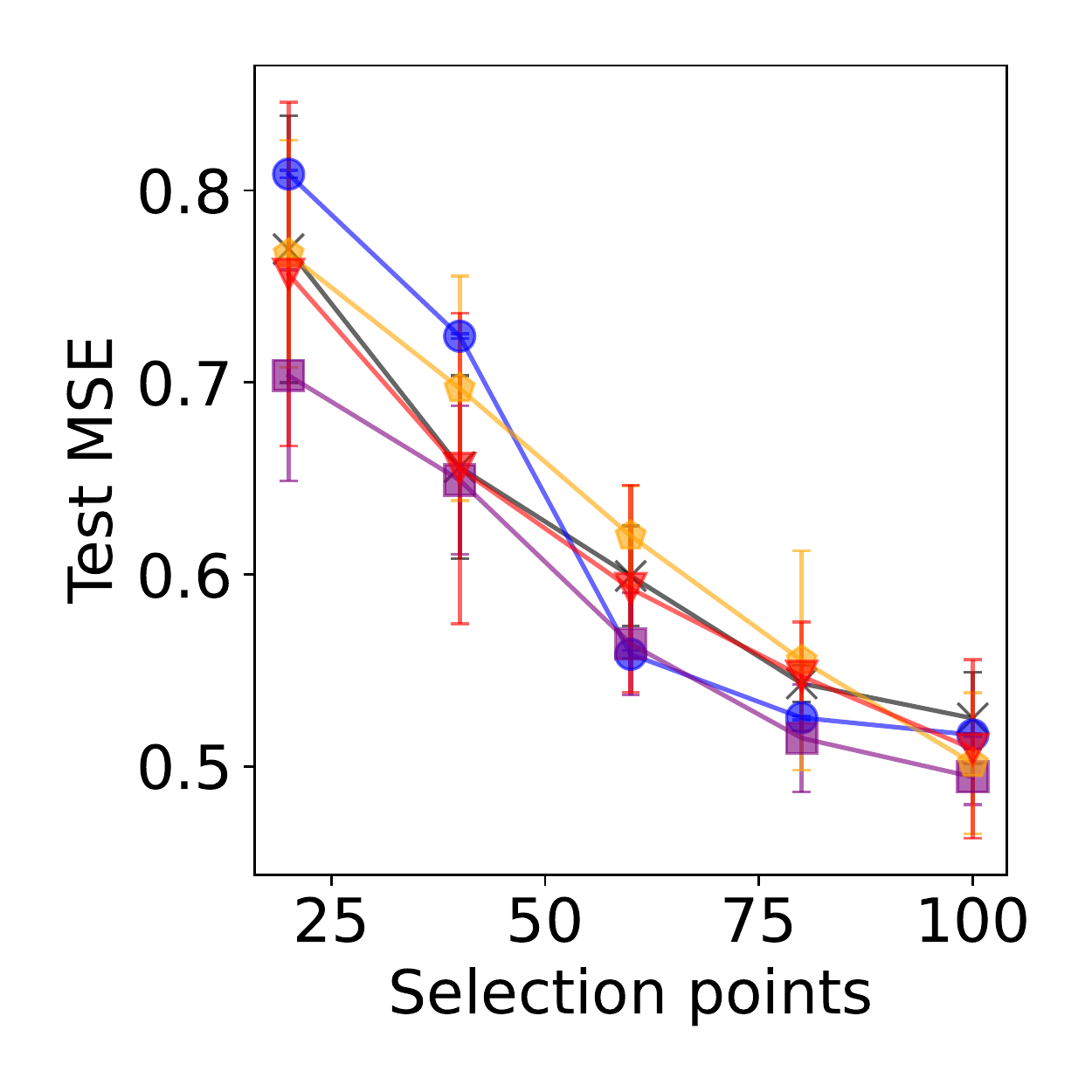} 
&
\includegraphics[width=0.2\linewidth]{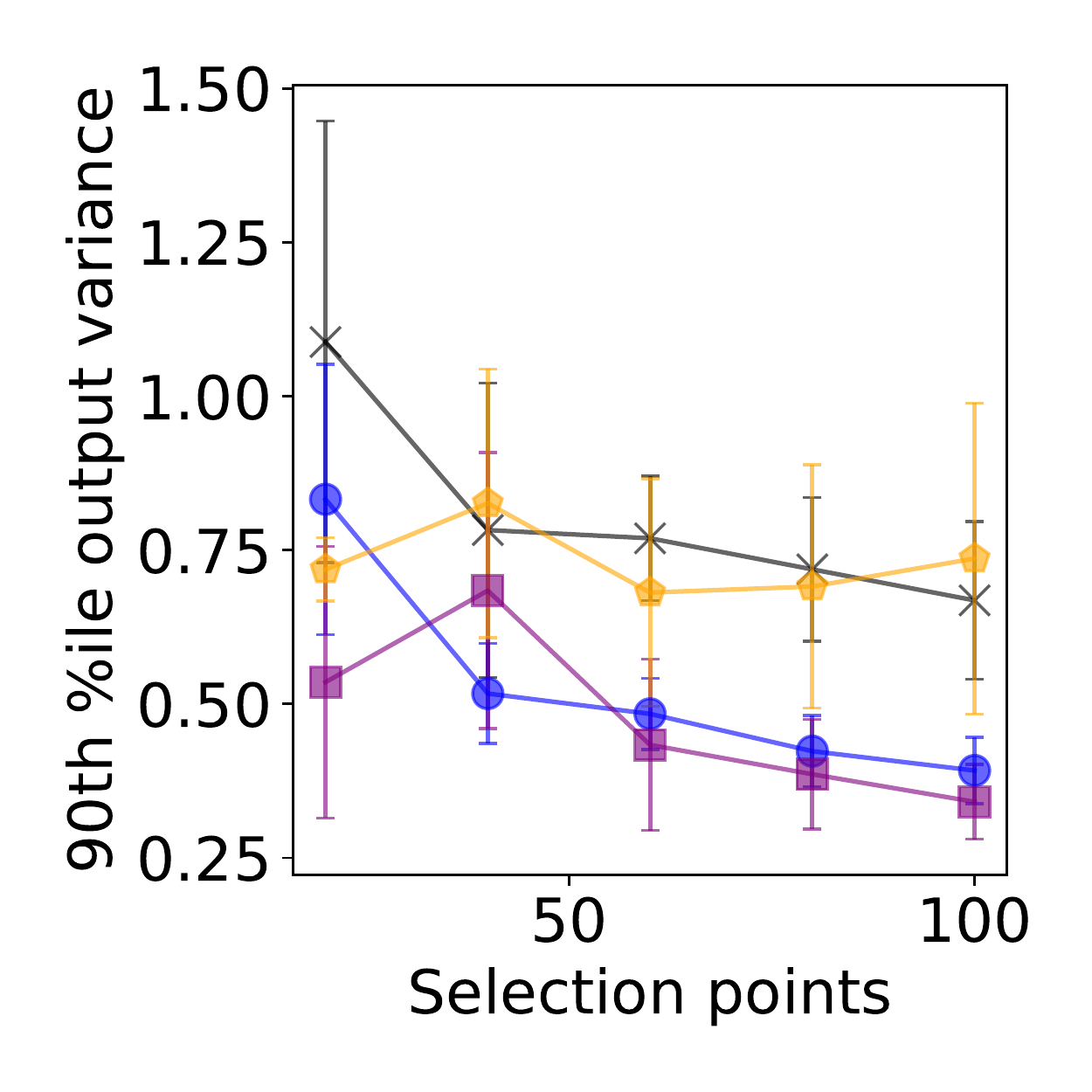}
\includegraphics[width=0.2\linewidth]{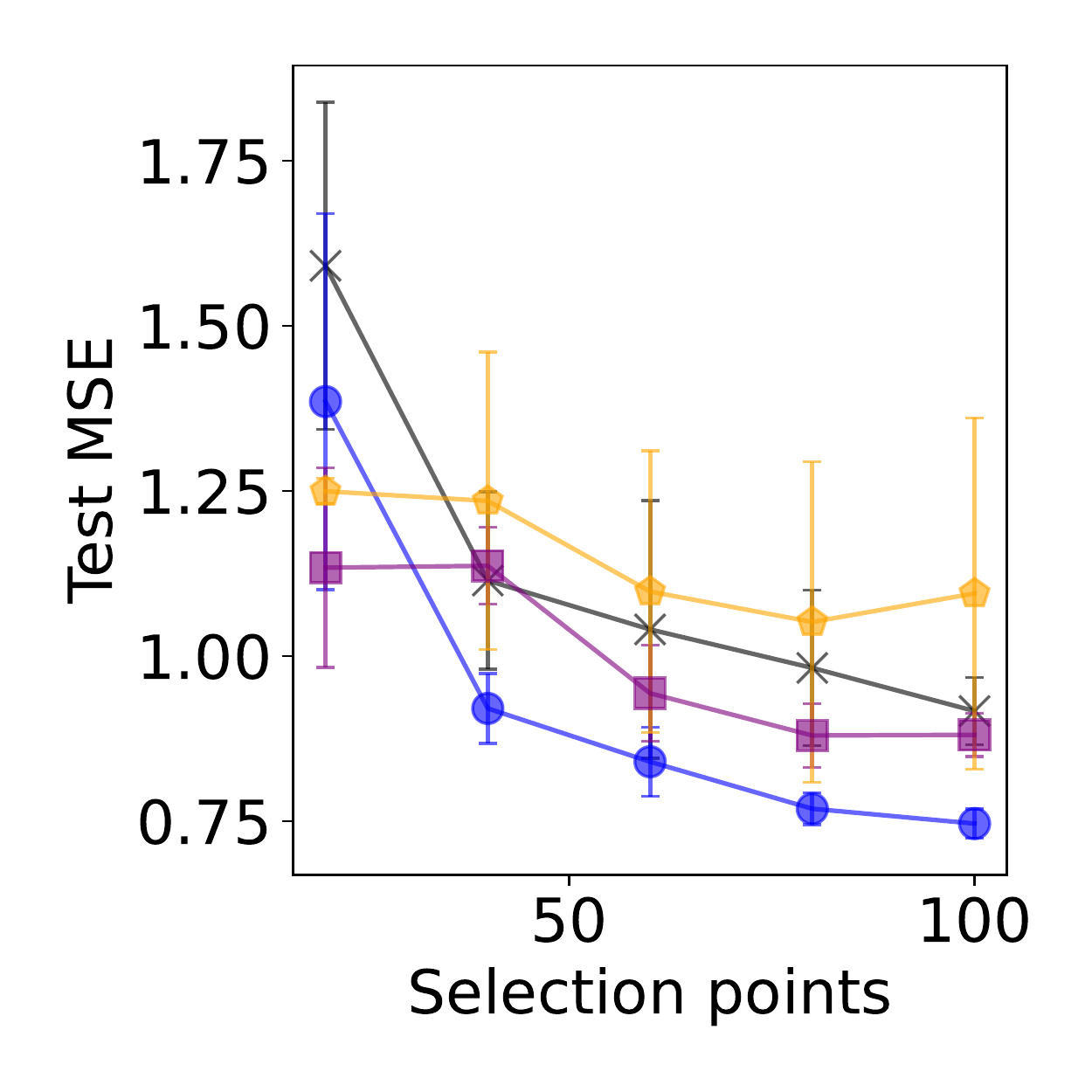} 
\end{tabular}

\includegraphics[width=0.3\linewidth]{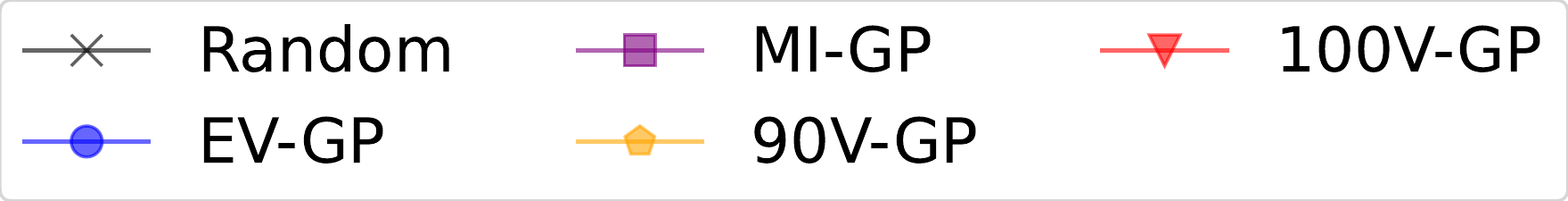}

\caption{Active learning on regression datasets with different criteria based on $\sigma^2_\text{NTKGP}$. The information shown is presented in the same manner as \cref{fig:small-set-batch} (see \cref{appx:fig-desc}).}
\label{fig:appx-small-set-batch-crits}
\end{figure}

\subsubsection{Results for Approximation Methods of NTKGP}

In \cref{fig:appx-regr-sntkgp-incremental,fig:appx-small-set-batch-sntkgp}, we present the results for active learning results for various approximation methods of the NTKGP, namely approximating the NTK empirically, and approximating the covariance using sparse GP techniques (as introduced in \cref{appx:sgp}).

We can see that the empirical approximation of the NTK sees some drop in performance compared to the theoretical NTK, however still provides useful enough approximation for the true NTK for our purposes.

Meanwhile, in the experiments where the covariance from the sNTKGP is used, we can see that the method also still provides a useful approximation for the criterion, and may be used instead of the true NTKGP, at least for smaller problems. Unfortunately, we find that applying a further linear approximation to the sNTKGP computation worsens the selection process. This suggests that the covariance function from the sNTKGP may not be smooth enough and so assuming a linear approximation method is inaccurate.

In practice, we also find that for smaller datasets (where the number of selected points is not smaller than the number of inducing points), the sNTKGP approximation is not as beneficial since the computation required for the full-rank NTKGP is already low and sNTKGP requires more matrix computation that would otherwise. Nonetheless, the sNTKGP may be a useful approximation method if can be scaled, and can be an interesting future direction of making our method more scalable.

\begin{figure}[ht]
\centering
\resizebox{\linewidth}{!}{
\begin{tabular}{cccc}
{\sffamily \small Boston} &
{\sffamily \small Naval}  & 
{\sffamily \small Robot Kinematics}  & 
{\sffamily \small Protein} 
\\ 
\includegraphics[width=0.23\linewidth]{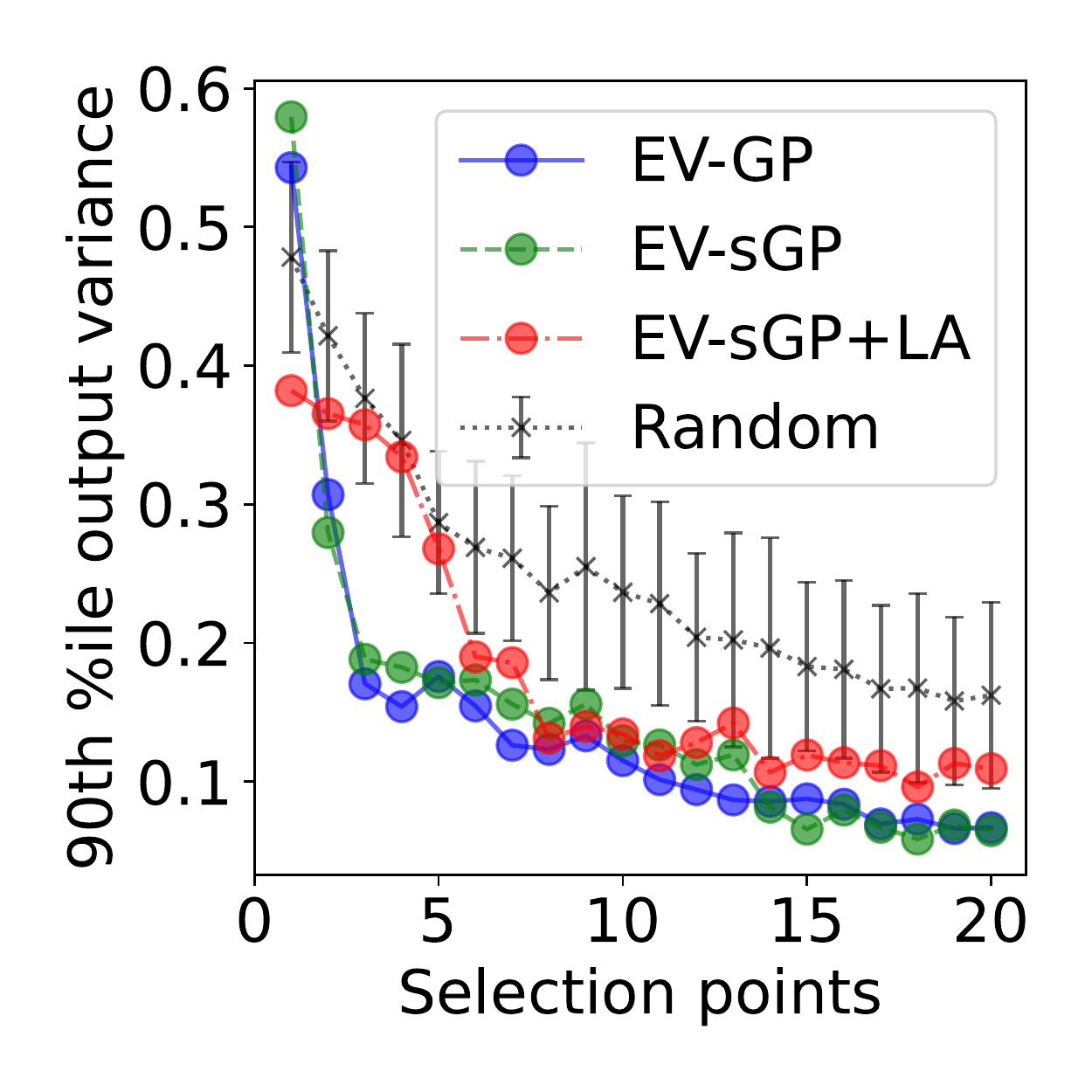} &
\includegraphics[width=0.23\linewidth]{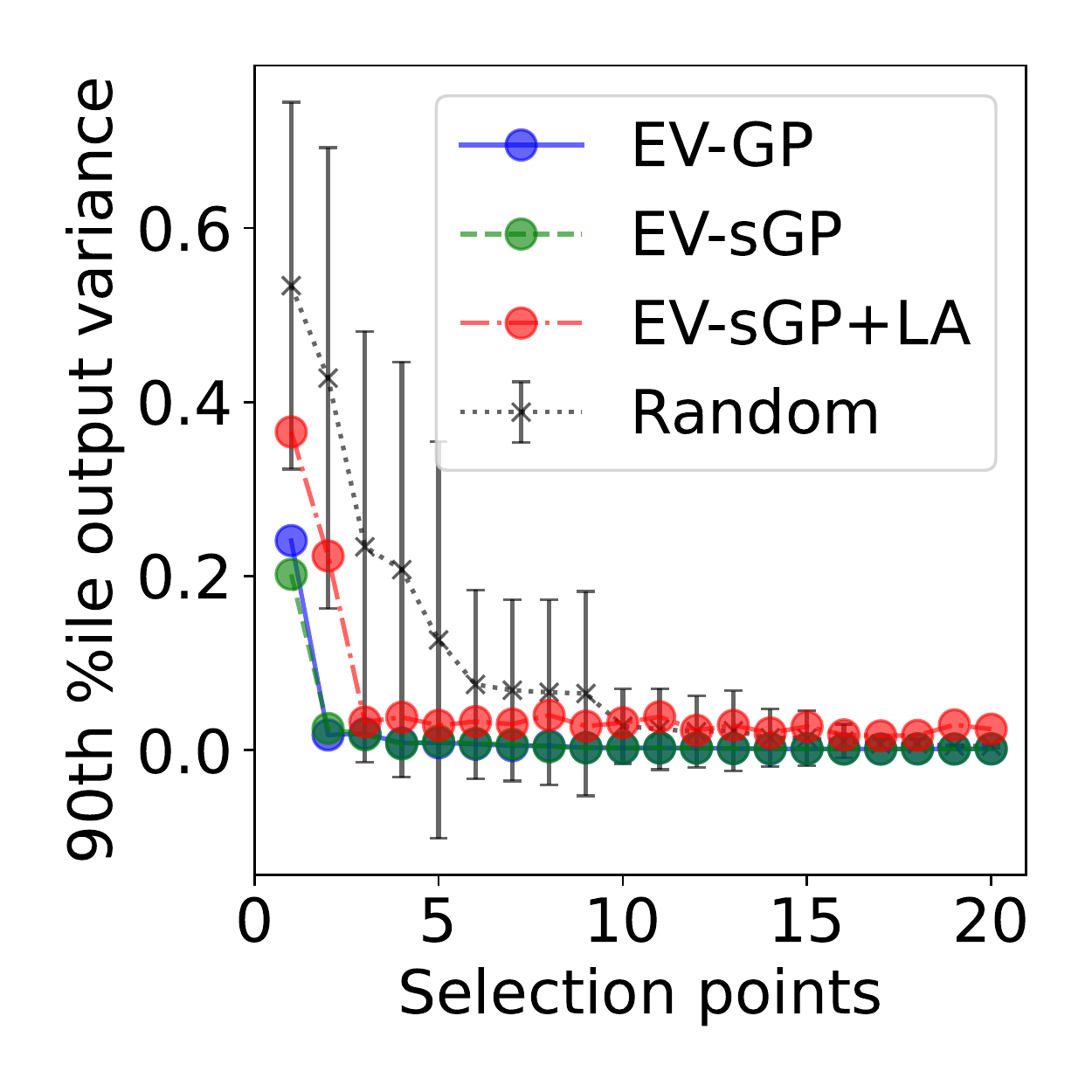} &
\includegraphics[width=0.23\linewidth]{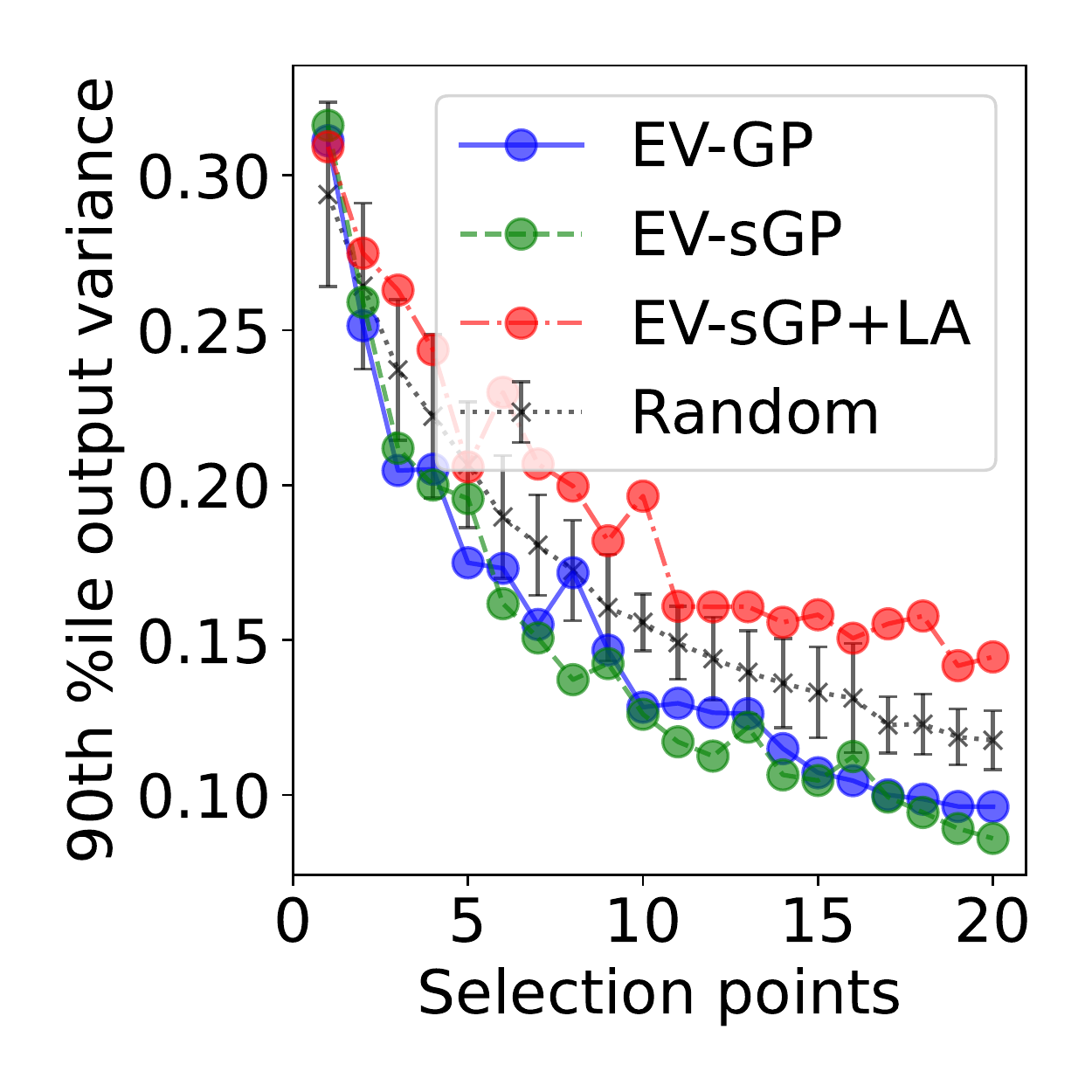} &
\includegraphics[width=0.23\linewidth]{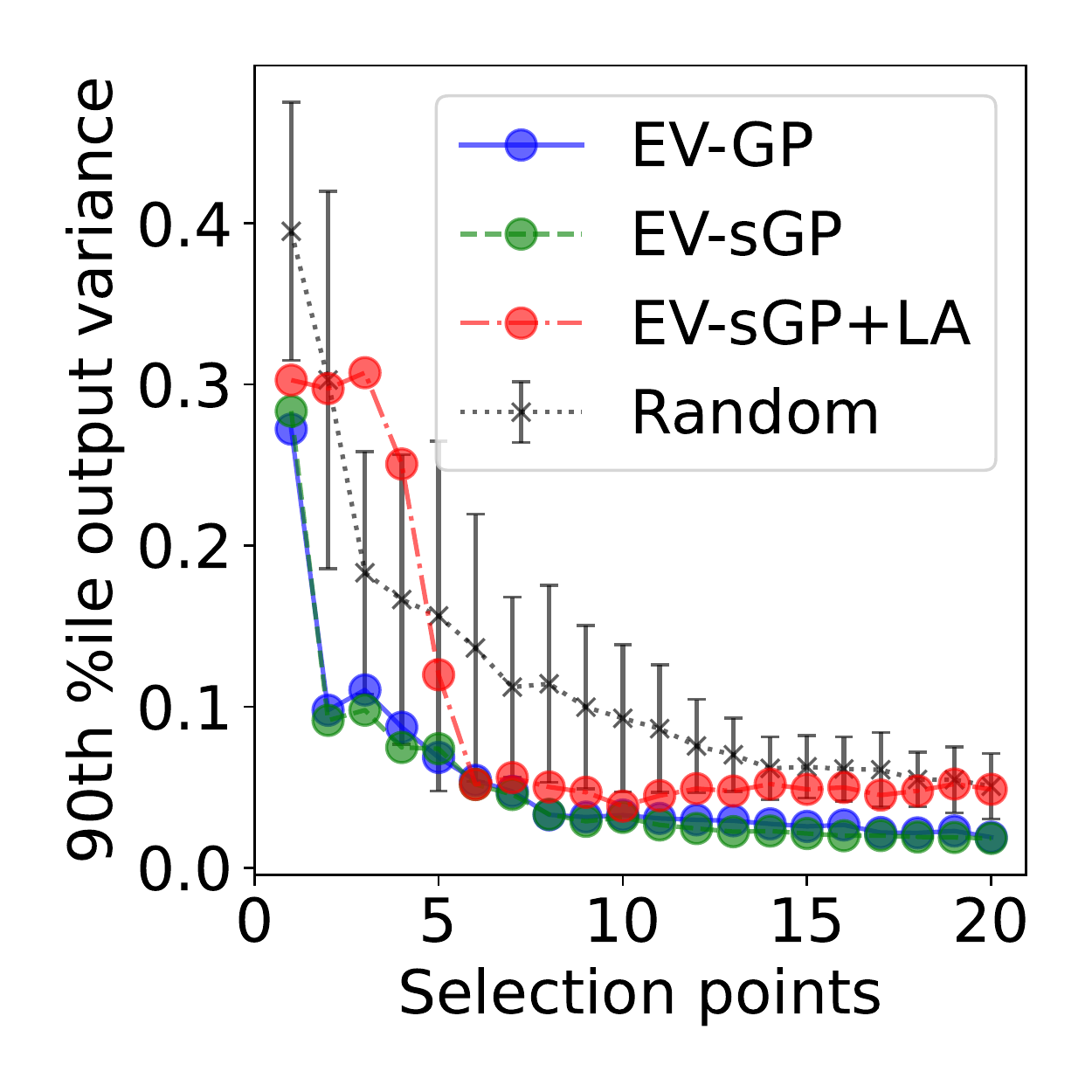} 
\end{tabular}
}
\caption{Active learning on regression problems with different approximations of $\sigma_\text{NTKGP}$. The graph shows the output variance of the test dataset when trained on the selected active set of various sizes, according to how it was approximated. The information presented is similar to that shown in the graphs on the right column of \cref{fig:small-set} (see \cref{appx:fig-desc}).}
\label{fig:appx-regr-sntkgp-incremental}
\end{figure}

\begin{figure}[ht]

\centering
\begin{tabular}{c|c}
{\sffamily \small Naval} &
{\sffamily \small Robot Kinematics}  \\ 
\includegraphics[width=0.2\linewidth]{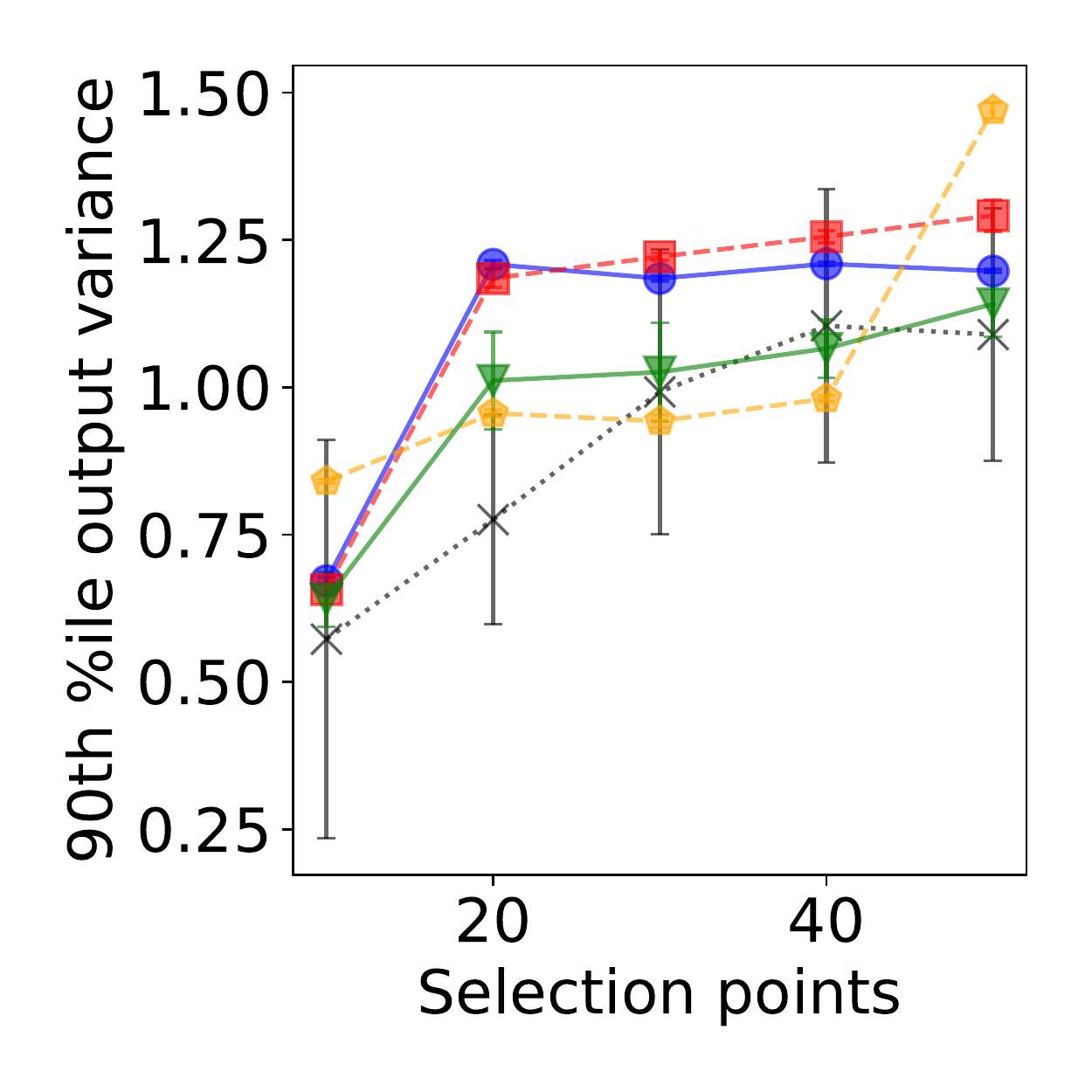}
\includegraphics[width=0.2\linewidth]{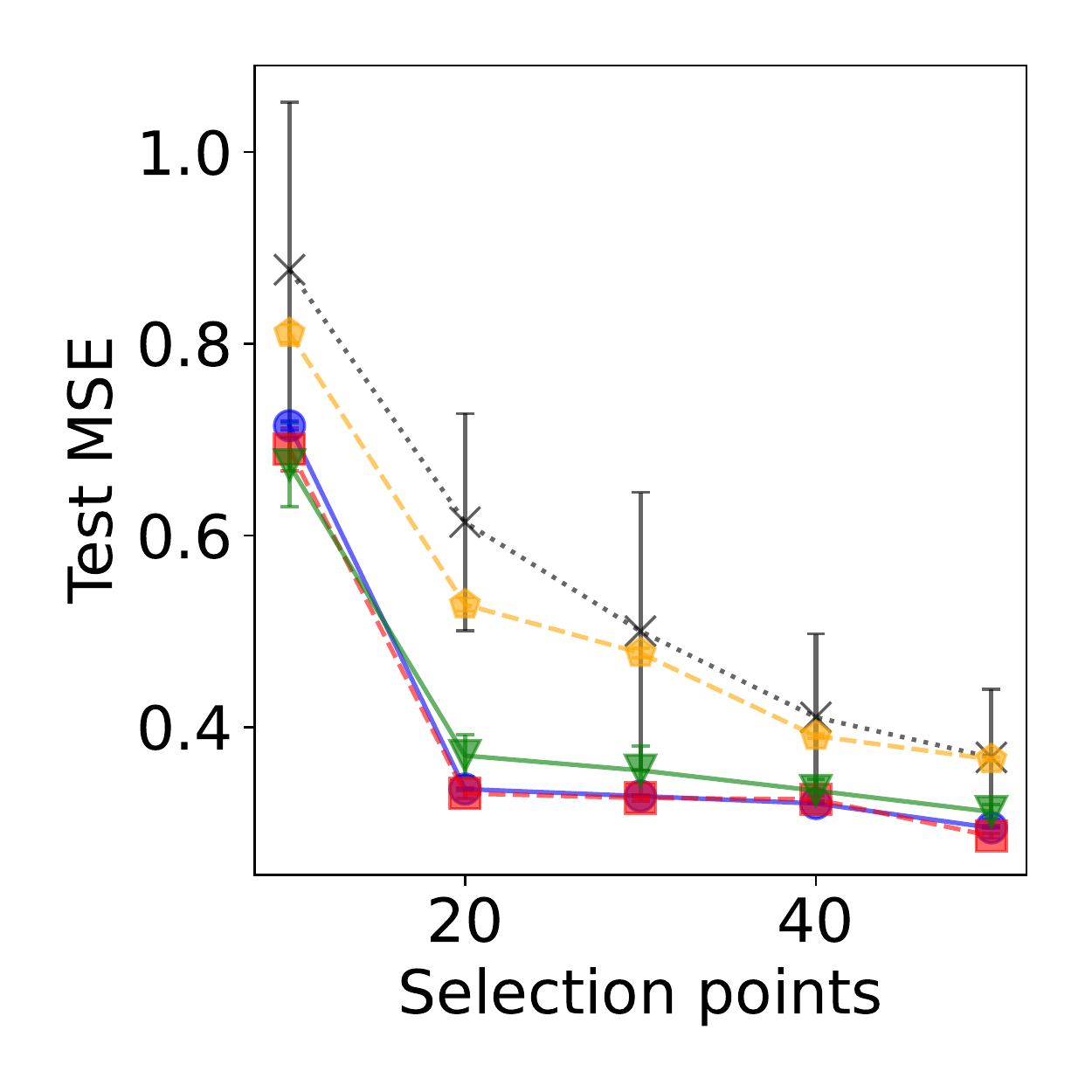} 
&
\includegraphics[width=0.2\linewidth]{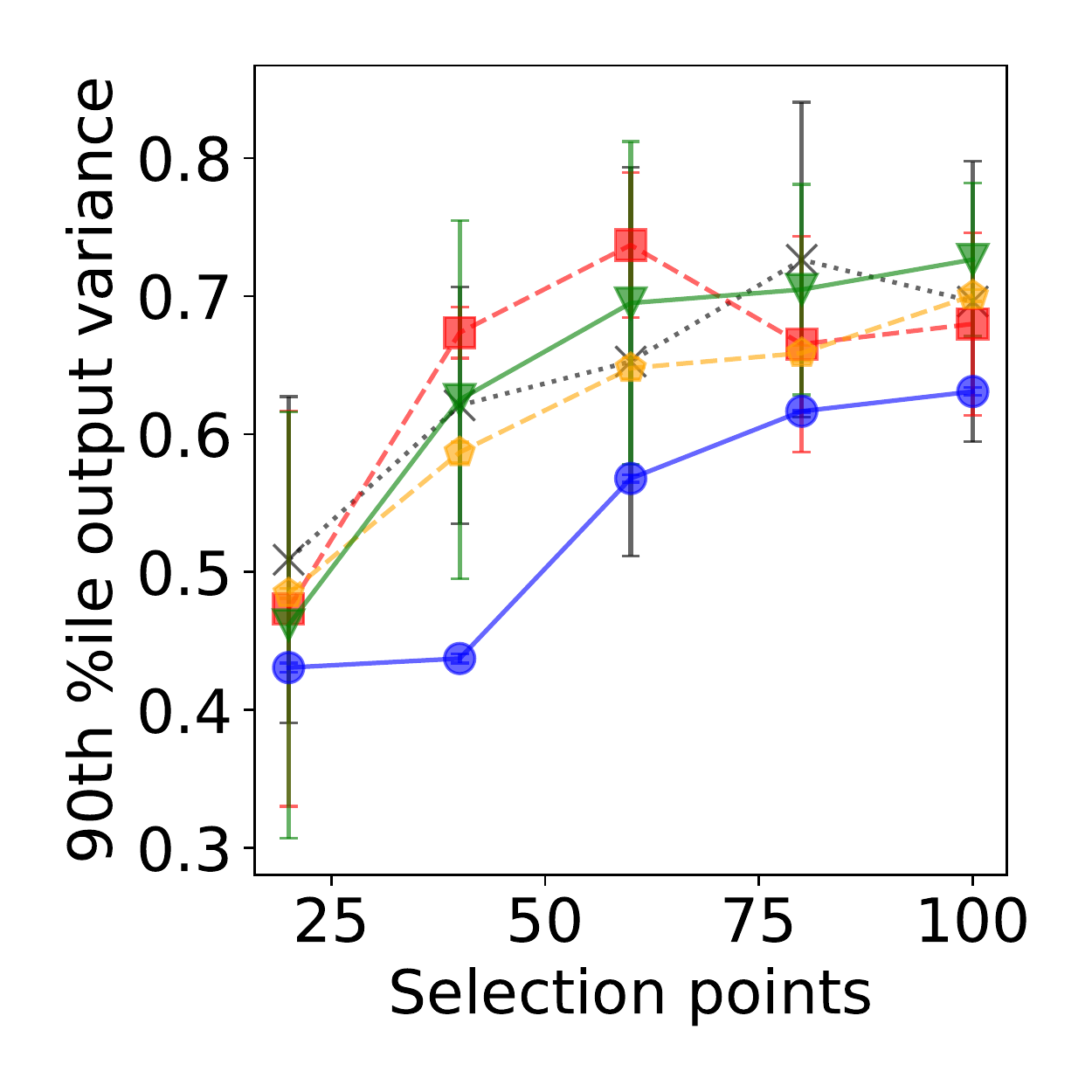}
\includegraphics[width=0.2\linewidth]{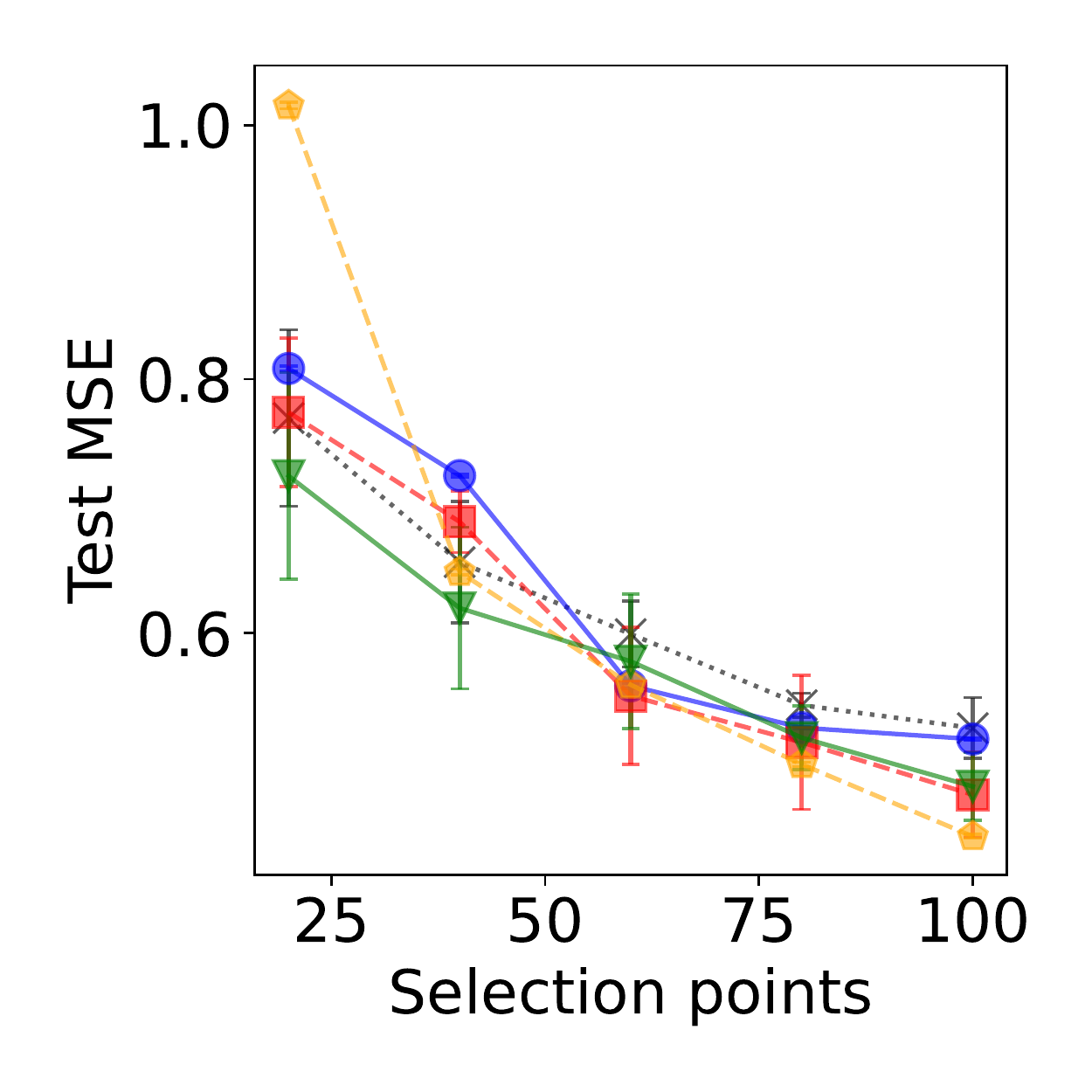} 

\end{tabular}

\includegraphics[width=0.3\linewidth]{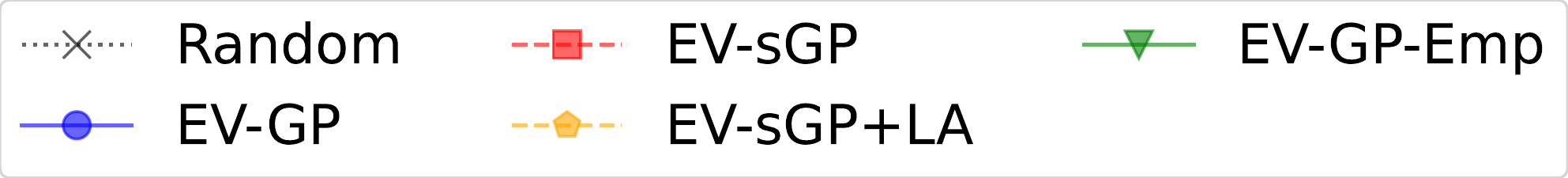}

\caption{Active learning on regression datasets with different approximations of $\sigma_\text{NTKGP}$ and with the empirical NTK. The information shown is presented in the same manner as \cref{fig:small-set-batch} (see \cref{appx:fig-desc}). }
\label{fig:appx-small-set-batch-sntkgp}
\end{figure}

\subsection{Additional Results From Classification Experiments}
\label{appx:exp-class}

\subsubsection{Additional Results From Experiments Involving Neural Networks with Convolutions}
\label{appx:exp-class-more-res}

In \cref{fig:results-class-cnn-extra}(a-b), we show the corresponding output entropy of the experiments from \cref{fig:results-class-all}(c-d). We find that even for convolution experiments our method can generally select points which result in good initialization-robustness. In \cref{fig:results-class-cnn-extra}(c) we provide results for active learning on SVHN dataset on ResNet18 (following the setup from \cite{ashDeepBatchActive2020}), while \cref{fig:results-class-cnn-extra}(d-e) present the results for active learning on CIFAR100 dataset.




\begin{figure}[t]
\centering
\begin{tabular}{cc|c}
\includegraphics[width=0.16\linewidth]{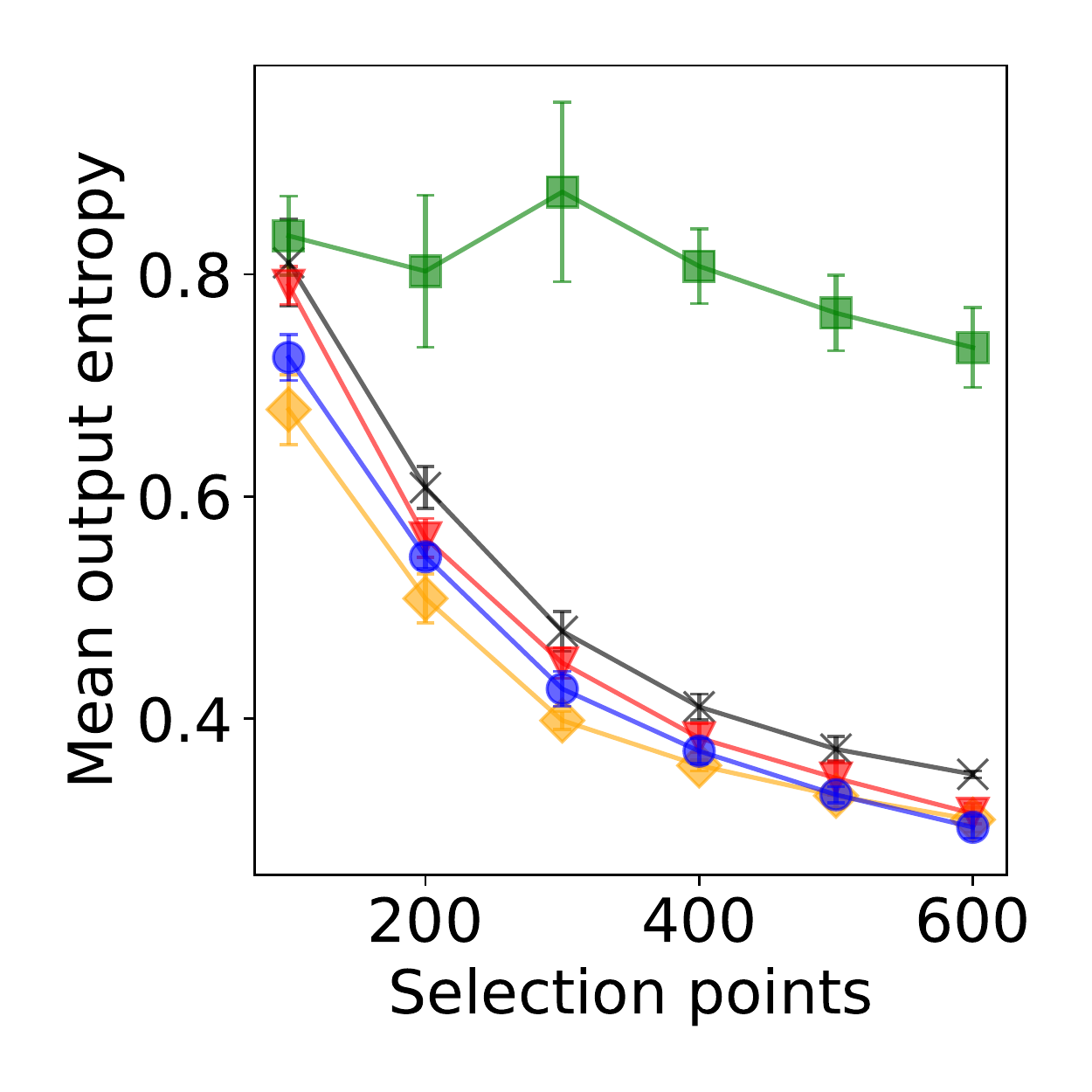} &
\includegraphics[width=0.16\linewidth]{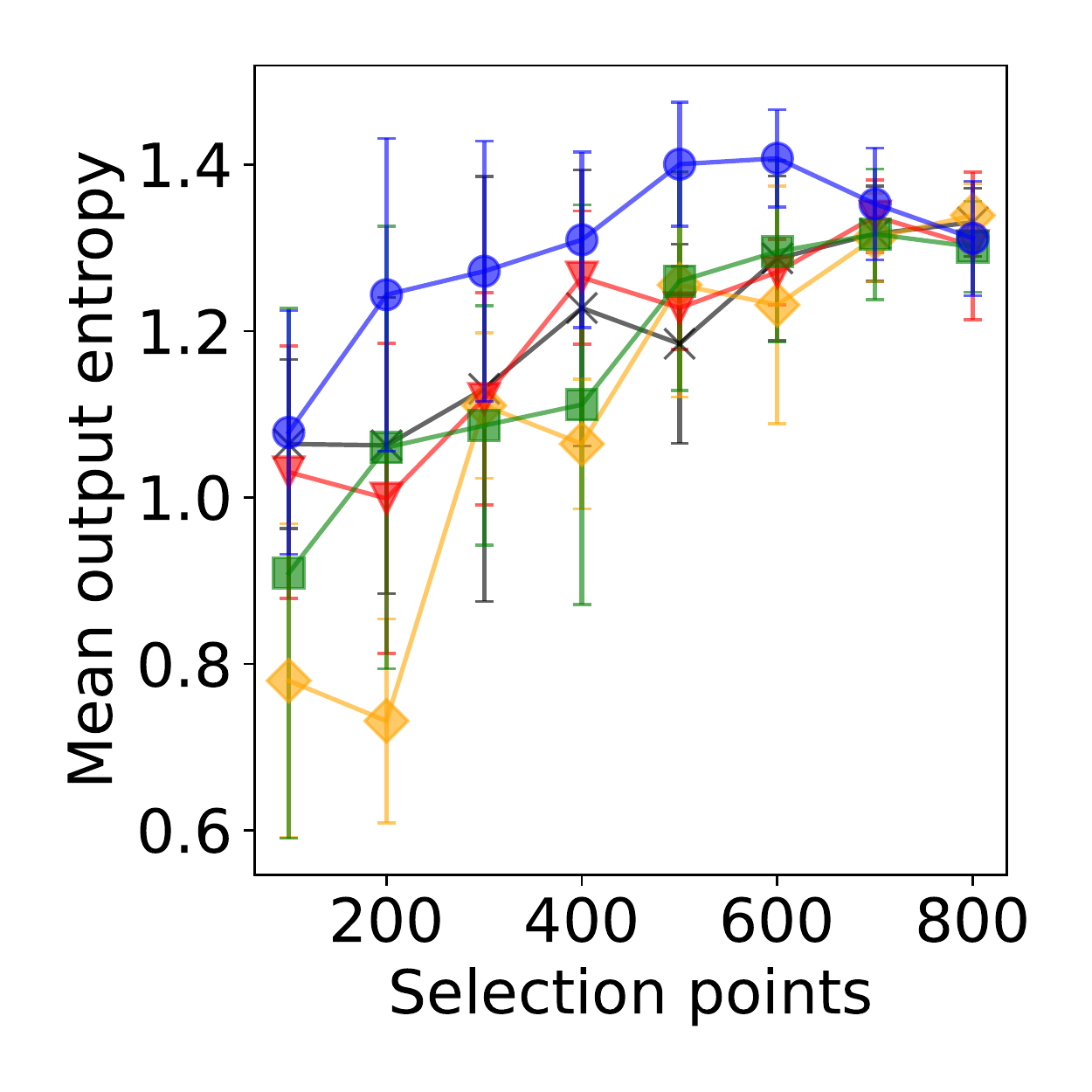} &
\includegraphics[width=0.16\linewidth]{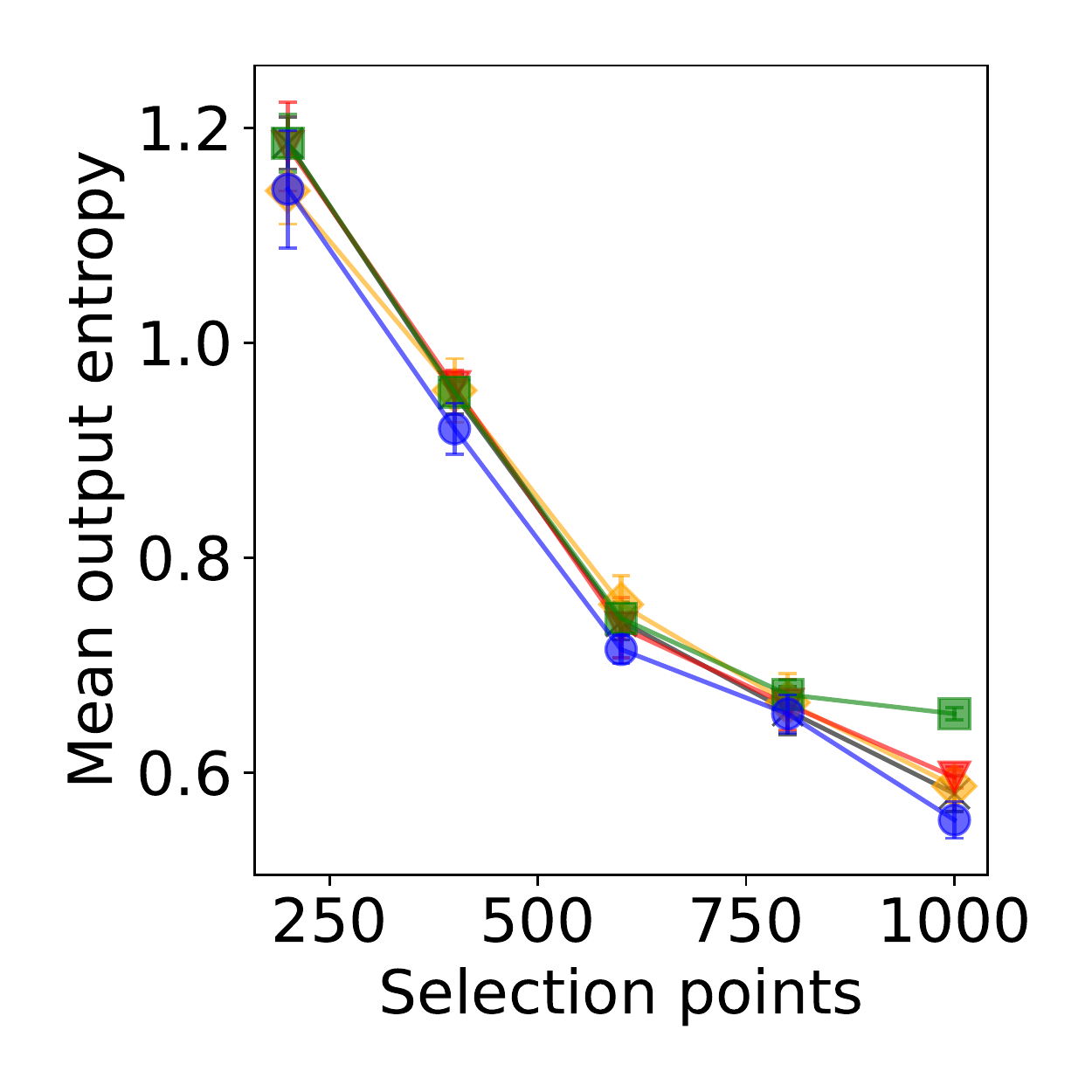}
\includegraphics[width=0.16\linewidth]{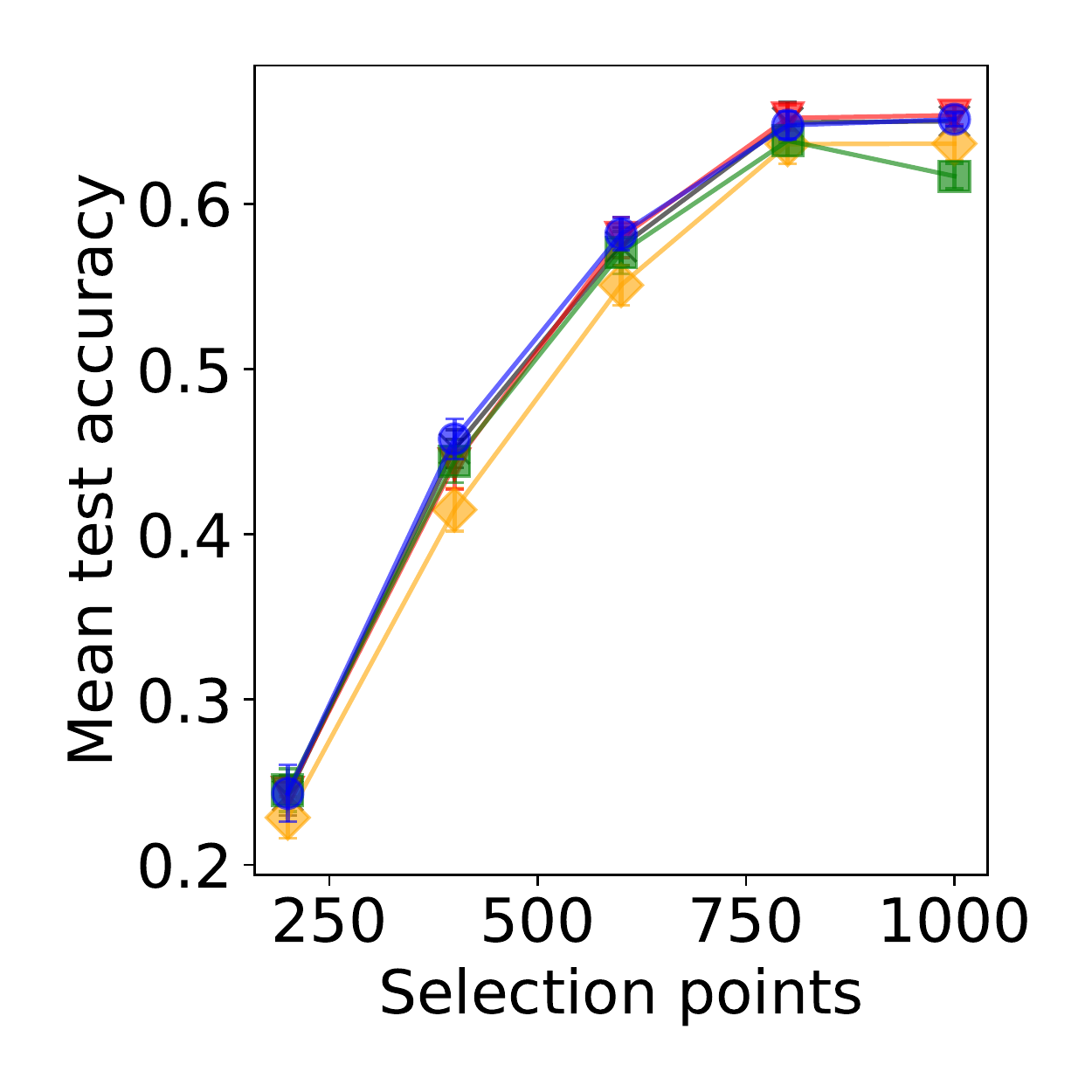} \\
{\sffamily \scriptsize (a) EMNIST, 2-layer CNN}  & 
{\sffamily \scriptsize (b) SVHN, WideResNet} &
{\sffamily \scriptsize (c) SVHN, ResNet18, Batch size 200, Budget 800} \\
\\
\multicolumn{2}{c|}{\includegraphics[width=0.16\linewidth]{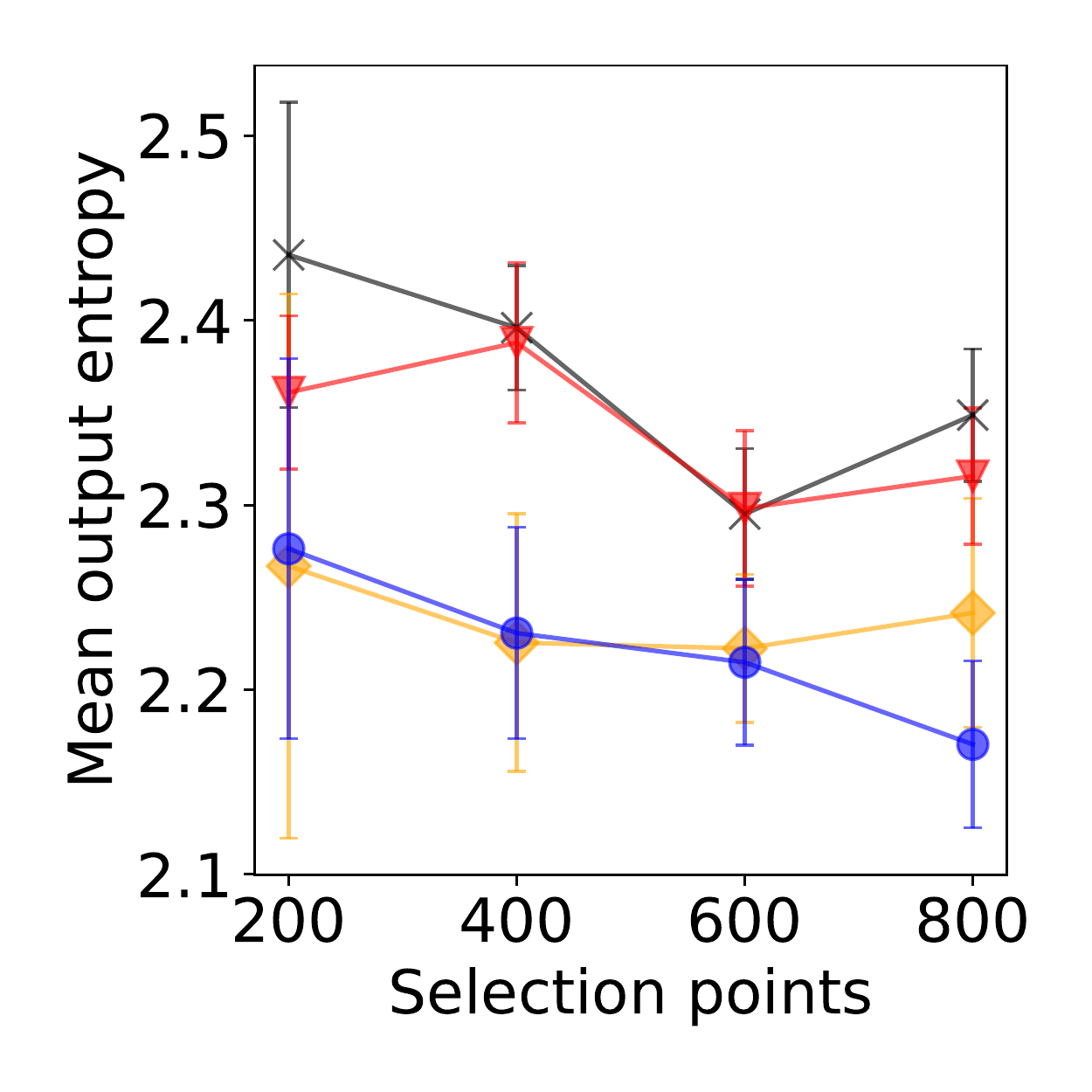}
\includegraphics[width=0.16\linewidth]{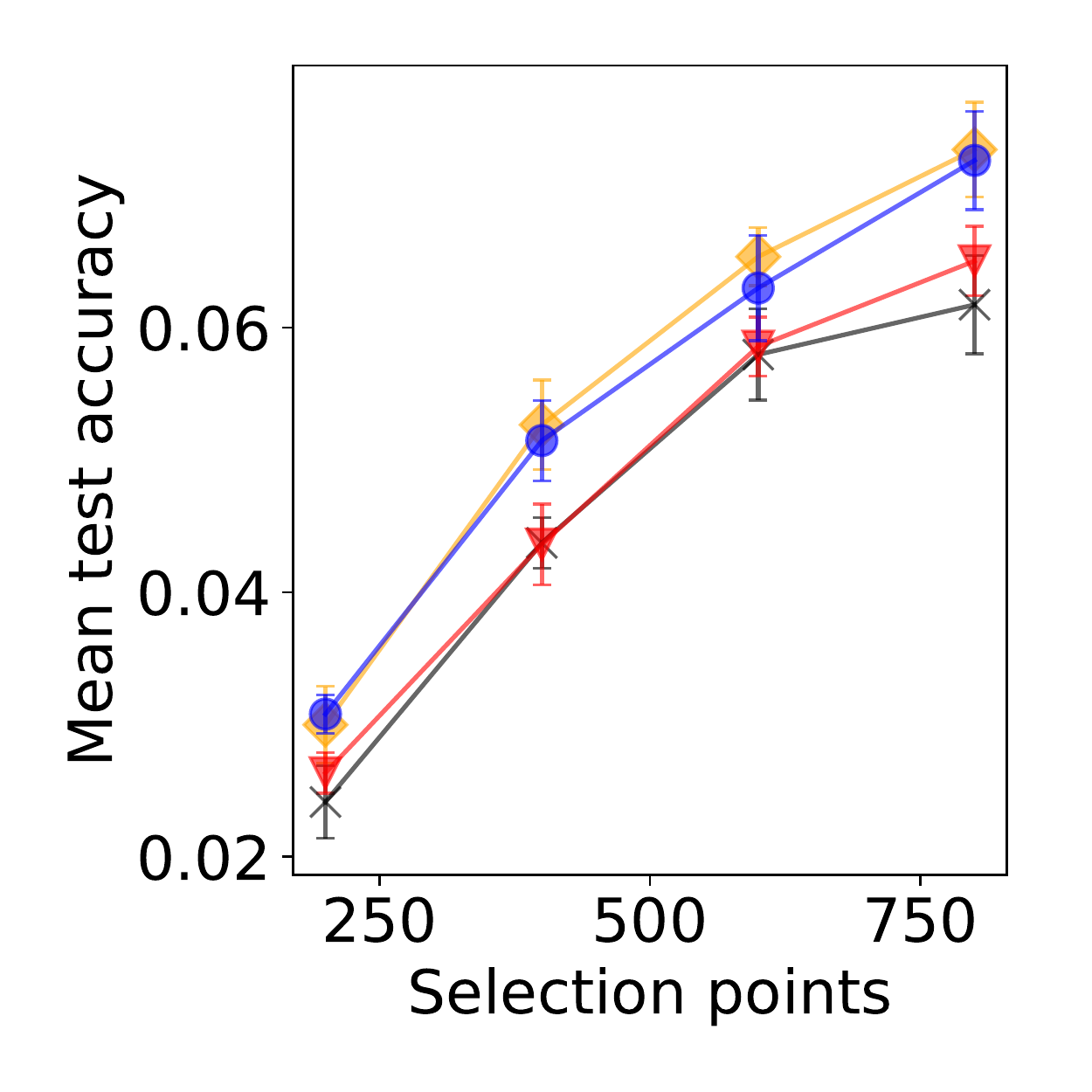}} &
\includegraphics[width=0.16\linewidth]{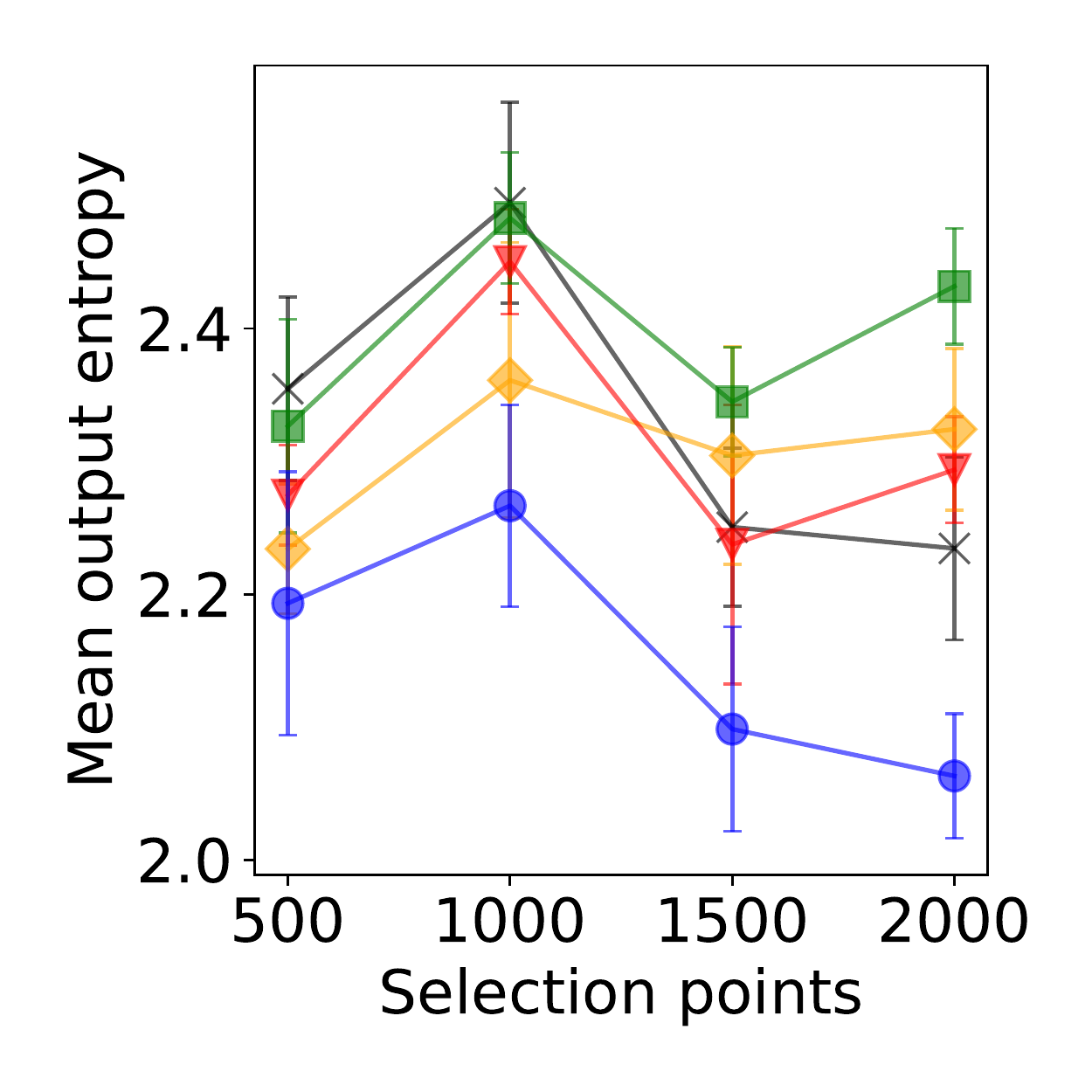}
\includegraphics[width=0.16\linewidth]{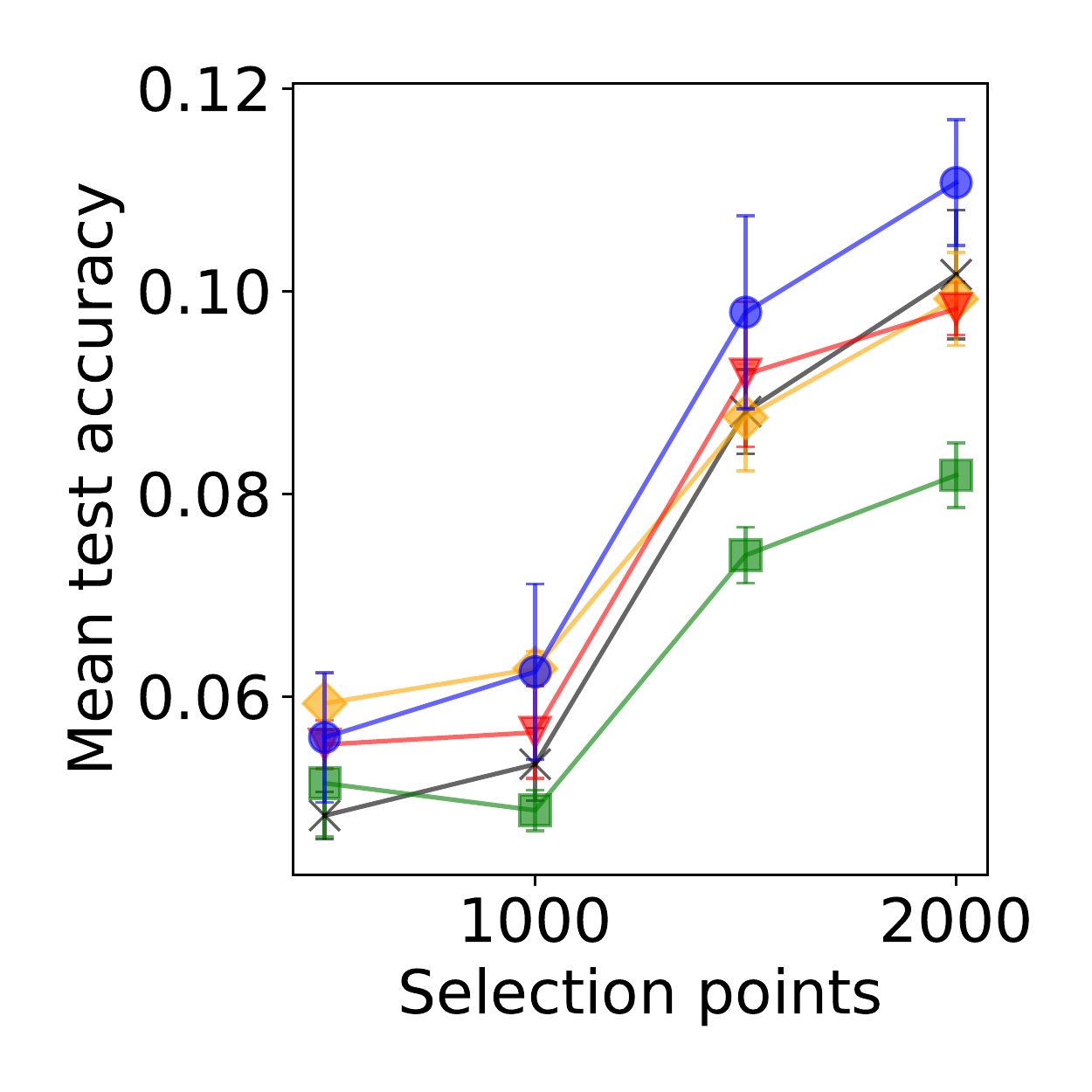} \\
\multicolumn{2}{c|}{\sffamily \scriptsize (d) CIFAR100, WideResNet, Batch size 200, Budget 800} &
{\sffamily \scriptsize (e) CIFAR100, WideResNet, Batch size 500, Budget 2000}\\ 
\end{tabular}

\includegraphics[width=0.3\linewidth]{fig/legends/legend_emp.pdf}

\caption{
(a-b) Achieved output entropy for active learning on corresponding classification experiments conducted for \cref{fig:results-class-all}(c-d). (c) Additional results on classification using ResNet18 on SVHN dataset. (d-e) Additional results on classification using WideResNet on CIFAR100 with different batch size and query budgets.
}
\label{fig:results-class-cnn-extra}
\end{figure}

\subsubsection{Comparison between \alg\xspace and \textsc{BatchBALD}}
\label{appx:exp-batchbald}

We also conducted experiments to compare \alg\xspace with \textsc{BatchBALD} \cite{kirschBatchBALDEfficientDiverse2019}, which is popular uncertainty-based active learning algorithm. \textsc{BatchBALD} was not included in the original benchmark due to its 

In \cref{fig:results-batchbald}, we present the results of the experiments. In these experiments, we find that \alg\xspace outperforms \textsc{BatchBALD}, which shows that \textsc{BatchBALD} is less suitable for scenarios with no initial training data or when large batch sizes are used (which is a use case not targeted at and hence not tested by \citet{kirschBatchBALDEfficientDiverse2019}). \textsc{BatchBALD} also does not take into consideration the overall distribution of the training dataset which may also harm its performances.

\begin{figure}[ht]
\centering

\resizebox{\linewidth}{!}{
\begin{tabular}{c|c}
\includegraphics[width=0.20\linewidth]{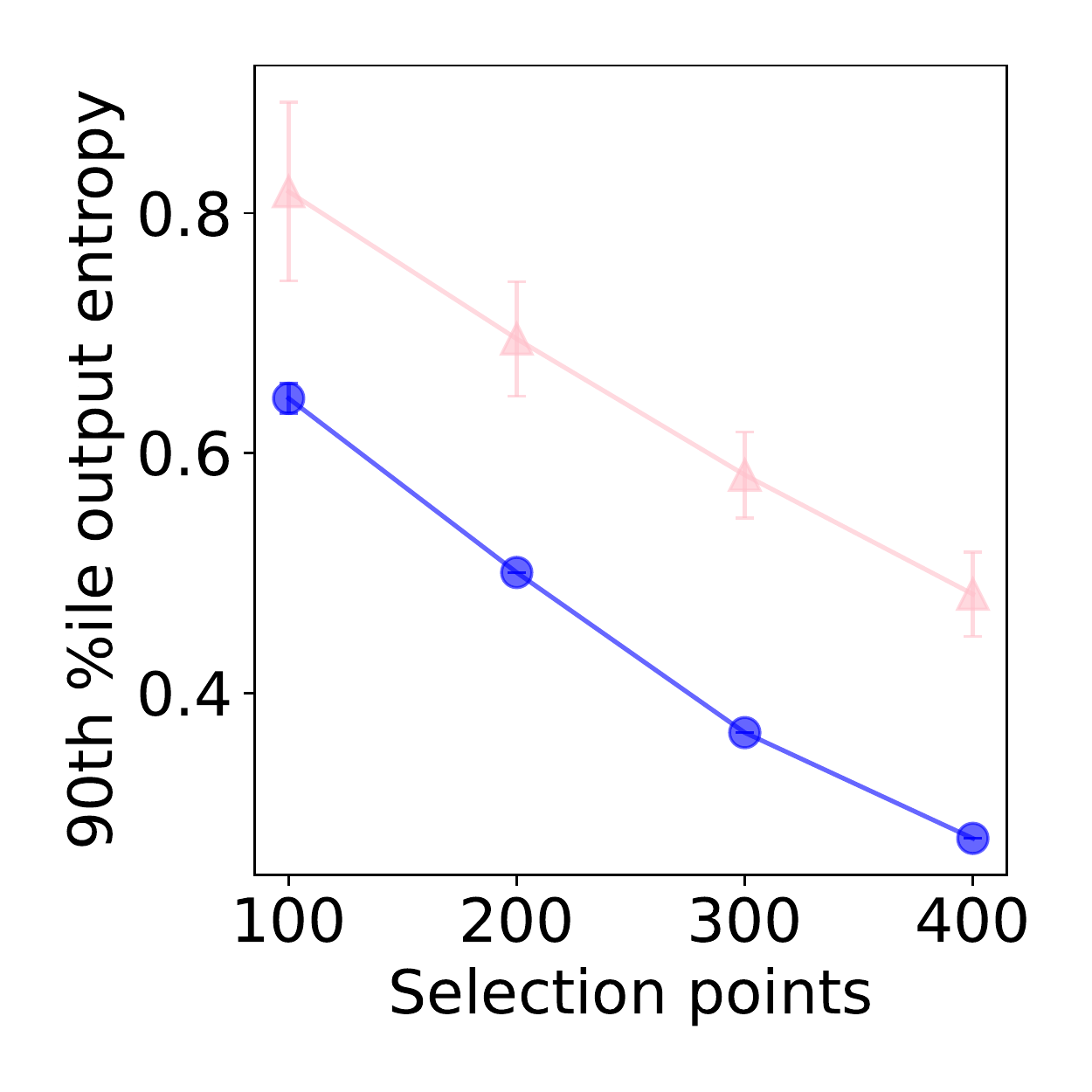}
\includegraphics[width=0.20\linewidth]{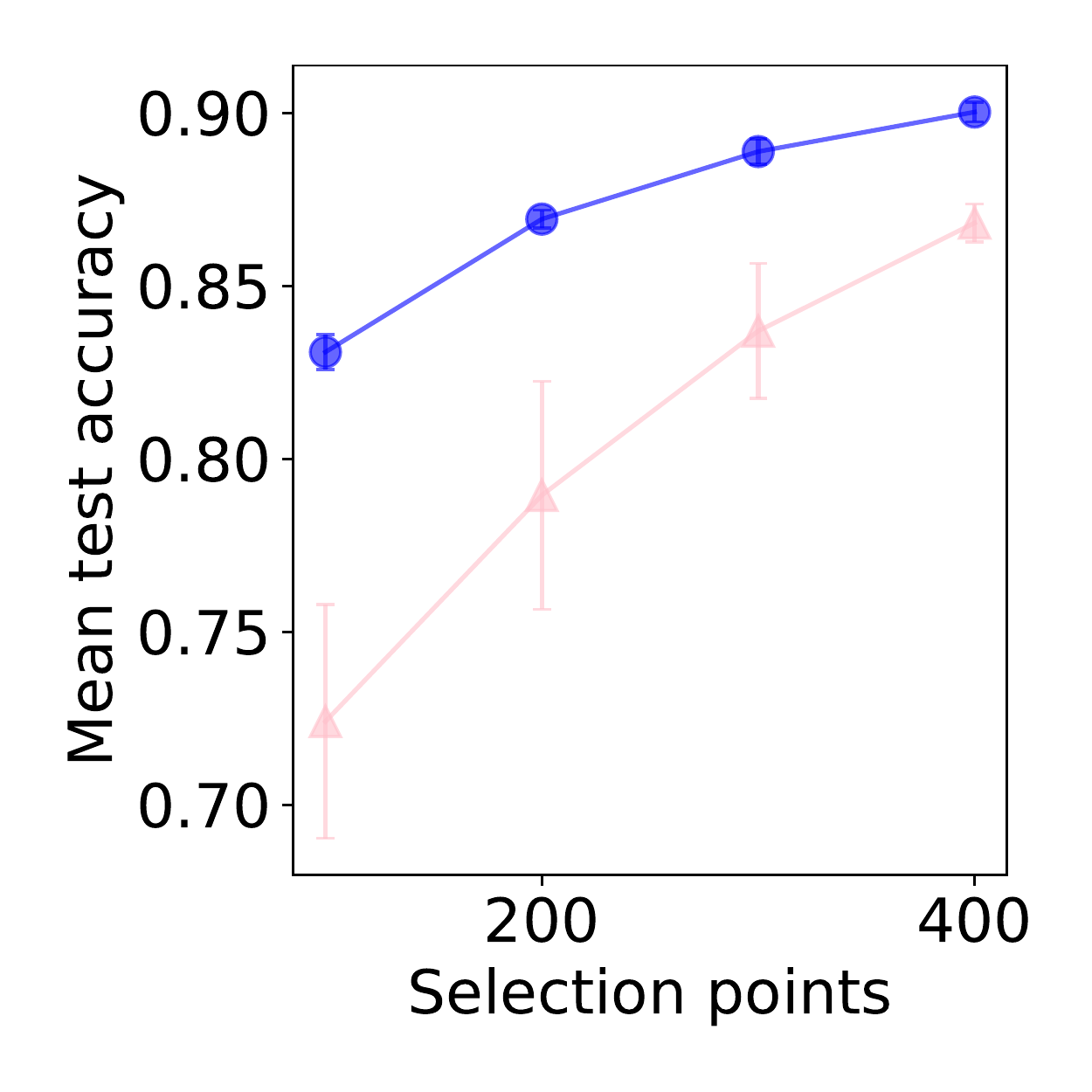} &
\includegraphics[width=0.20\linewidth]{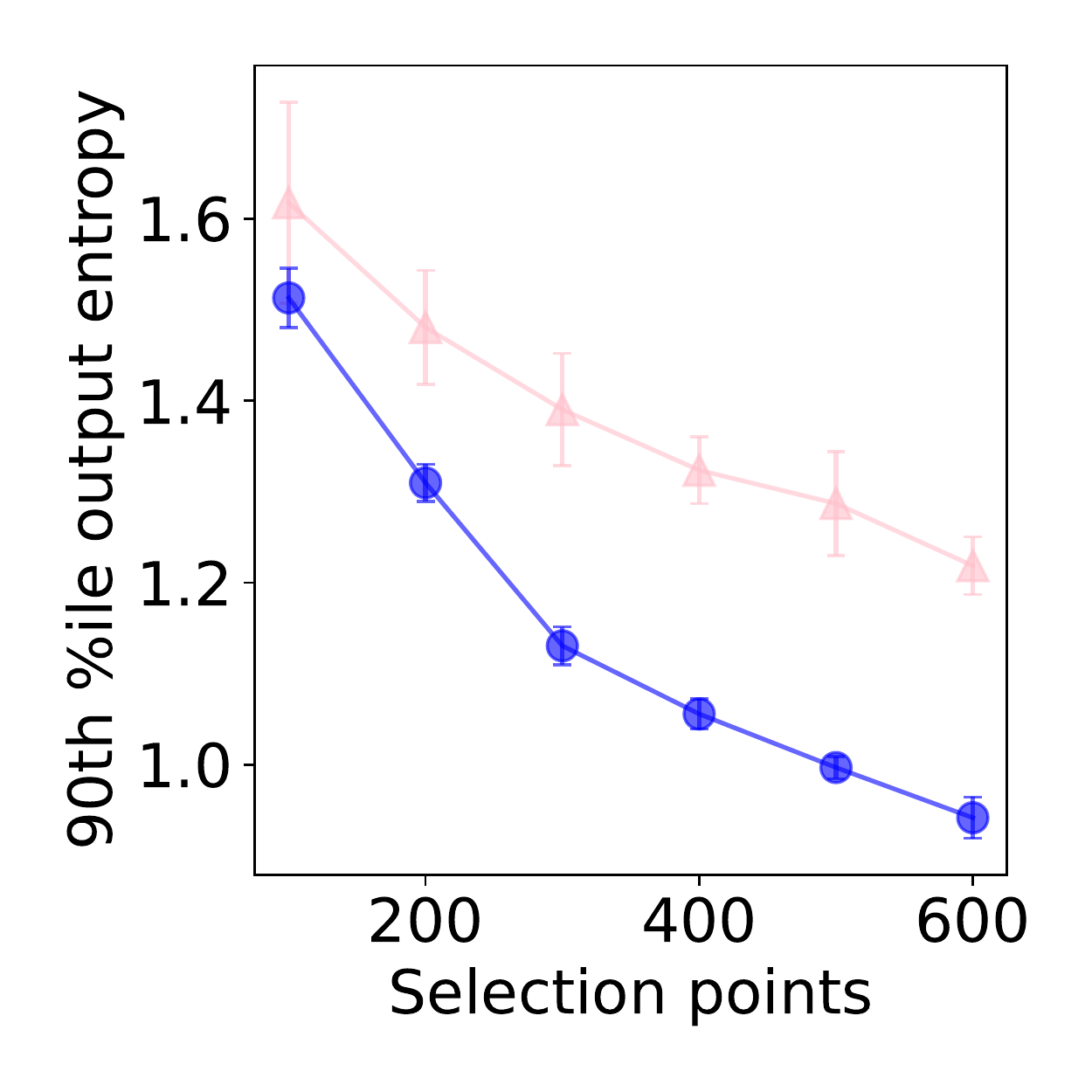}
\includegraphics[width=0.20\linewidth]{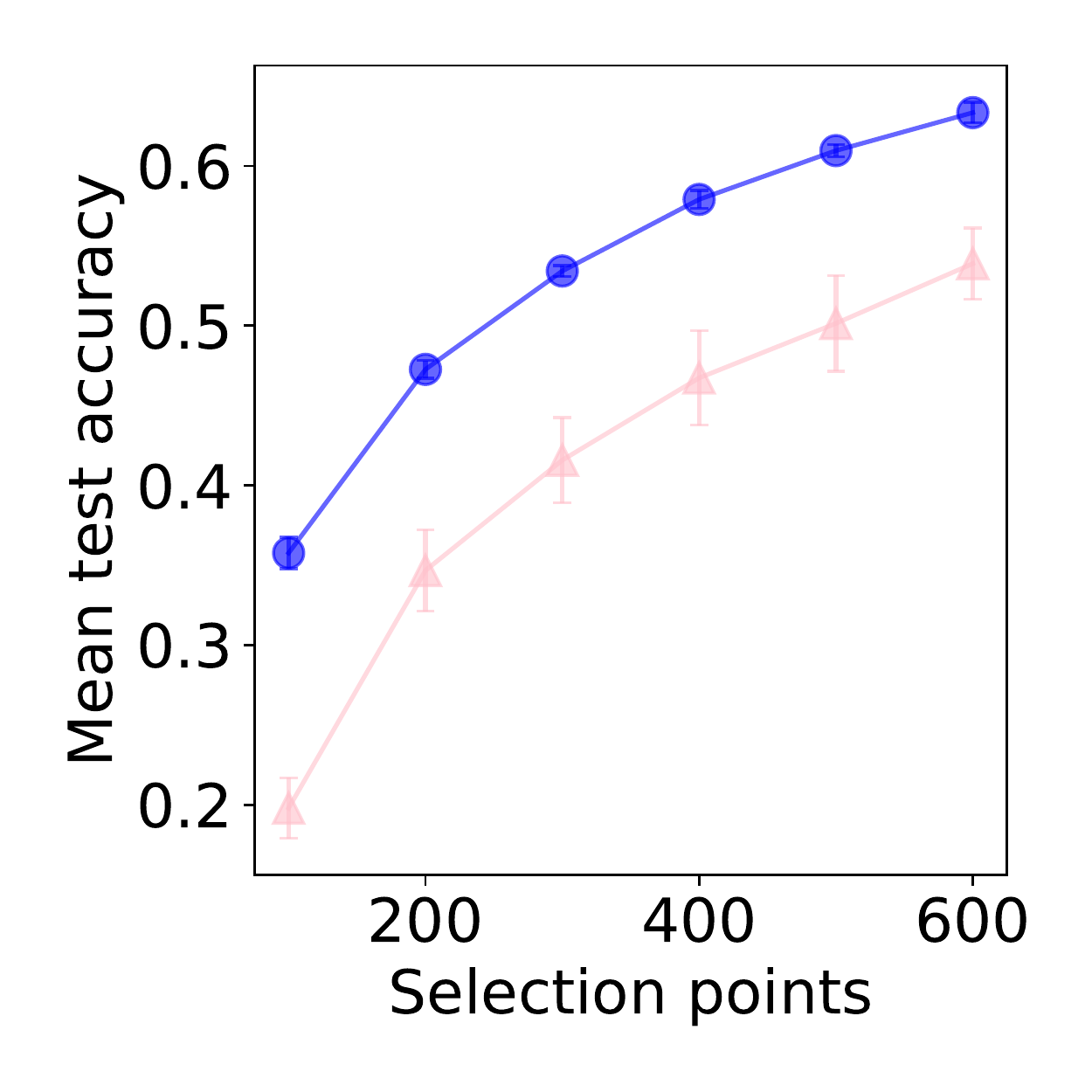} \\
\includegraphics[height=12px]{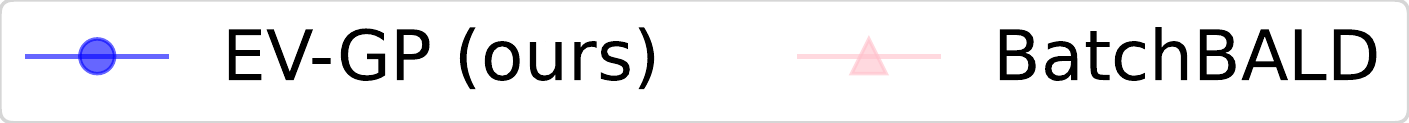} &
\includegraphics[height=12px]{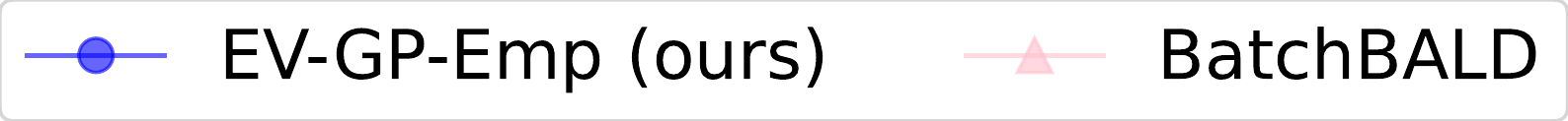} \\
{\sffamily \scriptsize MNIST (2-layer MLP)}  & 
{\sffamily \scriptsize EMNIST (CNN)}  \\ 
\end{tabular}
}
\caption{
Comparison between \alg\xspace and \textsc{BatchBALD}.
}
\label{fig:results-batchbald}
\end{figure}

\subsubsection{Additional Results From Varying Active Learning Batch Sizes}
\label{appx:exp-batch-sz}

In \cref{fig:results-vary-batch-appx}, we present further results for neural network active learning when the batch sizes vary. We find our algorithm shows little change in both accuracy and output entropy as the batch size gets larger. Note that when the empirical NTK is used, the algorithm seems to perform slightly worse in larger batch sizes case. This may be due to the chance that in some randomized network the empirical NTK provides a slightly less accurate representation of the true uncertainty. In the smaller iterations the random neural network used for NTK computation is reinitialized and so such an error can be reduced. However we find that in practice this effect is less significant and we still achieve a good active set regardless.

\begin{figure}[ht]
\centering

\resizebox{\linewidth}{!}{
\begin{tabular}{cc|cc}
\includegraphics[width=0.16\linewidth]{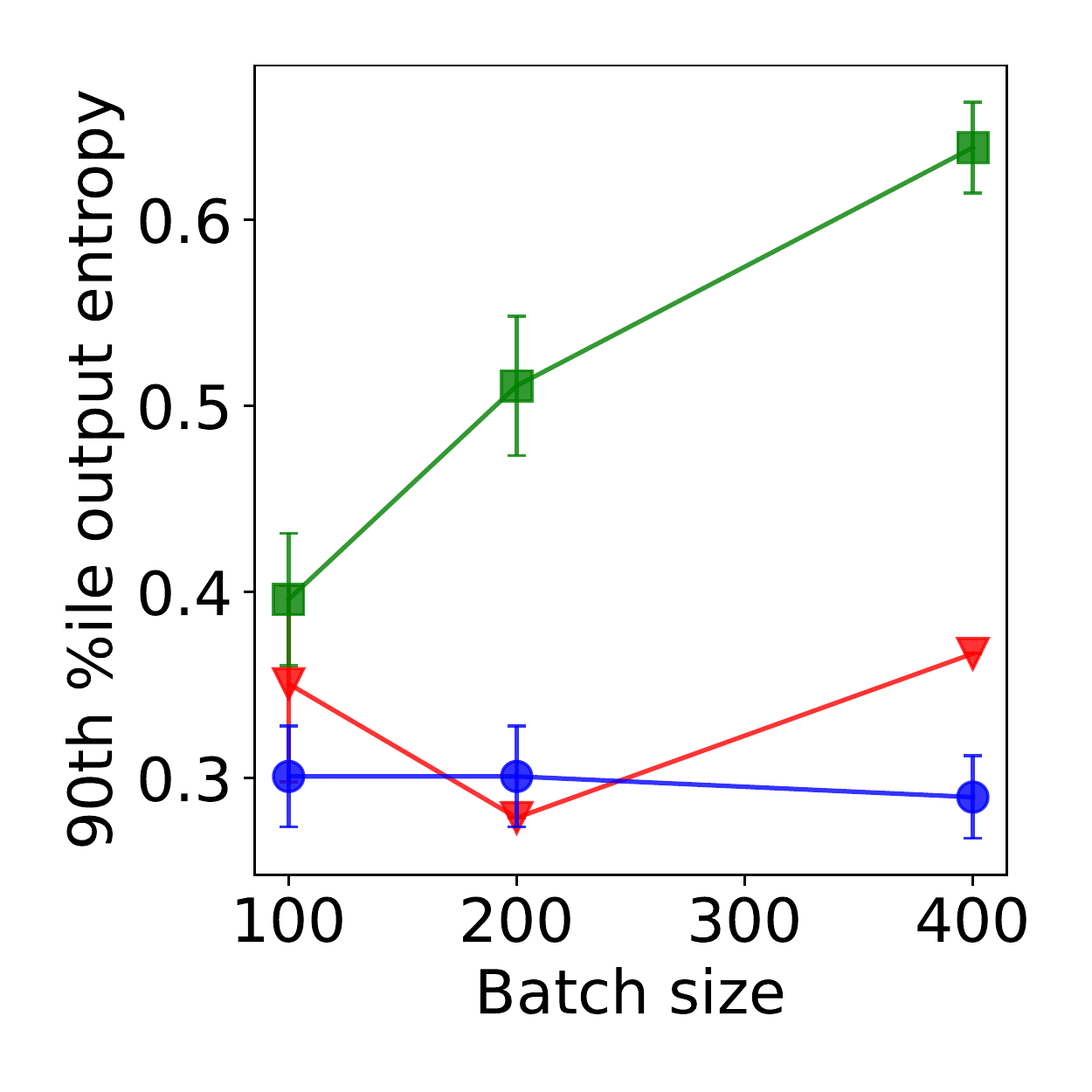}&
\includegraphics[width=0.16\linewidth]{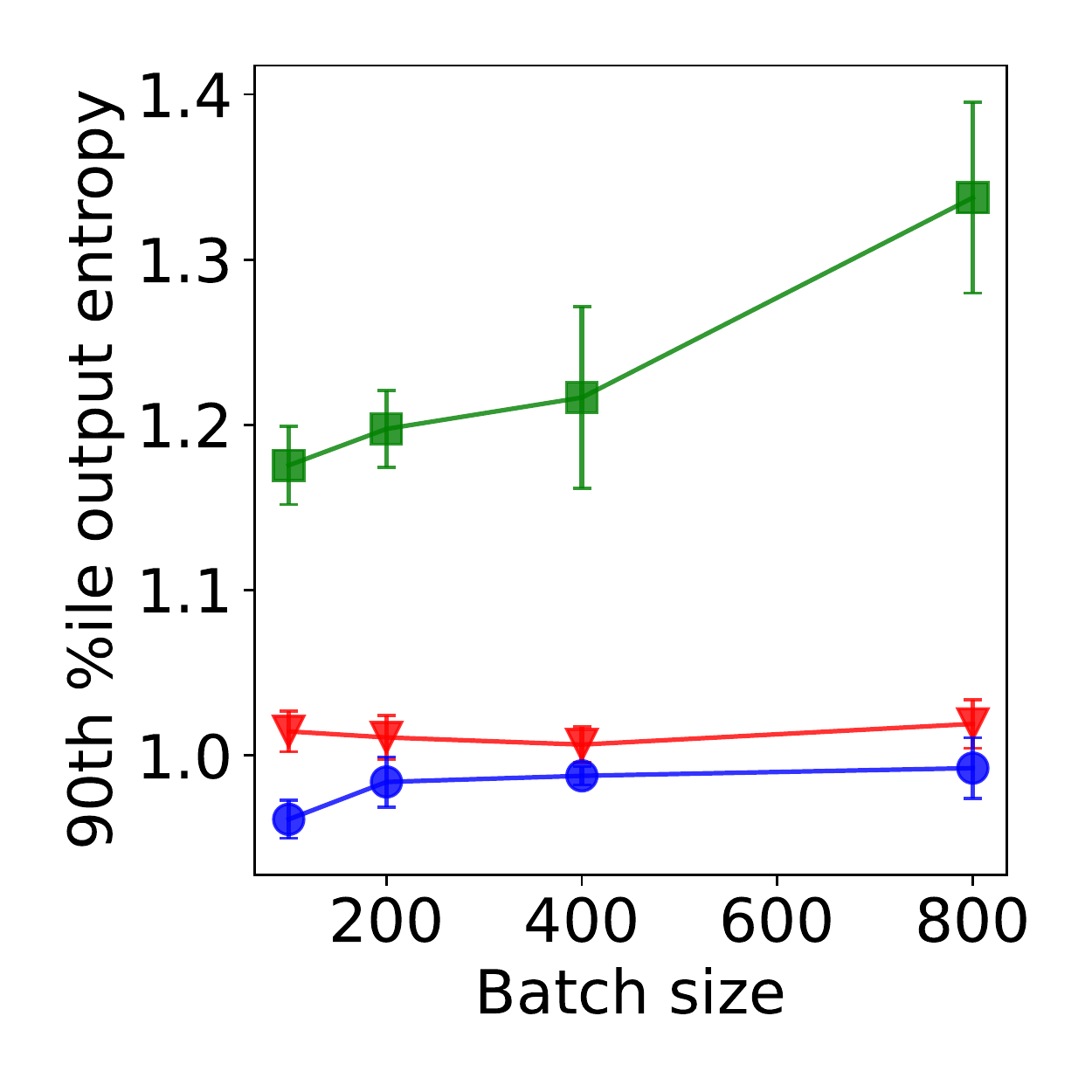}
\includegraphics[width=0.16\linewidth]{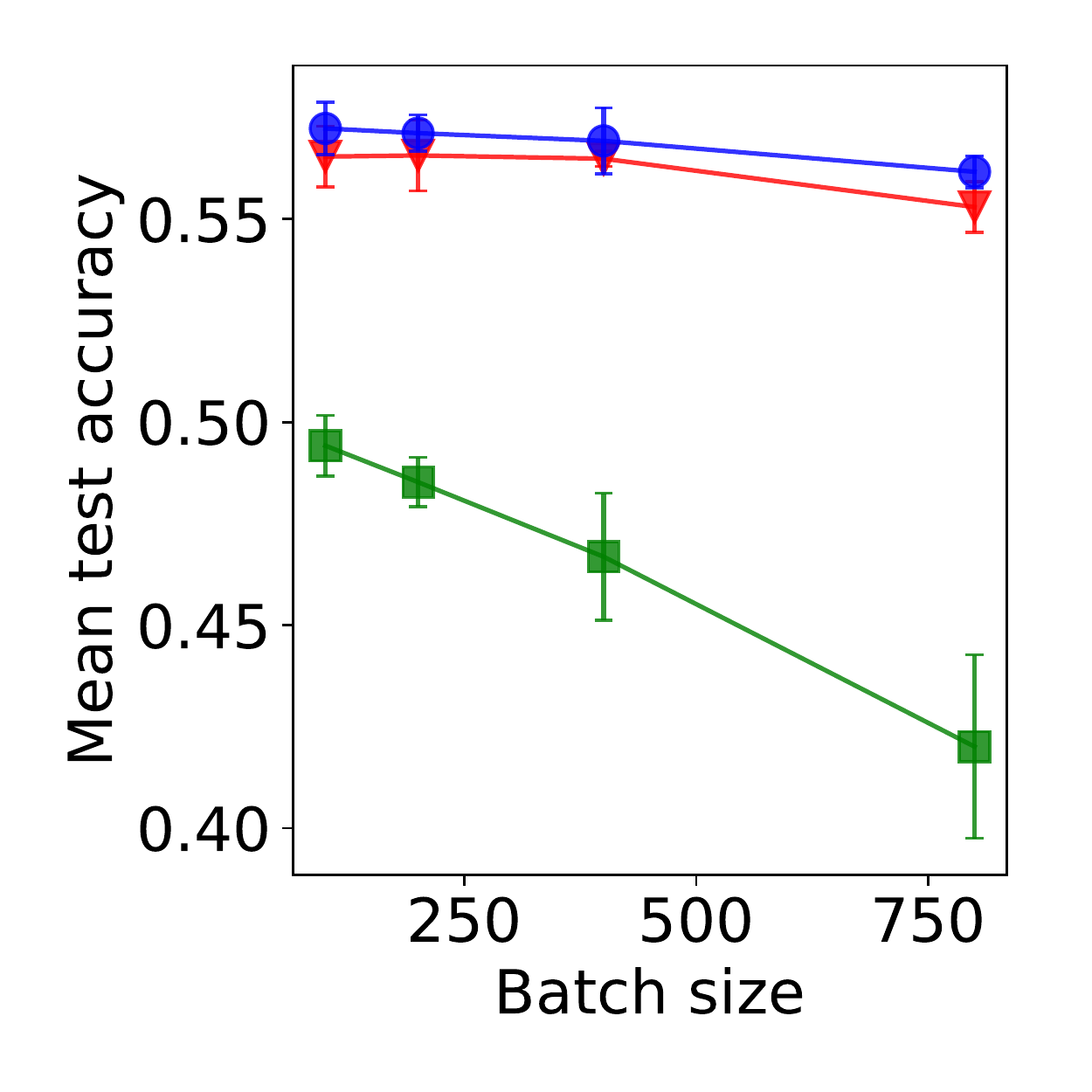}&
\includegraphics[width=0.16\linewidth]{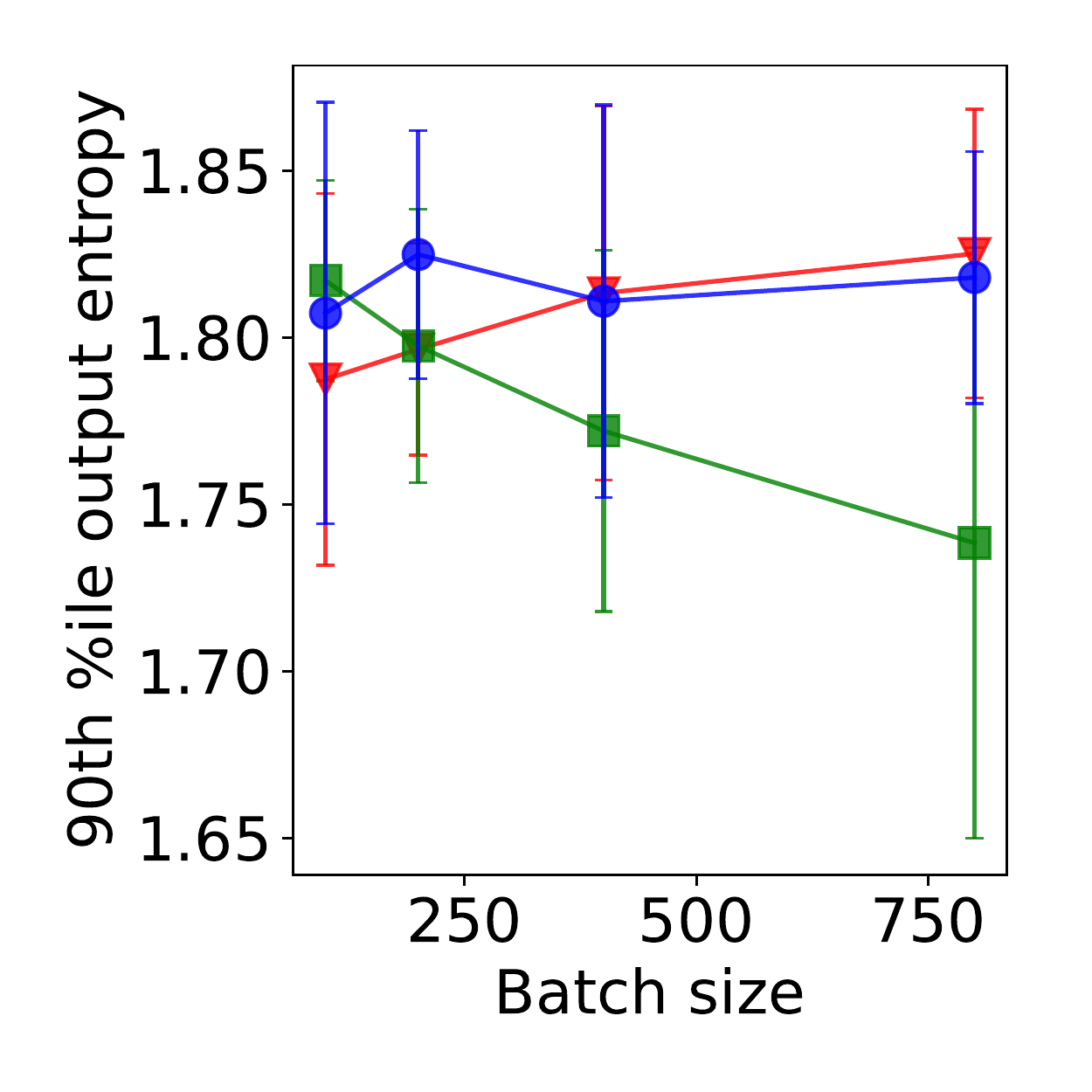}&
\includegraphics[width=0.16\linewidth]{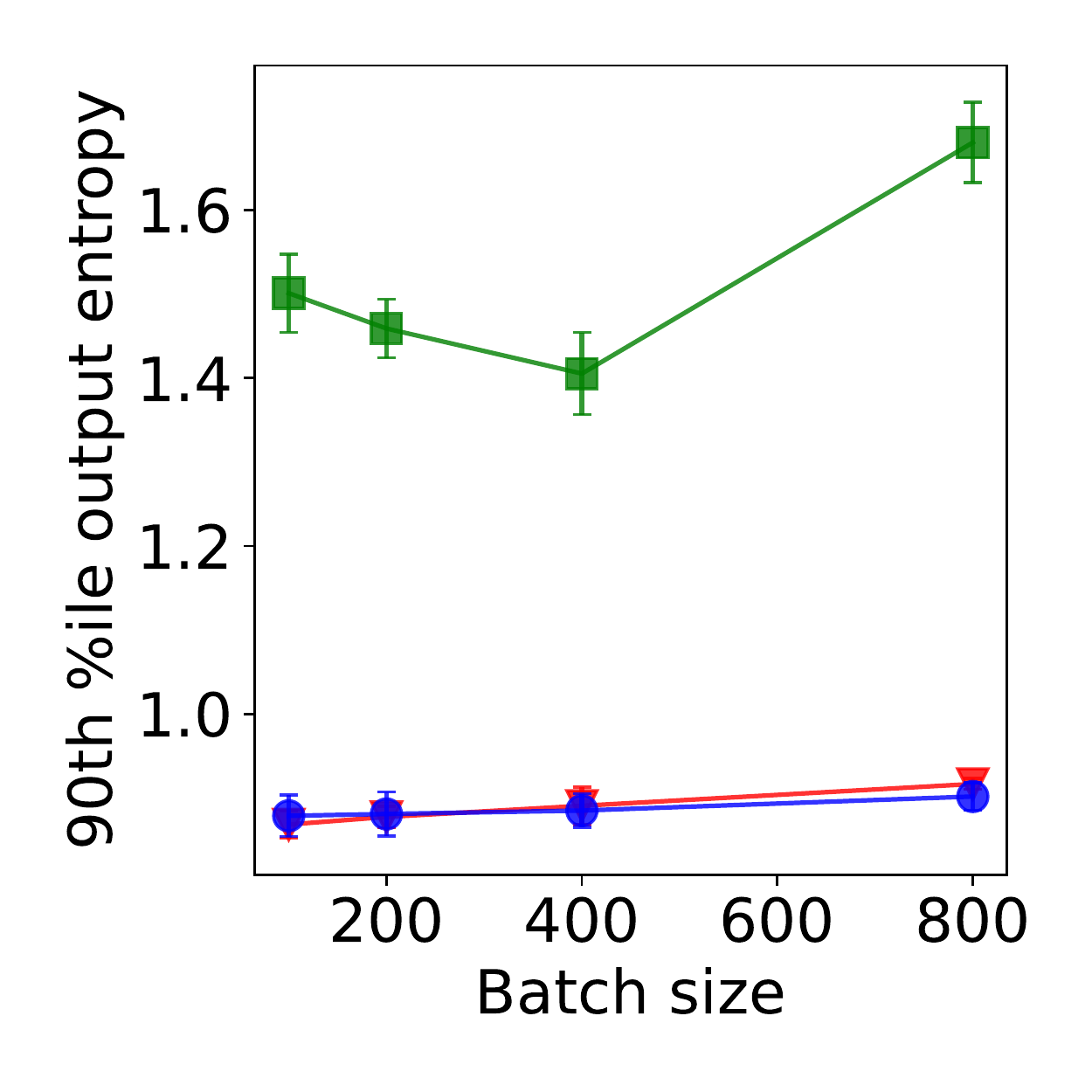} 
\includegraphics[width=0.16\linewidth]{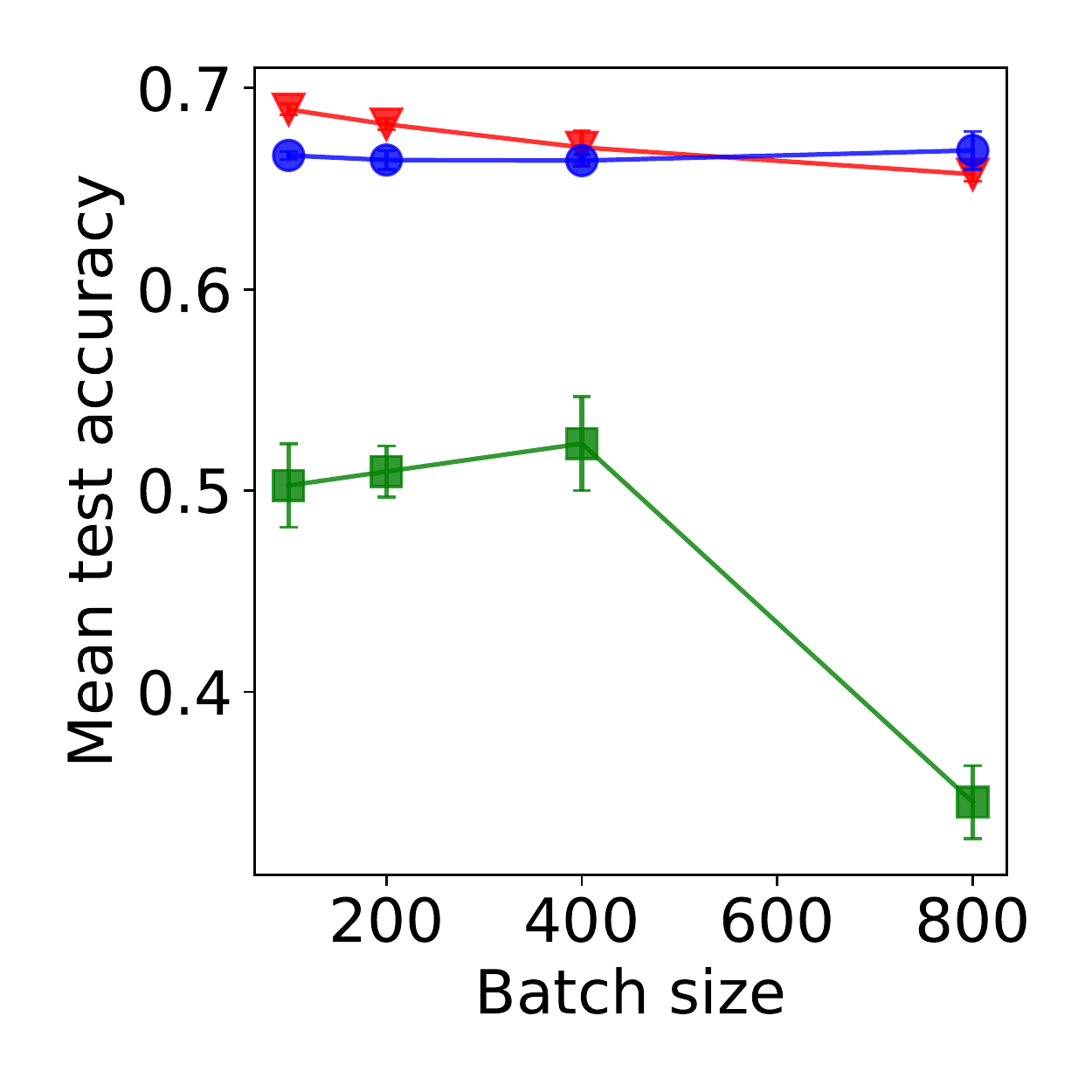} \\
{\sffamily \scriptsize MNIST (2-layer MLP)}  & 
{\sffamily \scriptsize EMNIST (3-layer MLP)}  & 
{\sffamily \scriptsize SVHN (WideResNet)} &
{\sffamily \scriptsize EMNIST (CNN)}  \\ 
\multicolumn{2}{c|}{\includegraphics[width=0.3\linewidth]{fig/legends/legend_theory.pdf}} &
\multicolumn{2}{c}{\includegraphics[width=0.3\linewidth]{fig/legends/legend_emp.pdf}}
\end{tabular}
}
\caption{
Results on classification with varying batch sizes.
}
\label{fig:results-vary-batch-appx}
\end{figure}

\subsubsection{Discussion on Running Time of Active Learning Algorithms}
\label{appx:exp-running-time}

Although the active set selection process can incur long running time, in the practical deployment of active learning, the time required for active set selection is easily dominated by the costs of querying data labels and training the neural networks. To this end, \alg\xspace is able to reduce the \textit{practical} running time in two aspects. Firstly, since active learning is usually required when querying data labels incurs high costs (e.g., costs in terms of time), an active learning algorithm which requires a smaller number of query batches (while ensuring good performances), is less time-consuming in practice. Secondly, significant source of the time costs in active learning results from the training of the neural networks, which is further aggravated in active learning problems involving a large amount of data. Since \alg\xspace is both batch size independent (as presented in \cref{fig:results-vary-batch}) and training-free, \alg\xspace can reduce the practical costs in these aspects.

Regardless, we may still analyze the amount of time our active learning algorithm requires in performing active set selection. We present the running time of our experiments, which does not account for the time costs of querying data labels. Specifically, in \cref{tab:time1}, we present the running time for active set selection on MNIST dataset on 2-layer MLP, which corresponds to experiments in \cref{fig:results-class-all}(a). Meanwhile, in \cref{tab:time2}, we present the running time for active set selection on SVHN dataset on ResNet18, which corresponds to experiments in \cref{fig:results-class-cnn-extra}(c)). These results show that the running time of \alg\xspace is not significantly larger than that of other algorithms which utilizes the NTK (e.g., MLMOC) or other sequential greedy-based active learning algorithms (e.g., \textsc{BatchBALD}). Therefore, our \alg\xspace outperforms the other algorithms in terms of both the predictive performances and initialization robustness (as shown in our previous results) without incurring significantly more running time. As a consequence, in practical settings where the cost of labelling is higher, we expect \alg\xspace to have comparable or better time efficiency than other active learning algorithms due to its ability to run in large batches with little sacrifice to its performances.

\begin{table}[t]
    \centering
    \caption{Running time for experiments on MNIST dataset on 2-layer MLPs (corresponding to experiments in \cref{fig:results-class-all}(a)).}
    \label{tab:time1}
    \begin{tabular}{|c|cccc|}
    \hline
        \multirow{2}{*}{Algorithm} & \multicolumn{4}{c|}{Time to query $n$ points (in seconds)} \\ \cline{2-5}
        & $n=100$ & $n=200$ & $n=300$ & $n=400$ \\ \hline
        BADGE & $12\pm 1$ & $23\pm 1$ & $35\pm 2$ & $47\pm 2$ \\ 
        \textsc{BatchBALD} & $2629\pm 1570$ & $4499\pm 1571$ & $6371\pm 1572$ & $8227\pm 1574$ \\ 
        MLMOC & $88\pm 11$ & $441\pm 62$ & $974\pm 98$ & $1637\pm 120$ \\ 
        \alg & $1120\pm 9$ & $2089\pm 12$ & $3159\pm 33$ & $4283\pm 41$ \\ \hline
    \end{tabular}
\end{table}

\begin{table}[t]
    \centering
    \caption{Running time for the experiments using ResNet18 on the SVHN dataset (corresponding to experiments in \cref{fig:results-class-cnn-extra}(c)).}
    \label{tab:time2}
    \begin{tabular}{|c|ccccc|}
    \hline
        \multirow{2}{*}{Algorithm} & \multicolumn{5}{c|}{Time to query $n$ points (in seconds)} \\ \cline{2-6}
        & $n=200$ & $n=400$ & $n=600$ & $n=800$ & $n=1000$ \\ \hline
        BADGE & $34\pm 4$ & $70\pm 8$ & $105\pm 9$ & $148\pm 13$ & $204\pm 15$ \\ 
        MLMOC & $2486\pm 3$ & $5485\pm 38$ & $9816\pm 43$ & $16564\pm 62$ & $26607\pm 457$ \\ 
        EV-GP & $3768\pm 10$ & $7846\pm 10$ & $13020\pm 52$ & $18755\pm 91$ & $25018\pm 115$ \\ \hline
    \end{tabular}
\end{table}

\subsubsection{Results From Varying Neural Network Width}
\label{appx:exp-model-width}

Here we investigate how the width of the NN affects the performance of our \textsc{EV-GP-Emp} algorithm, by changing the widths of both the NNs used for data selection (i.e., for calculating our \alg~criterion (\cref{subsec:active:learning:criterion})) and for training.
The left figure in \cref{fig:results-fixed-train-width} shows that increasing the width of either the NN for data selection (horizontal axis) or for training (different lines) improves the initialization robustness
This is because a wider NN for both data selection and training can reduce the error between the empirical NTK and exact NTK (\cref{lr:ntk}), and therefore improve the accuracy of approximation for $\sigma^2_\text{NTKGP}$ \eqref{eq:ntk-gp:sigma} and our \alg~criterion \eqref{eq:criterion}.
As a result, this leads to more accurate data selection and hence better initialization robustness.
The right figure in \cref{fig:results-fixed-train-width} shows that increasing the width of the NN for training improves test accuracy, which can be attributed to the better expressivity of wider NNs (for the fixed width). Meanwhile, in contrast to initialization robustness, increasing the width of the NN for data selection does not have a significant impact on the test accuracy.
This suggests that the NN does not need to be extremely wide in order to select data points that lead to a good predictive performance of the trained NN.




\begin{figure}[t]
\centering

{\sffamily \scriptsize MNIST (3-layer MLP)} \\
\includegraphics[width=0.4\linewidth]{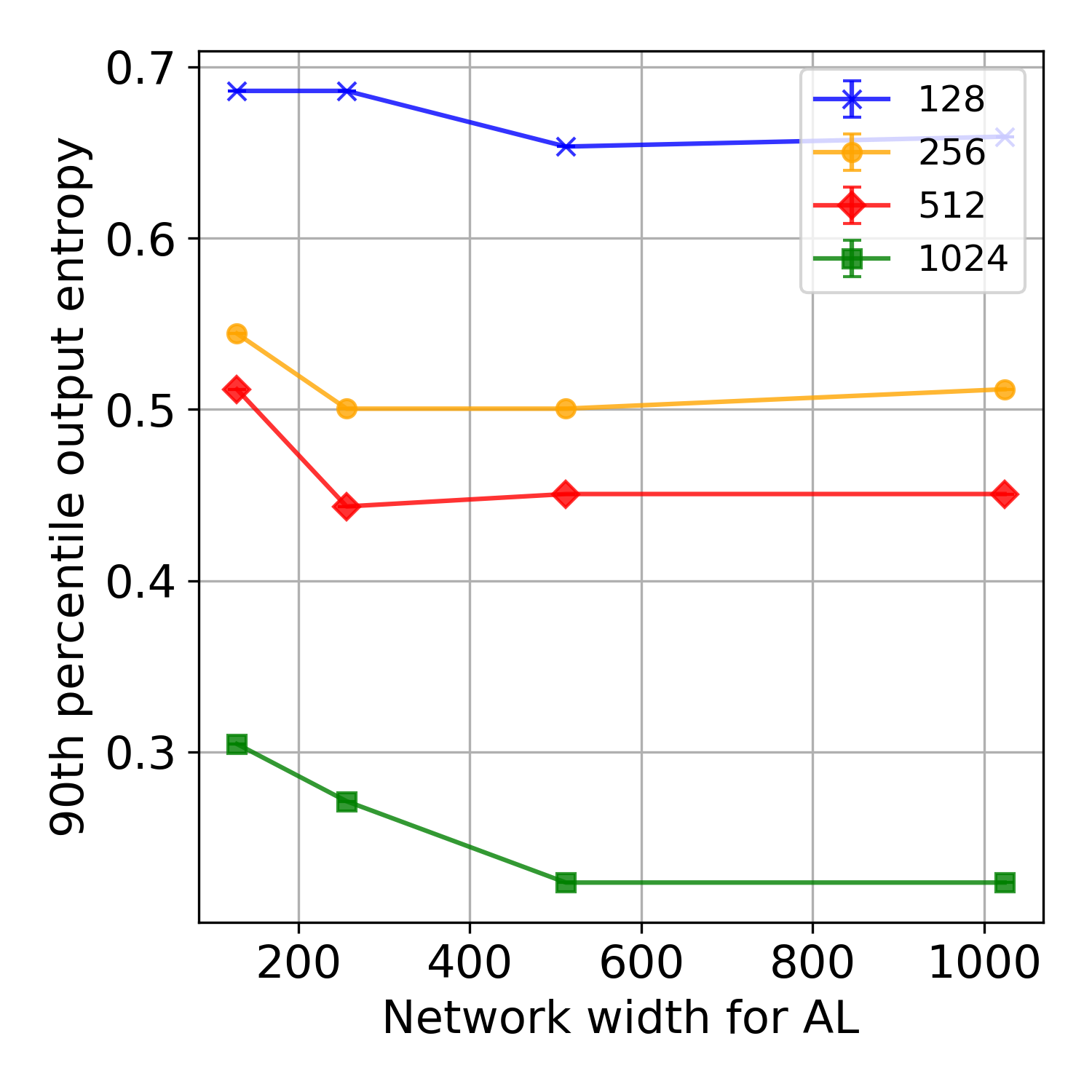}
\includegraphics[width=0.4\linewidth]{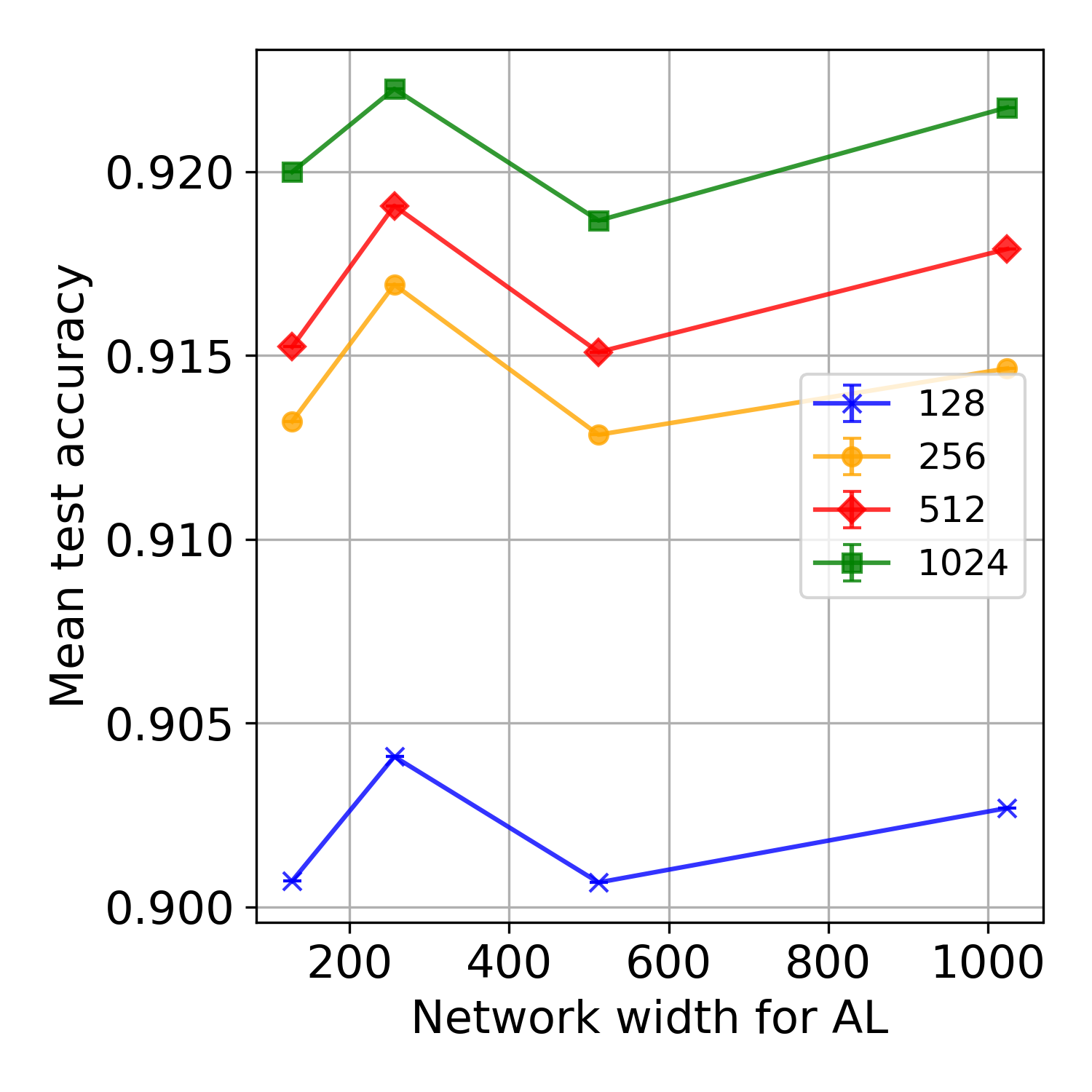} \\
\vspace{-6mm}
\caption{
Performances of our \textsc{EV-GP-Emp} with varying NN widths for data selection ($x$-axis) and NN training (different lines).
\label{fig:results-fixed-train-width}
}
\end{figure}

\subsubsection{Effect of Initial Points on Other Active Learning Benchmarks}

For BADGE and MLMOC algorithms, we find that the model performance is sometimes affected by how many initial labelled points were given to the algorithms. We plot the algorithms when some random data is available versus the case where no random initial data is given in \cref{fig:results-randinit}.

We find that BADGE shows little difference whether the first batch of data is randomized or not. This demonstrates the BADGE is able to select diverse points by itself without needing extra information to kickstart the algorithm. MLMOC, on the other hand, seems to be more reliant on the initial dataset, with a noticeable drop in performance when the algorithm is given no initial data. This suggests that MLMOC is a poor choice in the low-data regime, when the active learning algorithm has to provide good performance right away even when little prior knowledge about the data is given.

\begin{figure}[ht]
\centering
\begin{tabular}{ccc}
{\sffamily \scriptsize MNIST (2-layer MLP)}  & 
{\sffamily \scriptsize EMNIST (3-layer MLP)}  & 
{\sffamily \scriptsize SVHN (WideResNet)} \\
\includegraphics[width=0.25\linewidth]{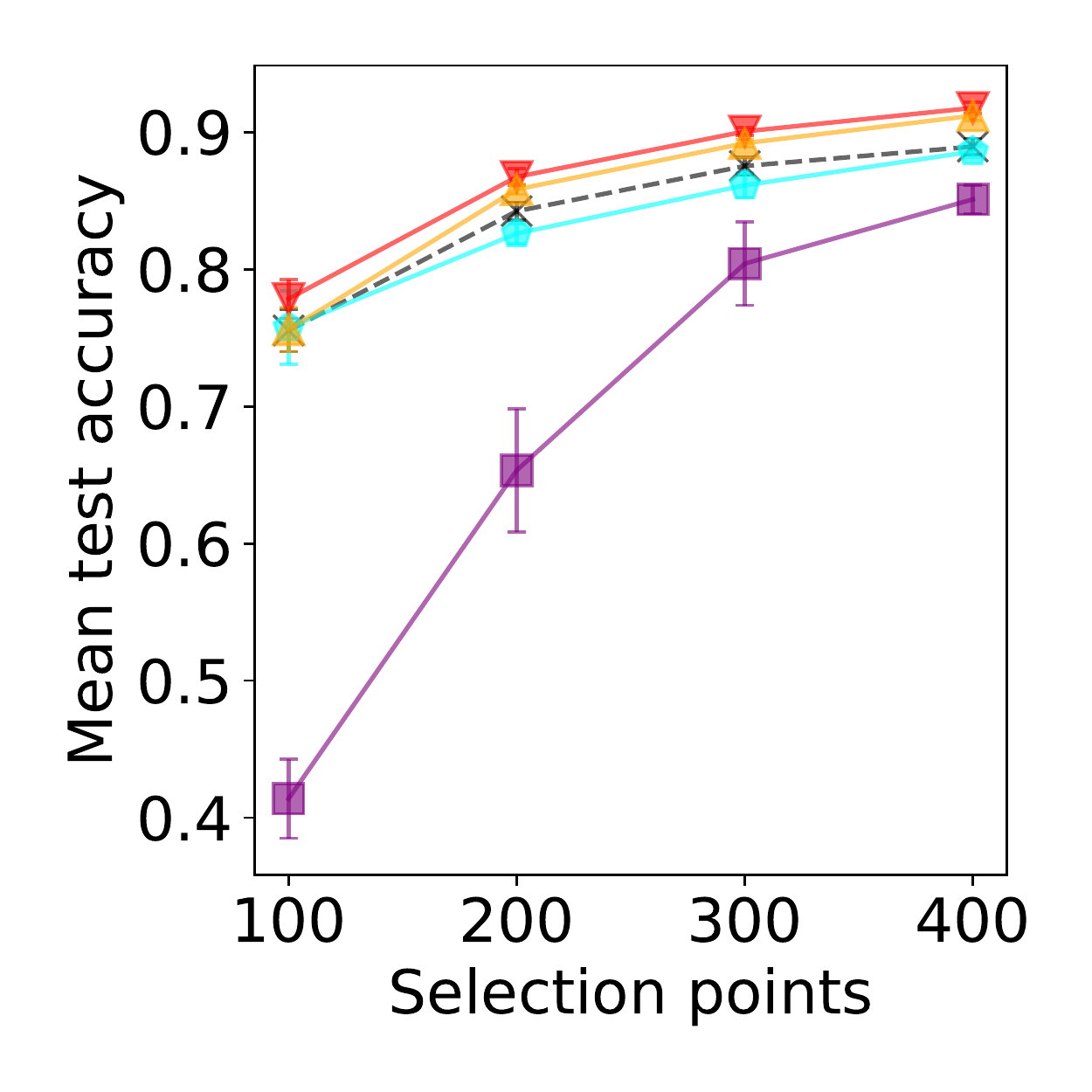}&
\includegraphics[width=0.25\linewidth]{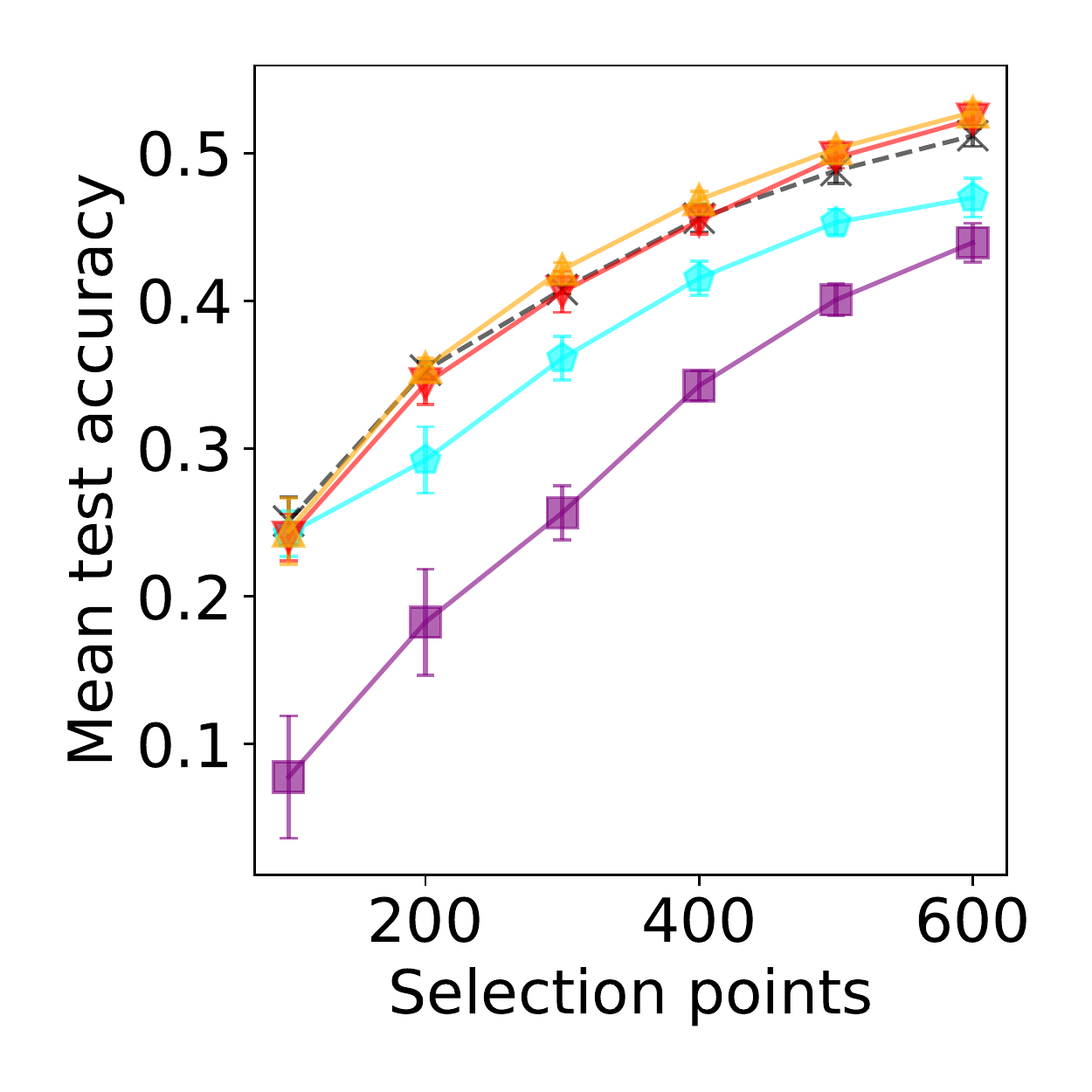} & 
\includegraphics[width=0.25\linewidth]{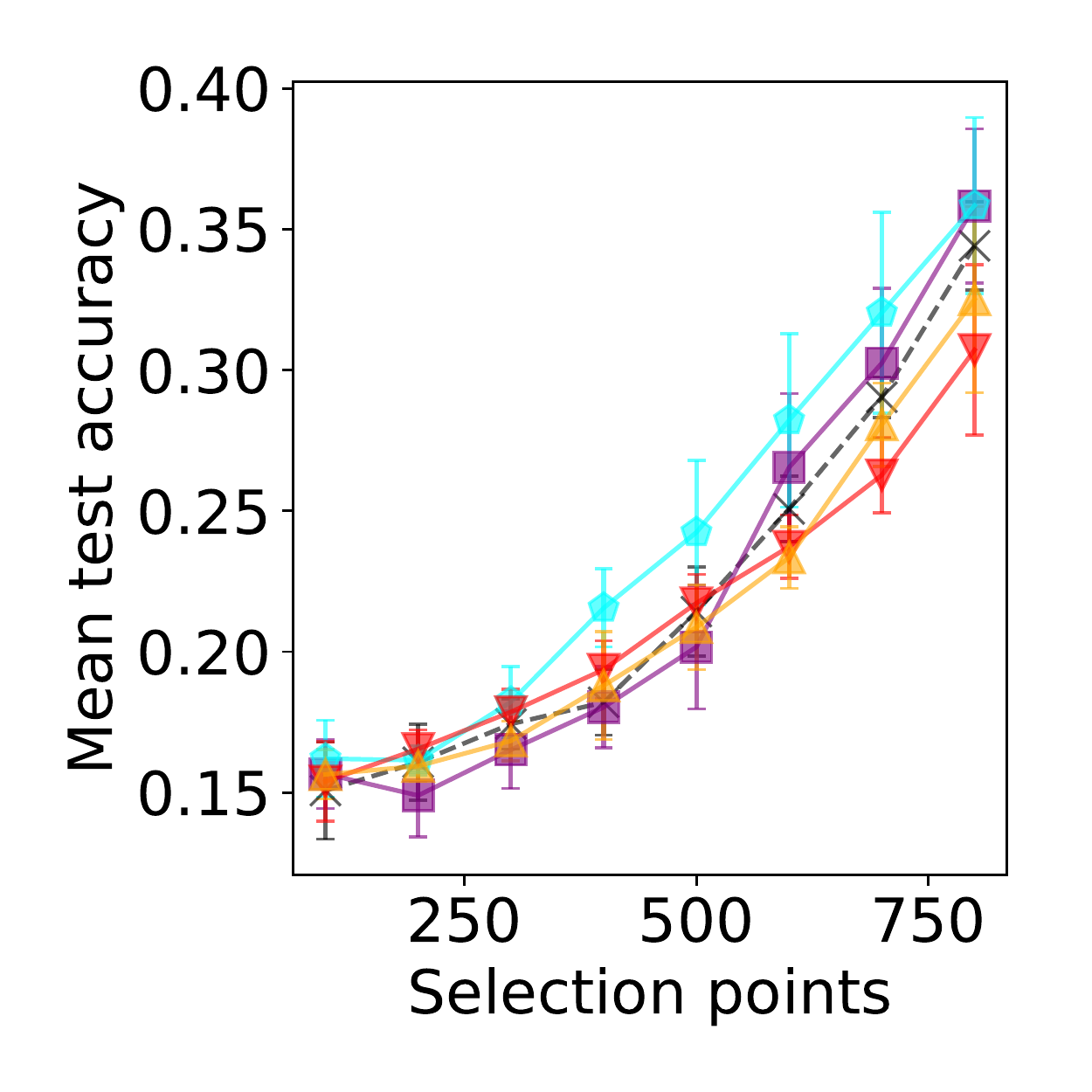}
\end{tabular}

\includegraphics[width=0.5\linewidth]{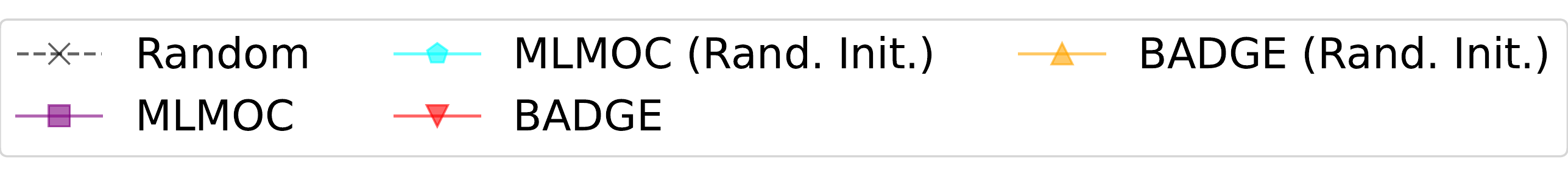}

\caption{
Results on classification with BADGE and MLMOC when random initial data is given. In the examples with {\sffamily (Rand. Init.)} quantifier, the algorithm randomizes the first batch of data instead of using their algorithm to first select the points.
}
\label{fig:results-randinit}
\end{figure}

\subsection{Additional Results for the Model Selection Algorithm}
\label{appx:ms-res}

\subsubsection{Relationship Between $\alpha_M$ and True Loss}

In this section we present how our model selection criterion relates to the true MSE loss. In \cref{fig:ms-res-ablation-appx}, we present how the model selection criterion $\alpha_M$ relates to the true achievable MSE loss. Note that we do not expect $\alpha_M$ to be a perfect reflection of the test MSE since during the active learning we do not know what the true test set is, and that $\alpha_M$ will be dependent on what $\cL$ is at the current point. Regardless, we find that $\alpha_M$ is able to give an accurate reflection of the MSE loss achievable. 

\begin{figure}[ht]
\centering

{\sffamily \small Robot Kinematics}\\
\begin{tabular}{ccc}
{\sffamily \scriptsize Scores at $|\cL| = 20$}  & 
{\sffamily \scriptsize Scores at $|\cL| = 200$} &
{\sffamily \scriptsize Scores progression}  \\ 
\includegraphics[width=0.2\linewidth]{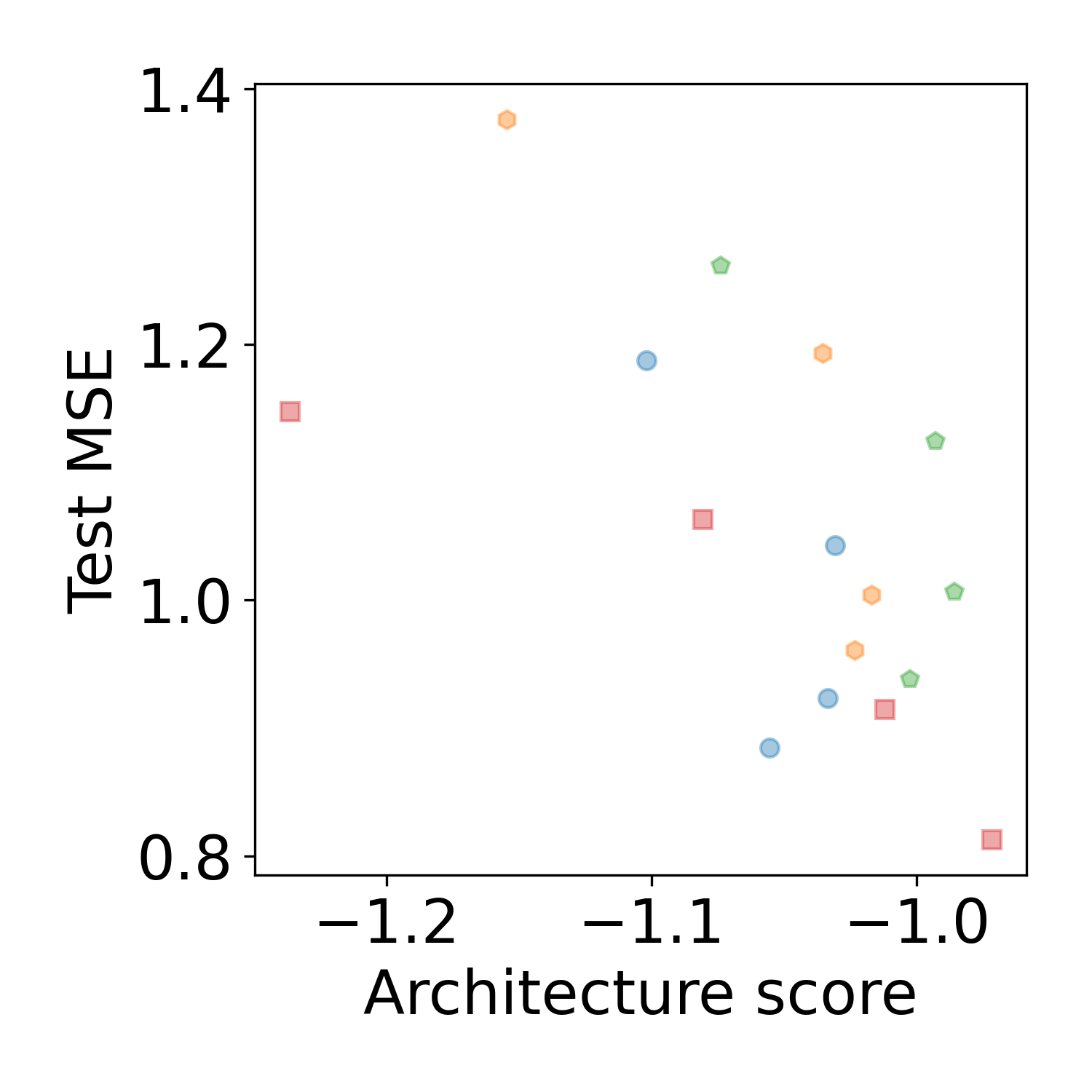}&
\includegraphics[width=0.2\linewidth]{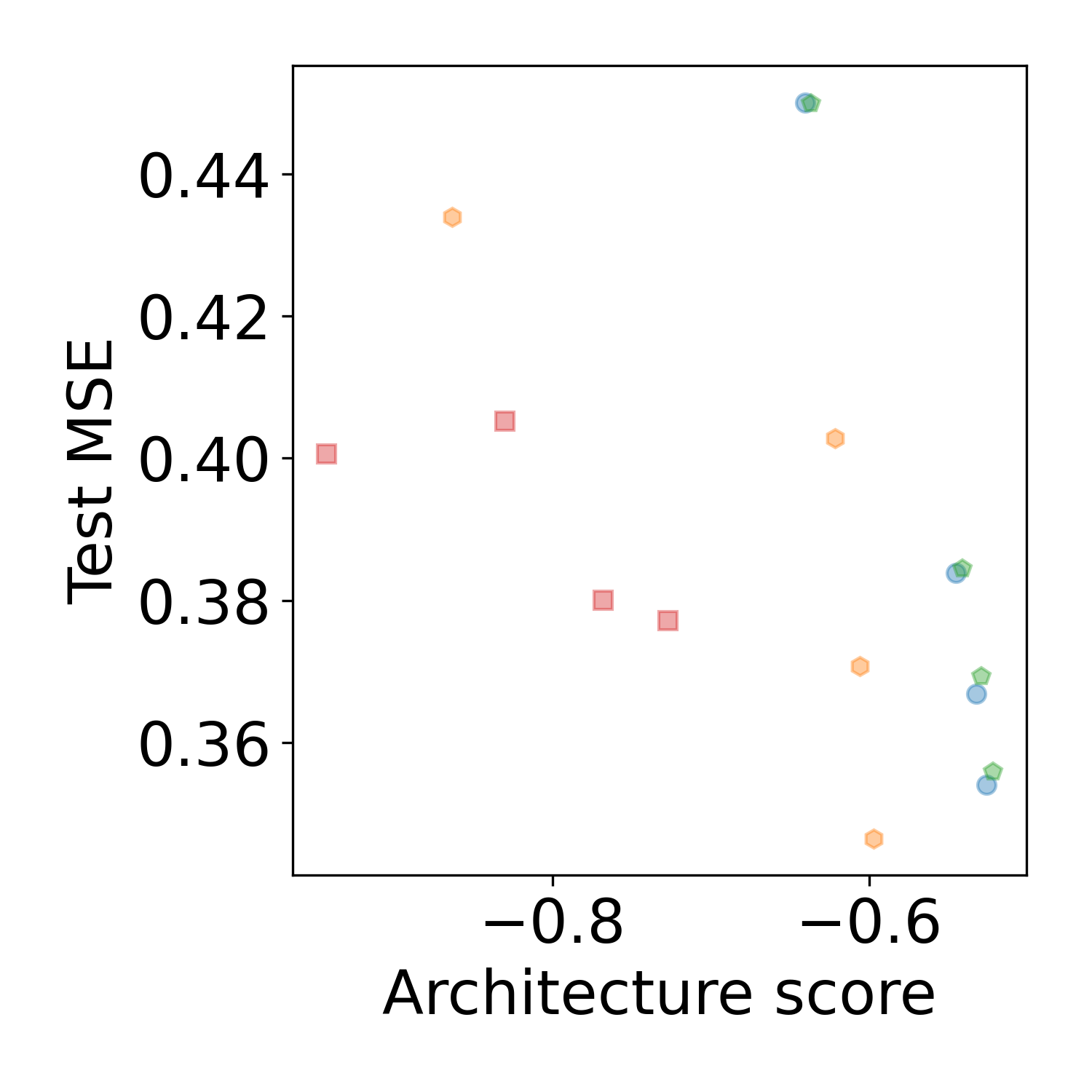}&
\includegraphics[width=0.2\linewidth]{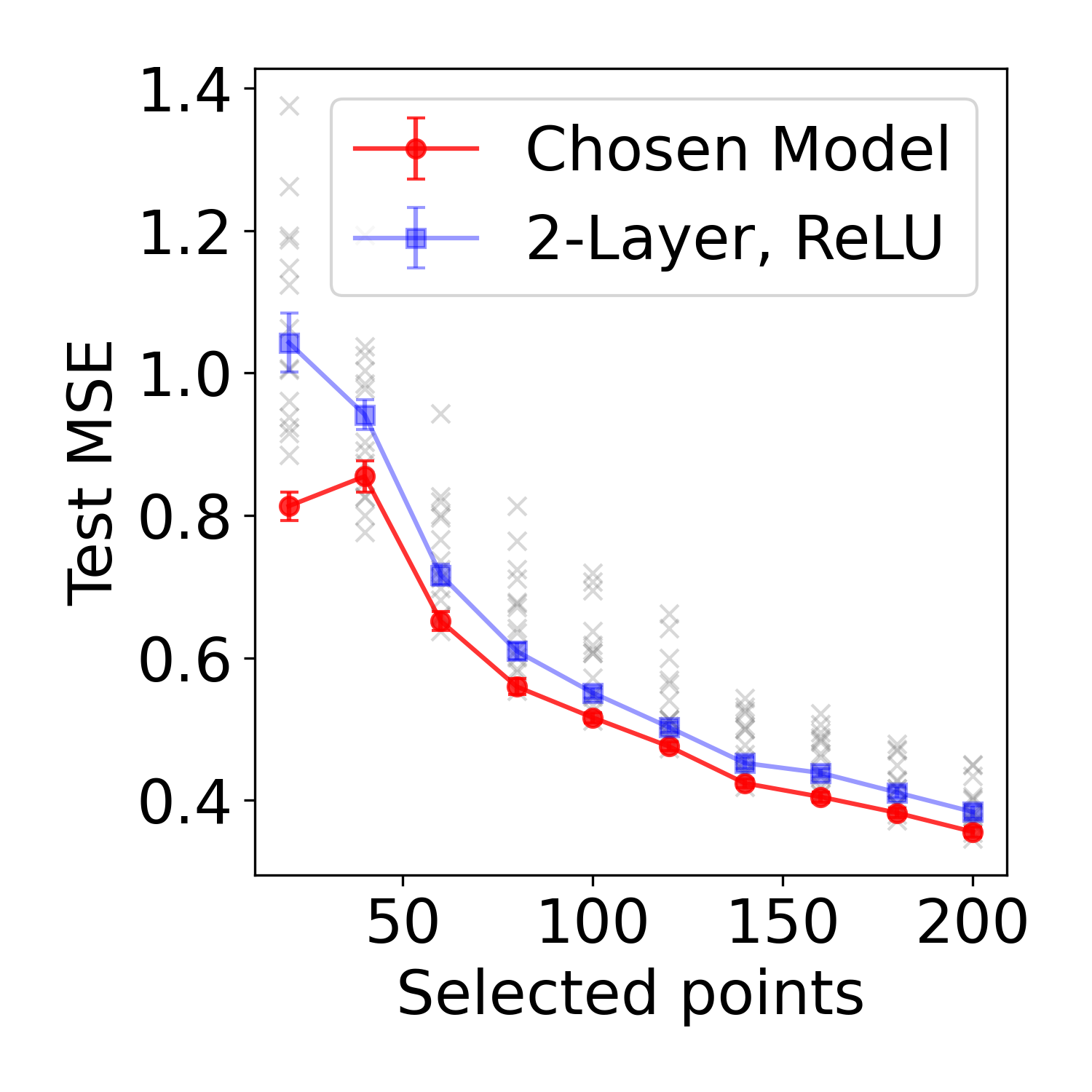}
\end{tabular}


{\sffamily \small Naval}\\
\begin{tabular}{ccc}
{\sffamily \scriptsize Scores at $|\cL| = 20$}  & 
{\sffamily \scriptsize Scores at $|\cL| = 200$} &
{\sffamily \scriptsize Scores progression}  \\ 
\includegraphics[width=0.2\linewidth]{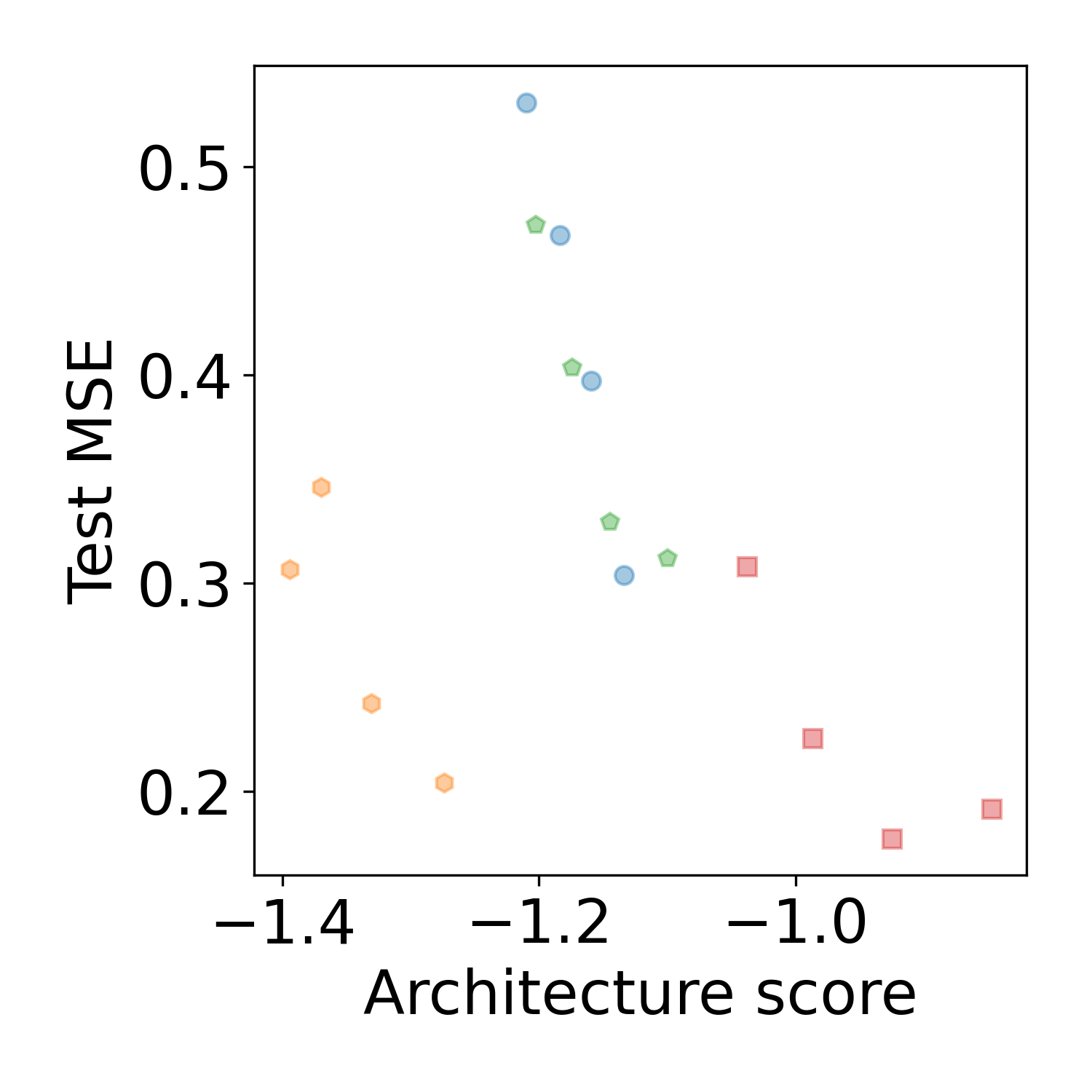}&
\includegraphics[width=0.2\linewidth]{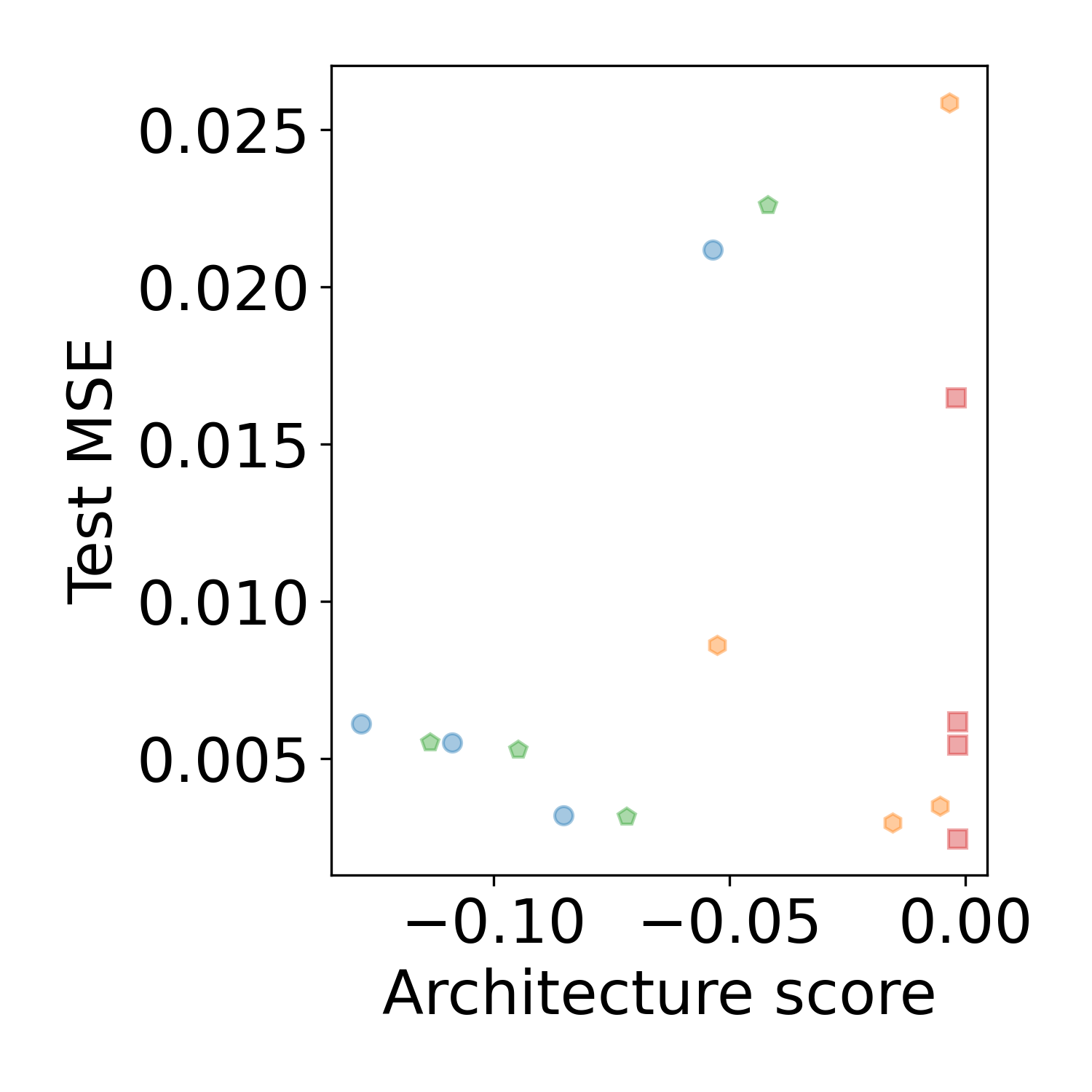}&
\includegraphics[width=0.2\linewidth]{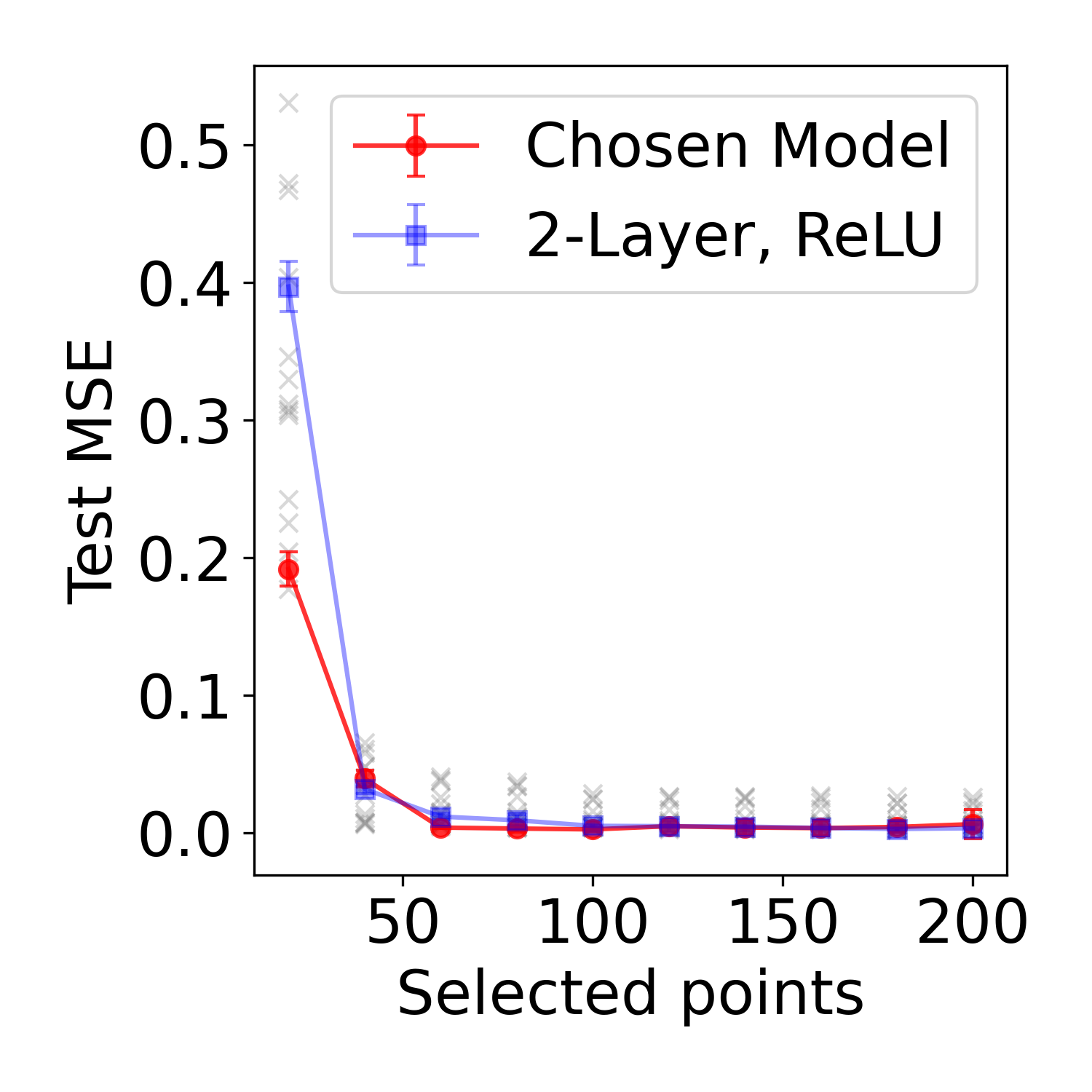}
\end{tabular}

\caption{
Further results related to \algms\xspace from ablation experiments conducted for \cref{fig:results-ms}. The graphs presents how \algms\xspace selects a model during one trial of active learning. \textit{Left and center:} the plots between $\alpha_M$ and the obtained test MSE for each model architecture at different iterations of the active learning process. Points with the same shapes are models with the same activation functions but varying depth. \textit{Right:} visualization of which models are selected by \algms\xspace in each round. The gray crosses represent the resulting training when using a particular architecture from $\cM$. The red line represents the choice from \algms\xspace, and the blue line represents the training if we stuck with the 2-layer MLP with ReLU activation throughout.
}
\label{fig:ms-res-ablation-appx}
\end{figure}

\subsubsection{Additional Results on Other Regression Datasets}

In \cref{fig:ms-res-extra-appx}, we present further results for \algms\xspace on other regression datasets. We find that in most datasets, our algorithm is able to select a suitable model architecture for the dataset, which leads to a low MSE loss at test time.

An issue we have found is that our algorithm is often sensitive to the obtained value of $\alpha_M$. In some cases, multiple models will have a similar value of $\alpha_M$, but in practice will have different MSE loss at training. This often leads to selecting a model which is suboptimal (although still provides acceptable results nonetheless). A future research direction would be on how the algorithm can be improved so that such noise can be mitigated. 

\begin{figure}[ht]
\centering


{\sffamily \small Protein}\\
\begin{tabular}{cccc}
{\sffamily \scriptsize Scores at $|\cL| = 20$}  & 
{\sffamily \scriptsize Scores at $|\cL| = 200$} &
{\sffamily \scriptsize Scores progression} &
{\sffamily \scriptsize Comparison}
\\ 
\includegraphics[width=0.2\linewidth]{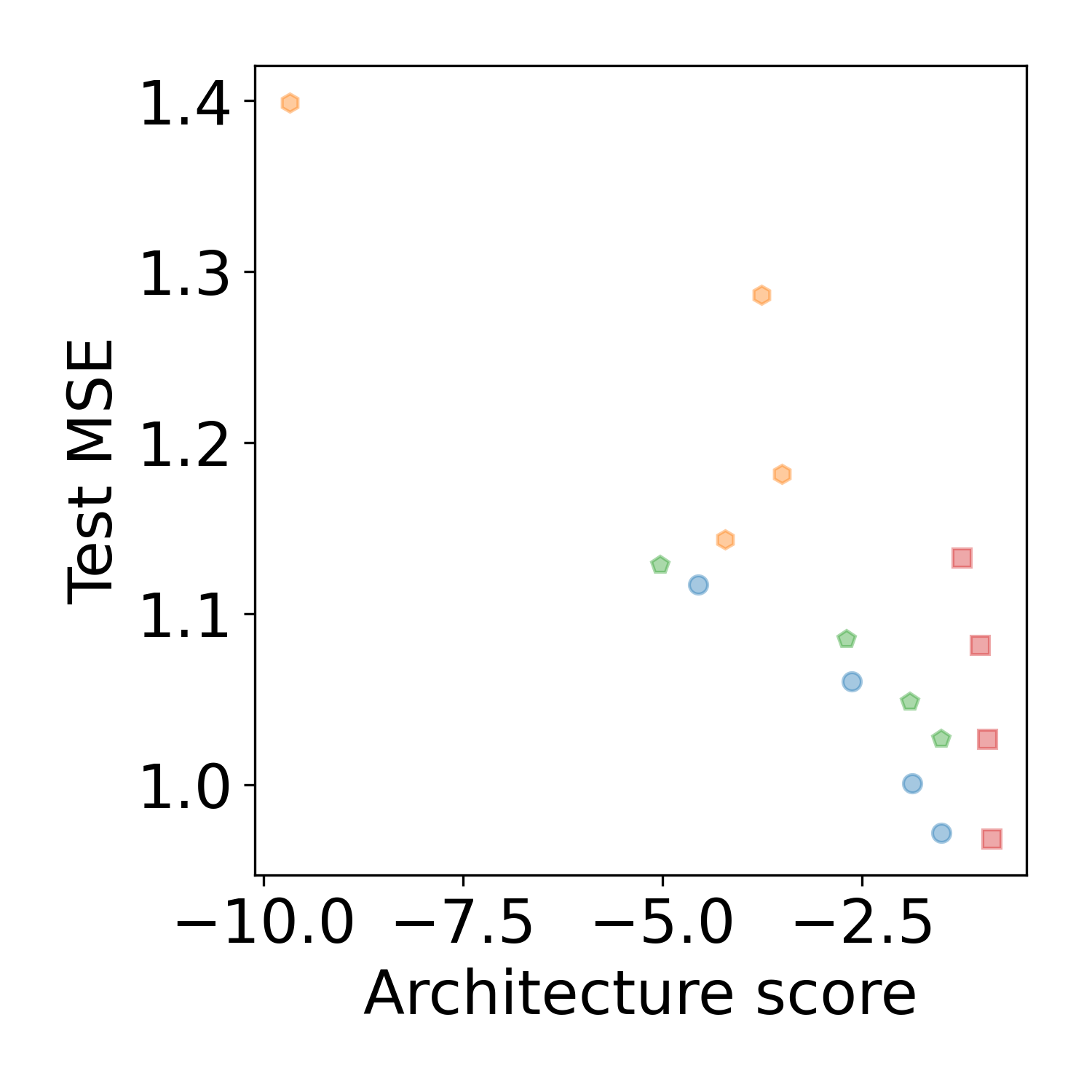}&
\includegraphics[width=0.2\linewidth]{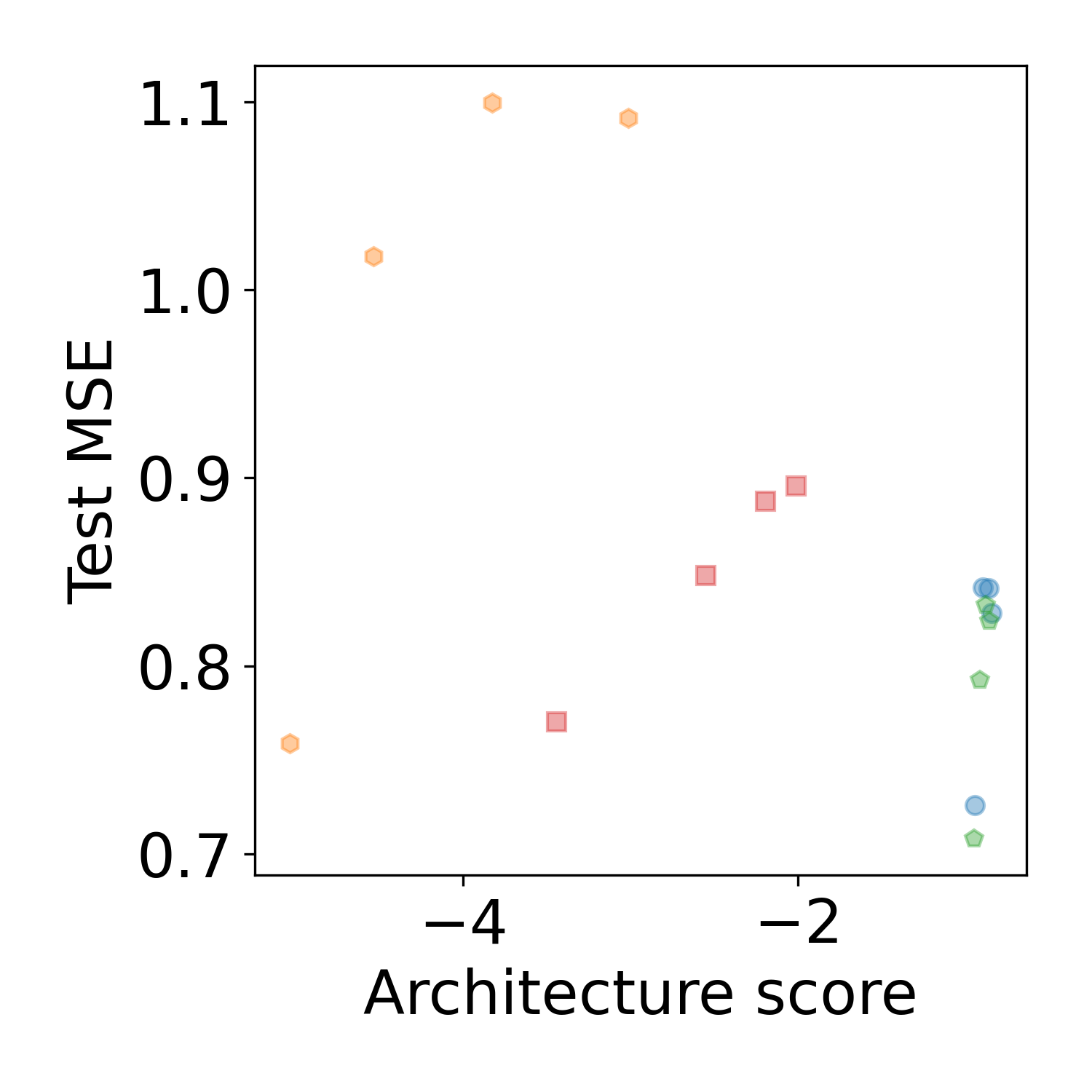}&
\includegraphics[width=0.2\linewidth]{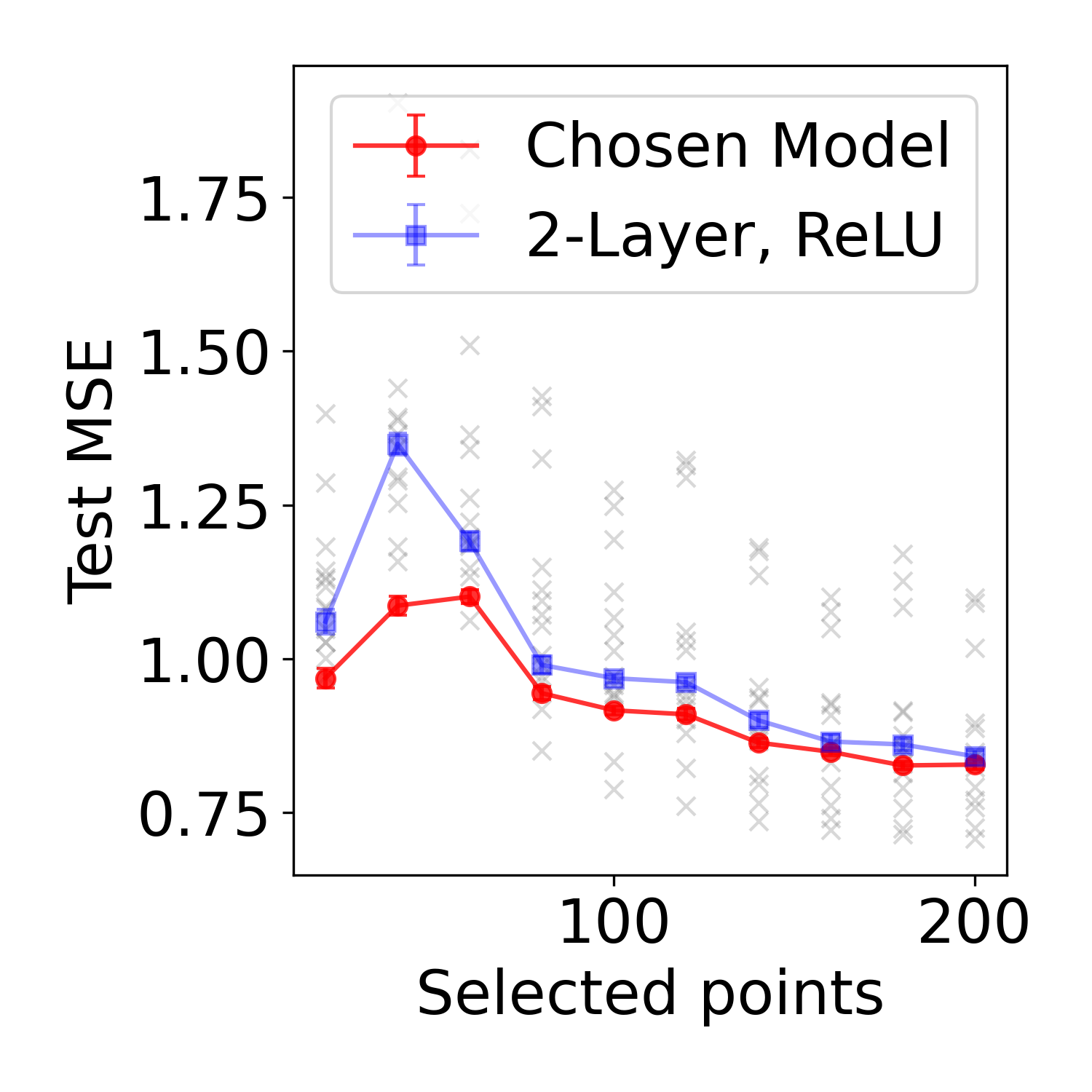}&
\includegraphics[width=0.2\linewidth]{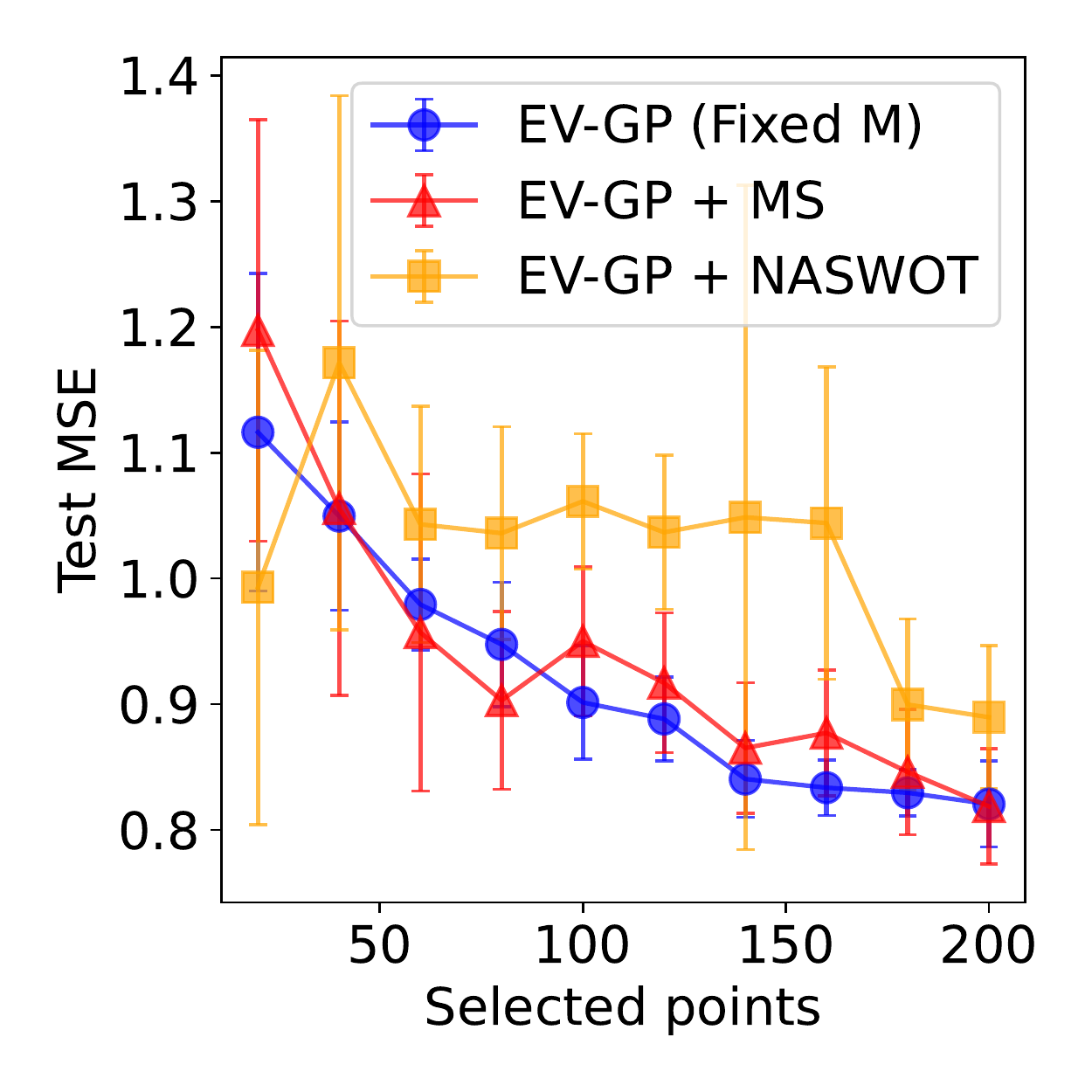}
\end{tabular}


{\sffamily \small Sinusoidal}\\
\begin{tabular}{cccc}
{\sffamily \scriptsize Scores at $|\cL| = 20$}  & 
{\sffamily \scriptsize Scores at $|\cL| = 200$} &
{\sffamily \scriptsize Scores progression} &
{\sffamily \scriptsize Comparison}
\\ 
\includegraphics[width=0.2\linewidth]{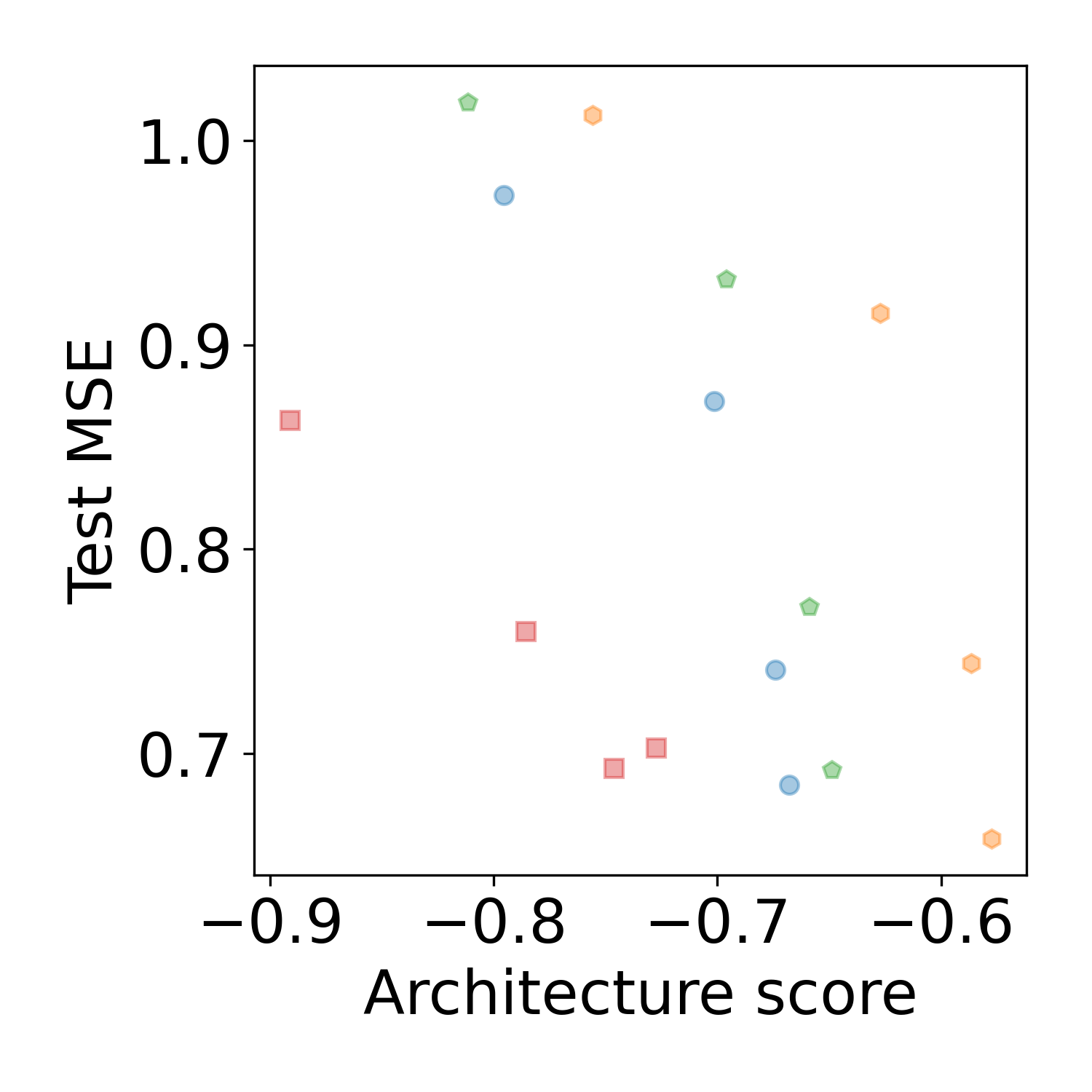}&
\includegraphics[width=0.2\linewidth]{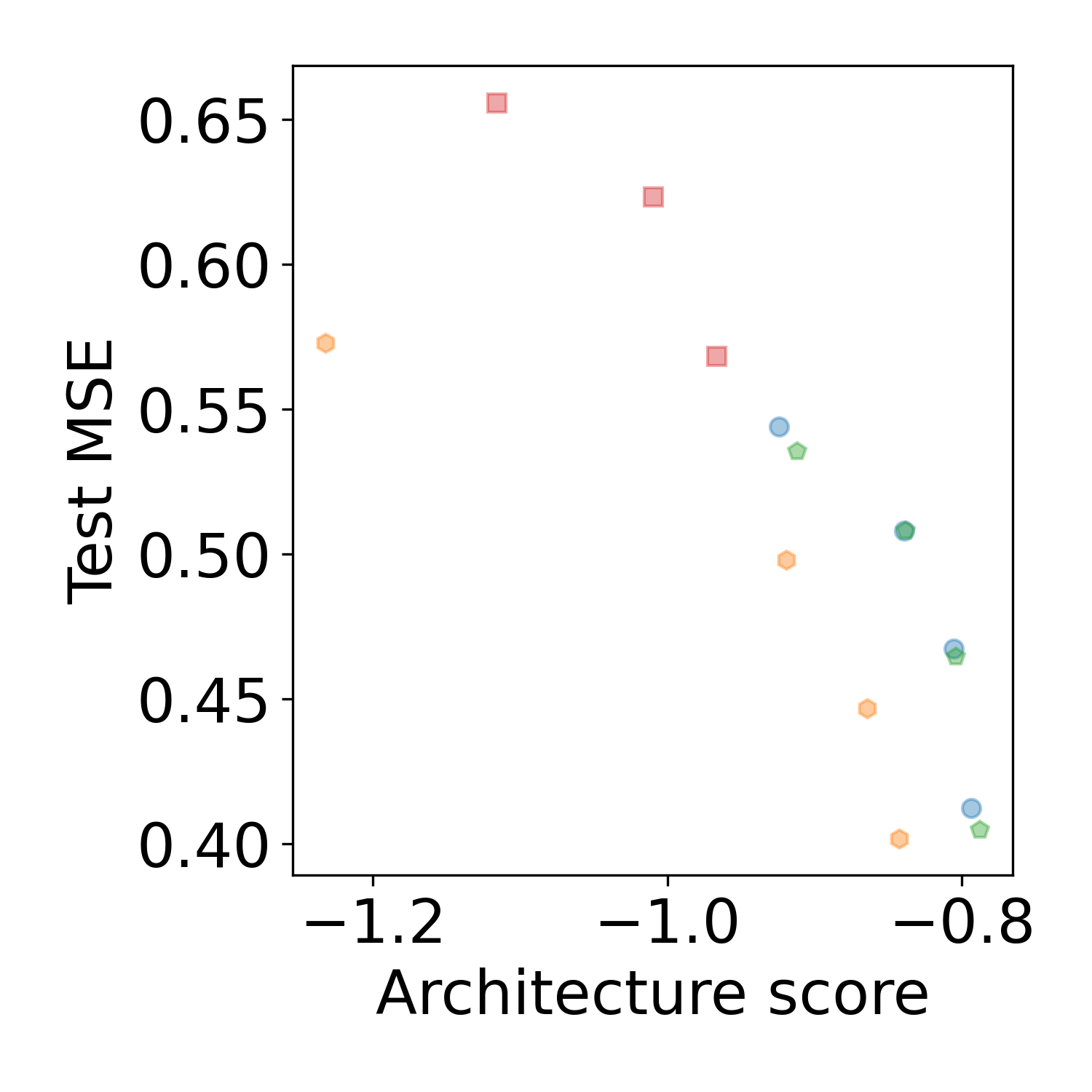}&
\includegraphics[width=0.2\linewidth]{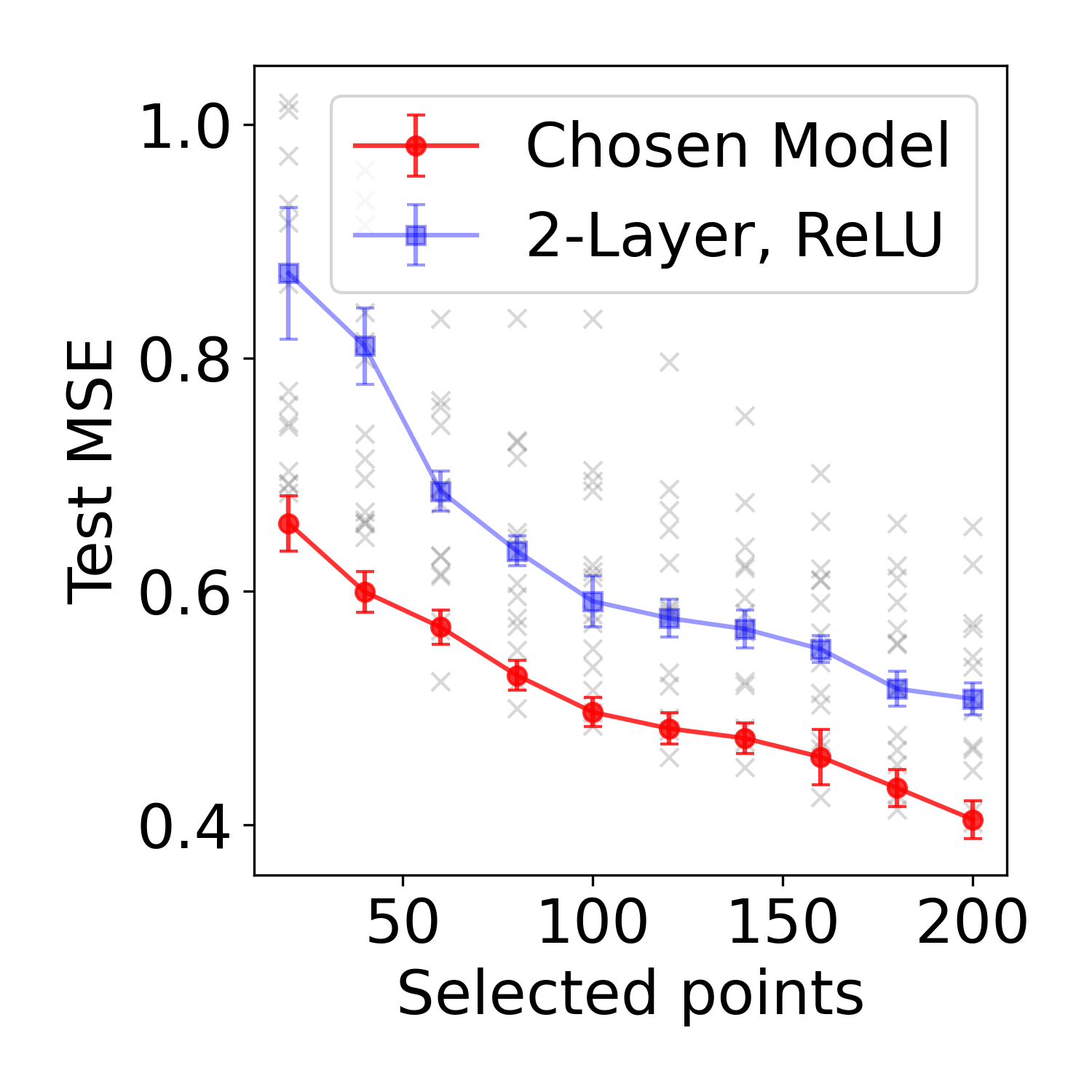}&
\includegraphics[width=0.2\linewidth]{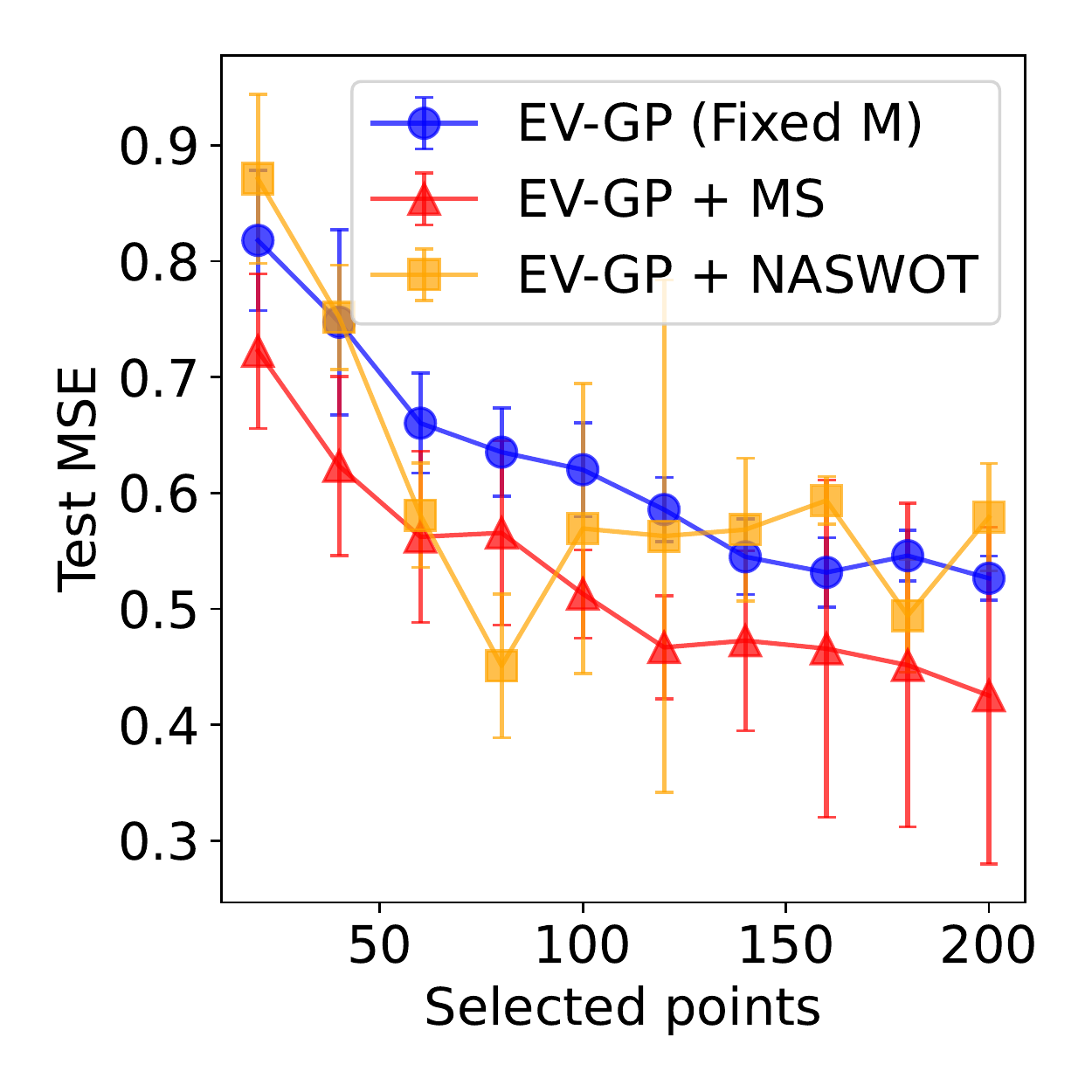}
\end{tabular}

\caption{
Additional experiments for \algms\xspace conducted on other datasets. The data presented is presented in the same manner as \cref{fig:results-ms,fig:ms-res-ablation-appx}.
}
\label{fig:ms-res-extra-appx}
\end{figure}

\end{document}